%% file: main.tex
\documentclass[a4paper,11pt,oneside,titlepage]{book}

\input{Include/Settings/Settings}
\begin{document}
\frontmatter 	   
\pagestyle{empty}  

\input{Include/Frontmatter/Titlepage} 
\clearpage
\input{Include/Frontmatter/Quote}

\newpage
\afterpage{\null\thispagestyle{empty}\clearpage} 
\input{Include/Frontmatter/Abstract} 
\begin{singlespace}
 \tableofcontents 	
 \addcontentsline{toc}{chapter}{\listfigurename}
 \listoffigures
 \addcontentsline{toc}{chapter}{\listtablename}
 \listoftables

 \printnomenclature
 \printglossaries
\end{singlespace}

\mainmatter	  
\clearpage
\pagestyle{fancy} 
\renewcommand{\chaptermark}[1]{\markright{\chaptername\ \thechapter.\ #1}{}}
\renewcommand{\sectionmark}[1]{\markright{\thesection.\ #1}}
\lhead{} 
\chead{}                   
\rhead{\slshape \rightmark} 
\lfoot{Lorenzo Giusti}
\cfoot{} 
\rfoot{\thepage}          
\renewcommand{\headrulewidth}{0.4pt} 
\renewcommand{\footrulewidth}{0.4pt}

\cleardoublepage
\pagenumbering{gobble}
\input{Include/Backmatter/Acknowledgement}


\input{Include/Chapters/1_Introduction/1_introduction}

\input{Include/Chapters/2_Background/2_background}
\input{Include/Chapters/3_Bottlenecks/3_bottlenecks}
\input{Include/Chapters/4_Enhancing/4_enhancing_graph_representations}

\input{Include/Chapters/5_Experiments/5_experiments}

\input{Include/Chapters/6_Conclusions/6_conclusions}

\cleardoublepage
\addcontentsline{toc}{chapter}{Bibliography}
\bibliography{Include/Backmatter/Bibliography}

\cleardoublepage
\appendix 
\input{Include/Backmatter/Appendix/Glossary}
\input{Include/Backmatter/Appendix/Appendix_Bottlenecks}
\input{Include/Backmatter/Appendix/Appendix_Simplicial_Attention_Networks}

\input{Include/Backmatter/Appendix/Appendix_Cell_Attention_Networks}
\input{Include/Backmatter/Appendix/Appendix_CINpp}

\end{document}

%% file: Include/Settings/Settings.tex
\usepackage[T1]{fontenc}		
\usepackage[utf8]{inputenc}		
\usepackage[english]{babel}		
\usepackage{geometry}			
\usepackage{setspace}			
\usepackage{fancyhdr}			

\usepackage{afterpage} 			
\usepackage{hyperref}			
\usepackage{url}
\usepackage{color}			
\usepackage{enumerate}			
\usepackage{enumitem}			
\usepackage{pstool}			
\usepackage{psfrag}
\usepackage{subcaption}
\usepackage{tabularx}			
\usepackage[bottom]{footmisc}		
\usepackage{booktabs}			
\usepackage{longtable}			

\usepackage{cellspace} 
\usepackage{caption}			
\usepackage{float}			
\usepackage[countmax]{subfloat}			
\usepackage{rotating}			
\usepackage{rotfloat}			
\usepackage{amsmath} 			
\usepackage{amssymb}			
\usepackage{cancel}                     
\usepackage{mleftright}
\usepackage{listings}                   
\usepackage[swapnames,norules,nouppercase]{frontespizio}
\usepackage[intoc]{nomencl}
\usepackage[acronym,toc,nomain,nopostdot,nonumberlist]{glossaries}
\usepackage[authoryear, round]{natbib}

\usepackage{verbatim}                   
\usepackage{wrapfig}                 
\usepackage{blindtext}

\usepackage{calligra}

\usepackage{microtype}
\usepackage{graphicx}
\usepackage{tcolorbox} 
\usepackage[normalem]{ulem}
\usepackage{mathtools}
\usepackage{amsfonts}
\usepackage{bm}
\usepackage{amsthm}
\usepackage[noabbrev,capitalise,nameinlink]{cleveref}
\hypersetup{colorlinks={true},linkcolor={darkred},citecolor=darkblue}

\usepackage{algorithm}
\usepackage[noend]{algpseudocode}
\newcommand{\algalign}[2]

\newlength{\bibitemsep}
\newlength{\bibparskip}
\let\oldthebibliography\thebibliography
\renewcommand\thebibliography[1]{%
  \oldthebibliography{#1}%
  \setlength{\parskip}{\bibitemsep}%
  \setlength{\itemsep}{\bibparskip}%
}

\usepackage{dsfont}

\usepackage{stmaryrd} 
\usepackage{tikz} 
\usepackage{tikz-cd} 
\usetikzlibrary{arrows.meta, calc, positioning, decorations.markings}
\usepackage{multirow}
\usepackage{changepage}
\usepackage{threeparttable}
\usepackage{makecell}
\usepackage{xspace}
\usepackage[colorinlistoftodos]{todonotes}
\usepackage{adjustbox} 

\usepackage{quiver}

\usepackage{comment}

\fancypagestyle{plain}{
                        \fancyhf{}                              
                        \fancyfoot[R]{\thepage}                 
                        \renewcommand{\headrulewidth}{0pt}      
                        \renewcommand{\footrulewidth}{0.4pt}    
                        }                                       
\definecolor{sapienza}{RGB}{130,36,51} 
\definecolor{cust1}{RGB}{85,85,85}
\definecolor{cust2}{RGB}{212,212,212}
\makenomenclature 
\makeglossaries 
\newenvironment{dedication}
{
  \phantom{.}
  \vspace{13cm}
  \begin{quote} \begin{flushright}}
{\end{flushright} \end{quote}}

\definecolor{bluegray}{rgb}{0.4, 0.6, 0.8}
\definecolor{electriclime}{rgb}{0.8, 1.0, 0.0}
\definecolor{malachite}{rgb}{0.04, 0.85, 0.32}
\definecolor{darkred}{rgb}{0.55, 0.0, 0.0}
\definecolor{darkblue}{rgb}{0.0, 0.0, 0.55}
\definecolor{darkgreen}{rgb}{0.0, 0.2, 0.13}
\definecolor{darkorchid}{rgb}{0.6, 0.2, 0.8}
\definecolor{simplexcolor}{RGB}{151,204,200}

\newcommand{\cheeg}{\mathsf{h}_{\mathsf{Cheeg}
}}
\newcommand{\res}{\mathsf{Res}}
\newcommand{\V}{\mathsf{V}}
\newcommand{\E}{\mathsf{E}}

\newcommand{\R}{\mathbb{R}}

\newcommand{\up}{\sigma}
\newcommand{\rs}{\mathsf{r}}
\newcommand{\mpas}{\mathsf{a}}
\newcommand{\agg}{\mathsf{agg}}
\newcommand{\com}{\mathsf{com}}
\newcommand{\out}{\mathsf{out}}
\newcommand{\mlp}{\mathsf{MLP}}
\newcommand{\cph}{\mathsf{C}} 
\newcommand{\gph}{\mathsf{G}} 
\newcommand{\kph}{\mathsf{K}} 
\newcommand{\pph}{\mathsf{P}} 
\newcommand{\xph}{\mathsf{X}} 

\newcommand{\MPNN}{\mathsf{MPNN}}
\newcommand{\rew}{\mathcal{R}}
\newcommand{\eigen}{\boldsymbol{\psi}}

\newcommand{\W}{\mathbf{W}}

\newcommand{\OMEga}{\boldsymbol{\Omega}}
\newcommand{\Anorm}{\boldsymbol{\mathsf{A}}}
\newcommand{\Bnorm}{\boldsymbol{\mathsf{B}}}
\newcommand{\Dnorm}{\boldsymbol{\mathsf{D}}}
\newcommand{\Hnorm}{\boldsymbol{\mathsf{H}}}
\newcommand{\Inorm}{\boldsymbol{\mathsf{I}}}

\newcommand{\Lnorm}{\boldsymbol{\mathsf{L}}}
\newcommand{\Pnorm}{\boldsymbol{\mathsf{P}}}
\newcommand{\Snorm}{\boldsymbol{\mathsf{S}}}
\newcommand{\Unorm}{\boldsymbol{\mathsf{U}}}
\newcommand{\Wnorm}{\boldsymbol{\mathsf{W}}}
\newcommand{\xnorm}{\boldsymbol{\mathsf{x}}}
\newcommand{\Xnorm}{\boldsymbol{\mathsf{X}}}
\newcommand{\Znorm}{\boldsymbol{\mathsf{Z}}}

\newcommand{\oper}{\boldsymbol{\mathsf{S}}_{\rs,\mpas}}
\newcommand{\obst}{\mathsf{O}}

\newcommand{\doarr}{{{}^{\downarrow}}}
\newcommand{\uparr}{{{}^{\uparrow}}}

\newcommand{\Ldo}{\Lnorm^{\doarr}}
\newcommand{\Lup}{\Lnorm^{\uparr}}

\theoremstyle{plain}
\newtheorem{theorem}{Theorem}[section]
\newtheorem{proposition}[theorem]{Proposition}

\newtheorem{corollary}[theorem]{Corollary}
\theoremstyle{definition}
\newtheorem{definition}[theorem]{Definition}
\newtheorem{assumption}[theorem]{Assumption}
\theoremstyle{remark}

\newcommand{\tikzcircle}[2][red,fill=red]{\tikz[baseline=-0.7ex]\draw[#1,radius=#2] (0,0) circle ;}%
\definecolor{gold}{RGB}{221, 196, 65}
\definecolor{silver}{RGB}{215, 215, 215}
\definecolor{bronze}{RGB}{205, 127, 50}

\newcommand{\peptides}{\texttt{Peptides}\xspace}
\newcommand{\pepfunc}{\texttt{Peptides-func}\xspace}
\newcommand{\pepstruct}{\texttt{Peptides-struct}\xspace}

\definecolor{lightgray}{gray}{0.95}
\definecolor{Gray}{gray}{0.85}
\definecolor{LightCyan}{rgb}{0.88,1,1}
\definecolor{LightPink}{HTML}{FCE1EF}
\definecolor{LightGreen}{HTML}{EEF7E1}
\newcolumntype{A}{>{\columncolor{white}}c}
\newcolumntype{B}{>{\columncolor{LightGreen}}c}
\newcolumntype{C}{>{\columncolor{LightPink}}c}

\definecolor{purpleheart}{rgb}{0.41, 0.21, 0.61}
\definecolor{mblue}{rgb}{0.2118, 0.3608, 0.5765}
\definecolor{mgreen}{rgb}{0.2314, 0.5765, 0.3961}
\definecolor{mred}{rgb}{0.5765, 0.2314, 0.3216}
\definecolor{mpurp}{rgb}{0.4824, 0.2314, 0.5765}
\definecolor{amber}{rgb}{1.0, 0.75, 0.0}

\definecolor{dark2green}{rgb}{0.1, 0.65, 0.3}
\definecolor{dark2orange}{rgb}{0.9, 0.4, 0.}
\definecolor{dark2purple}{rgb}{0.4, 0.4, 0.8}
\newcommand{\first}[1]{\colorbox{gold}{\textbf{#1}}}
\newcommand{\second}[1]{\colorbox{silver}{#1}}
\newcommand{\third}[1]{\colorbox{bronze}{{#1}}}

\definecolor{lightgreen}{RGB}{168,207,147}
\definecolor{lightblue}{RGB}{34,118,180}
\definecolor{lightyellow}{RGB}{255,226,149}

\newcommand{\face}{{\trianglelefteqslant\,}}

%% file: Include/Frontmatter/Titlepage.tex
\begin{frontespizio}
\Preambolo{\renewcommand{\fronttitlefont}{\fontsize{24}{24}\bfseries}}

\Margini{4cm}{3cm}{3cm}{3cm}				
\Logo[4cm]{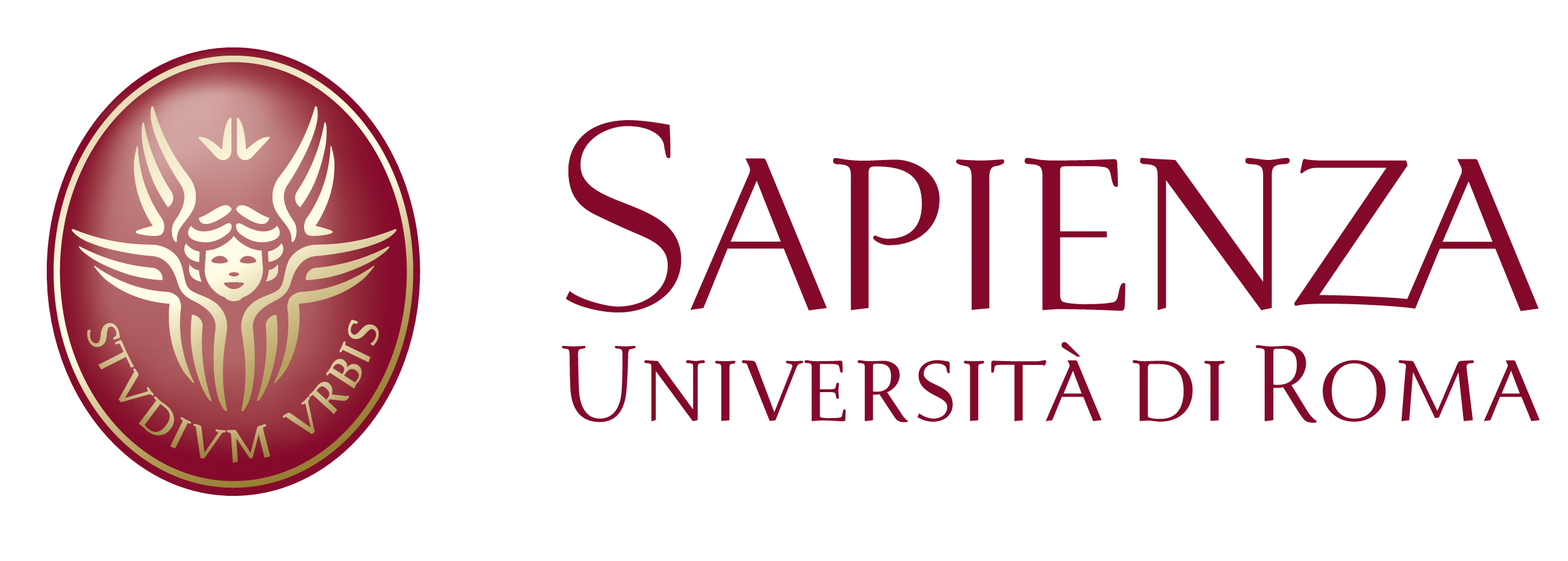}
\Istituzione{Sapienza University of Rome}
\Divisione{Department of Engineering}		
\Scuola{PhD in Data Science}
\Titoletto{Thesis For The Degree Of Doctor Of Philosophy}
\Titolo{Topological Neural Networks}
\Sottotitolo{Mitigating the Bottlenecks of Graph Neural Networks via Higher-Order Interactions}
\Punteggiatura{}					
\NCandidato{Candidate}					
\Preambolo{\renewcommand{\frontsmallfont}[1]{\small}}
\Candidato[]{Lorenzo Giusti}
\NRelatore{Thesis Advisor}{Advisor}			
\Relatore{Prof. Stefano Leonardi}
\Relatore{Prof. Pietro Li\`{o}}
\Piede{Academic Year MMXX-MMXXIII (XXXVI cycle)}					
\end{frontespizio}
\IfFileExists{\jobname-frn.pdf}{}{%
\immediate\write18{pdflatex \jobname-frn}}

%% file: Include/Frontmatter/Quote.tex
\begin{dedication}

{\fontfamily{calligra}\selectfont
{\Huge

To Whom it May Concern 
}
}
\end{dedication}

%% file: Include/Frontmatter/Abstract.tex
\thispagestyle{plain}			
\setlength{\parskip}{0pt plus 1.0pt}
\section*{Abstract}

The irreducible complexity of natural phenomena has led Graph Neural Networks to be employed as a standard model to perform representation learning tasks on graph-structured data. While their capacity to capture local and global patterns is remarkable, the implications associated with long-range and higher-order dependencies pose considerable challenges to such models. This work addresses these challenges by starting with the identification of the aspects that negatively impact the performance of graph neural networks in learning representations of events that strongly depend on long-range interactions. In particular, when graph neural networks require to aggregate messages among distant nodes, the message passing scheme performs an over-squashing of an exponentially growing amount of information into static vectors.

It is important to notice that for some classes of graphs (i.e., path, tree, grid, ring, and ladder) the underlying connectivity allows messages to travel along edges without encountering significant interference from other paths, thus reducing the growth of information to be linear in the number of messages exchanged.

When the underlying graph does not fall into the aforementioned categories, oversquashing arises because the propagation of information happens between nodes that are connected through edges, which induces a computational graph mirroring nodes' connectivity. This phenomenon causes nodes to become insensitive to information sent from remote parts of the graph. To offer a new perspective for designing architectures that mitigate such bottlenecks, a unified theoretical framework reveals the impact of network's width, depth, and graph topology on the over-squashing phenomena in message-passing neural networks.

The thesis then drifts towards the exploitation of higher-order interactions via Topological Neural Networks.  With a multi-relational inductive bias, topological neural networks propagate messages through higher-dimensional structures, effectively providing shortcuts or additional routes for information flow. With this construction, the underlying computational graph is no longer coupled with the input graph structure, thus mitigating the aforementioned bottlenecks while accounting also for higher-order interactions. Inspired by the masked self-attention mechanism developed in Graph Attention Networks alongside the rich connectivity provided by simplicial and cell complexes, two distinct attentional architectures are proposed: Simplicial Attention Networks and Cell Attention Networks.

The rationale behind these architecture is to leverage
the extended notion of neighbourhoods provided by the particular arrangement of groups of nodes within a simplicial or cell complex. In particular, these topological attention networks exploit the upper and lower adjacencies of the underlying complex to design anisotropic aggregations able to measure the importance of the information coming from different regions of the domain. By doing so, they capture dependencies that conventional Graph Neural Networks might miss.

Finally, a communication scheme between higher-order structures is introduced with Enhanced Cellular Isomorphism Networks, which augment topological message passing schemes by letting all the cells of a cell complex receive messages from their lower neighbourhood. This upgrade enables direct interactions among node groups within a cell complex, specifically arranged in ring-like structures. This augmented scheme offers more comprehensive representation of higher-order and long-range interactions, demonstrating very high performance on large-scale and long-range benchmarks.

\vfill
Keywords: Topological Deep Learning, Topological Neural Networks, Geometric Deep Learning, Graph Neural Networks.

\thispagestyle{empty}
\mbox{}

%% file: Include/Backmatter/Acknowledgement.tex
\thispagestyle{plain}			
\section*{Acknowledgements}
Lorem ipsum dolor sit amet, consectetur adipisicing elit, sed do eiusmod tempor incididunt ut labore et dolore magna aliqua. Ut enim ad minim veniam, quis nostrud exercitation ullamco laboris nisi ut aliquip ex ea commodo consequat. Duis aute irure dolor in reprehenderit in voluptate velit esse cillum dolore eu fugiat nulla pariatur. Excepteur sint occaecat cupidatat non proident, sunt in culpa qui officia deserunt mollit anim id est laborum.

\vspace{1.5cm}
\hfill
\copyright \, Lorenzo Giusti, Geneva, January 2024. All rights reserved.

\newpage				
\thispagestyle{empty}
\mbox{}

%% file: Include/Chapters/1_Introduction/1_introduction.tex
\chapter{Introduction}

\input{Include/Chapters/1_Introduction/1.1_motivation_and_background}

\input{Include/Chapters/1_Introduction/1.2_topological_deep_learning_4_science}

\input{Include/Chapters/1_Introduction/1.4_research_objectives_and_contrubition}

%% file: Include/Chapters/1_Introduction/1.1_motivation_and_background.tex
\section{On Graph Representation Learning}

In everyday life, we experience events that involve objects and relationships at all scales. From quantum physics~\citep{rovelli2021relational} to cosmology~\citep{makinen2022cosmic}, {\em nature communicates complex phenomena to us in terms of evolving systems of interconnected entities}~\citep{strogatz2004sync}. Examples at the human scale include: brain networks, where neurons are the entities and linked through synapses~\citep{bassett2017network}; molecules, with atoms glued together by chemical bonds~\citep{balaban1985applications} and social networks, where persons are connected through friendships~\citep{ohtsuki2006simple}. The mathematical language to describe such systems is known as \textbf{graph}, a tool able to represent nature's complexity by modelling entities as nodes and relationships as links between them~\citep{velivckovic2023everything}.

In the past decade, the machine learning community has recognized an outstanding template to perform learning tasks on data defined over relational domains. Such models are referred to as \textbf{Graph Neural Networks}(GNNs)~\citep{sperduti1994encoding, sperduti1997supervised, scarselli2008graph, gori2005new}. This success was possible due to their efficiency in combining the representational power of neural networks with a relational inductive bias~\citep{battaglia2018relational} provided by a prior knowledge of the relationships between objects. Within the realm of graph neural networks, the \textbf{message-passing paradigm}~\citep{gilmer2017neural} has emerged as an efficient scheme to realize graph neural networks, It enables nodes in a graph to update their representation with three operations: (1) {\bf communication} between the nodes and their neighbours, (2) {\bf aggregation} of the information received from the neighbours and (3) {\bf update} of the internal representation using the information received from the neighbours. The simplicity of the message passing paradigm has led to significant breakthroughs in scientific challenges like protein folding~\citep{jumper2021highly} and algorithmic reasoning~\citep{velivckovic2021neural}.

\

Although graph neural networks can learn {\em almost} any representation of interconnected systems and the successes of these class of neural networks are a proof of their exceptional ability, their original design face several limitations in representing data coming from more complex systems~\citep{battiston2020networks}. For example, scientists in biology~\citep{lee2013transcriptional, sever2015signal}, physics~\citep{parisi1983order}, sociology~\citep{granovetter1978threshold, sumpter2006principles}, network neuroscience~\citep{giusti2016two} and chemistry~\citep{steed2022supramolecular} may argue that events often involve groups of entities interacting concurrently in a cooperative or adversarial manner. For instance, in such fields of science, group dynamics often play a role, where the interaction of three or more entities can lead to outcomes different from pairwise~\citep{wooldridge2009introduction}.

In particular, when this happens, the underlying phenomena is said to exhibit {\bf higher-order interactions}~\citep{ahn2010link}. Applications in which higher-order interactions alter the state of an interconnected system might be found in most of real-world scenarios.

Although such interactions might contribute only a small amount of information, their effect might have a huge impact on the evolution of complex systems.

For instance, in biochemical networks, multiple proteins interacting together can lead to a cascading signal transduction that would not occur with simple pairwise interactions~\citep{barabasi2011network}. Similarly, in functional brain networks, the disruption or alteration of activity in a critical hub region, can propagate throughout the entire network leading to widespread changes in brain function and behavior which might impact various cognitive tasks and even contribute to neurological disorders~\citep{greicius2004default}. In such cases, traditional graph representations may fall short,  requiring models that can capture higher-order arrangements of entities in a principled fashion.

This manuscript focuses on developing tools for phenomena in which \textbf{the complexity goes beyond simple node-edge representations} and higher-order models are {\bf essential} to completely describe the the complex nature of events.

%% file: Include/Chapters/1_Introduction/1.2_topological_deep_learning_4_science.tex
\section{Topological Neural Networks for Science}\label{sec:intro:tnns_4_science}

The previous section highlights the necessity of a mathematical framework that allows for learning the representation of events involving  non-trivial relationship schemes among the entities that are involved. Although graph neural networks can be employed for learning \textit{almost} every representation of complex systems, in certain situations, traditional graph representations may not sufficiently capture the entire complexity of such systems.

In scientific fields, such as \textit{biology, neuroscience, physics and chemistry}, it has been observed that considering higher-order relationships reveal aspects of the underlying phenomenon that would be hidden if only mutual connections are taken into account.

This section aims to highlight the common threads across diverse scientific fields from the perspective of higher-order interactions.

In particular, it will be discussed how the dynamics of gene regulatory networks often involve multiple genes, how neurons in brain networks fire together, the way in which the degrees of freedom of spin glasses are related to the adversarial interactions among spins (atoms or ions) on a lattice structure and which molecular properties are determined by the relations among chemical rings.

\input{Include/Chapters/1_Introduction/1.2.1_tnns_in_biology}
\input{Include/Chapters/1_Introduction/1.2.2_tnns_in_neurosci}
\input{Include/Chapters/1_Introduction/1.2.3_tnns_in_physics}

\input{Include/Chapters/1_Introduction/1.2.4_tnns_in_superchem}

%% file: Include/Chapters/1_Introduction/1.2.1_tnns_in_biology.tex
\subsubsection{Biology}

Biology aims to understand the complex nature of life at the molecular level. For this purpose, computational biology employs {\bf gene regulatory network}~\citep{levine2005gene} as a tool to study systems of molecular interactions that govern the expression of genes within cells. These networks encode which genes are turned on or off within the cells at a specific time, and in response to biological signals~\citep{kauffman1969metabolic, karlebach2008modelling}. Within gene regulatory networks, higher-order interactions have been shown to enable a finer-grain control over gene expression and cellular functions~\citep{lee2013transcriptional}. The dynamics of gene regulatory networks often involve interactions between multiple genes, transcription factors, and other regulatory elements, leading to a cascade of biological effects, shaping the dynamical behaviors of cellular systems~\citep{davidson2010regulatory}. These interactions are expressed as non trivial regulatory feedback loops involving a synergy between multiple genetic and epigenetic entities~\citep{lee2013transcriptional}. For instance, the epigenetic modifications that occur at multiple levels of DNA regulation form a complex interplay with gene expression~\citep{bird2007perceptions}. 

\

At the core of higher-order interactions lies the notion of {\bf complexes}. These mathematical structures serves as a combinatorial domains naturally able to represent higher-order interactions in complex systems. While the formal definition of (simplicial and cell) complexes and signals will be provided later in the thesis using algebraic topology~\citep{hatcher2005algebraic}, an informal understanding of these ideas  will suffice the current discussion.

\begin{figure}[!htb]
    \centering
    \includegraphics[width=.75\textwidth]{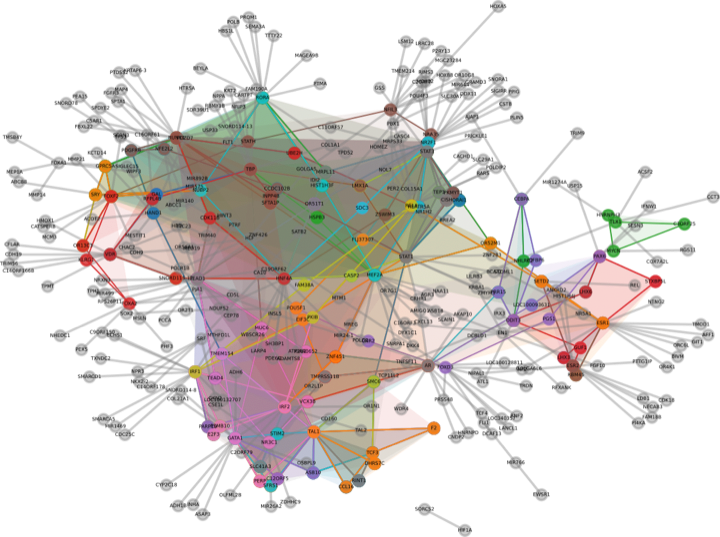}
    \vspace{-2pt}
    \caption{Gene Regulatory Complex}
    \label{fig:intro:gene_regulatory_complex}
\end{figure}

\begin{definition}[Complex (informal)]
    A complex $\xph$ is a mathematical tool for capturing how entities relate and interact. It consist of a set of nodes $\V$ and a structured collection $\mathsf{S}$ denoting the different ways they connect. Here, a $k$-th order interaction is represented by an ordered collection of $k+1$ nodes $\sigma^k$ called \textit{k-cell}.
\end{definition}

In this framework, a single node can be a standalone point; two nodes might connect as a line, symbolizing a second-order interaction; three nodes might form a triangle, indicating a third-order relationship, and so on~(\Cref{fig:intro:gene_regulatory_complex}).

\

Therefore, complexes can improve the representation power of gene regulatory networks by encoding genes as nodes, and $k$-cells as interactions among $k+1$ genes~\citep{berwald2013critical}. In this way, the set $\mathsf{S}$ contains different types of hierarchical relations. This structure takes the name of {\bf gene regulatory complex}, and constitutes a principled way to model genes and higher-order relationships among them which has revealed a landscape of attractors and bifurcations that govern cellular differentiation and response to environmental stimuli~\citep{perkins2017visualizing}.

\begin{proposition}
    The dynamics of genes interacting in higher-order feedback loops can be naturally exploited through gene regulatory complexes while simple networks might miss them~\citep{masoomy2021topological}.
\end{proposition}

To further model the dynamics of gene regulatory complexes, it is necessary to introduce the notion of regulatory functions, which will be represented as signals attached on each $k$-cells.

\begin{definition}[Regulatory Functions]
    For a $k$-cell $\sigma^k$ in a gene regulatory complex, the regulatory function $f_{\sigma^k}$ maps the state of the genes to a new state, capturing the combined effect of their interactions:
\[
    f_{\sigma^k}:\{0,1\}^{k+1} \rightarrow \{0,1\}
\]
\end{definition}

Where the domain represents the gene states (e.g., on/off or expressed/silenced) and the codomain captures the resulting state from their interaction. Notice that the binary framework for gene states offers a simplified abstraction. However, real-world gene expressions exhibit a broad spectrum of gene expression states which can manifest with arbitrary degrees of freedom. While this model serves as a starting point, advanced constructs can provide a more fine-grained gene expression profile.

\paragraph{Biological Implications}
Interactions captured by these higher-order cells are fundamental to various biological phenomena. For example, epigenetic modifications often result from the complex interplay of multiple genes and regulatory proteins, and can be expressed via specific configurations~\citep{bird2007perceptions}. Moreover, the landscape of attractors and bifurcations in the gene regulatory network dynamics, essential to cellular differentiation and response, can be more appropriately described considering these higher-order interactions~\citep{kauffman1969metabolic}.

\begin{proposition}\label{prop:alterations_in_grns}
Disruptions in higher-order interactions, represented by alterations in a gene regulatory complex, can lead to pathological states~\citep{vogelstein2013cancer}.
\end{proposition}

{\em By incorporating higher-order interactions via topological constructs is it possible to have a clear comprehension of the delicate balance of gene regulation. This perspective not only enhances the understanding of the regulatory processes but also opens for improving therapeutic approaches that target these higher-order interactions}~\citep{sever2015signal}.

%% file: Include/Chapters/1_Introduction/1.2.2_tnns_in_neurosci.tex
\subsubsection{Network Neuroscience}

The extraordinary complexity of neuronal connectivity shapes emotions, cognitive processes, and fundamentally, the essence of human experience. The human brain, {\em composed of approximately $86$ billion neurons}, forms a vast network of neurons linked together through synapses. Therefore, neural reactions are not just a random occurrence, but rather the result of elaborated labyrinths of neurons being activated via signals mediated by synapses. These reactions are denoted as {\bf neural pathways}. Such pathways are shaped mostly by past experiences, genetic predispositions, and environmental factors~\citep{kandel2001molecular}. To study functional and structural properties of such pathways in brain networks, the field of {\bf network neuroscience}~\citep{bassett2017network} aims to provide a framework from the perspective of graph theory. However, through the analysis of neural activity of large-scale human brain networks it has been recognized that the brain's functions are deeply rooted in the collective actions of several neurons rather than dyadic activity~\citep{petri2014homological, giusti2016two, reimann2017cliques}.

\begin{definition}
    A \textit{higher-order interaction} in a neural network refers to a synchronized activity ensemble of $n : n > 2$ neurons, where their combined activity cannot be reduced with the sum of their pairwise interactions.
\end{definition}

\begin{figure}[ht]
    \centering
    \includegraphics[width=.9\textwidth]{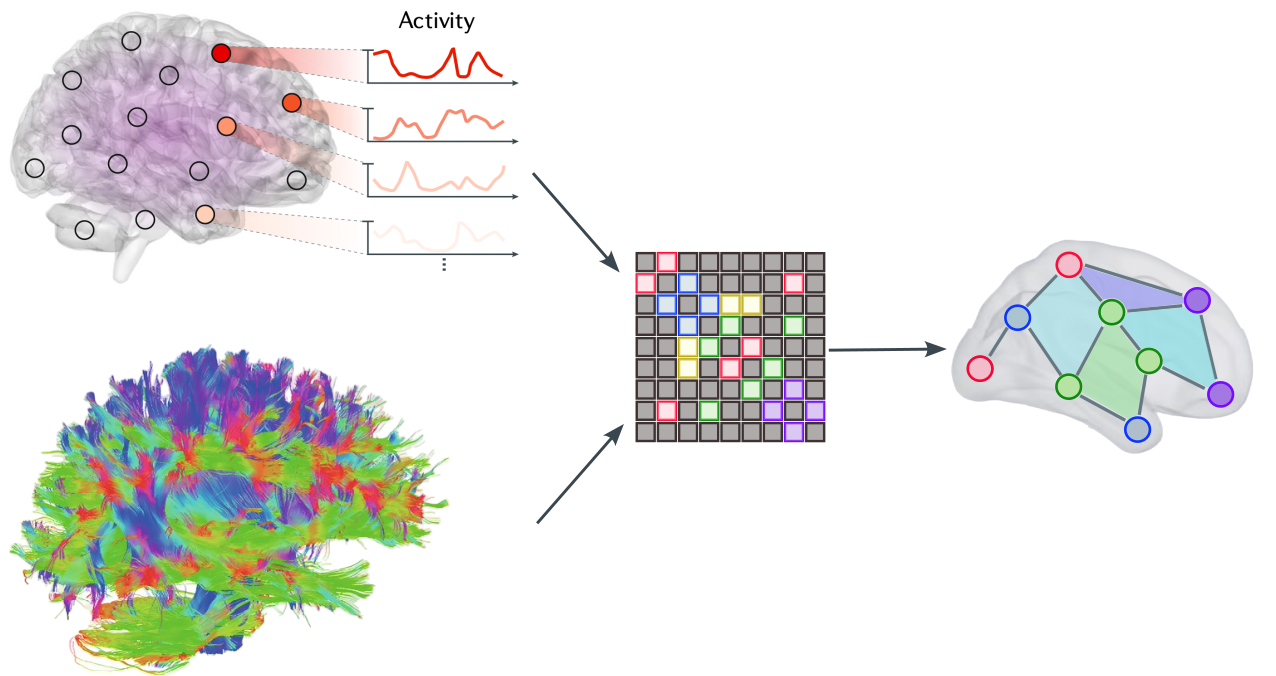}
    \caption{An illustration of a brain complex built from structural and functional neural patterns. This represents of how complex cognitive processes, such as memory formation, might emerge. Adapted from~\cite{lynn2019physics}.}
    \label{fig:intro:brain_complex}
\end{figure}

As visualized in~\Cref{fig:intro:brain_complex}, a group of neurons can form a complex where each node represents a neuron and higher-dimensional groups are associated to higher-order interactions. This structure elegantly captures the multi-neuronal patterns of activation.

For example, consider a triplet of neurons $A$, $B$, and $C$. If neurons $A$ and $B$, and neurons $B$ and $C$ have pairwise exchange of signals during certain cognitive processes, it does not necessarily imply that $A$, $B$, and $C$ are part of a higher-order interaction. 

However, {\em a synchronized firing pattern displayed by all three neurons that cannot be obtained by simply aggregating their pairwise activities indicates a higher-order interaction}.

\begin{proposition}
    If a set of neurons exhibits a higher-order interaction, the collective dynamics of this set cannot be entirely described using the sum of all possible pairwise interactions among the neurons.
\end{proposition}

Formally, let $\V$ be a set of $n$ neurons. The collective dynamics of $\V$ can be represented as:

\begin{equation}
    D(\V) = \sum_{i=1}^{n} d(v_i) + \sum_{i \neq j} d(v_i, v_j) + \sum_{i \neq j \neq k} d(v_i, v_j, v_k) + \ldots + d(v_1, v_2, \ldots , v_n), \nonumber
\end{equation}

Where $d(v_i)$ is the activity of neuron $v_i$, $d(v_i, v_j)$ represents pairwise interaction of neurons $v_i$ and $v_j$,  $d(v_i, v_j, v_k)$  is a third-order interaction between neurons $v_i$, $v_j$ an $v_k$ while $d(v_1, v_2, ... , v_n)$ characterizes the higher-order interaction of all neurons in set $\V$.

The key observation here is that {\em the terms after $\sum_{i \neq j} d(v_i, v_j)$, are non-trivial and group dynamics should be considered when processing brain signals to gain deeper insights into the brain's functionality}~\citep{ohki2005functional, schneidman2006weak}.

%% file: Include/Chapters/1_Introduction/1.2.3_tnns_in_physics.tex
\subsubsection{Physics}

A similar paradigm of higher-order interactions can be observed in condensed matter physics, particularly in {\bf spin glasses}~(\Cref{fig:intro:spin-glass}), disordered magnetic systems with competing interactions presenting several metastable states, which are local minima in their energy landscape where the system can get trapped for extended periods~\citep{binder1986spin}. Grasping higher-order interactions in spin glasses is a key challenge for understanding their role in phase diagrams and dynamical behaviors of complex systems.

\begin{figure}[ht]
    \centering
    \includegraphics[width=0.5\textwidth]{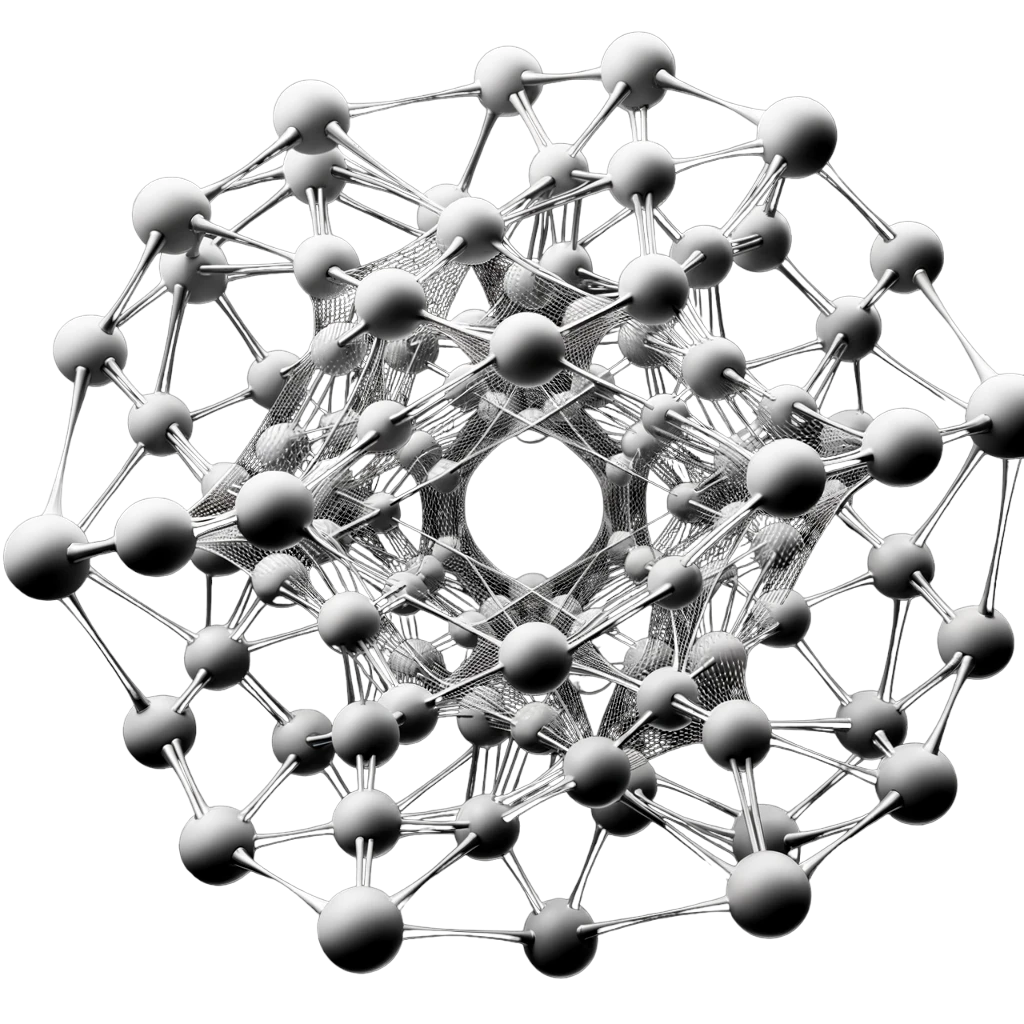}
    \caption{A spin glass lattice with nodes interconnected by edges for pairwise interactions, and polygons connecting multiple nodes to emphasize interactions among groups of spins}
    \label{fig:intro:spin-glass}
\end{figure}

\begin{definition}[Spin Glass]
    A spin glass is a disordered magnetic system characterized by the presence of random and competing ferromagnetic and antiferromagnetic interactions amongst the spin sites.
\end{definition}

Traditionally, these systems were described by pairwise interactions, often represented by the {\bf Ising model} -- a mathematical model in statistical mechanics that describes the magnetic properties of certain materials~\citep{ising1925contribution}. For a system with $N$ spins arranged on a $d$-dimensional lattice (e.g., a regular graph $\gph$), the Hamiltonian of the system that considers only pairwise interactions is given by:

\[
H_{2} = -\sum_{ \langle i j \rangle} J_{ij} \sigma_i \sigma_j 
\]

Here, $\sigma_i$ is a spin that can be oriented either upward, assuming the value of $+1$ or downward, assuming the value of $-1$. The value of $J_{ij}$ denotes the random interaction strength between spins $i$ and $j$. In particular, it represents {\bf cooperative or competitive behaviors amongst spins} $\sigma_i$ and $\sigma_j$. If $J_{ij} < 0$, the interaction between $\sigma_i$ and $\sigma_j$ is said to be antiferromagnetic, while $J_{ij} > 0$ denotes a ferromagnetic interaction between  $\sigma_i$ and $\sigma_j$. The case in which  $J_{ij} = 0$ happens if and only if $\sigma_i$ and $\sigma_j$ do not interact with each other.

However, while the Hamiltonian $H_{2}$ captures pairwise relationships among spins, providing insights into basic magnetization patterns, including higher-order interactions uncovers collective behaviors between spins that might alter their phase diagrams~\citep{edwards1975theory}. 
These diagrams map out different phases, or states of matter, that a system can exhibit under various conditions, such as temperature or pressure. For spin glasses, these phase diagrams can be profoundly shaped by interactions beyond just the pairwise ones.

\begin{definition}[Higher-Order Interaction in Spin Glasses]
    A \textit{higher-order interaction} in a spin glass system involves more than two spins simultaneously interacting, where the outcome cannot be factored into pairwise interactions.
\end{definition}

Incorporating three way relationships in spin glasses, leads to a Ising model of third-order interactions:

\[
H_{3} = -\sum_{\langle i  j  k \rangle} J_{ijk} \sigma_i \sigma_j \sigma_k 
\]

Where $J_{ijk}$ denotes the strength of the third-order interaction between spins $i$, $j$, and $k$. The notation $\langle i  j  k \rangle$ refers to a group of three arbitrary connected spins (i.e., three spins arranged on the vertices of a triangle).

\begin{proposition}
    Higher-order interactions alter the phase space of a spin glass system, leading to new metastable states and altered dynamical properties.
\end{proposition}

\paragraph{Physical Implications of Higher-Order Interactions in Spin Glasses: Critical Phenomena and Dynamical Responses}
Research suggests that the inclusion of higher-order interactions in spin glasses leads to profound implications in understanding their behavior, especially near critical points. For instance, while pairwise interactions predominantly influence the low-temperature phase of spin glasses, higher-order interactions can potentially modulate the dynamical responses, relaxation patterns, and aging phenomena of these systems~\citep{mezard1987spin}.

\begin{proposition}
    Higher-order interactions, when prominent, drastically affect the spin glass phase diagram, influencing critical temperatures, exponents, and susceptibility peaks.
\end{proposition}

Moreover, accounting for higher-order interaction in spin glassess can offer insights into a broader class of disordered systems such as the aforementioned networks of neurons~\citep{fuhs2006spin, tkacik2009spin}.

For general $k$-th order interactions, the Hamiltonian is given by:

\[ 
H_k = -\sum_{ \langle i_1  i_2  ...  i_k \rangle } J_{i_1 i_2 ... i_k} \sigma_{i_1} \sigma_{i_2} ... \sigma_{i_k},
\]
    
where $ J_{i_1 i_2 ... i_k} $ represent the strength and nature of the relationship between a set of $k$ spins interacting concurrently. The constraint $\langle i_1  i_2  ...  i_k \rangle$ ensures that each unique arrangement of $k$ spins is only considered once. It is importance to notice that, while models incorporating $k$-th order interactions provide a richer representation, they introduce non-trivial complexities, both computationally and analytically~\citep{newman1999monte}.

{\em An Ising model of spin glasses that accounts all the $k$-order interactions among spins is thus represented by sum of all the $k$ Hamiltonians $H_k$ that have a non-zero contribution to the total energy of the system}:

\[
H = \sum_{k} H_k
\]

{\em Such extensions capture the complexity behind spin glass systems more comprehensively, accounting for multi-spin interactions that are not reducible to pairwise ones. 
The interpretation of such interactions between spins can vary depending on the specific model or system under study, but they serve as a foundational mathematical tool for describing the complex behaviors observed in spin glasses.  }

%% file: Include/Chapters/1_Introduction/1.2.4_tnns_in_superchem.tex
\subsubsection{Supramolecular Chemistry}

Supramolecular chemistry~\citep{steed2022supramolecular}, often described as the {\bf chemistry beyond the molecule}, explores complex assemblies of molecules connected through a spectrum of weak bonds of varying strengths. These spontaneous secondary interactions include hydrogen bonding, dipole-dipole, charge transfer, van der Waals, and $\pi-\pi$ stacking interactions.

Supramolecular assemblies often exhibit complex chemical architectures and high-order self-assembly, giving rise to molecular machines~\citep{feringa2011molecular}, gas absorption~\citep{millward2005metal}, high-tech molecular sensing systems~\citep{allendorf2009luminescent}, nanoreactors~\citep{mattia2015supramolecular}, chemical catalysis~\citep{lee2009metal} and drug delivery systems~\citep{webber2017drug}. Intriguingly, molecular shape serves as a foundational design principle, thanks to the self-assembly~\citep{whitesides2002self} and self-healing~\citep{white2001autonomic} properties of supramolecules. These properties lead supramolecules to be categorized based on their curvature: zero (flat molecules), positive (bowl-shaped), and negative (saddle). 
Understanding these categories helps to distinguish the distinct behaviors and interactions of supramolecules in various contexts. These curvatures can restrict rotational and translational degrees of freedom in large stacked ensembles, leading to the formation of non-trivial scaling and directional graph-like architectures~\citep{lehn1995supramolecular}.

\begin{figure}[!htb]
     \centering
     \begin{subfigure}[t]{0.4\textwidth}
         \centering
         \includegraphics[width=.35\textwidth]{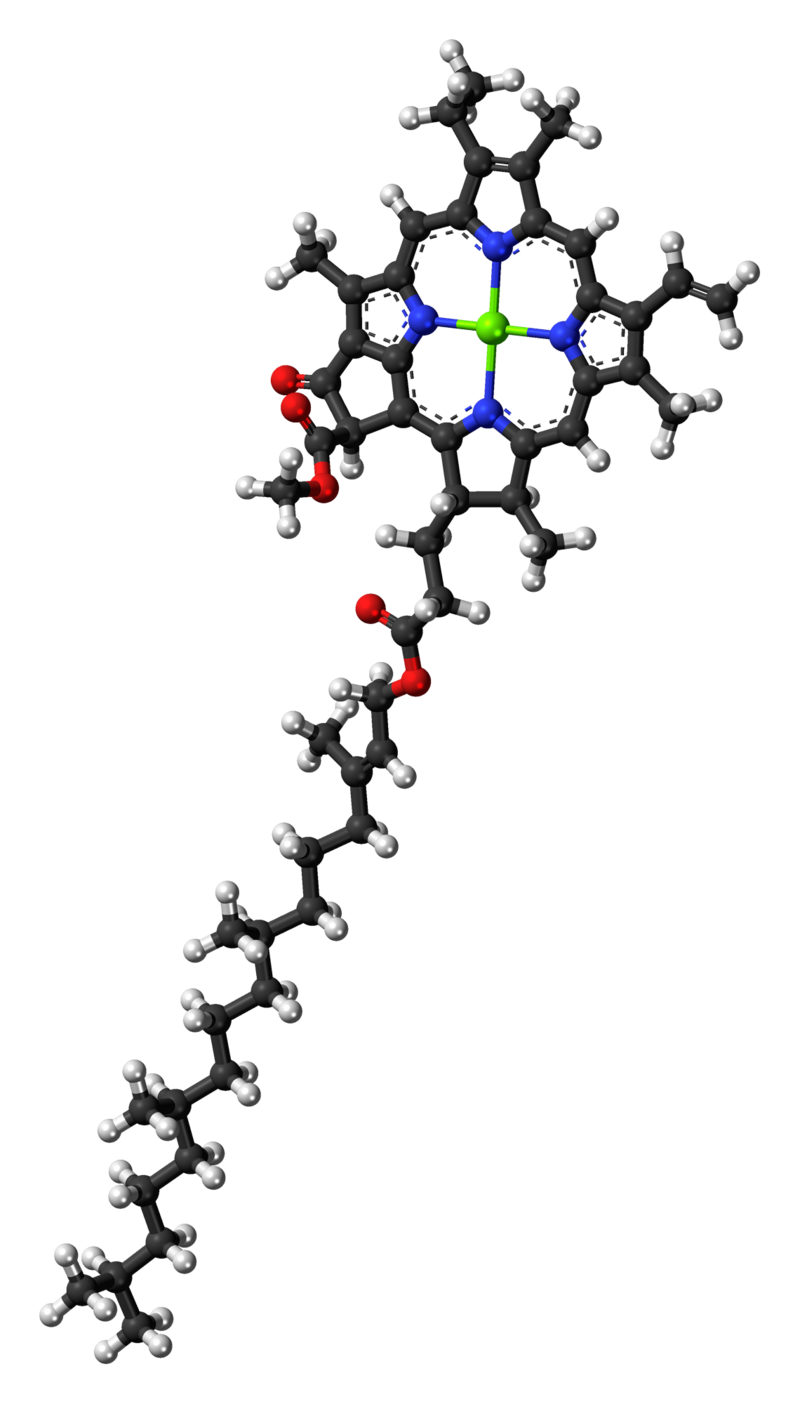}
    \caption{Molecular structure of Chlorophyll-a, the most common molecule in photosynthetic organisms.}
        \label{fig:intro:chlorophylla}
     \end{subfigure}
     \hspace{1cm}
     \begin{subfigure}[t]{0.4\textwidth}
         \centering
         \includegraphics[width=.6\textwidth]{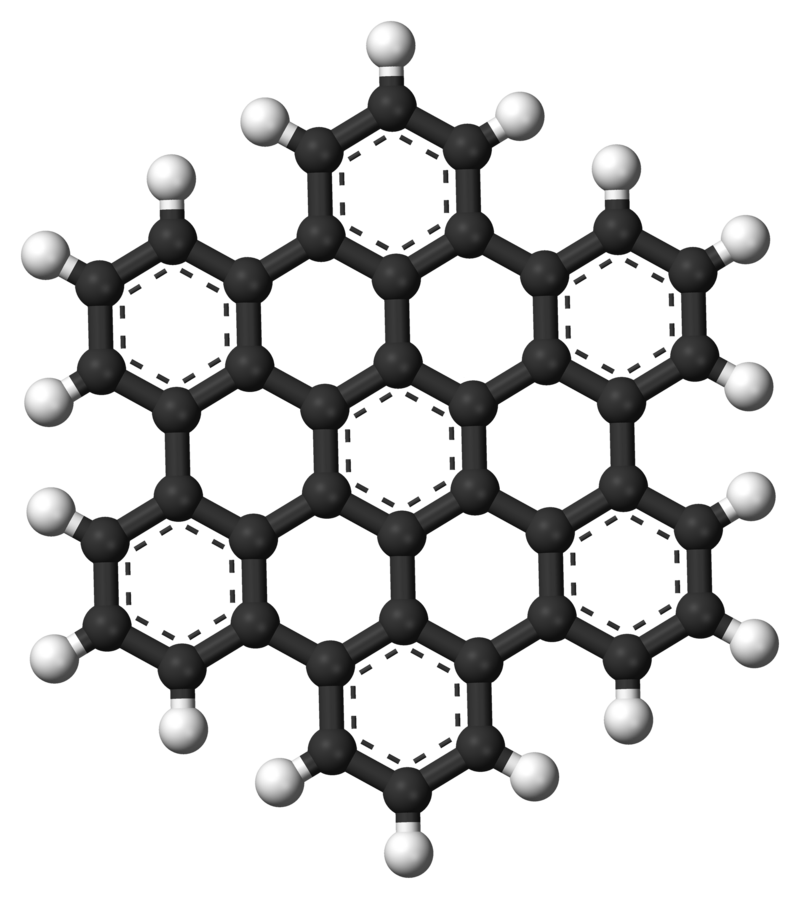}
    \caption{Molecular representation of hexabenzocoronene, a polycyclic aromatic hydrocarbon.}
    \label{fig:intro:pah}
     \end{subfigure}
     \caption{Illustrations of molecules in which long-range and higher-order interactions occurr spontaneously.}
\end{figure}

\

In supramolecular chemistry, long-range interactions refer to dependencies of molecular properties on elements far off from each other within a molecular system, typically spanning several bond lengths or more~\citep{gray2005long}. Of particular interest in this context are the interactions that arise in oxygenic photosynthesis. This is the process by which light energy is converted into chemical energy in the form of glucose or other sugars~\citep{barber2009photosynthetic}. This process is mediated by Chlorophyll-a~(\Cref{fig:intro:chlorophylla}), a cyclic tetrapyrrole molecule. Through its extensive conjugated $\pi$-system, Chlorophyll-a represents the basic building block of a photosystem. During photosynthesis, when a photon strikes a molecule of Chlorophyll-a, it excites an electron to a higher energy state. The energy produced is transferred from molecule to molecule within the light-harvesting complex via resonance energy transfer. Throughout this process, energy transfer manifests as a quantum-coherent phenomenon~\citep{engel2007evidence}, underlining the critical role of long-range interactions. Being able to capture them could lead to a positive impact in the development of efficient artificial photosynthetic systems~\citep{gust2001mimicking} and enhance solar energy technologies~\citep{green2021solar}. 

\

In addition to long-range interactions, higher-order interactions also play a fundamental role in chemical and biological processes. One example is the case of aromatic stacking. This process refers to the non-covalent interactions between aromatic rings, such as those found in the amino acid tryptophan or the nucleotide bases of DNA~\citep{hunter1990nature}, essential for biological processes including: {\em  protein folding, DNA/RNA structure, and ligand-receptor interactions}~\citep{meyer2003}. Another example of such interactions involves Polycyclic Aromatic Hydrocarbons (PAHs), molecules that have gained significant attention in astrophysics and astrobiology. PAHs~(\Cref{fig:intro:pah}) are thought to be among the most abundant and widespread organic molecules in the universe. They are identified in space via their unique infrared emission spectra~\citep{sandford2013infrared} and can form in the extreme conditions of space.

%% file: Include/Chapters/1_Introduction/1.4_research_objectives_and_contrubition.tex
\newpage

\section{Research Objectives, Outline and Contributions}

In the evolving landscape of deep learning, relational patterns present within data have become critical to tackle representation learning tasks on graph-structured data. 

With this perspective, this thesis explores the realm of Topological Neural Networks, highlighting the synergy between concepts from the field of algebraic topology to perform representation learning tasks on discrete topological spaces. The objectives of this work are structured to ensure both depth and breadth in understanding the higher-order interactions and their role in advancing neural architectures. Specifically, the goals of this thesis are: 

\begin{enumerate}
\item \textbf{Fundamentals:} Dive into the fields of graph theory and algebraic topology to understand how graph, simplicial complexes and cell complexes can be employed for constructing advanced neural architectures to perform representation learning tasks on topological spaces~(\Cref{chap:fundamentals}). 

\item \textbf{Challenges in Contemporary GNNs:} Dissect Graph Neural Networks (GNNs) to pinpoint their limitations, emphasizing the over-squashing phenomenon. By understanding the impact of network depth, width, and topology, the thesis sets the stage to demonstrate how topological approaches can mitigate the bottlenecks of graph neural networks when dealing with long-range interactions~(\Cref{chap:challenges}).

\item \textbf{Design Topological Extensions:} Develop novel architectures of topological neural networks as: Simplicial Attention Networks, Cell Attention Networks and Enhanced Topological Message Passing (CIN++), that integrate principles from algebraic topology to incorporate long-range and higher-order interactions~(\Cref{chap:tnns}).

\item \textbf{Empirical Evaluation:} Experimental assessments of the proposed models, confirming empirically the claims and comparing the proposed architectures with established state of the art methods in the field, highlighting the advantages and effectiveness of incorporating topological approaches in structured learning scenarios~(\Cref{chap:exp}).

\item \textbf{Broader Perspectives:} Implications of topological neural networks in various domains, while also discussing upon the limitations and provide future trajectories~(\Cref{chap:conclusions}). 
\end{enumerate}

\paragraph{Contributions}

This thesis grounds its contribution from five main researches: 

\begin{enumerate}
    \item {\em Francesco Di Giovanni, \textbf{Lorenzo Giusti}, Federico Barbero, Giulia Luise, Pietro Lio, and Michael
    Bronstein. On over-squashing in message passing neural networks: The impact of width, depth,
    and topology. In International Conference on Machine Learning, 2023.}~\citep{di2023over}. This work provides a theoretical understaning of one of the major bottleneck of message passing neural networks (the over-squashing phenomenon) from three different angles: the {\em width} (i.e., the number of hidden layers), the {\em depth} (i.e., the number layers) and the {\em topology} of the underlying graph. This work establish that while increasing the network's width can mitigate over-squashing, it does not aid in generalization, and depth (i.e., the number of hidden layers), on the other hand, is limited by vanishing gradients.  Most crucially, the paper highlights the profound impact of graph topology on over-squashing, revealing that it largely occurs between nodes with high commute times. In this study, L.G. and F.B. collaborated to empirically validate the theoretical concepts primarily developed by the CEO of oversquashing phenomena F.D.G.  Moreover, G.L., P.L., and M.B. provided insights with their expertise in the field as a senior supervisors of the research project.
    
    \item {\textbf{Lorenzo Giusti}*, Claudio Battiloro*, Paolo Di Lorenzo, Stefania Sardellitti, and Sergio Barbarossa.
    Simplicial Attention Networks\footnote{This work has been developed concurrently and independently from~\cite{goh2022simplicial}}.}~\citep{giusti2022simplicial}. This work extends the idea of masked self-attention for graph representation learning developed in Graph Attention Networks to data defined over simplicial complexes. In particular, the simplices have two distinct notions of neighbourhood: the upper and the lower ones, provided by the connectivity of the underlying domain. This implies that a simplex receives two types of messages, one coming from the upper neighbourhood and the other from the lower neighbouring simplices. To measure the relative importance of the information coming from messages sent by upper neighbouring simplices two independent masked self-attention mechanism are introduced in this work alongside a principled way to extract the harmonic component of a topological signal, according to the Hodge Theory. In this research, L.G. conceptualized and formulated the preliminary simplicial attention model. Further refinement of the model involved the participation of fratm C.B. which also wrote the method section of the work. L.G. implemented the experimental framework and executed the associated experiments. L.G. and C.B. equally contributed in design a model that respect the principles of the Hodge Theory. S.S. wrote the theoretical findings regarding the permutation equivariance and simplicial awareness of the model. P.d.L proposed the projection onto the harmonic subspace. S.B. provided a senior supervision to the overall research project.
    
    \item {\em \textbf{Lorenzo Giusti}, Claudio Battiloro, Lucia Testa, Paolo Di Lorenzo, Stefania Sardellitti, and Sergio
    Barbarossa. Cell attention networks, In International Joint Conference on Neural Networks (IJCNN).}~\citep{giusti2022cell}. This work further extends the masked self-attention scheme proposed in simplicial attention networks to introduce an architecture that tackles the task of graph representation learning by exploiting higher-order interactions provided by the rich connectivity structure provided by cell complexes. In particular, cell attention networks are able to lift data defined over graphs to features defined over the edges of a regular cell complex of dimension two. After the lifting operation, each layer of cell attention networks is composed by an attentinoal message passing scheme performed over the upper and lower neighbourhoods of the edges of the complex and a self-attention edge pooling procedure that selects the edges that contribute the most in the learning task using a differentiable pooling operation. In this study, L.G. was responsible for the design of cell attention networks and its conceptual framework. Additionally, L.G. developed the  experimental setup and carried out the related experiments. C.B. and L.T. contributed in writing a first version of the work. P.D.L., S.S., and S.B. provided a senior supervision of the work.
    
    \item  {\em \textbf{Lorenzo Giusti}, Teodora Reu, Francesco Ceccarelli, Cristian Bodnar, and Pietro Li\`{o}. CIN++:
    Enhancing topological message passing~\citep{giusti2023cin++}. } This work introduces CIN++, an extension of the Topological Message Passing scheme proposed with Cellular Isomorphism Networks (CINs), incorporating lower message exchanges within cell complexes. This augmentation enables better modeling of real-world complex interactions. The work also analizes  from a Weisfeiler and Lehman colouring procedure the faster convergence benefits in CINs by incorporating lower messages, allowing for direct ring interactions without waiting for upper messages. In this study, L.G. and C.B. were responsible for the initial conception and design of the enhanced topological message passing model. In this work, L.G. did not engage in studying color convergence speed between Cellular Isomorphism Networks and the method proposed in the research, which was done brilliantly by T.R. Also, L.G. conducted approximatively half of the experiments presented, the others were conducted by fratm F.C. The CEO of topological deep learning, C.B. alongside with the Jedi Master of life, P.L. provided a senior supervision to the work\footnote{For \textbf{any} concern about the relative contribution, feel free to reach out L.G. at \url{lorenzo.giusti@cern.ch}.}.
\end{enumerate}

Detailed mathematical proofs, supplementary information, and in-depth discussions supporting the content presented in the main chapters can be found in the appendices. Specifically,~\Cref{app:glossary}, contains the glossary of notation used throughout the thesis;~\Cref{app:on_oversq}, presents the proofs for the theoretical results for {\em Oversquashing in MPNNs};~\Cref{appendix:simmetries_tnns} provides a categorical approach to prove the symmetries of topological neural networks,~\Cref{app:comp_tan} provides a detailed analysis of the computational complexity and the number of learnable parameters involved in cell attention networks and ~\Cref{app:cinpp} contains the proof of CIN++'s expressivity alongside with insights on the enhanced topological message passing (CIN++) seen through the lens of Sheaf Theory.

%% file: Include/Chapters/2_Background/2_Background.tex
\chapter{Background and Related works}\label{chap:fundamentals}

\input{Include/Chapters/2_Background/2.1_graphs}
\input{Include/Chapters/2_Background/2.2_graph_sp}
\input{Include/Chapters/2_Background/2.3_graph_nns}

\input{Include/Chapters/2_Background/2.4_challenges}
\input{Include/Chapters/2_Background/2.5_simplicial_compelxes}
\input{Include/Chapters/2_Background/2.6_cell_complexes}
\input{Include/Chapters/2_Background/2.7_topo_sp}
\input{Include/Chapters/2_Background/2.8_topo_nns}
\input{Include/Chapters/2_Background/2.9_sota}

%% file: Include/Chapters/2_Background/2.1_graphs.tex
\section{Foundations of Graph Theory}\label{sec:background:graphs}

The mathematical abstraction that \textit{captures the essence of pairwise relationships between entities} takes the name of \textit{graph}. Graphs have been everywhere in various fields ranging from sociology (e.g., social networks,~\Cref{fig:back:real_graphs}, top-left), neuroscience (e.g., a brain network,~\Cref{fig:back:real_graphs}, top-right), natural sciences (e.g., molecular structures,~\Cref{fig:back:real_graphs}, bottom-left) to urban engineering (e.g., a transportation network,~\Cref{fig:back:real_graphs}, bottom-right). In real-world scenarios, an entity could symbolize a person in a social network, a neuron in a brain network, an atom in a molecule or a point of interest in an urban network. Moreover, connections could indicate friendships in social networks, synapses in brain networks, chemical bonds between two atoms in a molecule or roads in transportation networks~\citep{barabasi2013network}.

\begin{figure}[!htb]
    \centering
    \includegraphics[width=.8\textwidth]{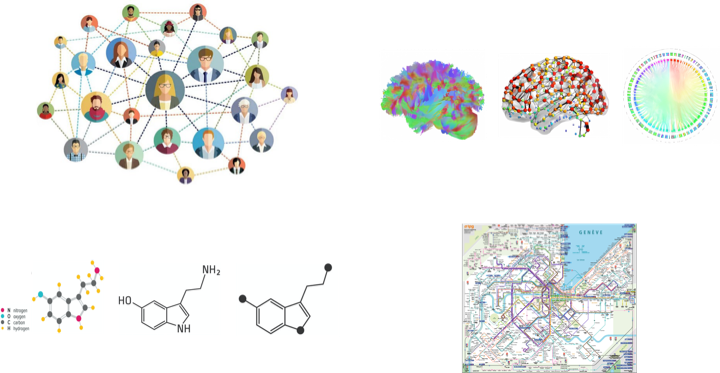}
    \caption{Illustrative examples of real-world scenarios where graphs play a key role: (top-left) A social network depicting friendships, (top-right) A brain network representing neural connections, (bottom-left) A molecular graph of a serotonin molecule showcasing atomic structures, and (bottom-right) The transportation network of Geneva, Switzerland. Adapted from~\cite{velickovic2021connected}.}
    \label{fig:back:real_graphs}
\end{figure}

\begin{definition}[Graph]
    A \textbf{graph} $\gph = (\V, \E)$ is a tuple composed of a set $\V$ of \textbf{nodes} (or vertices) representing the entities while relationships are encoded through a set $\E$ of \textbf{edges} (or links). For $u,v \in \V$, two nodes are \textit{connected through an edge} if $(u,v) \in \E$~\citep{bondy2008graph}. 
\end{definition}

\paragraph{Directedness} In $\gph$, the order of the node pair $(u,v) \in \E$ could be significant. It indicates a directed edge $e_i$ in which signals can only be propagated in one direction, from node $u$ to node $v$. When directionality matters, $\gph$ is said to be a \textbf{directed graph (Digraph)}~(\Cref{fig:back:direct}, right). In real-world applications, digraphs are fundamental structures to visualise and analyse neural information flows within brain networks~\citep{fornito2016fundamentals}. In this framework, each neuron corresponds to a node, and a directed edge $(u,v)$ represent a synapse where a pre-synaptic neuron $u$ transmits signals to a post-synaptic neuron $v$.

Conversely, when the sequence of the node pair $(u,v) \in \E$ is not significant, it results in an \textit{undirected edge} $e_i = (u,v)$ and $\gph$ is referred to be an \textbf{undirected graph}~(\Cref{fig:back:direct}, left). This implies that the relationship between nodes $u$ and $v$ is mutual, with no inherent order or direction. Undirected graphs are especially prevalent in modeling molecular structures, and study the topological properties of molecules, where atoms (nodes) are bound by chemical bonds (edges) without a notion of direction~\citep{trinajstic2018chemical}. In biochemistry, graph-based representation forms the foundation for a variety of applications, including the study of molecular dynamics, chemical reactivity, and structural biology.

\begin{figure}[!htb]
    \centering
    \includegraphics[width=.8\textwidth]{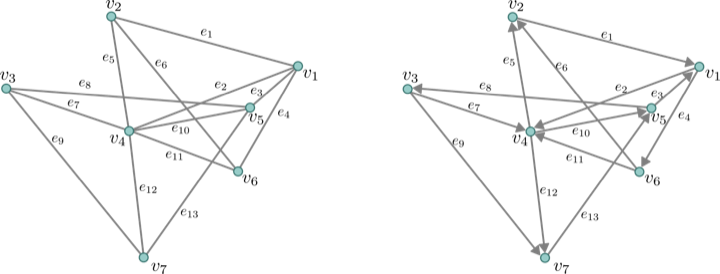}
    \caption{Comparative visualization of graph structures: (left) An \textit{undirected graph}, exemplifying mutual relationships without directionality, commonly used in molecular structures; (right) A \textit{directed graph (Digraph)}, representing one-way relationships, often observed in neural information flows within brain networks.}
    \label{fig:back:direct}
\end{figure}

\paragraph{Connectivity Representations}
The connectivity structure of $\gph$ is not limited to a visual characterization or a set-based definition. In fact, it can be precisely represented using the \textbf{adjacency matrix} $\Anorm$ and the \textbf{incidence matrix} $\Bnorm$, enabling a wide range of algebraic and analytical operations on graphs. 
\begin{definition}[Adjacency matrix]\label{def:adjacency}
For a graph $\gph=(\V,\E)$ with $n$ nodes, the adjacency matrix $\Anorm \in \mathbb{R}^{n\times n}$ have unitary entry in $\Anorm_{uv}$ if there is an edge between node $u$ and node $v$, and $0$ otherwise:

\begin{equation}\label{eq:adj}
    \Anorm_{uv} = \begin{cases} 
    1 & \text{if } (u,v) \in \E, \\
    0 & \text{otherwise}.
    \end{cases}
\end{equation}

\end{definition}

For undirected graphs, $\Anorm$ is symmetric (i.e., $\Anorm = \Anorm^\intercal$). In the case of directed graphs (or digraphs), $\Anorm$ can be asymmetric, indicating the direction of the edges. The adjacency matrix defined in \Cref{eq:adj} can also be generalized for graphs in which the edges are equipped with a scalar weight $w_{e_i}$ (weighted graphs) for $e_i = (u,v) \in \E$. In this case, the non-zero entries are replaced as: $\Anorm_{e_i} = w_{e_i}$~\citep{bondy2008graph}.  \emph{This work will focus mainly on connected, undirected and unweighted graphs.}

It exists a normalized representation of $\Anorm$, denoted as $\tilde{\Anorm} = \Dnorm^{-1/2} \Anorm \Dnorm^{-1/2}$,  where $\Dnorm$ is the degree matrix, a diagonal matrix such that $d_u$ is the \textit{degree} of node $u$, the number of its incident edges. $\tilde{\Anorm}$ is necessary to mitigate the influence of node degrees, thus allowing for a more uniform influence distribution across nodes in various graph algorithms~\citep{chung1997spectral}. Particularly, it is preferrable to use $\tilde{\Anorm}$ in contexts where the scale or magnitude of node connections could induce biases, ensuring that the intrinsic topology of the graph is preserved without being dominated by high-degree nodes like spectral clustering~\citep{von2007tutorial} or graph convolutional networks~\citep{kipf2017graph}.

To capture the topological characteristics of $\gph$ beyond the adjacency structure, the {\em incidence matrix} acts as map between each node $u$ and the edges $e_i$ that have $u$ as one of its endpoints.

\begin{definition}[Incidence matrix]\label{def:incidence} The incidence matrix $\Bnorm \in \mathbb{R}^{n \times e}$ (where $e$ is the number of edges) encodes, for each edge, which nodes are the endpoints:

\begin{equation}\label{eq:boundary}
    \Bnorm_{ij} = \begin{cases} 
    1 & \text{if node } i \text{ is on the tail of edge } j, \\
-1 & \text{if node } i \text{ is on the head of edge } j, \\
0 & \text{otherwise}.
    \end{cases}
\end{equation}
\end{definition}

In \Cref{eq:boundary}, the rows of $\Bnorm$ correspond to the nodes, while the columns represent the edges. The non-zero entries in each column denote the two nodes connected by that particular edge. In directed graphs, positive and negative entries indicate the tail and head of each directed edge, respectively. In~\Cref{fig:back:matrices} it is shown an example of a graph $\gph$ alongside its adjacency matrix $\Anorm$ and its incidence matrix $\Bnorm$.

\begin{figure}[!htb]
    \centering
    \includegraphics[width=.9\textwidth]{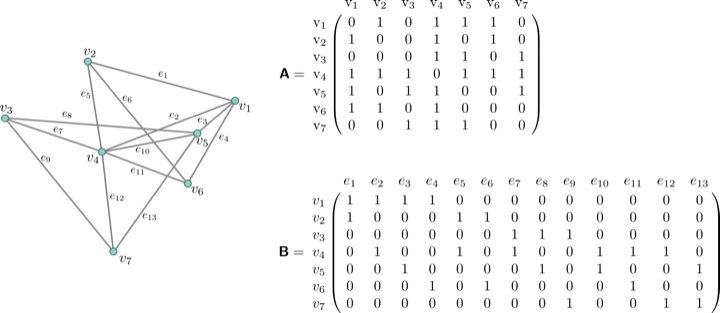}
    \caption{Algebraic Representations of an Undirected Graph: (left) Graph $\gph$; (top-right) its adjacency matrix $\Anorm$ and (bottom-right) unsigned incidence matrix $\Bnorm$.}
    \label{fig:back:matrices}
\end{figure}

Both the adjacency matrix $\Anorm$ and the incidence matrix $\Bnorm$ not only provide a structured way to visualize the graph's connectivity but are key components in applications of graph theory, such as determining the presence of specific subgraphs or analyzing graph properties and behaviors~\citep{barabasi2013network}.

\paragraph{Spectral Graph Theory}

Understanding the spectral properties of these matrices provides deep insights of features such as connectivity, clustering, and centrality, as well as the overall structural patterns within the graph. For instance, the spectrum of the adjacency matrix can reveal properties related to the graph's connectivity, its diameter, and even community structures within the graph. Similarly, building on this spectral framework, let $\Lnorm$ be the \textbf{Laplacian matrix}, a linear operator that had a key role in advancing the fields of spectral graph theory~\citep{chung1997spectral}, graph signal processing~\citep{shuman2013emerging} and acts as a bridge towards graph neural networks~\citep{gama2020graphs}.

\begin{definition}[Laplacian matrix]\label{def:laplacian}
    Given a graph $\gph = (\V, \E)$ having adjacency matrix $\Anorm$ and incidence matrix $\Bnorm$, the Laplacian matrix is defined as:
    \begin{equation}
        \Lnorm = \Bnorm \Bnorm^\intercal = \Dnorm - \Anorm.
    \end{equation}
\end{definition}

The Laplacian matrix, is a symmetric, positive semi-definite real matrix with non-negative eigenvalues $0=\lambda_0 \leq \lambda_1 \leq \lambda_2 \leq \ldots \leq \lambda_{n-1}$. The corresponding eigenvectors are denoted by $\mathbf{u}_0, \mathbf{u}_1, \ldots, \mathbf{u}_{n-1}$. Its eigendecomposition, $\Lnorm = \Unorm \boldsymbol{\mathsf{\Lambda}} \Unorm^\intercal$ provides insight into many graph properties, such as connectivity and expansion. As for the adjacency matrix, the Laplacian matrix admits a normalised representation: $\tilde{\Lnorm} = \Dnorm^{-1/2} \Lnorm \Dnorm^{-1/2} = \Inorm - \tilde{\Anorm}$, where $\Inorm$ is the {\em identity matrix}.

\begin{figure}[!htb]
    \centering
    \includegraphics[width=.5\textwidth]{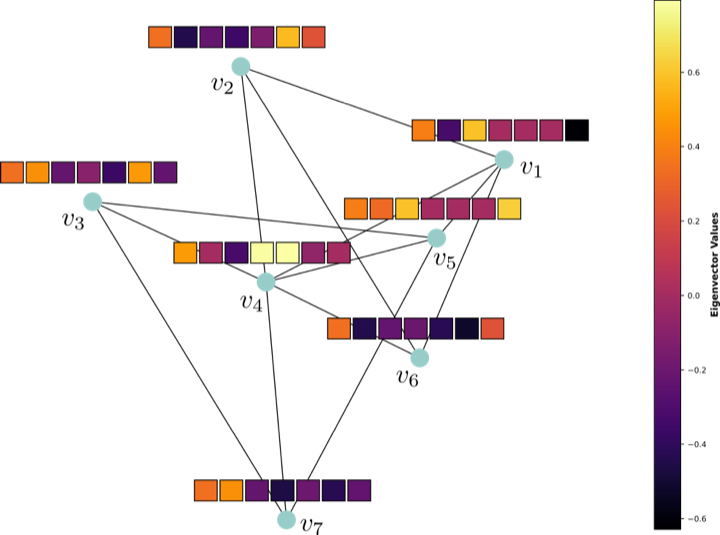}
    \caption{In a graph $\gph$, the set of orthonormal eigenvectors $\Unorm$ of the graph Laplacian $\Lnorm$ provide a unique fingerprint regarding the position of the node within the graph.}
    \label{fig:back:eigenvects}
\end{figure}

The spectral decomposition of the Laplacian matrix holds a wide range of applications. One of such is the \textit{spectral clustering}~\citep{von2007tutorial}. In particular, the eigenvalues ($\lambda_i$) and their associated eigenvectors ($\mathbf{u}_i$) reveal a low-dimensional fingerprint that reflects the community structure of the graph~(\Cref{fig:back:eigenvects}). The clustering is then obtained by applying a standard clustering algorithm, like {\em $k$-means}~\citep{macqueen1967some, lloyd1982least}, on the eigenvectors corresponding to the smallest non-zero eigenvalues, revealing clusters "hidden" in a graph. For example, as shown in~\Cref{fig:clust}, data scattered in a circular shape with a cluster at its center, traditional clustering methods might struggle, but spectral clustering can unveil the circle's structure and identify both clusters distinctly~\citep{ng2002spectral}.

\paragraph{Connectivity, Expansion and Cheeger’s Inequality:}

\begin{figure}[!t]
     \centering
     \begin{subfigure}[b]{0.45\textwidth}
         \centering
         \includegraphics[width=\textwidth]{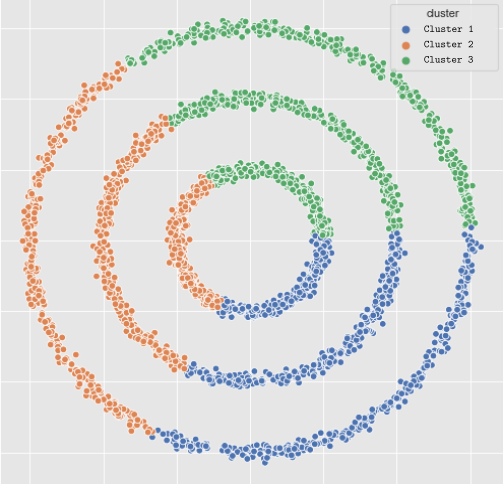}
         \caption{Standard k-means clustering.}
         \label{fig:kmeans}
     \end{subfigure}
     \hspace{5pt}
     \begin{subfigure}[b]{0.45\textwidth}
         \centering
         \includegraphics[width=\textwidth]{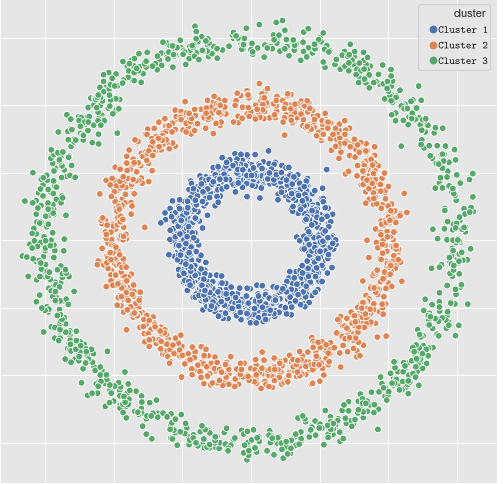}
         \caption{Spectral clustering.}
         \label{fig:spectral_clust}
     \end{subfigure}
    \caption{Comparison between (a) standard k-means clustering  and (b) spectral clustering  for a set of data with three distinct clusters formed by three nested circles. While k-means struggles to identify the true structure of the data, spectral clustering succeeds in revealing the patterns.}
    \label{fig:clust}
\end{figure}

The smallest eigenvalue $\lambda_0$ is always 0, and its multiplicity corresponds to the number of connected components in the graph. Moreover,the smallest positibe eigenvalue of the Laplacian matrix, $\lambda_1$ is called the \textbf{spectral gap} and is proportional to a measure of the graph's connectivity. Specifically, the smaller $\lambda_1$ is, the less connected the graph is. This is because a small$\lambda_1$ signifies a large spectral gap, indicating sparse connections between nodes. Conversely, if $\lambda_1$ is large, it means the spectral gap is small, suggesting a well-connected graph~\citep{chung1997spectral}. The spectral gap, is often used to gauge the graph's expansion properties via a quantity known as the \textbf{Cheeger constant}~\citep{cheeger69lower}.

\begin{definition}[Cheeger constant]
For a graph $\gph$, the Cheeger constant is
\begin{equation}
    h(\gph) = \min_{\mathsf{U}\subset \V}\frac{\lvert \{(u,v)\in\E: u\in \mathsf{U}, v\in \V\setminus \mathsf{U}\}\rvert}{\min(\mathrm{vol}(\mathsf{U}),\mathrm{vol}(\V\setminus \mathsf{U}))},
\end{equation}
\end{definition}
\noindent 

where $\mathrm{vol}(\mathsf{U}) = \sum_{u\in\mathsf{U}}d_u$, with $d_u$ the degree of node $u$. In partiuclar, a profound relationship between the eigenvalues of $\Lnorm$ and the expansion propoerties of $\gph$ is known as \emph{Cheeger inequality}: $\lambda_1 / 2 \leq h(\gph) \leq \sqrt{2 \lambda_1}$~\citep{cheeger69lower}. The previous result provides a connection between the algebraic properties of a graph through its eigendecomposition and its combinatorial structure via its expansion properties. {\em The smaller $h(\gph)$ is, the more it intimates the presence of a discernible bottleneck—illustrating two predominant node clusters sparsely interconnected. Conversely, a large value of $h(\gph)$ underscores a ubiquity of interconnections irrespective of any conceivable separation of the set of nodes, signifying the graph's resistance to simple partitioning.} 

%% file: Include/Chapters/2_Background/2.2_graph_sp.tex
\newpage
\section{Graph Signal Processing}\label{sec:gsp}

In real world graphs, nodes are often equipped with information that specify certain features of the entities they represent. For instance, let $G=(\V, \E)$ be a graph representing a social network. Each user is represented as a node $v \in \V$ and the profile information of such users as a vector $\xnorm_v$ capturing their interests, activities, or demographic details~\citep{konstas2009social}. In a molecular graph, the signal on each node (atom) might outline various atomic characteristics such as atomic weight, charge, or hybridization state~\citep{gilmer2017neural}.

\begin{figure}[!htb]
    \centering
    \includegraphics[width=.5\textwidth]{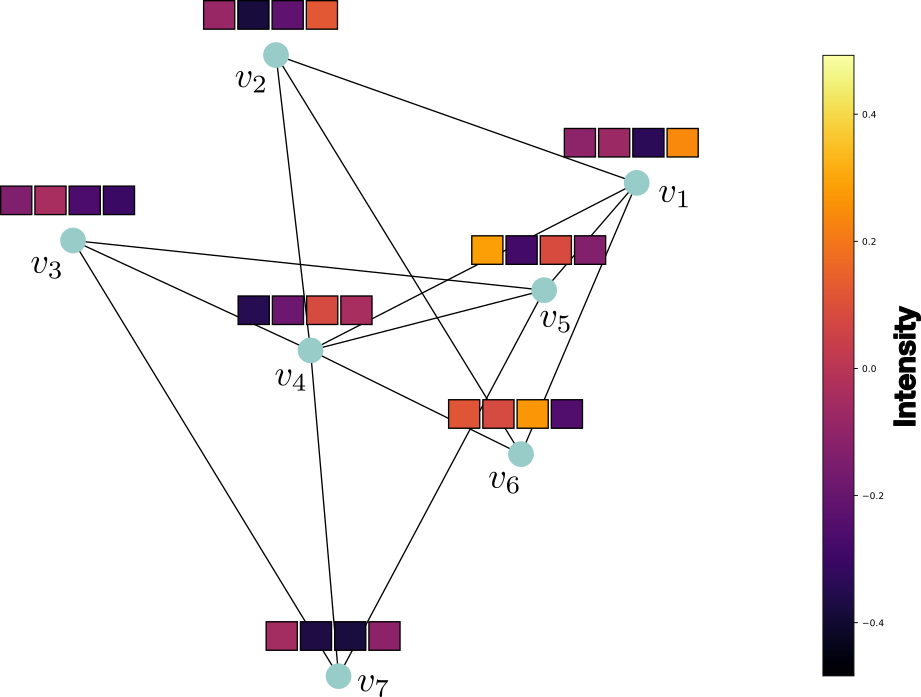}
    \caption{Illustration of a graph signal on a graph $\gph$.  Each node $v$ is associated with a four-dimensional feature vector $\mathbf{h}_u$.}
    \label{fig:back:graph_signal}
\end{figure}

In applied graph theory, this is achieved by extending the notion of temporal or spatial domains to a signal onto the domain defined by the graph topology. 

\begin{definition}
    A \textbf{graph signal} is a function, $s: \V \rightarrow \mathbb{R}^d$, that maps each node $v \in \V$ of a graph $\gph=(\V,\E)$ to a vector $\xnorm_v \in \mathbb{R}^d$. 
\end{definition}

In simpler terms, it assigns a $d$-dimensional vector to each node of $\gph$, thus enriching the node with additional information or features~\citep{shuman2013emerging}. In~\Cref{fig:back:graph_signal} it is shown a pictorial example of a four dimensional graph signal. Mathematically, if $\gph$ has $n$ nodes, a graph signal can be represented as a matrix $\Xnorm \in \mathbb{R}^{n \times d}$, where each row corresponds to the feature vector associated with a node. Later, it will be shown that this representation aligns well with graph representation learning frameworks, where node classification~\citep{kipf2017graph}, graph classification~\citep{xu2019powerful}, or link prediction~\citep{zhang2018link} grounded in the fact that $\Xnorm$ respects certain symmety properties.

By considering graph signals, one can perform graph-based signal processing, combining traditional signal processing techniques with the topological and structural characteristics of graphs.

\paragraph{Graph Shift Operator}
The \textbf{graph shift operator} defines localized operations on graph signals and it has been an integral component of the graph signal processing achievements. It plays a role analogous to the time shift in classical signal processing, encapsulating local interactions in the graph by capturing information from signal {\em shifts} across the neighboring nodes of a graph ~\citep{shuman2013emerging}.

\begin{figure}[!htb]
    \centering
    \includegraphics[width=\textwidth]{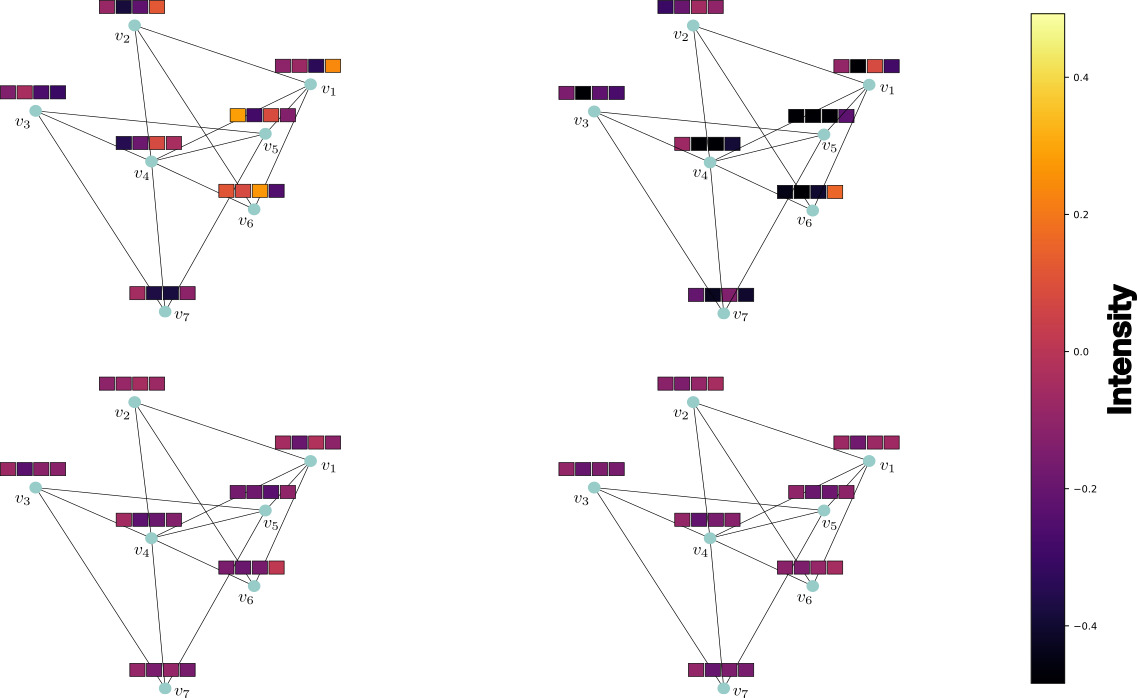}
    \caption{Visualization of the impact of the graph shift operator $\Snorm$  on the propagation of the signal across the neighbourhoods of $\gph$. The top-left figure shows $\Snorm = \Anorm$; the top-right figure shows $\Snorm = \widetilde{\Anorm}$; the bottom-left figure shows $\Snorm = \Lnorm$; the bottom-right figure shows $\Snorm = \widetilde{\Lnorm}$.}
    \label{fig:back:GSO}
\end{figure}

Formally, a graph shift operator is represented by a matrix $\Snorm \in \mathbb{R}^{n \times n}$ that incorporate the local interactions or connectivity structure of a graph $\gph=(\V,\E)$. The most prevalent choices for the graph shift operator are the adjacency matrix $\Anorm$, the Laplacian matrix $\Lnorm$ and their normalized version~\citep{ortega2018graph}. A pictorial overview of the effect of the particular choice for $\Snorm$ is depicted in~\Cref{fig:back:GSO}.

When a graph shift operation is applied to a graph signal $\Xnorm$, it transforms it as:
\begin{equation}
\Znorm = \Snorm \Xnorm
\end{equation}

The result, $\Znorm$, is a new graph signal where the value at each node is a localized combination of its neighbors' values, weighted by the structure captured in $\Snorm$. Intuitively, this can be thought of as a signal {\em propagation or diffusion} across the graph, mirroring the temporal shift of signals in the traditional signal processing paradigm~\citep{sandryhaila2013discrete}.

\

By leveraging powers of the graph shift operator  (i.e., $\Snorm^k$ for integer $k$), {\em one can model the effect of a filter at different local scales on the graph, capturing the influence of nodes further away in the graph topology}. This property makes the graph shift operator a versatile tool for designing  defining more complex graph signal processing operations that are sensitive to the underlying graph structure~\citep{hammond2011wavelets}. Its analogy with the time shift in classical signal processing links traditional methods with the complexities and nuances of processing signals on irregular, graph-structured data domains~\citep{gama2020graphs}.

\paragraph{Graph Fourier Transform}

The {\em graph Fourier transform acts as a bridge, from classical signal processing techniques towards their extensions to graph signals}. Analogous to the classical Fourier Transform, which decomposes signals into a basis of sines and cosines, the graph Fourier transform decomposes graph signals based on the eigenvectors of the graph Laplacian matrix~\citep{shuman2013emerging}.

Given the Laplacian matrix $\Lnorm$ of a graph $\gph=(\V,\E)$ and its eigendecomposition $\Lnorm=\Unorm \boldsymbol{\mathsf{\Lambda}}  \Unorm^\intercal$, where $\Unorm$ consists of the eigenvectors $\mathbf{u}_0, \mathbf{u}_1, \ldots, \mathbf{u}_{n-1}$ and $\boldsymbol{\mathsf{\Lambda}}$ is a diagonal matrix containing the corresponding eigenvalues $\lambda_0 \leq \lambda_1 \leq \lambda_2 \leq \ldots \leq \lambda_{n-1}$~\citep{chung1997spectral}.

These eigenvectors serve as the orthogonal basis functions in the graph spectral domain. Given a graph signal $\Xnorm$, its Graph Fourier Transform is given by:
\begin{equation}
    \hat{\Xnorm} = \Unorm^\intercal \Xnorm
\end{equation}
where $\hat{\Xnorm}$ represents the graph signal in the spectral domain.

The inverse Graph Fourier Transform, which retrieves the original graph signal from its spectral representation, is then:
\begin{equation}
   \Xnorm = \Unorm \hat{\Xnorm}
\end{equation}

This transformation has been critical to understand and process graph signals. It allows various graph signal processing tasks by considering operations to be performed in the spectral domain, which can provide insights into the signal's characteristics regarding the graph's connectivity by linking the graph's topological structure (through its Laplacian's eigenvectors) with the intrinsic properties of the signals residing on the graph~\citep{ortega2018graph}.

\

Additionally, similar to classical signal processing, operations like filtering can be efficiently achieved in the spectral domain, which, when mapped back to the vertex domain, translates to localized operations on the graph~\citep{hammond2011wavelets}.

\paragraph{Graph Filters}

Graph filters serve as essential tools in graph signal processing. They provide a dynamic way to understand and manipulate the propagation of information within the graph, much like classical filters operate on time or frequency-domain signals. These operators can modify graph signals either directly in the vertex domain or in the spectral domain by leveraging the eigendecomposition of the graph Laplacian~\citep{shuman2013emerging}. For example, is possible to represent and analyze how quickly and to which users this information disseminates in a social network as a graph signal. The graph filter then acts like a lens, allowing to 'zoom in' or 'zoom out' to see how strongly each user is influenced by the information, or to simulate what might happen if the speed or pattern of the spread changes. Given a graph signal $\Xnorm$, a filter $g$ in the spatial domain operates by directly modifying the signal values on the nodes, often accounting for their neighboring values. This is expressed as:

\begin{equation}
	\Znorm = g(\Snorm)\Xnorm,
\end{equation}

where $\Znorm$ is the filtered graph signal, and the operation $g(\Snorm)$ represents the local influence of neighboring nodes on the original signal values, encoding properties of the graph topology.

\begin{figure}[t]
    \centering
    \includegraphics[width=.9\textwidth]{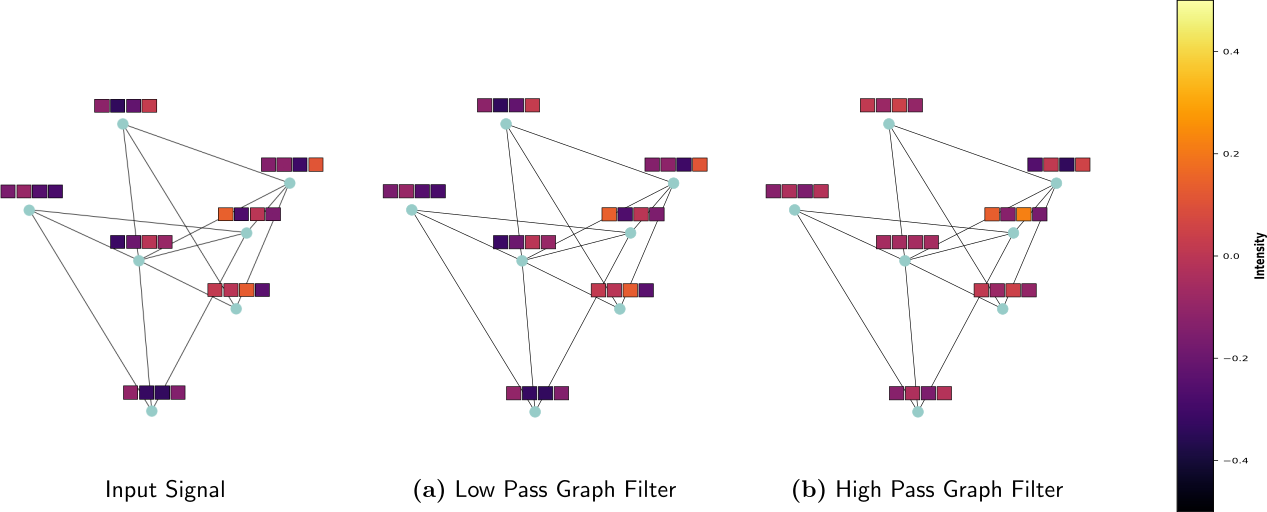}
    \caption{(a) Application of a Low Pass Filter (LPF) on a graph signal, retaining only the prominent, low-frequency features. (b) Application of a High Pass Filter (HPF), emphasizing the fine-grained, high-frequency features of the graph signal.}
    \label{fig:back:filters}
\end{figure}

On the other hand, filtering in the spectral domain involves manipulating the graph Fourier coefficients of the signal, an approach that echoes the filtering in the frequency domain in classical signal processing. Given the Graph Fourier Transform $\hat{\Xnorm} = \Unorm^\intercal \Xnorm$, a spectral filter $\tilde{g}$ is applied as:

\begin{equation}
	\hat{\Znorm} = \tilde{g}(\boldsymbol{\mathsf{\Lambda}}) \hat{\Xnorm},
\end{equation}

where $g(\boldsymbol{\mathsf{\Lambda}})$ is a diagonal matrix with entries formed by applying the filter function $\tilde{g}$ to the eigenvalues of $\Lnorm$. The filtered graph signal in the vertex domain is then recovered using the inverse Graph Fourier Transform: $\Znorm = \Unorm \hat{\Znorm}$ ~\citep{hammond2011wavelets}~(\Cref{fig:back:filters}).

Graph filters can be designed to enhance or suppress certain spectral components of the graph signal to manage tasks like noise reduction, signal smoothing, or feature enhancement. Notably, the design and application of these filters consider the graph's structure, making them adaptable to various graph topologies and catering to the specificities of the underlying data~\citep{ortega2018graph}, offering a powerful paradigm for processing and analyzing signals on graph structures, bridging the gap between classical signal processing techniques and the emerging challenges posed by data defined on irregular domains~\citep{sandryhaila2013discrete}.

\paragraph{Graph Convolution}

Graph convolution can be seen as an extension of classical convolution to graph-structured data when dealing with data that does not naturally fit into a regular grid. Instead of sliding a kernel across a regular grid as in the classical convolution, graph convolution operates by aggregating information from a node's local neighborhood, taking into account both the signal values and the underlying graph structure~\citep{shuman2013emerging}.

\begin{figure}[!htb]
    \centering
    \includegraphics[width=\textwidth]{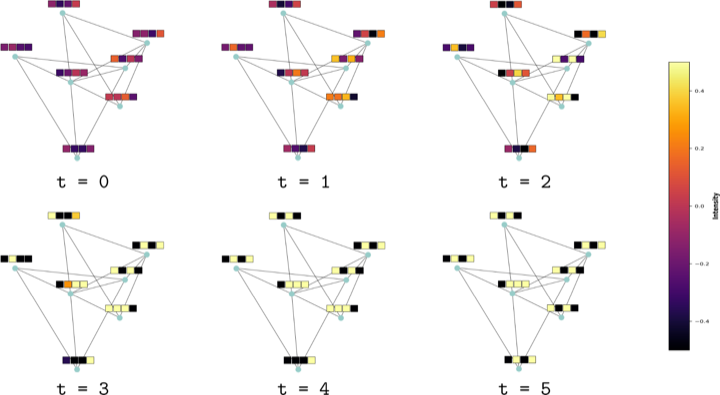}
    \caption{ The discrete diffusion process across the graph domain via Graph Convolution using the Laplacian as a graph shift operator. The Laplacian captures the local variations in the graph, and the convolution operation simulates the spread of information, coherently to how diffusion acts in physical systems. Notice how the signal $\xnorm$ starts to stabilizes to a steady state after a fixed $t_0$.}
    \label{fig:back:graph_diffusion}
\end{figure}

At its core, this operation defines how localized weights, analogous to those in a neural network kernel, interact with a graph signal. Given a graph signal $\Xnorm$ and a graph shift operator $\Snorm$, the graph convolution is typically expressed as a graph filtering operation, where the function $g(\Snorm)$, representing the graph filter, is a polynomial of the graph shift operator. 

\newpage

This polynomial expansion enables the aggregation of neighborhood information up to a specified degree, controlled by the polynomial's order, and its coefficients can be likened to the weights in a classical convolutional kernel~\citep{ortega2018graph}.
Mathematically,

\begin{equation}
    \Znorm = \sum_{t=0}^{T} g_t(\Snorm^t)\Xnorm,
\end{equation}

where $T$, representing the polynomial's order, serves as a conceptual measure of the {\em diffusion time}. This term indicates the reach of the convolution across the graph, specifying how far the information from a node is propagated through its neighborhood. The graph convolution can bee seen as a discrete analogous of the diffusion operation on curved surfaces via the Laplace operator~\citep{bronstein2017geometric}. It aggregates information from local neighborhoods, while the polynomial nature allows the convolution to consider information from extended neighborhoods (further hops away) by including higher-degree termsr~(\Cref{fig:back:graph_diffusion}).

This approach serves as the foundation for convolutional operations in Graph Neural Networks (GNNs)~\citep{bruna2013spectral, defferrard2016convolutional, kipf2017graph, hamilton2017inductive}. 

%% file: Include/Chapters/2_Background/2.3_graph_nns.tex
\newpage
\section{Graph Neural Networks}\label{sec:preliminaries_mpnns}
Let $\gph$ be a graph with nodes $\V$ and edges $\E$. The connectivity is encoded in the adjacency matrix $\mathbf{A}\in\mathbb{R}^{n\times n}$, where $n$ represents the number of nodes. Assume that $\gph$ is undirected, connected, and that there are features $\{\mathbf{h}_{v}^{(0)}\}_{v\in \V}\subset \R^{d}$. 
Graph Neural Networks (GNNs) are functions of the form $\mathsf{GNN}_{\theta}:(\gph,\{\mathbf{h}_v^{(0)}\}) \mapsto y_\gph$, with parameters $\theta$ estimated through training, where the output $y_\gph$ can be a node-level or graph-level prediction. The most studied class of GNNs, known as the Message Passing Neural Network ($\MPNN$)~\citep{gilmer2017neural}.

The $\MPNN$ computes node representations by performing $m$ independent message-passing rounds, formulated as:

\begin{equation}\label{eq:message-passing-scheme}
    \mathbf{h}_{v}^{\textsf{new}} = \mathsf{com}(\mathbf{h}_{v}, \underset{u \in \mathcal{N}(v)}{\mathsf{agg}}(\mathbf{h}_{u})),
\end{equation}

\noindent where  $\mathsf{agg}$ is some {\em aggregation} function invariant to node permutation, while $\com$ {\em combines} the node's current state with messages from its neighbours. Usually in $\MPNN$s, the aggregation takes the form:

\begin{equation}\label{eq:message-passing-operator}
    \underset{u \in \mathcal{N}(v)}{\mathsf{agg}}(\mathbf{h}_{u}) = \sum_{u} \mathsf{m} \big (\mathbf{h}_{u}, \mathbf{h}_{v}, \Snorm_{vu} \big ),
\end{equation}

\noindent where  $\Snorm\in\R^{n\times n}$ is a {\bf Graph Shift Operator}, meaning that $\Snorm_{vu} \neq 0$ if and only if $(v,u)\in \mathsf{E}$. Typically, $\Snorm$ is a (normalized) adjacency matrix that is also referred to as message-passing matrix. In~\Cref{eq:message-passing-operator}, $\mathsf{m}$ is the  {\bf message} function. In particular it is responsible to dispatch the information across the neighbourhoods. Although the particular choice of the message passing matrix $\Snorm$, the particular istance of the $\MPNN$ (i.e., GCN~\citep{kipf2017graph}, GAT~\citep{velivckovic2018graph}, SAGE~\citep{hamilton2017inductive}, GIN~\citep{xu2019powerful}) is fully determined by the choice of  $\mathsf{m}$ and  $\mathsf{com}$.

\begin{figure}[t]
    \centering
    \includegraphics[width=.3\textwidth]{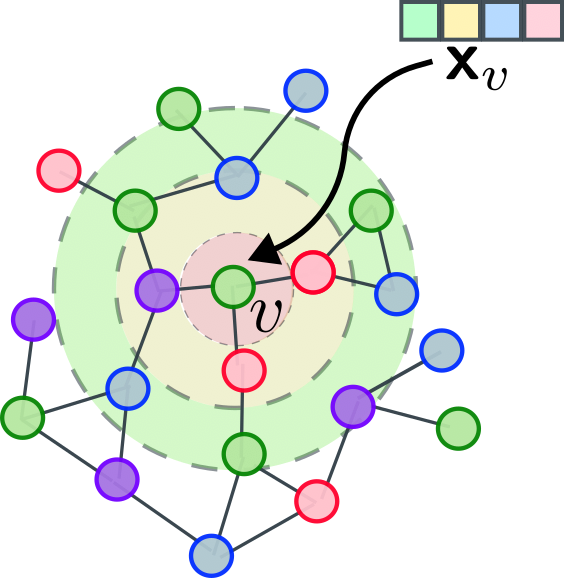}
     \caption{Illustration of a 3-hop receptive field of a node $v$ having features $\xnorm_v$. An $\MPNN$ must have at least three layers to include information coming from nodes $u$ that are no more that 3-hops away from $v$.}
     \label{fig:mpnn_rec_field}
\end{figure}

The common ground  of $\MPNN$s is that they all aggregate messages over the neighbours, such that in a layer, only nodes connected via an edge exchange messages~(\Cref{fig:mpnn_rec_field}).
\noindent 
%
This presents two advantages: $\MPNN$s  can capture graph-induced `short-range' dependencies well, 
and are efficient, since they can leverage the sparsity of the 
graph. 
Nonetheless, $\MPNN$s have been shown to suffer from a few drawbacks, including {\em limited expressive power} and {\em over-squashing}. The problem of expressive power 
stems from the equivalence of $\MPNN$s to the Weisfeiler-Leman graph isomorphism test 
\citep{xu2019powerful, morris2019weisfeiler}, which has been studied extensively 
\citep{jegelka2022theory}.  
%


%% file: Include/Chapters/2_Background/2.4_challenges.tex
\section{Challenges of Graph Neural Networks}\label{sec:oversquashing}

\paragraph{Long-Range Interactions} In an $\MPNN$, information from neighboring nodes is aggregated such that for a node $v$ to incorporate features from a distance $r$, the network requires at least $r$ layers~\citep{barcelo2019logical}~(\Cref{fig:mpnn_lr_interactions}). However, with the expansion of a node's receptive field, it has been observed that $\MPNN$s can lead to a phenomenon termed as {\em over-squashing}, where there is a potential loss of information~\citep{alon2020bottleneck}. 

\begin{figure}[!htb]
    \centering
    \includegraphics[width=.3\textwidth]{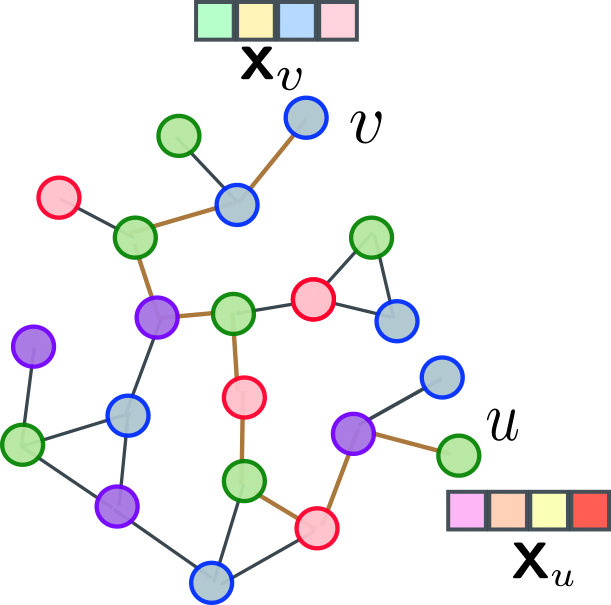}
     \caption{Pictorial overview of long-range interactions. Since the geodesic distance between $v$ and $u$, is equal to the diameter of the graph (i.e., $d_{\gph}(v,u) = 9$), an $\MPNN$ must have at least {\em nine} layers to include information coming from nodes $u$ when updating the representation of node $v$. This would causes $v$ to receive an exponential number of messages over-squashed into fixed size vectors, reducing the sensitivity of the underlying $\MPNN$.}
     \label{fig:mpnn_lr_interactions}
\end{figure}

\begin{proposition}[Sensitivity of $\MPNN$s]
    Consider an $\MPNN$ with a message-passing matrix $\Anorm$~(\Cref{eq:message-passing-operator}) and {\em scalar} features. Let also $v,u$ be a pair of nodes at distance $r$. The sensitivity of node features can be quantified as $\lvert \partial h_{v}^{(r)} / \partial h_{u}^{(0)}\rvert \leq c\cdot (\Anorm^{r})_{vu},$ 
\noindent with $c$ a constant depending on the Lipschitz regularity of the model. If 
$(\Anorm^{r})_{vu}$ decays exponentially with $r$, 
then the feature of $v$ is insensitive to the information contained at $u$. 
\end{proposition}
Moreover,~\citet{topping2022understanding} showed that over-squashing is related to the existence of edges with {\em high negative curvature}. Such characterization though only applies to propagation of information up to 2 hops. 

\paragraph{Higher-Order Interactions}\label{sec:group_interactions}

In graph representation learning, $\MPNN$s have focused on mutual node relationships, posing a challenge in modeling higher-order interactions. To understand this, consider $\mathbf{h}_{\mathsf{S}}$ as feature vectors representing interactions across subsets of nodes $\mathsf{S} \subseteq \V$ such that $\vert \mathsf{S} \vert = k$~\citep{majhi2022dynamics, bick2023higher}. To capture these interactions, one can aggregate features from a subset of nodes using a function such that $\mathbf{h}_{\mathsf{S}} = \agg({\mathbf{h}_{v}: v \in \mathsf{S}})$, where $\mathbf{h}_{\mathsf{S}}$ contains {\em collective state of nodes} in $\mathsf{S}$. 
As demonstrated by~\cite{perotti2015hierarchical}, the collective influence of nodes in a subset $\mathsf{S}$ on the entire graph $\gph$ can be quantified through the measure:
\begin{equation}\label{eq:interaction_strength}
I(\mathsf{S}) = \sum_{v \in \mathsf{S}} I(\mathbf{h}_{\mathsf{S}}; \mathbf{h}_{v}), 
\end{equation}

where $I(\mathbf{h}_{\mathsf{S}}; \mathbf{h}_{v})$ is the {\em mutual information} between $\mathbf{h}_{\mathsf{S}}$ and $\mathbf{h}_{v}$. If $I(\mathsf{S})$ is significantly greater than $0$, then the representation $\mathbf{h}_{\mathsf{S}}$ contributes beyond the individual information provided node representations.

\begin{figure}[!htb]
    \centering
    \includegraphics[width=.3\textwidth]{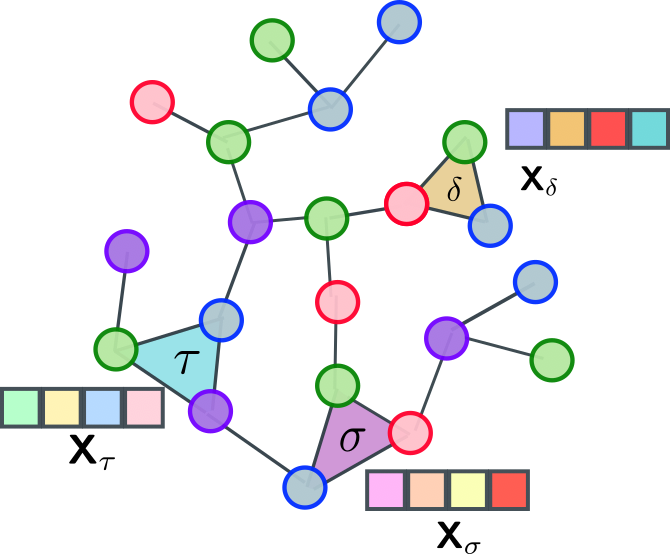}
     \caption{Visual intuition of higher-order interactions. Groups of nodes $\sigma, \tau$ and $\delta$  are equipped with features representing the state of the group. Notice that in this context, the features $\boldsymbol{\mathsf{X}}_\sigma, \boldsymbol{\mathsf{X}}_\tau, \boldsymbol{\mathsf{X}}_\delta \in \mathbb{R}^d$ cannot be reduced as the sum of the individual features attached to the nodes that compose $\sigma, \tau$ and $\delta$. }
     \label{fig:mpnn_ho_interactions}
\end{figure}

\begin{proposition}[Limitation of Pairwise Aggregations]
Let $\gph$ be a graph having a subset of nodes $\mathsf{S} \subseteq \V$. Let also $I(\mathsf{S})$ be the  information provided by the interaction among the nodes in $\mathsf{S}$. If $I(\mathsf{S}) > \varepsilon$, the total information provided by individual nodes in $\mathsf{S}$ do not fully captures the group dynamics of $\mathsf{S}$.  Moreover, when $I(\mathsf{S}) > \varepsilon$, MPNNs with only pairwise aggregations exhibit a drop in performance proportional to $\mathcal{O}(I(\mathsf{S}))$ in modelling the underlying phenomena.
\end{proposition}

As the discussion did not make specific assumptions about the choice of $\mathsf{S}$, the challenge lies in finding suitable groups $\mathsf{S}$ such that $\mathbf{h}_{\mathsf{S}}$ truly represent meaningful group interactions without considering all possible choices of $\mathsf{S} \subseteq \V$~\citep{benson2016higher}. A promising approach to identify meaningful node groups, like $\mathsf{S}$, is through the mathematical foundations of simplicial and cell complexes. These structures inherently model group interactions via their connectivity patterns~(\Cref{fig:mpnn_ho_interactions}). Moreover, message passing operations over simplicial and cell complexes can implement higher-order aggregations to group dependencies prevalent in complex systems naturally. 

%% file: Include/Chapters/2_Background/2.5_simplicial_compelxes.tex
\section{Simplicial Complexes}\label{sec:background:simplicial_complexes}

Simplicial complexes are mathematical objects able to capture the essence of continuity in topological spaces with a combinatorial framework. This thesis explores these structures by focusing on the connectivity properties provided by simplicial complexes. In particular, it is emphasized how these properties naturally model higher-order relationships among entities. In this context, simplicial complexes provide a generalization of graphs and expand upon the traditional idea of nodes and edges to encapsulate higher-dimensional relationships. These constructs are built by gluing collections of nodes into higher-order structures called \textit{simplices}. Like graphs, they have applications in a variety of domains, from computational topology~\citep{nanda21, edelsbrunner2022computational} to algebraic geometry~\citep{schenck2003computational}, and are particularly useful for modeling complex relationships between entities, such as multiple neurons firing together~\citep{giusti2016two} or multi-agent collaboration in computer science~\citep{munkres2018elements}. 

\begin{figure}[!htb]
    \centering
    \includegraphics[width=.9\textwidth]{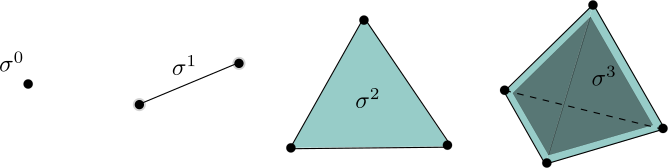}
    \caption{Simplices: node (0-simplex), edge (1-simplex), triangle (2-simplex), tetrahedron (3-simplex)}
    \label{fig:back:simplices}
\end{figure}

\begin{definition}[Simplex]\label{def:abstract_simplex}
Given a finite set of nodes $\V$, a \textit{$k$-simplex} is a collection $\sigma^k = \{v_1, \ldots, v_{k+1}\}$ of $k+1$ distinct elements of $\V$.
\end{definition}

For a finite set of nodes $\V$ situated in a $d$-dimensional real space $\mathbb{R}^d$, the simplices can be equipped with a geometric interpretation. Specifically, a geometric $k$-simplex represents the convex hull of $k+1$ nodes that are affinely independent. Affine independence in $\mathbb{R}^d$ denotes that no node in the set can be written as a linear combination of the others, ensuring the set spans a $k$-dimensional space for $k \leq d$~\cite{hatcher2005algebraic}. Consequently, is it possible to classify a nodes as a $0$-simplex, a line segment as a $1$-simplex, a triangle as a $2$-simplex, a tetrahedron as a $3$-simplex, and so forth~(\Cref{fig:back:simplices}). Intuitively, the dimension of a simplex $\sigma^k$ is $k$, which is one less than the number of its vertices. Given a simplex $\sigma^k$, subsets of its nodes that define lower-dimensional simplices. These are known as the \textit{faces} of $\sigma^k$. As shown in \Cref{fig:concept_of_face}, a $2$-simplex (triangle) $\sigma^2$, has three distinct edges as its faces. The combinatorial nature of simplices implies that higher-dimensional simplices have faces that represent all possible distinct node combinations of lower dimensions.

\begin{figure}[!htb]
    \centering
    \includegraphics[width=.3\textwidth]{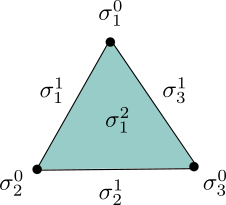}
 \caption{Depiction of the hierarchical face incidence relationships of a 2-simplex, $\sigma_1^2$ and its substructures. This simplex consists of three 1-simplices ($\sigma_1^1, \sigma_2^1, \sigma_3^1$) as its bounding edges. Each of these 1-simplices, in turn, is determined by two distinct 0-simplices as its endpoints. For example, $\sigma_1^1$ has $\sigma_1^0$ and $\sigma_2^0$ as its faces.}
    \label{fig:concept_of_face}
\end{figure}

\begin{definition}[Face]\label{def:face}
Given a $k$-simplex $\sigma^k$, a \textit{face} $\sigma^{k-1} \subset \sigma^k$ is a $(k-1)$-simplex obtained by omitting exactly one node from $\sigma^k$. In other words, $\sigma^{k-1} = \sigma^k \setminus \{v_i\}$ for some $v_i \in \sigma^k$.
\end{definition}

A simplex $\sigma^k$, will be referred as $\sigma$  if its dimension is clear from the context, or not relevant. Furthermore, the face incidence relation, will be referred as $\tau \face \sigma$, and reads as: "Simplex $\tau$ is a face of simplex $\sigma$".  Simplices serve as fundamental building blocks to represent multi-dimensional relationships between entities. Although individual simplices shed light on these interconnections, it is more common to deal with collections of interconnected simplices. Such assembly of simplices, spanning diverse dimensions, comes together in a structured framework known as \textit{simplicial complex}~(\Cref{fig:back:simplicial_complex}). This structure not only highlights the hierarchical structure among its components but also maintain specific topological consistencies~\citep{hatcher2005algebraic}.

\begin{figure}[!htb]
    \centering
    \includegraphics[width=.4\textwidth]{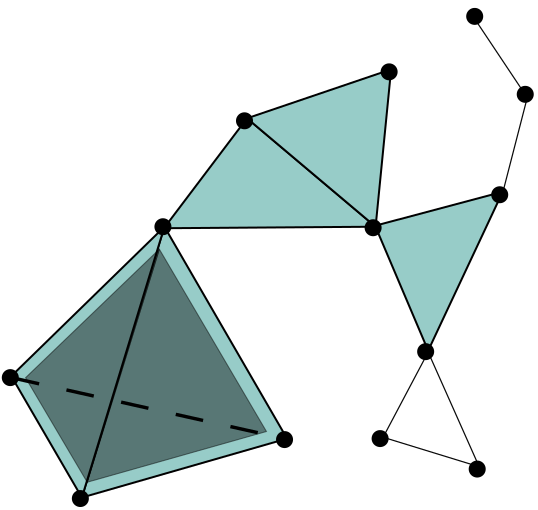}
    \caption{Geometric representation of a three-dimensional simplicial complex.}
    \label{fig:back:simplicial_complex}
\end{figure}

\begin{definition}[Simplicial Complex]\label{def:simplicial_complex_enhanced}
A simplicial complex $\kph = (\V, \mathsf{S})$ is a collection of \textit{simplices} $\mathsf{S}$ such that every face of any simplex in $\mathsf{S}$ must also belong to $\mathsf{S}$: $\sigma \in \mathsf{S}$ and $\tau \face \sigma \implies \tau \in \mathsf{S}$. Moreover, the intersection of two arbitrary simplices $\sigma, \tau \in \mathsf{S}$ is either empty or a face of both.
\end{definition}

Specifically, the collection $\mathsf{S}$ consists of sets of simplices of varying dimensions, such that $\mathsf{S} = \bigcup_{k=0}^{K} \Sigma^k$, where each $\Sigma^k = \{ \sigma^k_1, \sigma^k_2, \sigma^k_3, \ldots \}$ represents the set of all $k$-simplices. It is important to remark that Every singleton set that contains a node $\{v\}, v \in \V$ is represented as a $0$-simplex in $\kph$, pairs $\{u,v\}$ are represented as $1$-simplices, triplets with $2$-simplices and so on. The dimension of $\kph$, is the maximum dimension of any of its simplices and is referred as $\text{dim}(\kph)$. Notice that a simplicial complex $\kph = (\V, \mathsf{S})$ such that $\text{dim}(\kph) = 1$ is mathematically equivalent to a graph $\gph = (\V, \E)$ in which the nodes and the edges of $\gph$ correspond to the 0-simplices and the 1-simplices of $\kph$, respectively.

\paragraph{Orientation} 
Much like the directedness in graphs, simplices in a simplicial complex can be equipped with another symmetry, the \textit{orientation}. However, unlike arrow directions in graphs' edges, the orientation of simplices provides a richer structure. An orientation can be intuitively thought of as a consistent "clockwise" or "counterclockwise" assignment across the simplices $\sigma^k$ of a simplicial complex $\kph$ specified by the ordered $(k+1)$-tuple of $\sigma^k$. If every simplex of $\kph$ is equipped with an orientation, $\kph$ is said to be an \textit{oriented simplicial complex}. In other words, the orientation imparts a directionality to each simplex $\sigma^k$ of $\kph$, enabling a symmetry structure to be established between simplices, enabling advanced algebraic constructs~(\Cref{fig:orientation}). To capture the orientation compatibility between a lower order simplex $\sigma^{k-1}$, and a higher order simplex $\sigma^{k}$, the notation $\sigma^{k-1} \sim \sigma^{k}$ is employed. This notation illustrates that the orientation of $\sigma^{k-1}$ is coherent with that of $\sigma^{k}$. On the contrary, $\sigma^{k-1} \nsim \sigma^{k}$ refers to simplices that have opposite orientation.

\begin{figure}[!htb]
    \centering
    \includegraphics[width=.35\textwidth]{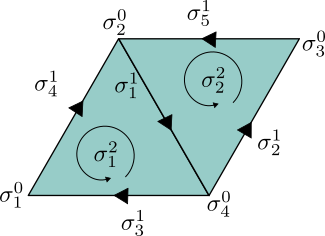}
    \caption{Illustrative example of an oriented simplicial complex of dimension 2. Notice that, between $\sigma_1^2$ and all its faces  the orientation remains coherent while for  $\sigma_2^2$, the orientation of its faces is opposite to the one of $\sigma_2^2$.}
    \label{fig:orientation}
\end{figure}

\paragraph{From Simplicial Complexes to Algebraic Structures}

To understand complex systems, it can be useful to look for structures that can represent entities and intricate relationships in a manner more robust than the familiar framework of graphs. Simplicial complexes are one possible choice for such structures, which can capture polyadic relationships and multi-facet interactions. While simplicial complexes provide a richer perspective, to perform algebraic operations on them, it is required to translate the representation defined so far into an algebraic structure. This need paves the way for the concept of \textit{$k$-chains}.

\begin{figure}[!htb]
    \centering
    \includegraphics[width=\textwidth]{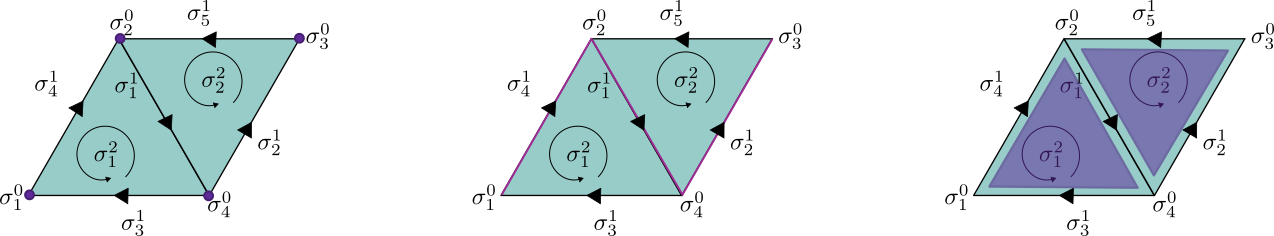}
    \caption{Visualization of hierarchical structures for chains within a 2D simplicial complex.}
    \label{fig:chains}
\end{figure}

\begin{definition}[Chains]\label{def:chains}
    Let $\kph = (\V, \mathsf{S})$ be an oriented simplicial complex, where $\V$ denotes the set of nodes and $\mathsf{S}$ represents the set of simplices. The \textit{$k$-chain} space, denoted by $C_k(\kph, \mathbb{R})$, is defined as the vector space formed by taking linear combinations, with real coefficients, of the oriented $k$-simplices of $\kph$. Any element belonging to $C_k(\kph, \mathbb{R})$ is called a \textit{$k$-chain}. For $k > \text{dim}(\kph)$ it holds $C_k(\kph, \mathbb{R}) = \emptyset$.
\end{definition}

An example of a k-chains are depicted in~\Cref{fig:chains}. In particular the chain  $\alpha_1 \sigma_1^0 + \alpha_2 \sigma_2^0 + \alpha_3 \sigma_3^0 + \alpha_4 \sigma_4^0 \in C_0(\kph, \mathbb{R})$ weights the nodes with real coefficients while the chain  $\beta_1 \sigma_4^1 + \beta_2 \sigma_1^1 + \beta_3 \sigma_2^1 \in C_1(\kph, \mathbb{R})$ is a combination of distinct 1-simplices (edges) of a complex $\kph$ with three arbitrary real values. Notice that, omitting a $k$-simplex from a $k$-chain is equivalent to consider its coefficient equal to zero.

While chains identify distinct regions of a complex $\kph$, to assign features $\xnorm^k$ to those regions of $\kph$ specified by a chain it is necessary to introduce the space of \textit{co-chains}. These are vector spaces of functionals defined on chains, essentially mapping chains to the real numbers. Intuitively ,while chains are combinations of simplices, co-chains offer a formal definition to assign values to these simplices.

\begin{definition}[Co-chains]\label{def:cochains}
Let $\kph = (\V, \mathsf{S})$ be an oriented simplicial complex. The \textit{k-co-chain} space, denoted by $C^k(\kph, \mathbb{R})$, is defined as the set of all real-valued functions on the oriented $k$-simplices of $\kph$. Any element $\xnorm_k \in C^k(\kph, \mathbb{R})$ is called a \textit{k-co-chain}. Here, for $k > \text{dim}(\kph)$ implies $C^k(\kph, \mathbb{R}) = \emptyset$.
\end{definition}

\paragraph{Transitioning Between Dimensions}
By progressing from the representation of simplicial complexes using $k$-chains, a natural interest to relate different chains, particularly, how to transition from a higher-dimensional simplex to its lower-dimensional faces might arise. The tool that serves this purpose, linking one chain to its adjacent, lower-dimensional chain, takes the name of \textit{boundary operator}.

\begin{definition}[Boundary]\label{def:boundary}
For an oriented simplicial complex $\kph$, the boundary operator is a linear map
\begin{equation}
\partial_k : C_k(\kph, \mathbb{R}) \to C_{k-1}(\kph, \mathbb{R}),
\end{equation}
which takes a $k$-chain and produces a $(k-1)$-chain representing the simplices that are on its boundary. Specifically, for a $k$-simplex given by an ordered sequence of nodes $\sigma^k = [v_0, v_1, \ldots, v_k]$, its boundary is given by:
\begin{equation}
\partial_k \big( \sigma^k \big) = \sum_{i=1}^{k} (-1)^i [v_1, \ldots, \hat{v_i}, \ldots, v_k],
\end{equation}
where $\hat{v_i}$ indicates the omission of the node $v_i$.
\end{definition}

The alternating sign ensures that the orientation is respected when taking boundaries. A fundamental property that follows from this definition is that the boundary of a boundary is always zero (i.e., $\partial_{k-1} \circ \partial_k = 0$). 

\

In the realm of algebraic topology, \textit{k-chains} and \textit{boundary operators} define rigorously algebraic operations on simplicial complexes, bridging the gap between the topological properties and computations over discrete spaces~\citep{hatcher2005algebraic}.  Much like how the boundary operator allows to transition from higher-dimensional simplices to their lower-dimensional faces, there exists a dual operator which allows for a transition in the opposite direction: from lower-dimensional co-chains to higher-dimensional ones. This dual map takes the name of \textit{Co-boundary} operator.  

\begin{definition}[Co-Boundary]\label{def:coboundary}
For an oriented simplicial complex $\kph$, the co-boundary operator is a linear map
\begin{equation}
\delta^k : C^k(\kph, \mathbb{R}) \to C^{k+1}(\kph, \mathbb{R}),
\end{equation}
which takes a $k$-co-chain and maps it to a $(k+1)$-co-chain. If $\phi$ is a $k$-co-chain, then for any $(k+1)$-simplex $\sigma^{k+1} = [v_0, \ldots, v_{k+1}]$ in $\kph$, the action of the co-boundary operator is defined by:
\begin{equation}
\delta^k \big( \sigma^{k} \big) = \sum_{i=1}^{k+1} (-1)^i \phi([v_1, \ldots, \hat{v_i}, \ldots, v_{k+1}]),
\end{equation}
where again, $\hat{v_i}$ indicates the omission of the node $v_i$.
\end{definition}

Importantly, the co-boundary operator has a relationship with the boundary operator. In the same way that the boundary of a boundary is always zero (i.e., $\partial_{k-1} \circ \partial_k = 0$), the co-boundary of a co-boundary also vanishes (i.e., $\delta^{k+1} \circ \delta^k = 0$).

Co-chains and co-boundary operators are the building blocks of cohomology, which is a fundamental concept in algebraic topology. Just as homology captures the "holes" or missing simplices in a space, cohomology captures the functions or signals defined on that space. 

\paragraph{Connectivity Structure of Simplcial Complexes}

\begin{figure}[!htb]
    \centering
    \includegraphics[width=.4\textwidth]{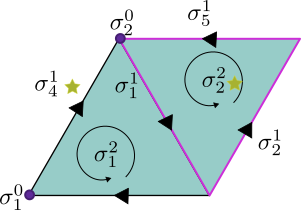}
    \caption{Visualization of a 2D simplicial complex highlighting boundary neighbourhoods. Simplices involved in the boundary computation are marked with a \textcolor{gold}{\Large$\star$}.}
    \label{fig:boundary_adj_sc}
\end{figure}

Within the field of algebraic topology, the way in which simplices' adjacencies are arranged is critical to understanding how simplicial complexes model  relationships. This is due to the fact that, {\em given a simplicial complex $\kph = (\V, \mathsf{S})$, an arbitrary $k$-simplex yields four different neighbourhoods in contrast of the canonical adjacency provided by a graph $\gph = (\V, \E)$.}
\

\textbf{Boundary Adjacency:}  A $(k-1)$-simplex $\sigma^{k-1} \in \kph$ is said to be  {\it boundary adjacent} to a $k$-simplex $\sigma^{k}$ if it holds $\sigma^{k-1} \face \sigma^{k}$.  For a $k$-simplex $\sigma^{k}$, the set of boundary adjacent simplices is denoted by $\mathcal{B}(\sigma^k)$. For example, in~\Cref{fig:boundary_adj_sc}, the boundary neighbourhood of the edge represented by the 1-simplex $\sigma_4^1$ is a set $\mathcal{B}(\sigma_4^1) = \{\sigma_1^0, \sigma_2^0\}$ that contains the 0-simplices (nodes) that are at the ends of $\sigma_4^1$. For the triangle represented by the 2-simplex $\sigma_2^2$, the boundary neighbourhood is $\mathcal{B}(\sigma_2^2) = \{\sigma_2^1, \sigma_5^1, \sigma_1^1 \}$. Notice that for an oriented simplicial complex, the boundary of an oriented $k$-simplex consists of the union of oriented $(k-1)$-simplices, each given an orientation induced from that of the $k$-simplex. This induced orientation guarantees that by assembling the $(k-1)$-simplices according to their orientation, the original $k$-simplex with its given orientation is recovered.

\begin{figure}[!htb]
    \centering
    \includegraphics[width=.4\textwidth]{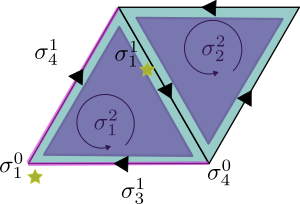}
    \caption{Visualization of 2D simplicial complex emphasizing co-boundary relationships. Simplices under consideration for showing the co-boundary computation are marked with a \textcolor{gold}{\Large$\star$}.}
    \label{fig:coboundary_adj_sc}
\end{figure}

\

\textbf{Co-boundary Adjacency:}  A $(k+1)$-simplex $\sigma^{k+1} \in \kph$ is said to be  {\it co-boundary adjacent} to a $k$-simplex $\sigma^{k}$ if it holds $\sigma^k \face \sigma^{k+1}$. For a $k$-simplex $\sigma^{k}$, the set of co-boundary adjacent simplices is denoted by $\mathcal{C}o(\sigma^k)$. For example, in~\Cref{fig:coboundary_adj_sc} it is shown that, the co-boundary neighbourhood of the 0-simplex (node) $\sigma_1^0$, is composed by the 1-simplices (edges) $\sigma_3^1$ and $\sigma_4^1$. That is $\mathcal{C}o(\sigma_1^0) = \{\sigma_3^1, \sigma_4^1\}$. Moreover, for the 1-simplex (edge) $\sigma_1^1$, the co-boundary is the set $\mathcal{C}o(\sigma_1^1) = \{\sigma_1^2, \sigma_2^2 \}$ that contains the 2-simplices triangles that have $\sigma_1^1$ as one of their faces. Notice that, in a two-dimensional simplicial complex $\kph$, only 0-simplices and 1-simplices might have a non-empty co-boundary neighbourhood.

\begin{figure}[!htb]
    \centering
    \includegraphics[width=.4\textwidth]{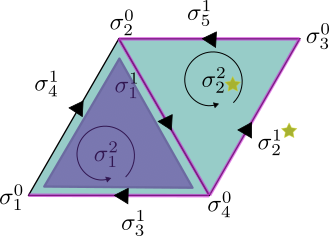}
    \caption{Visualization of lower-neighbourhood relationships within a 2D simplicial complex. Simplices marked by a \textcolor{gold}{\Large$\star$} highlight the focus when determining the lower neighbourhood.}
    \label{fig:lower_neigh_adj_sc}
\end{figure}

\

\paragraph{Upper Adjacency} 
Let $\sigma^k$ and $\tau^k$ be two arbitrary $k$-simplices within a simplicial complex $\kph$. If $\sigma^k$ and $\tau^k$ share a mutual relationship as faces of a $(k+1)$-simplex, they are {\it upper adjacent} ($\sigma^k \in \mathcal{N}_{\uparr}(\tau^k)$ and vice-versa). In other words, $\sigma^k$ and $\tau^k$ are both are faces of a simplex $\delta^{k+1}$ of one dimension higher ($\sigma^k \face \delta^{k+1}$ and $\tau^k \face \delta^{k+1}$). In a 2-dimensional simplicial complex,two 0-simplices (nodes) are upper adjacent if they have an edge that joins them while two 1-simplices (edges) are upper adjacent if both are sides of a common 2-simplex (triangle)~(\Cref{fig:upper_neigh_adj_sc}).

\begin{figure}[!htb]
    \centering
    \includegraphics[width=.4\textwidth]{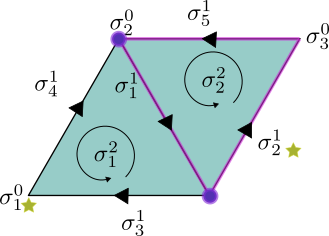}
    \caption{Visualization of a 2D simplicial complex emphasizing upper adjacency. Simplices denoted with a \textcolor{gold}{\Large$\star$} are the ones for which their corresponding upper adjacent simplices are highlighted.}
    \label{fig:upper_neigh_adj_sc}
\end{figure}

\

\paragraph{Lower Adjacency}
Conversely, two simplices $\sigma^k$ and $\tau^k$  are {\it lower adjacent} ($\sigma^k \in \mathcal{N}_{\doarr}(\tau^k)$ and vice-versa) if they jointly possess a shared face of order $k-1$ within $\kph$. So, it exists a simplex $\delta^{k-1}$ such that $\delta^{k-1} \face \sigma^k$ and $\delta^{k-1} \face \tau^k$). For example, consider the two triangles (2-simplices $\sigma_1^2, \sigma_2^2$) in~\Cref{fig:lower_neigh_adj_sc}, they are lower adjacent because they have a shared face $\sigma_1^1$ of one dimension lower.

\paragraph{Algebraic Representation of Simplicial Complexes}

In the study of algebraic topology, simplicial complexes serve as combinatorial models that provide a bridge between topological spaces and algebraic structures that provide a rich connectivity structure . This relationship facilitates a wide array of calculations and analyses. Among the tools used to represent the algebraic structure of simplicial complexes are the \textit{incidence (or boundary) matrices} $\Bnorm_k$ and the \textit{higher-order Laplacian matrices} $\Lnorm_k$~\citep{goldberg2002combinatorial}. In particular, $\Bnorm_k$ is the algebraic representation of the boundary operator $\partial_k$ while $\Lnorm_k$ is the extention,  to higher dimensional simplices of the canonical graph Laplacian~\citep{grady2010discrete}. Furthermore, a set of incidence matrices $\Bnorm_k$ for $k = 1,\ldots, K$, is sufficient to define the connectivity structure of an oriented simplicial complex $\kph$ of order $K$. Specifically, the entries of $\Bnorm_k$ establish which $k$-simplices are incident to which $(k+1)$-simplices and if they have a coherent orientation. Formally:

\begin{equation}\label{eq:sc_incidence_matrix}
  \big[\Bnorm_{k} \big]_{ij}=\left\{\begin{array}{rll}
  0, & \text{if} ~ \sigma_i^{k-1} \mkern5mu\not\mkern-5mu\face \sigma_j^{k},\\
  1,& \text{if} ~ \sigma^{k-1} \face \sigma_j^{k} ~ \text{and} ~ \sigma_i^{k-1}  \sim \sigma_j^{k},\\
  -1,& \text{if} ~  \sigma_i^{k-1} \face \sigma_j^{k} ~ \text{and} ~ \sigma_i^{k-1}  \not\sim \sigma_j^{k}\\
  \end{array}\right.
\end{equation}

The incidence matrices reflect the geometric structure and mutual relationships between simplices within a simplicial complex~\citep{hatcher2005algebraic}. However, to derive the spectral properties of these complexes it is required to introduce the higher-order Laplacian matrices. For a simplicial complex $\kph$, these extend the notion of the traditional graph Laplacian to capture multi-dimensional interactions. Formally, 

\begin{align}\label{eq:back:laplacians}
&\Lnorm_0=\Bnorm_{1}\Bnorm_{1}^T,\\
&\Lnorm_k=\underbrace{\Bnorm^T_k\Bnorm_k}_{\Ldo_k}+ \underbrace{\Bnorm_{k+1}\Bnorm_{k+1}^T}_{\Lup_k}, \quad k=1, \ldots, K-1\label{eq:back:L_int},\\
&\Lnorm_K=\Bnorm_{K}^T\Bnorm_{K}.
\end{align}

Notice that all Laplacians of intermediate order~(\Cref{eq:back:L_int}), contain two terms expressing the lower and upper adjacencies of $k$-order simplices. The former term $\Bnorm^T_k\Bnorm_k$, it is the {\it lower Laplacian}, $\Ldo_k$; the latter ($\Bnorm_{k+1}\Bnorm_{k+1}^T$) is the {\it upper Laplacian} $\Lup_k$~\citep{barbarossa2020topological}.

%% file: Include/Chapters/2_Background/2.6_cell_complexes.tex
\section{Cell Complexes}\label{sec:background:cell_complexes}

Simplicial complexes are powerful combinatorial structures able to represent naturally polyadic interactions and a unique connectivity provided by the different neighborhoods of the simplices. However, this very property can make them inflexible, since the face inclusion property~(\Cref{def:simplicial_complex_enhanced}) ensures that such domains are built by sticking simplices together. In some cases, this solution is too rigid since it necessitates the explicit representation of all $(k-1)$-dimensional faces when only the full $k$-th order interaction needs to be represented~(\Cref{fig:back:simplicial_complex}). This can lead to superfluous information and unnecessary computational overhead, particularly when the primary interest lies in capturing specific higher-order interactions without being constrained by the need to represent all their constituent sub-interactions. This can be achieved by switching $k$-simplices with spaces 'like' closed $k$-dimensional disks, overcoming this limitation with the introduction of {\em regular cell complexes}.

\

\begin{figure}[!htb]
    \centering
    \includegraphics[width=.4\textwidth]{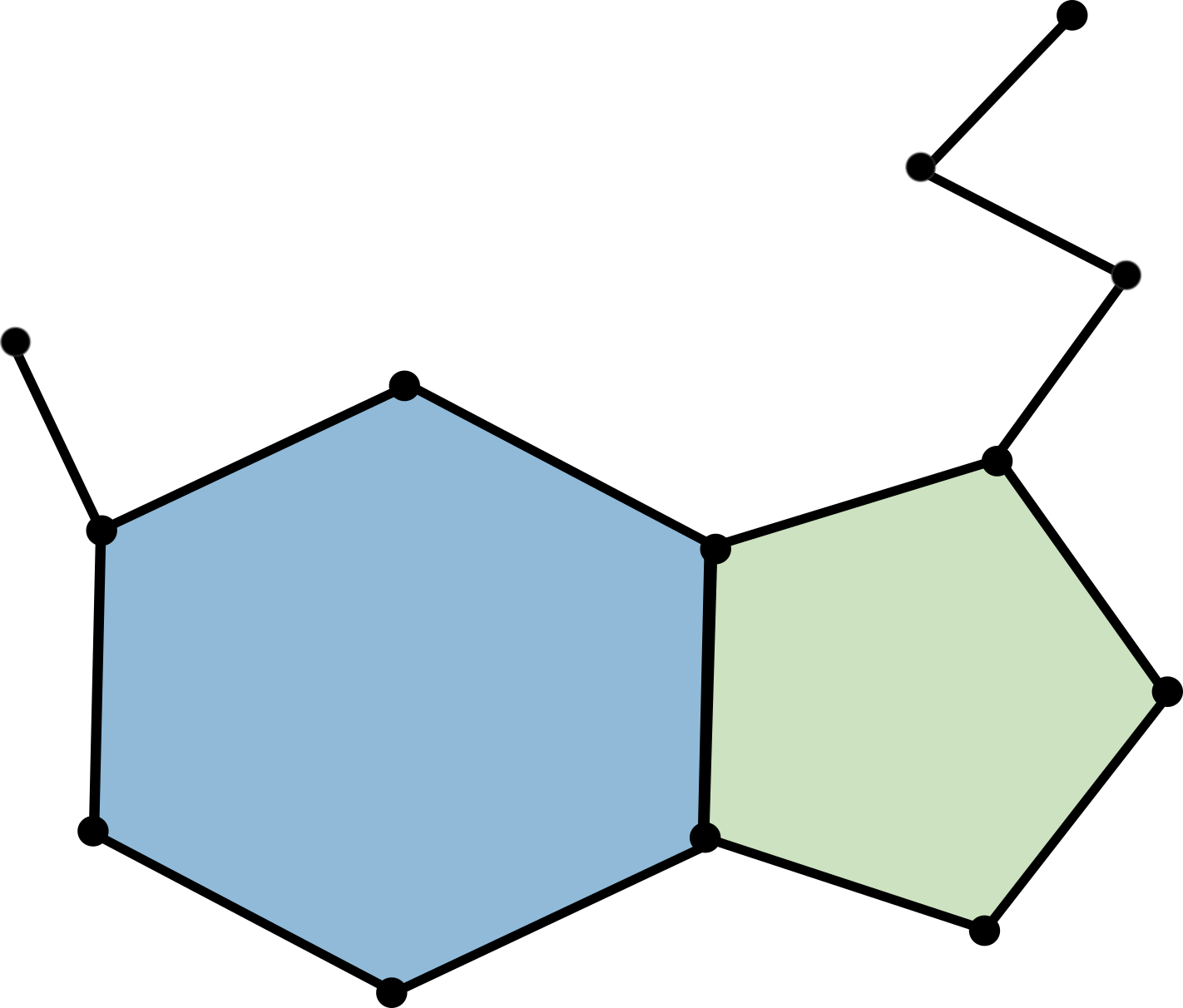}
    \caption{A cell complex $\cph$ representing a serotonin molecule. Notice that, nodes can be arranged as rings without the necessity of representing sub-structures as required by simplicial complexes via the face inclusion principle~(\Cref{def:simplicial_complex_enhanced}).}
    \label{fig:serotonin_molecule}
\end{figure}

Cell complexes are discrete topological spaces able to represent complex interconnected systems, generalizing graphs~\Cref{sec:background:graphs} and simplicial complexes~\Cref{sec:background:simplicial_complexes}. In particular, cell complexes naturally relax the constraint imposed by the face inclusion property required by simplicial complexes~(\Cref{fig:serotonin_molecule}). However, before defining operations over these flexible domains, it is necessary to ensure that a cell complex $\cph$ respects certain regularity conditions. 

\

\begin{definition}[Regular Cell Complex]\label{def:cell_complex}
\cite{hansen2019toward} \textit{A {\it regular cell complex} is a topological space $\cph$ together with a partition $\{\cph_{\sigma}\}_{\sigma \in \mathcal{P}_{\cph}}$ of subspaces $\cph_{\sigma}$ of $\mathcal{C}$ called $\mathbf{cells}$, where $\mathcal{P}_{\cph}$ is the indexing set of $\mathcal{C}$, such that}

\begin{enumerate}
    \item For each cell $\sigma \in \cph$, every sufficiently small neighbourhood of $\sigma$ intersects finitely many cells $\cph_{\sigma}$;  
    \item \label{it:closure} For all $\tau$, $\sigma$ in $\cph$, it holds that $\cph_{\tau} \cap \overline{\cph}{\sigma} \neq \varnothing$ iff $\cph{\tau} \subseteq \overline{\cph}{\sigma}$, where $\overline{\cph}{\sigma}$ denotes the closure of the cell;
    \item \label{it:homeomo} Every $\cph_{\sigma}$ is homeomorphic to $\mathbb{R}^{k}$ for some $k$;
    \item \label{it:posed} For every $\sigma$ $\in$ $\mathcal{P}_{\cph}$ there is a homeomorphism $\phi$ of a closed ball in $\mathbb{R}^{k}$ to $\overline{\cph}_{\sigma}$ such that the restriction of $\phi$ to the interior of the ball is a homeomorphism onto $\cph_{\sigma}$.
\end{enumerate}
\end{definition}

Condition~(\ref{it:closure}) implies that the indexing set $\mathcal{P}_{\cph}$ has a poset structure, given by $\tau$ $\leq$ $\sigma$ iff $\cph_{\tau}$ $\subseteq$ $\overline{\cph_\sigma}$. This is known as the face poset of $\cph$. The regularity condition~(\ref{it:posed}) implies that all topological information about $\cph$ is encoded in the poset structure of $\mathcal{P}_{\cph}$. Then, a regular cell complex can be identified with its face poset. For this reason, from now on, the cell $\cph_{\sigma}$ will be referred with its corresponding face poset element $\sigma$ which dimension $\dim({\sigma})$ is equal to the dimension of the space homeomorphic to $\cph_{\sigma}$.

\begin{figure}[!htb]
    \centering
    \includegraphics[width=.95\textwidth]{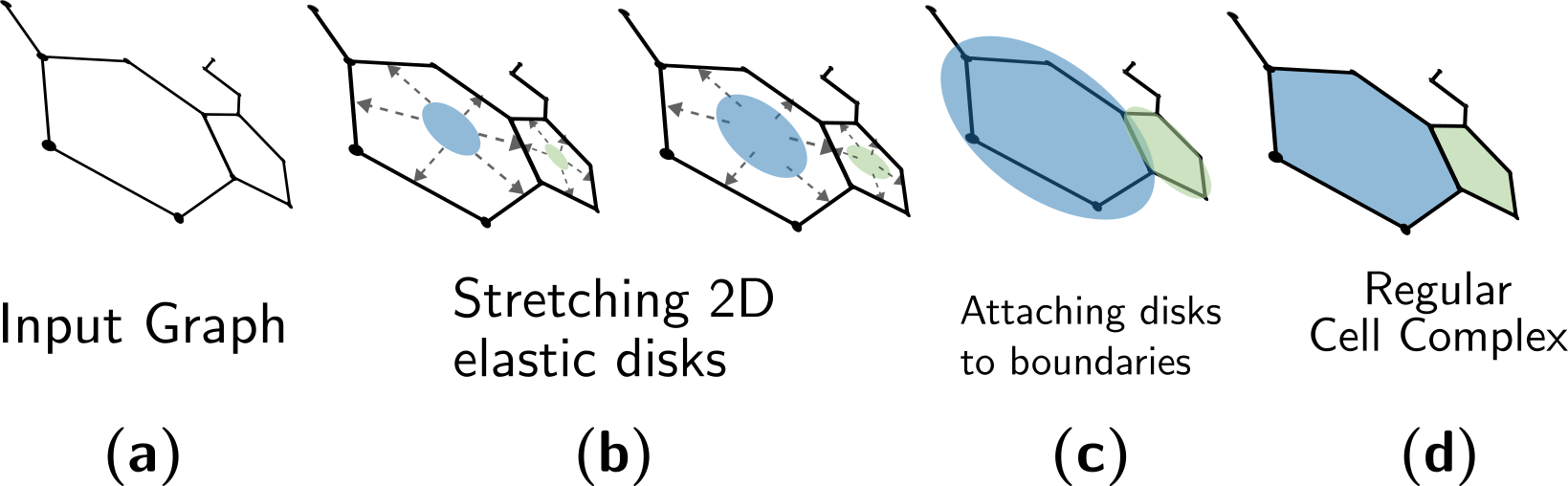}
    \caption{Illustration of a skeleton-preserving lifting procedure: Attaching two-dimensional cells to the induced cycles of a graph $\gph$, preserving node and edge features to form a regular cell complex $\cph$ such that $sk_1(\cph) = \gph$.}
    \label{fig:sl}
\end{figure}

In this context, a graph $\gph = (\V, \E)$ can be viewed as a particular case of a regular cell complex $\cph$. Specifically, a graph is a cell complex  where the set of $2$-cells is the empty set. The vertices of the graph correspond to the $0$-cells in $\cph$, while the edges of the graph are then represented by its $1$-cells, connecting pairs of vertices. Throughout this thesis, only regular cell complexes $\cph$ built using {\em skeleton-preserving cellular lifting maps}~\citep{bodnar2021weisfeilercell} from an input graph $\gph$ will be considered. A pictorial example of this operation is provided in~\Cref{fig:sl}, where filled rings are attached to closed paths of edges having no internal chords.

\begin{figure}[t]
    \centering\includegraphics[width=.9\textwidth]{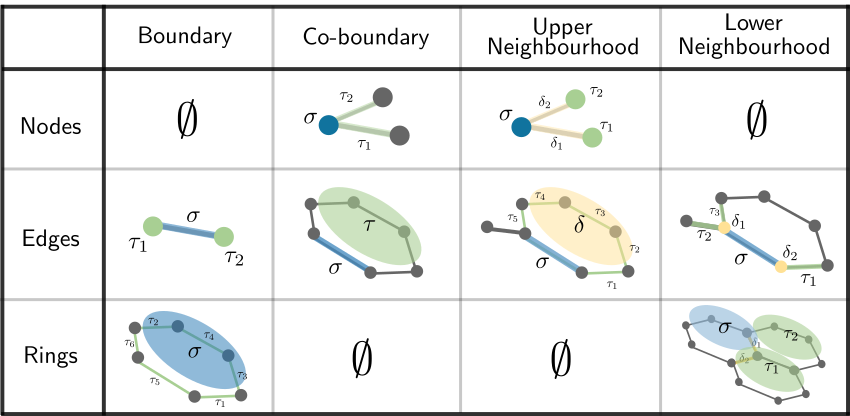}
    \caption{Visual representation of adjacencies within cell complexes. The reference cell, $\sigma$, is showcased in \textcolor{lightblue}{blue}, with adjacent cells $\tau$, highlighted in \textcolor{lightgreen}{green}. Any intermediary cells $\delta$ mediating the connectivity are depicted in \textcolor{lightyellow}{yellow}.}
    \label{fig:adjs}
\end{figure}

\

\paragraph{Connectivity Structure of Cell Complexes}

The connectivity structure of a regular cell complex is similar to the one provided by simplicial complexes. Cell complexes have a unique connectivity blueprint thanks to their flexibility in modelling higher-order structures with a relatively simple combinatorial domain. A glossary of the neighbourhoods of a two dimensional regular cell complex $\cph$ is depicted in~\Cref{fig:adjs}.

\begin{definition}[Boundary Relation]\label{def:boundary_rel}
    Given two cells $\sigma, \tau \in \cph$. The boundary relation $\sigma \face \tau$ holds \emph{iff} $\dim({\sigma}) < \dim({\tau})$ and there does not exist a $\delta \in \cph$ such that $\sigma \face \delta \face \tau$.
\end{definition}

\

For a cell $\sigma$, the \textbf{boundary neighbourhood} is a set $\mathcal{B}(\sigma) = \{ \tau \, \vert \,  \tau \face \sigma \}$ composed by the lower-dimensional cells that respect the boundary relation~(\Cref{def:boundary_rel}). For example, in a cell complex of dimenison two, nodes don’t possess a boundary neighborhood as they represent isolated points within the complex; an edge is bounded by the nodes at its endpoints; the boundary of a ring is defined by the edges that circumscribe it.

\

The \textbf{co-boundary neighbourhood} is a set $\mathcal{C}o(\sigma) = \{\tau \, \vert \, \sigma \face \tau \}$ of higher-dimensional cells with $\sigma$ on their boundary. In a two-dimensional cell complex, the co-boundary of a node is constituted by the edges originating from or terminating at it; for an edge, the co-boundary includes the rings for which the edge is a ring's border. Tings do not possess a co-boundary neighborhood in this scenario.

\

The \textbf{upper neighbourhood} are the cells of the same dimension as $\sigma$ that are on the boundary of the same higher-dimensional cell as $\sigma : \mathcal{N}_{\uparr}{(\sigma)} = \{ \tau \, \vert \, \exists \delta : \sigma \face \delta \wedge \tau \face \delta\}$. The upper neighborhood of a node is provided by the set of nodes directly connected to via edges, which is the canonical graph adjacency; for an edge, the upper neighbourhood include edges surrounding the rings for which the edge is a boundary element; 

\

The \textbf{lower neighbourhood} is composed by the cells of the same dimension as $\sigma$ that share a lower dimensional cell on their boundary: $\mathcal{N}_{\doarr}{(\sigma)} = \{ \tau \, \vert \, \exists \delta : \delta \face \sigma \wedge \delta \face \tau\}$.  In regular cell complexes, nodes do not have a lower neighborhood; the lower adjacent cells of an edge are the edges that share a common vertex with the edge in consideration; in a $2$-complex, rings do not have upper adjacent cells. The lower adjacent cells of a ring are the rings sharing a common boundary edge with the ring itself. 

\

By combining a flexible connectivity structure with a minor complexity overhead, cell complexes find applications in several real-world scenarios, including: molecular modelling (e.g., molecular graphs and molecular surfaces can be represeted as~\Cref{fig:serotonin_molecule}); material science (e.g., topological insulators~\citep{hasan2010colloquium}); computer graphics (e.g., polygonal meshes~\citep{crane2018discrete}; physics (e.g., general relativity, space-time can be modelled using 4D cell complexes~\citep{tonti1975formal}). 

\

{\em Although the theory of presented so far and the methods proposed afterwards apply to cell complexes of arbitrary dimension, in this thesis, only cell complexes with cells of maximum dimension equal to $2$ are considered.} 

%% file: Include/Chapters/2_Background/2.7_topo_sp.tex
\section{Topological Signal Processing}

\begin{figure}[!htb]
    \centering
    \includegraphics[width=.9\textwidth]{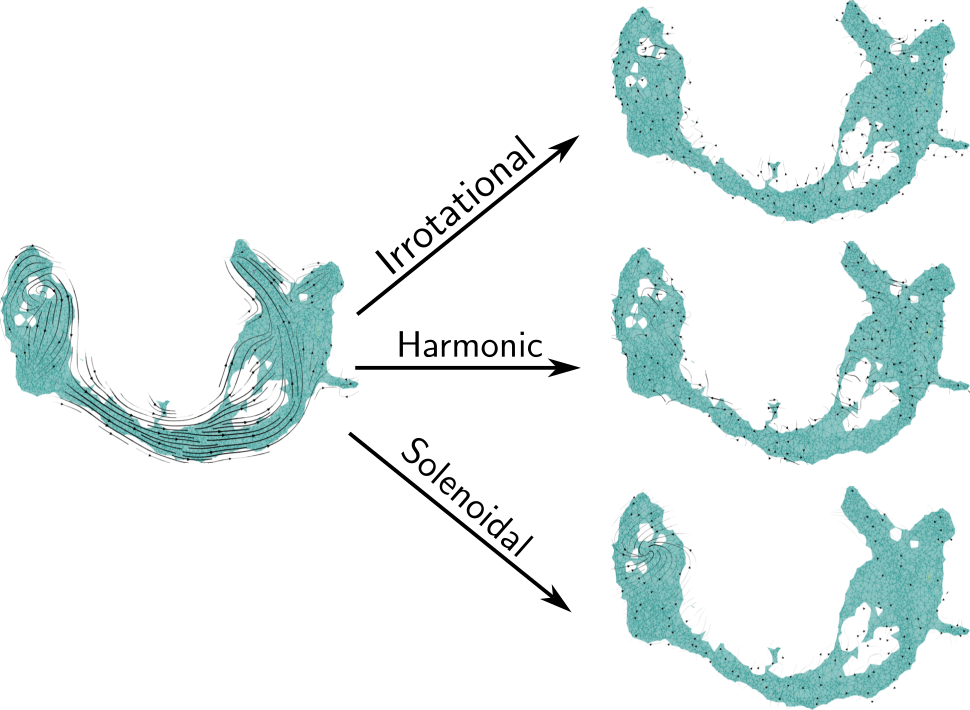}
    \caption{Visual representation of the Hodge Decomposition applied to RNA velocity fields from~\cite{la2018rna}. It showcases the separation of flow components into: {\em irrotational, harmonic, and solenoidal}. }
    \label{fig:back:hodge}
\end{figure}

The mathematical formalism for extending the graph signal processing techniques, as defined in~\Cref{sec:gsp}, to process signals defined over complex topological spaces is known in the literature as {\bf topological signal processing}.
In particular, this section provides fundamental tools to analyze signals defined over topological spaces. Moreover, processing topological signals over cell complexes includes signals over simplicial complexes and graphs as particular instances of the framework~\citep{barbarossa2020topological, schaub2020random, schaub2021signal, sardellitti2022cell, roddenberry2022signal, yang2021finite, yang2022simplicial}. Therefore, this section will focus on processing signals defined on cell complexes without loss of generality. To this aim, let $\xph$ be a discrete topological space. In this context, $\xph$ can be either a simplicial complex $\kph = (\V, \mathsf{S})$ or a cell complex $\cph = (\V, \mathcal{P}_{\cph})$. As mentioned before, processing signals defined on $\xph$ does not require it to be materialized in one of the two particular instances. In particular, ensuring that each cell in $\xph$ is equipped with defined features and proper neighborhoods guarantees consistent signals' flow, independently of whether $\xph$ is instantiated as a simplicial or cell complex. 

It is worth noting that, while both cell complexes and simplicial complexes can represent $\xph$. The choice between them should be influenced by the specific application and the nature of the data.

\begin{definition}\label{def:cell_signal}
Let $\cph = (\V, \mathcal{P}_{\cph})$ be a two dimensional cell complex having a set of nodes $\V$, edges $\E$ and rings $\mathsf{R}$ incorporated within the indexing set $\mathcal{P}_{\cph}$. A cell signal is defined as a function that assigns a value from field $\mathbb{F}$ to each cell of $\cph$:
\begin{equation}\label{eq:sc:sc_features}
\xnorm_{\sigma}: {\cph_\sigma} \rightarrow \mathbb{F}.
\end{equation}
\end{definition}

In this context, $\mathbb{F}$ typically represents a $d$-dimensional vector space (${\cph_\sigma} \rightarrow \mathbb{R}^{d})$, where the dimension $d$ can vary across different cells without loss of generality.\footnote{\Cref{def:cell_signal} alongside a notion of "bridges" between dimensions grounds the Sheaf Theory~\citep{bredon2012sheaf}.}.

\paragraph{Hodge decomposition} High order Laplacians admit a Hodge decomposition~\citep{lim2020hodge},
leading to three orthogonal subspaces. In particular, the $k$-simplicial signal space can be decomposed as:

\begin{equation}\label{hodge_spaces}
\mathbb{R}^{d} = \text{im}(\Bnorm_{k}^\intercal\big) \oplus \text{im}\big(\Bnorm_{k+1}\big) \oplus \text{ker}\big(\Lnorm_{k}\big).
\end{equation}

Thus, any topological signal $\xnorm_{\sigma}$ can be decomposed as:
\begin{equation}\label{eq:hodge_decomposition}
    \xnorm_{\sigma}=\underbrace{\Bnorm_{k}^\intercal\, \xnorm_{\tau}}_{\text{irrotational}} +\underbrace{\Bnorm_{k+1}\, \xnorm_{\delta}}_{\text{solenoidal}} +\underbrace{\xnorm_{h}}_{\text{harmonic}} 
\end{equation}

where $dim(\sigma) = k \implies \xnorm_\sigma \in \mathbb{R}^d$ and $dim(\tau) = k-1$ and $dim(\delta) = k+1$

To provide an interpretation of the three orthogonal components in~\Cref{eq:hodge_decomposition} consider $k=1$ and edge flows \big(i.e., $\xnorm_{\sigma}\big|_{\sigma \in \E}$ \big)~\citep{barbarossa2020topological}. The matrix $\Bnorm_{1}$ is the \textbf{discrete divergence operator}, applied to an edge flow $\xnorm_{\sigma}$ computes, for each node $v \in \V$ its net flow that is the amount of flow going towards $v$ minus the flow going from $v$ outward its neighbours. Its adjoint $\Bnorm_{1}^\intercal $ differentiates a node signal $\xnorm_{\tau}\big|_{\tau \in \V}$ along the edges to induce an edge flow $\Bnorm_{1}^\intercal \xnorm_{\tau}$.

The component $\Bnorm_{1}^\intercal \xnorm_{\tau}$ is referred to as \textbf{irrotational component} of $\xnorm_{\sigma}$ and $\text{im}(\Bnorm_{k}^\intercal)$ the gradient space. Applying matrix $\Bnorm_{2}^\intercal$ to an edge flow $\xnorm_{\sigma}$ means computing its circulation along each cell, thus $\Bnorm_{2}^\intercal$ is called a curl operator. Its adjoint $\Bnorm_{2}$ induces an edge flow $\xnorm_{\sigma}$ from a cell signal $\xnorm_{\delta} $.

The component $\Bnorm_{2}\xnorm_{\delta}$ is referred to as the \textbf{solenoidal component} of $\xnorm_{\sigma}$ and $\text{im}(\Bnorm_{2})$ the curl space. The remaining component $\xnorm_{h}$ is  the \textbf{harmonic component} since it belongs to$\text{ker}(\Lnorm_{1})$ that is called the harmonic space. Any harmonic flow $\xnorm_{h}$ has zero divergence and curl.

In the sequel the focus will be on topological signal processing techinques for edge signals, without loss of generality. Therefore, let  $\xnorm := \xnorm_{1}$, $\Lnorm := \Lnorm_{1}$,  $\Lnorm^{\doarr} := \Lnorm_1^{\doarr}$ and $\Lnorm^{\uparr} := \Lnorm_1^{\uparr}$, such that $\Lnorm = \Lnorm^{\doarr} + \Lnorm^{\uparr}$. Also, let $\mathcal{N}_{\doarr}(e)$ and $\mathcal{N}_{\uparr}(e)$ be the lower and upper neighbors of edge $e$, respectively.

\paragraph{Topological filters} The Hodge decomposition in~\Cref{eq:hodge_decomposition} suggests to separately filter the irrotational, solenoidal and harmonic components of the signal. Thus, generalizing the approach proposed in~\cite{yang2021finite}, consider a simplicial convolutional filter given by:
\begin{equation}\label{eq:topological_filters}
    \Hnorm =   \underbrace{\sum_{k = 1}^{K_{\doarr}} w^{\doarr}_k \big(\Lnorm^{\doarr}\big)^k}_{\Hnorm^{\doarr}} + \underbrace{\sum_{k = 1}^{K_{\uparr}}  w^{\uparr}_k \big(\Lnorm^{\uparr}\big)^k}_{\Hnorm^{\uparr}} + \underbrace{w_{h} \Pnorm_h}_{\Hnorm^{h}}
\end{equation}

where $\mathbf{w}^{\doarr} = \Big[w^{\doarr}_1,...,w^{\doarr}_{K^{\doarr}}\Big]$, $\mathbf{w}^{\uparr} = \Big[w^{\uparr}_1,...,w^{\uparr}_{K^{\uparr}}\Big]$ and $w_{h}$ are the filter's weights.
The order of the irrotational and solenoidal filters are represented by 
$K_{\doarr}$  and $K_{\uparr}$, respectively. The filter in~\Cref{eq:topological_filters} resembles the Hodge decomposition and it is a proper generalization to simplicial signals  of the linear-shift-invariant graph filters~\citep{shuman2013emerging}. In particular, the terms $\mathbf{H}^{\doarr}$ and $\mathbf{H}^{\uparr}$ of~\Cref{eq:topological_filters} allows to independently filter the input signal based on its lower and upper simplicial neighbourhoods (encoded into the Laplacians $\Lnorm^{\doarr}$ and $\Lnorm^{\uparr}$), thus processing its irrotational and solenoidal components, respectively. The term $\mathbf{H}^{h}$ extracts and scales the harmonic component of the signal, with $\Pnorm_h \in \mathbb{R}^{E \times E}$ being a  projection operator onto the harmonic space $\text{ker}\big(\Lnorm\big)$. From~\Cref{hodge_spaces} and~\Cref{eq:back:laplacians},  harmonic  signals  can  be represented  as  linear  combination  of  a  basis  of  eigenvectors  spanning  the kernel of $\Lnorm$. However, since there is no unique way to identify a basis for such a subspace, the approximation can be driven by ad-hoc criteria to choose a specific basis, as in~\cite{sardellitti2022cell}, or just finding an approximated projector $\widehat{\Pnorm}_h$ of any of the possible bases, but with some desirable property as sparsity. In the latter case, the true harmonic projection operator is equal to $\Pnorm = U_h U_h^T$, where $U_h$ is the set eigenvectors of $\Lnorm$ corresponding to the smallest eigenvalue  A sparse approximation of $\Pnorm_h$ can thus be obtained as \cite{olfati2004consensus}:
\begin{align}\label{eq:harmonic_projector_approx}
    \widehat{\Pnorm}_h = \big(\mathbf{I} - \varepsilon \, \Lnorm\big)^{K^{h}},
\end{align}
where $K^{h} > 0$ and $0<\varepsilon\leq\frac{2}{\lambda_{\text{max}}(\Lnorm)}$. It can be shown that  for $\widehat{\Pnorm}_h$ in~\Cref{eq:harmonic_projector_approx} it holds~\citep{olfati2004consensus}:
\begin{equation}\label{harmonic_convergence}
\displaystyle\lim_{K^{h} \rightarrow \infty} \; \widehat{\Pnorm} = \Pnorm_h.
\end{equation}
The matrix $\mathbf{H}^{h}$ in~\Cref{eq:topological_filters} and~\Cref{eq:harmonic_projector_approx} is known as the \textit{harmonic filter}.

\

\paragraph{Spectral Interpretaion.} A frequency response of the filter in~\Cref{eq:topological_filters} can be derived, based on the work from \cite{barbarossa2020topological} and the definition of Simplicial Fourier Transform  from \cite{yang2021finite}, therefore further details can be found therein. Assume that $\widehat{\Pnorm}_h$ in~\Cref{eq:harmonic_projector_approx} is in the asymptotic regime in~\Cref{harmonic_convergence} (i.e., $\widehat{\Pnorm}_h = \Pnorm_h$).The Simplicial Fourier Transform  $\mathbf{s} \in \mathbb{R}^{E} $ of a signal $\xnorm \in \mathbb{R}^{E}$ is defined as its projection onto the basis of the eigenvectors $\mathbf{U} \in \mathbb{R}^{E \times E}$  of $\Lnorm\in \mathbb{R}^{E \times E}$ (which is a symmetric and positive semi-definite matrix by definition):

 \begin{equation}\label{eq:background:sc:simplicial_fourier_transform}
    \mathbf{s} = \mathbf{U}^T\xnorm.
 \end{equation}
Given the transform in~\Cref{eq:background:sc:simplicial_fourier_transform}, the filter frequency response is defined as:
\begin{equation}
     \boldsymbol{\Sigma} =  \mathbf{U}^T\mathbf{H}\mathbf{U},
 \end{equation}
where $\boldsymbol{\Sigma} \in \mathbb{R}^{E \times E}$ is a diagonal matrix representing a mask in the frequency domain.
Due to~\Cref{hodge_spaces} and~\Cref{eq:hodge_decomposition}, the matrix $\boldsymbol{\Sigma}$ can be seen as a block-diagonal matrix made of a diagonal matrix $\boldsymbol{\Sigma}^{\doarr}\in \mathbb{R}^{N_{\doarr} \times N_{\doarr}}$ containing the frequency mask associated to non-zeros eigenvalues $\boldsymbol{\lambda}^{\doarr} \in \mathbb{R}^{N_{\doarr}}$ of $\Lnorm^{\doarr}$, a diagonal matrix $\boldsymbol{\Sigma}^{\uparr}\in \mathbb{R}^{N_{\uparr} \times N_{\uparr}}$ containing the frequency mask associated to the non-zeros eigenvalues $\boldsymbol{\lambda}^{\uparr} \in \mathbb{R}^{N_{\uparr}}$ of $\Lnorm^{\uparr}$  and a constant diagonal matrix  $\boldsymbol{\Sigma}^{h}\in \mathbb{R}^{N_h \times N_h}$ containing the constant frequency mask associated to the zero eigenvalues of $\Lnorm$. Therefore, $N_{\doarr}$ is the dimension of the gradient space,
 $N_{\uparr}$ is the dimension of the curl space and $N_h$ is the dimension of the harmonic space, such that $N = N_{\doarr} + N_{\uparr} + N_h$. This fact allows to characterize the frequency response in terms of irrotational, solenoidal and harmonic frequencies responses, enhancing the perspective of~\Cref{eq:topological_filters} as three parallel filtering branches. In particular, it holds:
 \begin{align}
     &[\boldsymbol{\Sigma}^{\doarr}]_{ii} = \sum_{k = 1}^{K_{\doarr}} w_k^{\doarr} (\lambda^{\doarr}_i)^k, \label{spctral_lower}\\
     &[\boldsymbol{\Sigma}^{\uparr}]_{ii} = \sum_{k = 1}^{K_{\uparr}} w_k^{\uparr} (\lambda^{\uparr}_i)^k, \label{spctral_upper} \\ &[\boldsymbol{\Sigma}_{h}]_{ii} = w^{h},\label{spctral_harmonic}
\end{align}
which represent the frequency masks of the irrotational, solenoidal, and harmonic component, respectively.

%% file: Include/Chapters/2_Background/2.8_topo_nns.tex
\section{Topological Neural Networks}\label{sec:tnn}

Let $\xph$ be a discrete topological space (i.e., a simplicial or cell complex) with nodes $\V$ and a set $\mathcal{P}_{\xph}$ that indexes higher-order cells contained in $\xph$, including the set of nodes $\V$ as 0-cells and the set of edges $\E$ as 1-cells. The connectivity of $\xph$ is encoded in the set of incidence matrices ${\Bnorm_{k}}{k=1}^{dim(\xph)}$, where each $\Bnorm{k}$ maps $k$-cells to the $(k+1)$-cells on their co-boundary. For example, if $\xph$ is a cell complex of dimension 2, its connectivity is fully encoded in the set $\{\Bnorm_1, \Bnorm_2\}$ such that $\Bnorm_1 \in \R^{n \times e}$ maps each node $v$ to the edges that have $v$ on their boundary and $\Bnorm_{2} \in \R^{e \times r}$ acounting for the connectivity between edges and rings. Assume that $\xph$ is connected, undirected, unweighted, unoriented, and that there are features $\{\mathbf{h}_{\sigma}\}_{\sigma \in \mathcal{P}_{\xph}}\subset \R^{d}$. Topological Neural Networks (TNNs) are functions of the form:

\begin{equation}\label{eq:tnn_def}
    \mathsf{TNN}_{\theta}:(\xph,\{\mathbf{h}_\sigma\}) \mapsto y_\xph,
\end{equation}

with parameters $\theta$ learned via a training procedure and whose output $y_\xph$ is either a cell-level or complex-level prediction. 

\

From the broad class of topological nerual networks~\citep{papillon2023architectures}, this manuscript will focus on message passing schemes defined over topological spaces, known as Topological Message Passing~\citep{bodnar2021weisfeiler} that compute cell representations by stacking layers of the form:

\begin{align}\label{eq:topological-message-passing-scheme}
    \mathbf{h}_{\mathcal{B}} &= \underset{\tau \in \mathcal{B}(\sigma)}{\agg} \big(\mathsf{m}_{\mathcal{B}}\big( \mathbf{h}_{\sigma}, \mathbf{h}_{\tau}\big)\big),\\[1.5pt]
    \mathbf{h}_{\mathcal{C}o} &= \underset{\tau \in \mathcal{C}o(\sigma)}{\agg} \big(\mathsf{m}_{\mathcal{C}o}\big( \mathbf{h}_{\sigma},  \mathbf{h}_{\tau}\big)\big), \\[1.5pt]
    \mathbf{h}_{\uparr} &= \underset{\tau \in \mathcal{N}_\uparr(\sigma)}{\agg} \big(\mathsf{m}_{\uparr}\big( \mathbf{h}_{\sigma}, \mathbf{h}_{\tau}\big)\big), \\[1.5pt]
    \mathbf{h}_{\doarr} &= \underset{\tau \in \mathcal{N}_\doarr(\sigma)}{\agg} \big(\mathsf{m}_{\doarr}\big( \mathbf{h}_{\sigma}, \mathbf{h}_{\tau} \big) \big), \\[1.5pt]
    \mathbf{h}_{\sigma}^{\textsf{new}} &= \mathsf{com} \big( \mathbf{h}_{\sigma},  \mathbf{h}_{\mathcal{B}} , \mathbf{h}_{\mathcal{C}o}, \mathbf{h}_{\uparr},     \mathbf{h}_{\doarr} \big).\label{eq:tmp-comb}
\end{align}

It is important to highlight that for any given cell $\sigma$, certain neighborhoods may be empty, meaning they lack adjacent cells. In such cases, the associated representations are considered as zeros. For example, nodes ($0$-cells) do not have neither a boundary neighbourhood nor the lower one. In this case for $\sigma$ being a $0$-cell,~\Cref{eq:tmp-comb} reduces to compute the latent representation of co-boundary (i.e., $\mathbf{h}_{\mathcal{C}o}$) and upper (i.e., $\mathbf{h}_{\uparr}$) messages to combine them into a new representation of $\sigma$ as: $\mathbf{h}_{\sigma}^{\textsf{new}} = \mathsf{com} \big( \mathbf{h}_{\sigma}, \mathbf{h}_{\mathcal{C}o}, \mathbf{h}_{\uparr} \big)$

\

Although message passing schemes on domains that extend beyond traditional graphs have been extensively studied\footnote{Saying {\em "going beyond message passing"} is still an argument of (friendly) discussion among scientist within the field of graph representation learning~\citep{velivckovic2022message}.}, {\em this section focuses specifically on topological neural networks for representation learning over simplicial and cellular complexes}. Many real-world applications yield data in the form of attributed graphs, though this is not an exhaustive representation of all data types encountered in practice. By employing topological domains with higher complexity than simplicial and cell complexes, hypergraphs~\citep{feng2019hypergraph} or combinatorial complexes~\citep{hajijtopological} might require additional inductive biases which usually go far beyond the domain knowledge and will not be considered in this thesis.

%% file: Include/Chapters/2_Background/2.9_sota.tex
\section{State-of-the-Art and Related Works}\label{sec:related_work}

Multiple solutions to face the challenges of graph neural networks have already been proposed. For clarity, it is convenient to introduce the following notion:

\begin{definition}\label{def:rewiring}
Consider an $\MPNN$, a graph $\gph$ with adjacency $\mathbf{A}$, and a map $\mathcal{R}:\R^{n\times n}\rightarrow \R^{n\times n}$. A graph $\gph$ is said to be {\bf rewired} by $\mathcal{R}$, if the messages are exchanged on $\mathcal{R}(\gph)$ instead of $\gph$, with $\mathcal{R}(\gph)$ the graph with adjacency $\mathcal{R}(\mathbf{A})$. 
\end{definition}
\noindent Recent approaches to address over-squashing share a common idea: replace the graph $\gph$ with a rewired graph $\mathcal{R}(\gph)$ enjoying better connectivity \Cref{fig:effective-resistance}. These works are then distinguished based on the choice of the rewiring $\mathcal{R}$. 
\paragraph{Spatial methods.} Since $\MPNN$s fail to propagate information to distant nodes, a solution consists in replacing $\gph$ with $\rew(\gph)$ such that $\mathrm{diam}(\rew(\gph)) \ll \mathrm{diam}(\gph)$. 

Typically, this is achieved by either explicitly adding edges (possibly attributed) 
between distant nodes 
\citep{bruel2022rewiring,abboud2022shortest,gutteridge2023drew} or by allowing distant nodes to communicate through higher-order structures (e.g., cellular or simplicial complexes, \citep{bodnar2021weisfeilercell,bodnar2021weisfeiler}, which requires additional domain knowledge 
and incurs a computational overhead). 
{\em Parametrizing $\mathcal{R}(\cdot)$} This is achieved by considering the rewiring of $\gph$ a function whose parameters can be learned via backpropagation~\citep{hint86}. In~\cite{chen2020iterative, chen2019deep} proposed an end-to-end graph learning framework for jointly and iteratively learning the GCN parameters and an optimal graph topology, as a refinement of the initially available graph. The work in~\cite{tang2019joint}  proposed a dynamic procedure for joint learning of graphs and GCN parameters based on pairwise similarities of convolutional features in each layer. In~\cite{franceschi2019learning}, the authors provided a method for joint learning of graph and GCN parameters based on solving a bilevel program that learns a discrete probability distribution at the edges of the graph.
{\em Graph-Transformers} 
can be seen as an extreme example of rewiring, 
where $\rew(\gph)$ is a {\em complete graph} with edges weighted via attention \citep{kreuzer2021rethinking, mialon2021graphit, ying2021transformers,rampavsek2022recipe}. While these methods do alleviate over-squashing, since they {\em bring all pair of nodes closer}, they come at the expense of making the graph $\rew(\gph)$ much denser. In turn, this has an impact on computational complexity and introduces the risk of mixing local and non-local interactions. 

This group includes~\citep{topping2022understanding} and~\citep{banerjee2022oversquashing},  
where the rewiring is {\em surgical} -- but requires specific pre-processing -- in the sense that 
$\gph$ is replaced by $\rew(\gph)$ where edges have only been added to `mitigate' bottlenecks as identified, for example, by negative curvature \citep{ollivier2007ricci, di2022heterogeneous}. 

Spatial rewiring, intended as accessing information beyond the 1-hop when updating node features, is common to many existing frameworks 
\cite{abu2019mixhop,klicpera2019diffusion, chen2020supervised, ma2020path, wang2020multi, nikolentzos2020k}. However, 
this is usually done via powers of the adjacency matrix, which is the main culprit for over-squashing \citep{topping2022understanding}.
Accordingly, although the diffusion operators $\Anorm^{k}$ allow to aggregate information over non-local hops, they are not suited to mitigate over-squashing. 
\paragraph{Spectral methods.}
The connectedness of a graph $\gph$ can be measured via a quantity known as the {\em Cheeger constant}, defined as follows \citep{chung1997spectral}:

\begin{definition}\label{def:cheeger}
For a graph $\gph$, the Cheeger constant is
\begin{equation*}
    \cheeg = \min_{\mathsf{U}\subset \V}\frac{\lvert \{(u,v)\in\E: u\in \mathsf{U}, v\in \V\setminus \mathsf{U}\}\rvert}{\min(\mathrm{vol}(\mathsf{U}),\mathrm{vol}(\V\setminus \mathsf{U}))},
\end{equation*}
\noindent where $\mathrm{vol}(\mathsf{U}) = \sum_{u\in\mathsf{U}}d_u$, with $d_u$ the degree of node $u$.
\end{definition}
\noindent The Cheeger constant $\cheeg$ represents the energy required to disconnect $\gph$ into two communities. A small $\cheeg$ means that $\gph$ generally has two communities separated by only few edges -- over-squashing is then expected to occur here {\em if} information needs to travel from one community to the other. While $\cheeg$ is generally intractable to compute, thanks to the Cheeger inequality it holds $\cheeg \sim \lambda_1$, where $\lambda_1$ is 
the positive, smallest eigenvalue of the graph Laplacian. Accordingly, a few new approaches have suggested to choose a rewiring that depends on the spectrum of $\gph$ and yields a new graph satisfying $\cheeg(\rew(\gph)) > \cheeg(\gph)$.
s \citet{arnaiz2022diffwire,deac2022expander, karhadkar2022fosr}. 
It is claimed that sending messages over such a graph $\rew(\gph)$
alleviates over-squashing, however this has not been shown analytically yet. 

\textbf{Pooling in MPNNs}: In message passing neural networks the pooling operation refer to a procedure that aims to reduce the number of nodes of the input graph $\gph$ through the layers of the $\MPNN$,  typically it follows a hierarchical scheme in which the pooling regions correspond to graph clusters that are
combined to produce a coarser graph~\citep{bruna2013spectral,defferrard2016convolutional, gama2018convolutional, mesquita2020rethinking}.

\textbf{Advent of Topological Deep Learning}
To cope with the limitations of long-range and group interactions, the field of topological deep learning~\citep{bodnar2022topological} provides the fundamental principles to overcome several limitations of the message passing schemes previously mentioned. In~\cite{bodnar2021weisfeiler} the authors proposed a Simplicial Weisfeiler-Lehman (SWL) colouring procedure for distinguishing non-isomorphic simplicial complexes and a provably powerful message passing scheme based on SWL, that generalise Graph Isomorphism Networks~\citep{xu2019powerful}. This was later refined in~\cite{bodnar2021weisfeilercell}, where the authors introduced CW Networks (CWNs), a hierarchical message-passing on cell complexes proven to be strictly more powerful than the WL test and not less powerful than the 3-WL test. In~\cite{hajij2020cell}, the authors provide a general message-passing mechanism over cell complexes however, they do not study the expressive power of the proposed scheme, nor its complexity. Furthermore, they did not experimentally validate its performance. The works in~\cite{bodnar2022neural, suk2022surfing} introduced Neural Sheaf Diffusion Models, neural architectures that learn a sheaf structure on graphs to improve learning performance on transductive tasks in heterophilic graphs. For a more detailed examination of the architectures developed in the field of topological deep learning, it is worth to read the survey of Papillon~\citep{papillon2023architectures}. Recent works considered also rings within the message passing scheme by means of Junction Trees (JT)~\citep{Fey2020_himp} and by augmenting node features with information about cycles~\citep{bouritsas2022improving}.

%% file: Include/Chapters/3_Bottlenecks/3_bottlenecks.tex
\chapter{On the Limitations of Graph Neural Networks and How Mitigate Them}\label{chap:challenges}
\input{Include/Chapters/3_Bottlenecks/3.1_On_Oversquashing}

%% file: Include/Chapters/3_Bottlenecks/3.1_On_Oversquashing.tex
\section{On Over-Squashing in Message Passing Neural Networks}\label{sec:on_oversq}

The message-passing paradigm, realized via Message-Passing Neural Networks ($\MPNN$s)~\citep{gilmer2017neural}, has been criticized for its limitations related to expressivity~\citep{xu2019powerful}, over-smoothing~\citep{li2018deeper}. Graphs, at their core, represent a basic form of \emph{topological space}, and as such, {\em they often fall short in consistently modeling group and long-range interactions inherent in more complex topologies~\citep{bodnar2022topological}}. When $\MPNN$s propagate messages across distant nodes, many messages are condensed into fixed-size vectors issueing a phenomena known in the literature as \textbf{over-squashing}~\citep{alon2020bottleneck}. While this concern has been recognized and partially linked to graph-topological attributes like edges with \textbf{high negative curvature}~\citep{topping2022understanding} and \textbf{high commute time}~\citep{velingker2022affinity}, several pertinent questions remain unanswered. Among these are the roles of model depth and width in mitigating over-squashing and its relation to graph spectrum~\citep{karhadkar2022fosr} and underlying topology~\citep{deac2022expander}.


\paragraph{The goal of this section.} The analysis of~\citet{topping2022understanding} represents the current theoretical understanding of the over-squashing problem. 
However, it leaves some important open questions which are addressed in this section: 
(i) The role of the {\bf width} in mitigating over-squashing; (ii) What happens when the {\bf depth} exceeds the distance among two nodes of interest; (iii) How over-squashing is related to the graph structure (beyond local curvature-bounds) and its {\bf spectrum}. Therefore, this section provides {\em a unified framework to explain how spatial and spectral approaches alleviate over-squashing}.

\paragraph{Contributions and outline.} An $\MPNN$ is generally constituted by two main parts: a choice of architecture, and an underlying graph over which it operates. This section provides an investigatation how these factors participate in the over-squashing phenomenon focusing on the width and depth of the $\MPNN$, as well as on the graph-topology. 

\begin{itemize}
    \item  \Cref{sec:width} formally state, how the \emph{width} can mitigate over-squashing (\Cref{cor:bound_MLP_MPNN}), albeit at the potential cost of 
    generalization. 
    \item \Cref{sec:depth}, shows that depth may not be able to alleviate over-squashing. In particular, two regimes are identified: the first one, the number of layers is comparable to the graph diameter, and \Cref{cor:over-squasing_distance}  proves that over-squashing is likely to occur among distant nodes. In fact, the distance at which over-squashing happens is strongly dependent on 
    the graph topology. In the second regime, an arbitrary (large) number of layers are considered. Therefore, due to \Cref{thm:vanishing}, in this stage the $\MPNN$ is, generally, 
    dominated by vanishing gradients. This result is of independent interest, since it characterizes analytically conditions of vanishing gradients of the loss for a large class of $\MPNN$s that also include residual connections.
    \item \Cref{sec:topology} shows that the \emph{topology} of the graph has the greatest impact on over-squashing. In fact, \Cref{thm:effective_resistance} states that over-squashing happens among nodes with high commute time. This provides a unified framework to explain why all spatial and spectral {\em rewiring} approaches (discussed in \Cref{sec:related_work}) do mitigate over-squashing.
\end{itemize}


\begin{figure}[t]
    \centering
    \includegraphics[width=.8\textwidth]{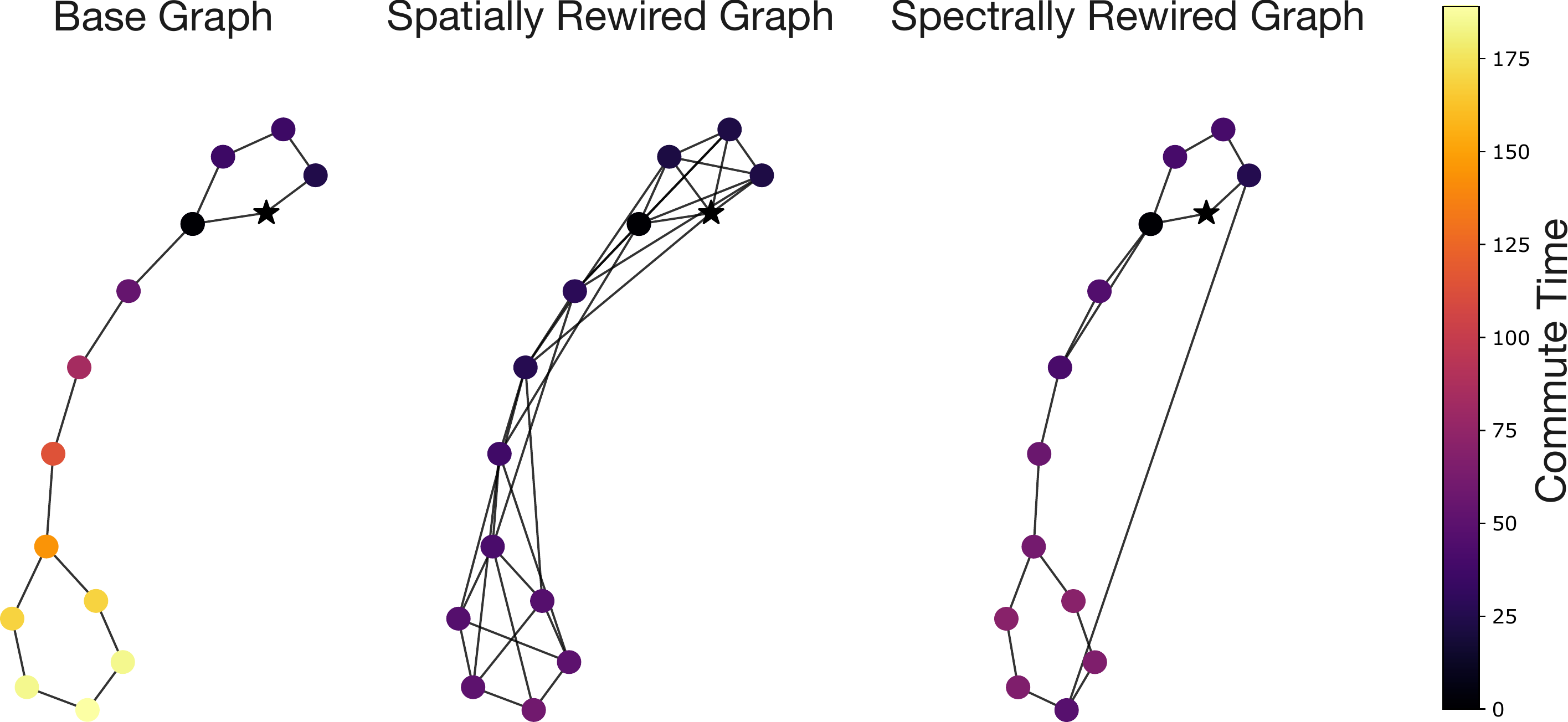}
    \vspace{-1.9mm}
    \caption{Effect of different rewirings $\mathcal{R}$ on the graph connectivity. 
    The colouring 
    denotes Commute Time -- defined in~\Cref{sec:topology} -- w.r.t. to the star node. 
    From left to right, the graphs shown are: the base, spatially rewired and spectrally rewired. The added edges significantly reduce the Commute Time and hence mitigate over-squashing in light of~\Cref{thm:effective_resistance}. }
    \vspace{-3mm}\label{fig:effective-resistance}
\end{figure}

\subsection{The impact of width}\label{sec:width}
This section addresses whether the width of the underlying $\MPNN$ can mitigate over-squashing and to what extent this is possible. In order to do that, the sensitivity analysis in \citet{topping2022understanding} is extended to higher-dimensional node features. In particular, consider a class of $\MPNN$s parameterised by neural networks, of the form: 
\begin{equation}\label{eq:MPNN_mlp}
    \mathbf{h}_{v}^{(t+1)} = \up\Big(c_{\rs}\W_{\rs}^{(t)} \mathbf{h}_{v}^{(t)} + c_{\mpas}\W_{\mpas}^{(t)}\sum_{u}\Anorm_{vu}\mathbf{h}^{(t)}_{u}\Big),
\end{equation}
\noindent where $\sigma$ is a pointwise-nonlinearity, $\W_{\rs}^{(t)},\W_{\mpas}^{(t)}\in\R^{p\times p}$ are learnable weight matrices and $\Anorm$ is a graph shift operator. Note that~\Cref{eq:MPNN_mlp} includes common $\MPNN$s such as $\mathsf{GCN}$~\citep{kipf2017graph}, $\mathsf{SAGE}$ \citep{hamilton2017inductive}, and $\mathsf{GIN}$ \citep{xu2019powerful}, where $\Anorm$ is one of $\mathbf{D}^{-1/2}\mathbf{A}\mathbf{D}^{-1/2}$, $\mathbf{D}^{-1}\mathbf{A}$ and $\mathbf{A}$, respectively, with $\mathbf{D}$ the diagonal degree matrix. In~\Cref{app:sec_width}, this analysis is extended to a more general class of $\MPNN$s (see~\Cref{thm:bound_general}), which includes stacking multiple nonlinearities. It is worth noting that the positive scalars $c_{\rs},c_{\mpas}$ represent the weighted contribution of the residual term and of the aggregation term, respectively. To simplify notations, a set of message-passing matrices that depend on $c_{\rs},c_{\mpas}$ are introduced. 

\begin{definition}\label{def:oper}
For a graph shift operator $\Anorm$ and constants $c_{\rs},c_{\mpas} > 0$,  define $\oper := c_{\rs}\mathbf{I} + c_{\mpas}\Anorm\in\R^{n\times n}$ to be the message-passing matrix adopted by the $\MPNN$.
\end{definition}

As in \citet{xu2018representation} and \citet{ topping2022understanding}, this section analyse the propagation of information in the $\MPNN$ via the Jacobian of node features after $m$ layers.  

\begin{theorem}[\textbf{Sensitivity bounds}]\label{cor:bound_MLP_MPNN}
Consider an $\MPNN$ as in~\Cref{eq:MPNN_mlp} for $m$ layers, with $c_{\up}$ the Lipschitz constant of the nonlinearity $\sigma$ and $w$ the maximal entry-value over all weight matrices. 
For $v,u\in\V$ and width $p$, it holds
\begin{equation}\label{eq:upper_bound_mlp}
    \left\| \frac{\partial \mathbf{h}_{v}^{(m)}}{\partial \mathbf{h}_{u}^{(0)}}\right\|_{L_1} \leq  (\underbrace{\vphantom{\oper^m}c_{\up}wp}_{\mathrm{model}})^{m}\underbrace{(\oper^{m})_{vu}}_{\mathrm{topology}},
\end{equation}
with $\oper^{m}$ the $m^{th}$-power of $\oper$ introduced in \Cref{def:oper}.
\end{theorem}
\noindent 
Over-squashing occurs if the right hand side of Eq.~\eqref{eq:upper_bound_mlp} is too small -- this will be related to the distance among $v$ and $u$ in 
\Cref{subsec:shallow}. A small derivative of $\mathbf{h}_v^{(m)}$ with respect to $\mathbf{h}_u^{(0)}$ means that after $m$ layers, {\em the feature at $v$ is mostly insensitive to the information initially contained at $u$}, and hence that messages have not been propagated effectively. \Cref{cor:bound_MLP_MPNN} clarifies how the model can impact over-squashing through (i) its Lipschitz regularity $c_{\up}, w$ and (ii) its width $p$. In fact, given a graph $\gph$ such that $(\oper^{m})_{vu}$ decays exponentially with $m$, the $\MPNN$ can compensate by increasing the width $p$ and the magnitude of $w$ and $c_\sigma$.  
This confirms 
analytically the discussion in \citet{alon2020bottleneck}: \textbf{a larger hidden dimension $p$ does mitigate over-squashing}. However, this is not an optimal solution since increasing the contribution of the model (i.e. the term $c_{\up}wp$) may lead to over-fitting and poorer generalization \citep{bartlett2017spectrally}. Taking larger values of $c_{\up},w,p$ affects the model {\em globally} and does not target the sensitivity of specific node pairs induced by the topology via $\oper$. 



\begin{tcolorbox}[boxsep=0mm,left=2.5mm,right=2.5mm]
\textbf{Message of the Section:} {\em The Lipschitz regularity, weights, and width of the underlying $\MPNN$ can help mitigate the effect of over-squashing. However, this is a remedy that comes at the expense of generalization and does not address the real culprit behind over-squashing: the graph-topology}.   
\end{tcolorbox}

\subsection{The impact of depth}\label{sec:depth}
Consider a graph $\gph$ and 
a task with `long-range' dependencies, meaning that there exists (at least) a node $v$ whose embedding has to account for information contained at some node $u$ situated at a considerably large distance $r \gg 1$. One natural attempt at resolving over-squashing amounts to increasing the number of layers $m$ to compensate for the distance. However, evidence suggests that simply increasing the depth of an $\MPNN$ does not effectively mitigate over-squashing. The findings reveal that: (i) If depth $m$ mirrors the distance, over-squashing is bound to occur among distant nodes. Moreover, the distance at which this occurs is intrinsically linked to the underlying topology; (ii) Upon incorporating a high number of layers to encompass long-range interactions, certain precise conditions are outlined under which $\MPNN$s face the vanishing gradients problem.



\subsection{The shallow-diameter regime: over-squashing occurs among distant nodes}\label{subsec:shallow}

Consider the scenario above, with two nodes $v,u$, whose interaction is important for the task, at distance $r$. First, focus on the regime $m \sim r$ referred to this as the {\em shallow-diameter} regime, since the number of layers $m$ is comparable to the diameter of the graph. 

From now on, let $\Anorm = \mathbf{D}^{-1/2}\mathbf{A}\mathbf{D}^{-1/2}$, and recall that $\mathbf{A}$ is the adjacency matrix and $\mathbf{D}$ is the degree matrix. This is not restrictive, but allows to derive more explicit bounds and, later, bring into the equation the spectrum of the graph. Notice that results can be extended easily to $\mathbf{D}^{-1}\mathbf{A}$, given that this matrix is similar to $\Anorm$, and, in expectation, to $\mathbf{A}$ by normalizing the Jacobian as in \citet{xu2019powerful} and Section A in the Appendix of \citet{topping2022understanding}. 


\begin{theorem}[\textbf{Over-squashing among distant nodes}]\label{cor:over-squasing_distance} Given an $\MPNN$ as in~\Cref{eq:MPNN_mlp}, with $c_{\mpas} \leq 1$, let $v,u\in\mathsf{V}$ be at distance $r$. Let $c_{\up}$ be the Lipschitz constant of $\sigma$, $w$ the maximal entry-value over all weight matrices, $d_{\mathrm{min}}$ the minimal degree of $\gph$, and $\gamma_{\ell}(v,u)$ the number of walks from $v$ to $u$ of maximal length $\ell$. For any $0 \leq k < r$, there exists $C_{k} > 0$ {\bf independent} of $r$ and of the graph, such that 
\begin{equation}\label{eq:cor_distance}
   \left\| \frac{\partial \mathbf{h}_{v}^{(r+k)}}{\partial \mathbf{h}_{u}^{(0)}}\right\|_{L_1} \leq C_{k}\gamma_{r+k}(v,u)\Big(\frac{2c_{\up}wp}{d_{\mathrm{min}}}\Big)^r.
\end{equation}
\end{theorem}
To understand the bound above, fix $k < r$ and assume that nodes $v,u$ are `badly' connected, meaning that the number of walks $\gamma_{r+k}(v,u)$ of length at most $r+k$, is small. If $2\, c_{\up}wp < d_{\mathrm{min}}$, then the bound on the Jacobian in~\Cref{eq:cor_distance} {\em decays exponentially with the distance} $r$. 
Note that the bound above considers $d_{\mathrm{min}}$ and $\gamma_{r+k}$ as a worst case scenario. If one has a better understanding of the topology of the graph, sharper bounds can be derived by estimating $(\oper^{r})_{vu}$. 
\Cref{cor:over-squasing_distance} implies that, when the depth $m$ is comparable to the diameter of $\gph$,
\textbf{\em over-squashing becomes an issue if the task depends on the interaction of nodes $v,u$ at `large' distance $r$.} 
In fact, 
\Cref{cor:over-squasing_distance} 
shows that the distance at which the Jacobian sensitivity falls below a given threshold, depends on both the model, via $c_{\up}, w,p$, and on the graph, through $d_{\mathrm{min}}$ and $\gamma_{r+k}(v,u)$. This implies that~\Cref{cor:over-squasing_distance} generalizes the analysis in \citet{topping2022understanding} in multiple ways: (i) it holds for any width $p > 1$; (ii) it includes cases where $m > r$; (iii) it provides explicit estimates in terms of number of walks and degree information. 

{\bf Remark.} What if $2c_{\up}wp > d_{\mathrm{min}}$? Taking larger weights and hidden dimension increases the sensitivity of node features. However, this occurs {\em everywhere} in the graph the same. Accordingly, nodes at shorter distances will, on average, still have sensitivity exponentially larger than nodes at large distance. This is validated in the synthetic experiments in~\Cref{app:on_oversq}, where the weights do not have constraints on.


\paragraph{The deep regime: vanishing gradients dominate}\label{subsec:deep}
Now the focus will be on the regime where the number of layers $m \gg r$ is 
large. 
In this case, vanishing gradients can occur and make the entire model insensitive. 
Given a weight $\theta^{(k)}$ entering a layer $k$, one can write the gradient of the loss after $m$ layers as \citep{pascanu2013difficulty}
\begin{equation}
    \frac{\partial \mathcal{L}}{\partial \theta^{(k)}} = \sum_{v,u\in V}\Big(\frac{\partial \mathcal{L}}{\partial \mathbf{h}^{(m)}_v}\frac{\partial \mathbf{h}_u^{(k)}}{\partial \vphantom{\mathbf{h}^{(k)}_u}\theta^{(k)}}\Big)\underbrace{\frac{\partial \mathbf{h}_v^{(m)}}{\partial \mathbf{h}_{u}^{(k)}}}_{\mathrm{sensitivity}}
\end{equation}
\noindent Here there are 
provided the {\bf exact conditions} for $\MPNN$s to incur the vanishing gradient problem, intended as the gradients of the loss decaying exponentially with the number of layers $m$. 
\begin{theorem}[\textbf{Vanishing gradients}]\label{thm:vanishing} Consider an $\MPNN$ as in Eq.~\eqref{eq:MPNN_mlp} for $m$ layers with a quadratic loss $\mathcal{L}$. Assume that (i) $\sigma$ has Lipschitz constant $c_{\up}$ and $\sigma(0) = 0$, and (ii) weight matrices have spectral norm bounded by $\mu > 0$. 
Given any weight $\theta$ entering a layer $k$, there exists a constant $C > 0$ independent of $m$, such that
\begin{align}
    \left\vert \frac{\partial \mathcal{L}}{\partial \theta}\right\vert &\leq C\left(c_{\up}\mu(c_{\rs} + c_{\mpas})\right)^{m-k}\left(1  +\left(c_{\up}\mu(c_{\rs} + c_{\mpas})\right)^{m}\right).
\end{align}
\noindent In particular, if $c_{\up}\mu(c_{\rs} + c_{\mpas}) < 1$, then the gradients of the loss decay to zero exponentially fast with $m$.
\end{theorem}

\noindent The problem of vanishing gradients for graph convolutional networks have been studied from an empirical perspective~\citep{li2019deepgcns,li2021training}.~\Cref{thm:vanishing} provides sufficient conditions for the vanishing of gradients to occur in a large class of $\MPNN$s that also include (a form of) residual connections through the contribution of $c_{\rs}$ in~\Cref{eq:MPNN_mlp}. This extends a behaviour studied for Recurrent Neural Networks~\citep{bengio1994learning,hochreiter1997long, pascanu2013difficulty, rusch2021coupled, rusch2021unicornn} to the $\MPNN$ class. Some discussion on vanishing gradients for $\MPNN$s can be found in~\cite{ruiz2020gated} and~\cite{rusch2022graph}. A few final comments are in order. (i) The bound in \Cref{thm:vanishing} seems to `hide' the contribution of the graph. This is, in fact, because the spectral norm of the graph operator $\oper$ is $c_{\rs} + c_{\mpas}$ -- An investigation of more general graph shift operators~\citep{dasoulas2021learning} is left to future work. (ii)~\Cref{cor:over-squasing_distance} shows that if the distance $r$ is large enough and the number of layers $m$ is chosen such hat. $m \sim r$, over-squashing arises among nodes at distance $r$. Taking the number of layers large enough though, may incur the vanishing gradient problem 
\Cref{thm:vanishing}. 
In principle, there might be an intermediate regime where $m$ is larger than $r$, but {\em not} too large, in which the depth could help with over-squashing before it leads to vanishing gradients. Given a graph $\gph$, and bounds on the Lipschitz regularity and width, there exists $\tilde{r}$, depending on the topology of $\gph$, such that if the task has interactions at distance $r > \tilde{r}$, no number of layers can allow the $\MPNN$ class to solve it. This is left for future work. 



\begin{tcolorbox}[boxsep=0mm,left=2.5mm,right=2.5mm]
\textbf{Message of the Section:} \em Increasing the depth $m$ will, in general, not fix over-squashing. As $m$ increases, 
$\MPNN$s transition from over-squashing (\Cref{cor:over-squasing_distance}) to vanishing gradients (\Cref{thm:vanishing}). 
\end{tcolorbox}
\vspace{-5pt}

\subsection{The impact of topology }\label{sec:topology}
This section discusses the impact of graph topology, particularly the graph spectrum, on over-squashing. This allows to draw a unified framework that shows why existing approaches manage to alleviate over-squashing by either spatial or spectral rewiring (\Cref{sec:related_work}).  

\paragraph{On over-squashing and access time} Throughout the section over-squashing is related to well-known properties of random walks on graphs. To this aim, it is worth to review basic concepts about random walks.

\paragraph{Access and commute time.} A Random Walk (RW) on a graph $\gph$ is a Markov chain where, at each step, it moves from a node $v$ to one of its neighbors with probability proportional to $1/d_v$, where $d_v$ is the degree of node $v$. 
Several properties about RWs have been studied. Of particular interest in this context are the notions of {\em access time} $\mathsf{t}(v,u)$ and {\em commute time} $\tau(v,u)$ (see~\Cref{fig:effective-resistance}). The access time $\mathsf{t}(v,u)$ (also known as {\em hitting time}) is the  expected number of steps
before node $u$ is visited for a RW starting from node $v$. The commute time instead, represents the expected number of steps in a
RW starting at $v$ to reach node $u$ and {\em come back}. A high access (commute) time means that nodes $v,u$ generally struggle to visit each other in a RW -- this can happen if nodes are far-away, but it is in fact more general and strongly dependent on the topology.

Some connections between over-squashing and the topology have already been derived (\Cref{cor:over-squasing_distance}), but up to this point `topology' has entered the picture through 'distances' only. 
In this section, over-squashing is further linked to other quantities related to the topology of the graph, such as access time, commute time and the Cheeger constant. Ultimately this section provides a unified framework to understand how existing approaches manage to mitigate over-squashing via graph-rewiring.


\paragraph{Integrating information across different layers.} Consider a family of $\MPNN$s of the form
\begin{equation}\label{eq:mpnn_simplified}
    \mathbf{h}_{v}^{(t)} = \mathsf{ReLU}\Big(\W^{(t)}\Big(c_{\rs}\mathbf{h}_v^{(t-1)} + c_{\mpas}(\Anorm\mathbf{h}^{(t-1)})_v\Big)\Big).
\end{equation}
\noindent Similarly to \citet{kawaguchi2016deep,xu2018representation}, the following assumptions are required:


\begin{assumption}\label{assumption_main_body} All paths in the computation graph
of the model are activated with the same probability of
success $\rho$.
\end{assumption}
Take two nodes $v\neq u$ at distance $r\gg 1$ and consider an $\MPNN$ that sends information {\em from $u$ to $v$}. Given a layer $k < m$ of the $\MPNN$, by~\Cref{cor:over-squasing_distance} it might be expected that $\mathbf{h}_v^{(m)}$ is much more sensitive to the information contained {\em at the same} node $v$ at an earlier layer $k$, i.e. $\mathbf{h}_{v}^{(k)}$, rather than to the information contained at a distant node $u$, i.e. $\mathbf{h}_{u}^{(k)}$. 
\noindent Accordingly, consider the following quantity: 
\begin{align*}
\mathbf{J}_{k}^{(m)}(v,u) := \frac{1}{d_v}\frac{\partial \mathbf{h}_{v}^{(m)}}{\partial \mathbf{h}_{v}^{(k)}} - \frac{1}{\sqrt{d_v d_u}}\frac{\partial \mathbf{h}_{v}^{(m)}}{\partial \mathbf{h}_{u}^{(k)}}. 
\end{align*}
\noindent Notice that the normalization by degree stems from the choice $\Anorm = \mathbf{D}^{-1/2}\mathbf{A}\mathbf{D}^{-1/2}$. Here it is provided an intuition for this term. 
Say that node $v$ at layer $m$ of the $\MPNN$ is mostly insensitive to the information sent from $u$ at layer $k$. Then, on average, $\|\partial\mathbf{h}_v^{(m)} / \partial \mathbf{h}_u^{(k)} \| \ll \| \partial\mathbf{h}_v^{(m)} / \partial \mathbf{h}_v^{(k)} \| $. In the opposite case instead, on average, $\| \partial\mathbf{h}_v^{(m)} / \partial \mathbf{h}_u^{(k)} \| \sim \| \partial\mathbf{h}_v^{(m)} / \partial \mathbf{h}_v^{(k)} \| $. Therefore $\| \mathbf{J}^{(m)}_k(v,u)\|$ will be {\em larger} when $v$ is (roughly) independent of the information contained at $u$ at layer $k$. Therefore, the same argument can be extended by accounting for messages sent at each 
layer $k \leq m$. 
\begin{definition}\label{def:obst}
The Jacobian obstruction of node $v$ with respect to node $u$ after $m$ layers is $
    \obst^{(m)}(v,u) =  \sum_{k = 0}^{m}\|\mathbf{J}_k^{(m)}(v,u)\|.
$
\end{definition}
\noindent 
\noindent 
As motivated above, a larger $\obst^{(m)}(v,u)$ means that, after $m$ layers, the representation of node $v$ is more likely to be insensitive to information contained at $u$ and conversely, a small $\obst^{(m)}(v,u)$ means that nodes $v$ is, on average, able to receive information from $u$. 
Differently from the Jacobian bounds of the earlier sections, here the contribution coming from all layers $k \leq m$ is considered (note the sum over layers $k$ in~\Cref{def:obst}). 
\begin{theorem}[\textbf{Over-squashing and access-time}]\label{thm:access} Consider an $\MPNN$ as in Eq.~\eqref{eq:mpnn_simplified} and let Assumption \ref{assumption_main_body} hold. If $\nu$ is the smallest singular value across all weight matrices and $c_{\rs},c_{\mpas}$ are such that $\nu (c_{\rs} + c_{\mpas}) = 1$, then, 
in expectation,  
\[
\obst^{(m)}(v,u) \geq \frac{\rho}{\nu c_{\mpas}}\frac{\mathsf{t}(u,v)}{2\lvert \E\rvert}  + o(m), 
\]
\noindent with $o(m)\rightarrow 0$ exponentially fast with $m$. 

\end{theorem} 
\noindent Notice that an exact expansion of the term $o(m)$ is reported in~\Cref{app:on_oversq}. 
Also observe that more general bounds are possible if $\nu(c_{\rs} + c_{\mpas}) < 1$ -- however, they will progressively become less informative in the limit $\nu (c_{\rs} + c_{\mpas}) \rightarrow 0$. \Cref{thm:access} shows that the obstruction is a function of the access time $\mathsf{t}(u,v)$; {\bf high access time, on average, translates into high obstruction for node $v$ to receive information from node $u$ inside the $\MPNN$}. This resonates with the intuition that access time is a measure of how easily a `diffusion' process starting at $u$ reaches $v$. In particular, the obstruction provided by the access time cannot be fixed by increasing the number of layers and in fact this is independent of the number of layers, further corroborating the analysis in \Cref{sec:depth}. 
Next, over-squashing is related to commute time, and hence, to effective resistance.

\paragraph{On over-squashing and commute time}\label{subsec:effective}
Let's restrict the attention to a slightly more special form of over-squashing. To this aim, consider nodes $v,u$ exchanging information both ways -- differently from before where node $v$ receives information from node $u$. Following the same intuition described previously, consider the symmetric quantity: 
\begin{align*}
\tilde{\mathbf{J}}_{k}^{(m)}(v,u) &:= \Big(\frac{1}{d_v}\frac{\partial \mathbf{h}_{v}^{(m)}}{\partial \mathbf{h}_{v}^{(k)}} - \frac{1}{\sqrt{d_v d_u}}\frac{\partial \mathbf{h}_{v}^{(m)}}{\partial \mathbf{h}_{u}^{(k)}}\Big)
\\ &+ \Big(\frac{1}{d_u}\frac{\partial \mathbf{h}_{u}^{(m)}}{\partial \mathbf{h}_{u}^{(k)}}  - \frac{1}{\sqrt{d_v d_u}}\frac{\partial \mathbf{h}_{u}^{(m)}}{\partial \mathbf{h}_{v}^{(k)}}\Big). 
\end{align*}
Once again,  $\|\tilde{\mathbf{J}}^{(m)}_k(v,u)\|$ is expected to be larger if nodes $v,u$ are failing to communicate in the $\MPNN$, and conversely to be smaller whenever the communication is sufficiently robust. Similarly, merge the information collected at each layer $k\leq m$.
\begin{definition}
The symmetric Jacobian obstruction of nodes $v,u$ after $m$ layers is $
    \tilde{\obst}^{(m)}(v,u) =  \sum_{k = 0}^{m}\|\tilde{\mathbf{J}}_k^{(m)}(v,u)\|.
$
\end{definition}
\noindent 
\noindent 
\noindent The intuition of comparing the sensitivity of a node $v$ with a different node $u$ and to itself, and then swapping the roles of $v$ and $u$, resembles the concept of commute time $\tau(v,u)$. In fact, this is not a coincidence:
\noindent

\begin{theorem}[\textbf{Over-squashing and commute-time}]\label{thm:effective_resistance} Consider an $\MPNN$ as in Eq.~\eqref{eq:mpnn_simplified} with $\mu$ the maximal spectral norm of the weight matrices and $\nu$ the minimal singular value. Let Assumption \ref{assumption_main_body} hold. 
If $\mu (c_{\rs} + c_{\mpas}) \leq 1$, then there exists $\epsilon_\gph$, independent of nodes $v,u$, such that in expectation, 
    \begin{align*}
\epsilon_\gph(1 - o(m))\frac{\rho}{\nu c_{\mpas}}\frac{\tau(v,u)}{2\lvert\E\rvert} \leq \tilde{\obst}^{(m)}(v,u) \leq \frac{\rho}{\mu c_{\mpas}}\frac{\tau(v,u)}{2\lvert\E\rvert},
    \end{align*}
with $o(m)\rightarrow 0$ exponentially fast with $m$ increasing.
\end{theorem}
Notice that an explicit expansion of the $o(m)$-term is reported in the proof of the Theorem in the Appendix.
\noindent By the previous discussion, a {\bf smaller} $\tilde{\obst}^{(m)}(v,u)$ means 
$v$ is more sensitive to $u$ in the $\MPNN$ (and viceversa when $\tilde{\obst}^{(m)}(v,u)$ is large). Therefore, \Cref{thm:effective_resistance} implies that nodes at small commute time will exchange information better in an $\MPNN$ and conversely for those at high commute time.
This has some {\bf important consequences}:
\begin{itemize}
    \item [(i)] When the task only depends on local interactions, the property of $\MPNN$ of reducing the sensitivity to messages from nodes with high commute time {\em can} be beneficial since it decreases harmful redundancy. 
    \item[(ii)] Over-squashing is an issue when the task depends on the interaction of nodes with high commute time. 
    \item[(iii)] The commute time represents an obstruction to the sensitivity of an $\MPNN$ which is {\em independent of the number of layers}, since the bounds in  \Cref{thm:effective_resistance} are independent of $m$ (up to errors decaying exponentially fast with $m$). 
\end{itemize}

\noindent Notice that the same comments hold in the case of access time as well if, for example, the task depends on node $v$ receiving information from node $u$ but not on $u$ receiving information from $v$.
\paragraph{A unified framework}\label{subsec:unified} 

\paragraph{Why spectral-rewiring works.}First, it it discussed and justified why the spectral approaches discussed in \Cref{sec:related_work} mitigate over-squashing. This comes as a consequence of \citet{lovasz1993random} and \Cref{thm:effective_resistance}:
\begin{corollary}\label{cor:spectral_methods}
 Under the assumptions of \Cref{thm:effective_resistance}, for any $v,u\in\mathsf{V}$, it holds:
 \begin{equation*}
 \tilde{\obst}^{(m)}(v,u)\leq \frac{4}{\rho\mu c_{\mpas}}\frac{1}{\cheeg^2}.
 \end{equation*}
\end{corollary}
\noindent \Cref{cor:spectral_methods} essentially tells that the obstruction among {\em all} pairs of nodes decreases (so better information flow) if the $\MPNN$ operates on a graph $\gph$ with larger Cheeger constant. This rigorously justifies why recent works like \citet{arnaiz2022diffwire, deac2022expander, karhadkar2022fosr} manage to alleviate over-squashing by 
propagating information on a rewired graph $\rew(\gph)$ with larger Cheeger constant $\cheeg$. 
This result also highlights why bounded-degree expanders are particularly suited 
- as leveraged in \citet{deac2022expander} -- given that their commute time is only $\mathcal{O}(\lvert\E\rvert)$ \citep{chandra1996electrical}, making the bound in \Cref{thm:effective_resistance} scale as $\mathcal{O}(1)$ w.r.t. the size of the graph. In fact, the concurrent work of \citet{black2023understanding} leverages directly the effective resistance of the graph $\mathsf{Res}(v,u) = \tau(v,u)/2\lvert\mathsf{E}\rvert$ to guide a rewiring that improves the graph connectivity and hence mitigates over-squashing.

\paragraph{Why spatial-rewiring works.} 
\cite{chandra1996electrical} proved that the commute time satisfies: $\tau(v,u) = 2\lvert \mathsf{E}\rvert \res(v,u)$, with $\res(v,u)$ the {\bf effective resistance} of nodes $v,u$. 
$\res(v,u)$ measures the voltage difference between nodes $v,u$ if a unit current flows through the graph from $v$ to $u$ and each edge is taken to represent a unit resistance~\citep{thomassen1990resistances,dorfler2018electrical}, and has also been used in~\cite{velingker2022affinity} as a form of structural encoding. Therefore, consider that~\Cref{thm:effective_resistance} can be \textbf{\em equivalently rephrased as saying that nodes at high-effective resistance struggle to exchange information in an $\MPNN$} and viceversa for node at low effective resistance. A result known as Rayleigh's monotonicity principle~\citep{thomassen1990resistances}, asserts that the \emph{total} effective resistance $\res_{\gph} = \sum_{v,u} \res{(v,u)}$ decreases when adding new edges -- which offers a new interpretation as to why spatial methods help combat over-squashing. 

\paragraph{What about curvature?} This analysis also sheds further light on the relation between over-squashing and curvature derived in \citet{topping2022understanding}. If the effective resistance is bounded from above, this leads to lower bounds for the resistance curvature introduced in \citet{devriendt2022discrete} and hence, under some assumptions, for the Ollivier curvature too \citep{ollivier2007ricci, ollivier2009ricci}. This analysis then recovers why preventing the curvature from being `too' negative has benefits in terms of 
reducing over-squashing. 

\begin{tcolorbox}[boxsep=0mm,left=2.5mm,right=2.5mm]
\textbf{Message of the Section:} \em  $\MPNN$s struggle to send information among nodes with high commute (access) time (equivalently, effective resistance). This connection between over-squashing and commute (access) time provides a unified framework for explaining why spatial and spectral-rewiring approaches manage to alleviate over-squashing.
\end{tcolorbox}

\subsection{Discussion}\label{sec:conclusion}

{\bf What was done?} In this section, the role played by width, depth, and topology in the over-squashing phenomenon have been investigated. In particular, this section proved that, while width can partly mitigate this problem, depth is, instead, generally bound to fail since over-squashing spills into vanishing gradients for a large number of layers. In fact, as shown, the graph-topology plays the biggest role, with the commute (access) time providing a strong indicator for whether over-squashing is likely to happen independently of the number of layers. As a consequence of this analysis, is possible to draw a unified framework where rigorously justifications are provided regarding all recently proposed rewiring methods do alleviate over-squashing.

\paragraph{Limitations.} 
The analysis in this work primarily applies to $\MPNN$s that assign uniform weight to each edge contribution, subject to degree normalization. In the opposite case, which, for example, includes $\mathsf{GAT}$ \citep{velivckovic2018graph} and $\mathsf{GatedGCN}$ \citep{bresson2017residual}, over-squashing can be further mitigated 
by pruning the graph, 
hence alleviating the dispersion of information. However, the attention (gating) mechanism can fail if it is not able to identify which branches to ignore and can even amplify over-squashing by further reducing `useful' pathways. In fact, $\mathsf{GAT}$ still fails on the $\mathsf{Graph}$ $\mathsf{Transfer}$ task of \Cref{sec:depth}, albeit it seems to exhibit slightly more robustness. Extending the Jacobian bounds to this case is not hard, but will lead to less transparent formulas: a thorough analysis of this class, is left for future work. 
Moreover, determining when the sensitivity is 'too' small is generally also a function of the resolution of the readout, which have not been considered. Finally, \Cref{thm:effective_resistance} holds in expectation over the nonlinearity and, generally, \Cref{def:obst} encodes an average type of behaviour: a more refined (and exact) analysis is left for future work. 


\paragraph{Where to go from here.} This section shows the necessity of further analysis on the relation between over-squashing and vanishing gradient deserves. In particular, it seems that there is a phase transition that $\MPNN$s undergo from over-squashing of information between distant nodes, to vanishing of gradients at the level of the loss. In fact, this connection suggests that traditional methods that have been used in RNNs and GNNs to mitigate vanishing gradients, may also be beneficial for over-squashing. On a different note, this section has not touched on the important problem of over-smoothing; the theoretical connections derived so far, based on the relation between over-squashing, commute time, and Cheeger constant, suggest a much deeper interplay between these two phenomena. Finally, while this analysis confirms that both spatial and spectral-rewiring methods provably mitigate over-squashing, it does not tell which method is preferable, when, and why. The theoretical investigation of over-squashing provided here also help tackle this important methodological question.

%% file: Include/Chapters/4_Enhancing/4_enhancing_graph_representations.tex
\chapter{Enhancing Graph Representation with Topological Approaches}\label{chap:tnns}

\input{Include/Chapters/4_Enhancing/4.1_topological_attention_networks}
\input{Include/Chapters/4_Enhancing/4.2_topological_mp}

%% file: Include/Chapters/4_Enhancing/4.1_topological_attention_networks.tex
\input{Include/Chapters/4_Enhancing/4.1.1_simplicial_attention_networks}

\input{Include/Chapters/4_Enhancing/4.1.2_cell_attention_networks}

%% file: Include/Chapters/4_Enhancing/4.1.1_simplicial_attention_networks.tex
\section{Simplicial Attention Networks}\label{sec:san}

It should be clear at this point that Message Passing Neural Networks (MPNNs)~\citep{gilmer2017neural} are able to provide exceptional performance in graph representation learning tasks. However, motivated by the ability of these discrete domains to capture higher-order connectivity structures, there has recently been a shift beyond traditional graphs towards more complex topological spaces like simplicial~\citep{ebli2020simplicial, bunch2020simplicial, bodnar2021weisfeiler, yang2022simplicial} and cell complexes~\citep{bodnar2021weisfeiler, hajij2020cell}..
The introduction of Simplicial Neural Networks (SNNs) has opened up new directions in tasks like missing data imputation~\citep{ebli2020simplicial}, link prediction~\citep{chen2022bscnets}, graph classification~\citep{bunch2020simplicial, bodnar2021weisfeiler}, and trajectory prediction~\citep{bodnar2021weisfeiler, roddenberry2021principled}. However, a critical aspect in these methods is the strong coupling between the computational graph induced by the message passing operations and the combinatorial structure of the underlying domain. Moreover, they consider \textit{isotropic} aggregations, meaning that a simplex $\sigma$ aggregates the messages from its neighbours without accounting for the importance of the message. This results in a dramatic drop in expressive power, with the consequence being that models lack generalisation capabilities for out-of-distribution data. Inspired by graph attention networks~\citep{velivckovic2018graph} and dynamic graph attention networks~\citep{brody2021attentive}, this section introduces Simplicial Attention Networks (SAN). This class of neural models learn to dynamically adapt their focus based on the relevance of the simplices' features.

\

\paragraph{Simplicial Attention}\makeatletter\def\@currentlabel{Simplicial Attention}\makeatother\label{par:simplicial_attention}
The core idea behind the adaptability of these architectures is grounded in the {\em simplicial attention}, two independent topology-aware self-attention mechanisms designed to separately calibrate the information being aggregated from simplices within the {\em upper and lower neighbourhoods} of a simplex $\sigma$. 

Let $\kph = (\V, \mathsf{S})$ be a simplicial complex of order $K$, such that $\sigma, \tau \in \kph$ and $\tau \in \mathcal{N}_{\uparr}(\sigma)$ or $\tau \in \mathcal{N}_{\doarr}(\sigma)$. Both simplices $\sigma$ and $\tau$ are also equipped with latent representations, $\mathbf{h}_\sigma \in \mathbb{R}^d$ for simplex $\sigma$ and $\mathbf{h}_\tau \in \mathbb{R}^d$  for simplex $\tau$. For clarity, references to the $l$-th layer in equations will be assumed implicit and thus omitted.

\begin{figure}[t]
    \centering
    \includegraphics[width=.7\textwidth]{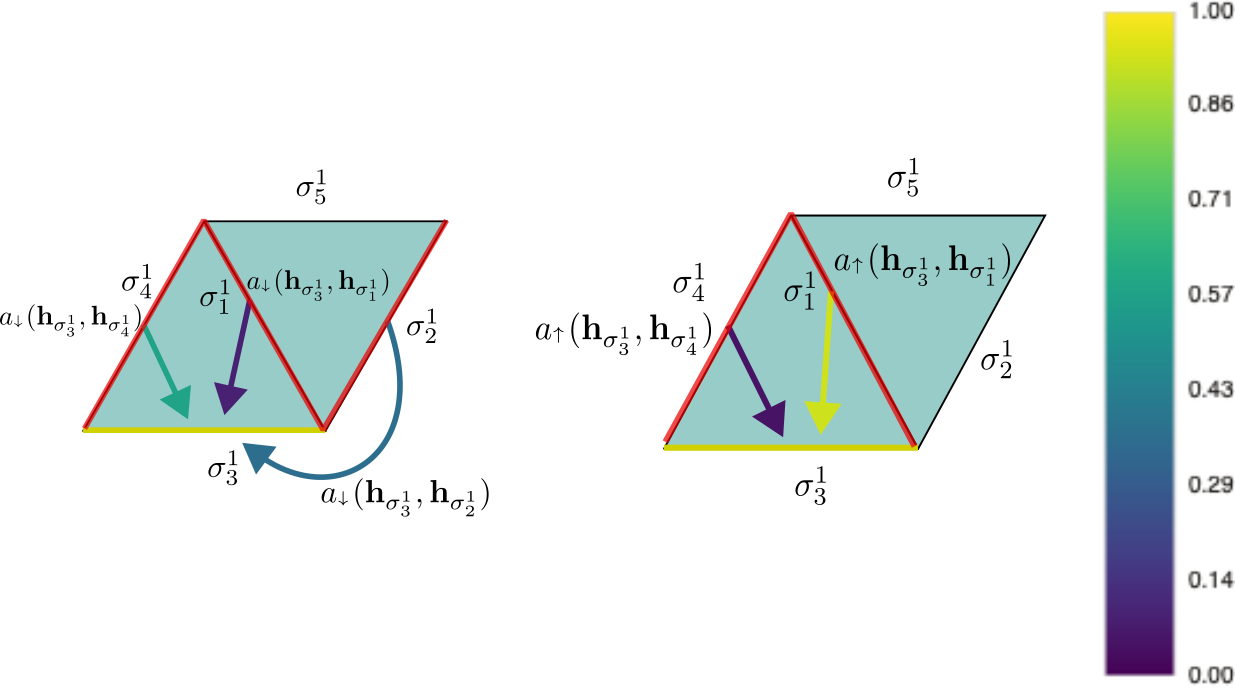}
     \caption{Illustration of the Simplicial Attention mechanism. The left panel illustrates the Lower Attention, it evaluates the reciprocal importance of two 1-simplices (edges) sharing a common 0-simplex (node). The right panel showcases the Upper Attention, emphasizing the significance of edges within the same triangle. In \textcolor{amber}{yellow} it is indicated the receiver while \textcolor{red}{red} is used for senders.}
     \label{fig:simp_att}
\end{figure}

Therefore, the importance of the latent representations within the {\em upper neighbourhood} of $\sigma$ is measured by the \textbf{upper scoring} function $s_{\uparr} : \mathbb{R}^d \times \mathbb{R}^d \rightarrow \mathbb{R}$ while for the features of {\em lower neighbouring} simplices this task is handled by  the \textbf{lower scoring} function $s_{\doarr} : \mathbb{R}^d \times \mathbb{R}^d \rightarrow \mathbb{R}$ functions. In particular, $s_{\uparr}$ and $s_{\doarr}$ are parametrised using two independent neural networks as:

\begin{align}
s_{\uparr} (\mathbf{h}_\sigma, \mathbf{h}_\tau) &= \text{LeakyReLU}\left( \mathbf{a}_{\uparr}^\top \big[  \mathbf{W}_{\uparr} \mathbf{h}_\sigma \parallel \mathbf{W}_{\uparr}  \mathbf{h}_\tau \big]  \right),\label{eq:upper_attention_scores_v1} \\[1.5pt]
s_{\doarr} (\mathbf{h}_\sigma, \mathbf{h}_\tau) &= \text{LeakyReLU}\left( \mathbf{a}_{\doarr}^\top \big[  \mathbf{W}_{\doarr} \mathbf{h}_\sigma \parallel \mathbf{W}_{\doarr}  \mathbf{h}_\tau \big]  \right),\label{eq:lower_attention_scores_v1}
\end{align}

where $\mathbf{W}_{\uparr}, \mathbf{W}_{\doarr} \in \mathbb{R}^{d' \times d}$  are {\em learnable} weight matrices\footnote{Imposing $\mathbf{W}_{\uparr} = \mathbf{W}_{\doarr}$ leads to the SAT architecture~\citep{goh2022simplicial}.} while $\mathbf{a}_{\uparr}, \mathbf{a}_{\uparr} \in \mathbb{R}^{2d'}$ are {\em learnable} vectors of attention coefficients. Here $\parallel$ denotes concatenation. It is worth emphasizing that involving two distinct set of parameters is a design choice made to separate the different topological properties contained in the upper and lower neighborhoods of a simplex $\sigma$.

\

Although the scoring functions defined in~\Cref{eq:upper_attention_scores_v1} and in~\Cref{eq:lower_attention_scores_v1} are those originally developed in graph attention networks~\citep{velivckovic2018graph},
they provide what is called a {\em static} attention mechanism. This fact might restrict the ranking of attention scores to be unconditioned on the query node, limiting its expressiveness. To increase the expressive power of the model, a dynamic masked self-attention can be employed by replacing the scoring functions~\citep{brody2021attentive}:

\begin{align}
s_{\uparr} (\mathbf{h}_\sigma, \mathbf{h}_\tau) &= \mathbf{a}_{\uparr}^\top \text{LeakyReLU}\left( \mathbf{W}_{\uparr}  [\mathbf{h}_\sigma \parallel \mathbf{h}_\tau]  \right), \label{eq:upper_attention_scores_v2} \\[1.5pt]
s_{\doarr} (\mathbf{h}_\sigma, \mathbf{h}_\tau) &= \mathbf{a}_{\doarr}^\top \text{LeakyReLU}\left( \mathbf{W}_{\doarr} [\mathbf{h}_\sigma \parallel \mathbf{h}_\tau]  \right),\label{eq:lower_attention_scores_v2}
\end{align}

where $\mathbf{W}_{\uparr}, \mathbf{W}_{\doarr} \in \mathbb{R}^{d' \times 2d}$ are learnable weight matrices respectively responsible for the upper and lower attention, $a_{\uparr}, a_{\doarr} \in \mathbb{R}^{d'}$ are learnable vectors of attention coefficients. 

\

Regardless of the particular choice of scoring functions, it is critical to ensure that the magnitude of the scores does not disproportionately affect the aggregation operation, thus preventing unstable or biased learning. The {\em standard approach to induce a more stable model} is to scale them to sum up to one via the softmax function across the neighbours:

\begin{align}
\alpha^{\uparr}_{\sigma, \tau} &= \underset{\tau \in \mathcal{N}_{\uparr}(\sigma)}{\text{softmax}}\,(s_{\uparr} (\mathbf{h}_\sigma, \mathbf{h}_\tau)), \label{eq:softmax_upper_attention} \\[1.5pt]
\alpha^{\doarr}_{\sigma, \tau} &= \underset{\tau \in \mathcal{N}_{\doarr}(\sigma)}{\text{softmax}}\,(s_{\doarr} (\mathbf{h}_\sigma, \mathbf{h}_\tau)). \label{eq:softmax_lower_attention}
\end{align}

This operation ensures that the {\em normalised attention coefficients} are comparable across different neighborhoods. Moreover, it provides a probabilistic interpretation of the scores to better understand how the model is allocating its attention across different parts of $\kph$.

Therefore, the normalised attention coefficients are used to compute a combination of the features corresponding to them, to obtain the final latent representations:

\begin{align}
    a_{\uparr} (\mathbf{h}_\sigma, \mathbf{h}_\tau) &= \alpha^{\uparr}_{\sigma, \tau} \mathbf{W}_{\uparr}, \label{eq:upper_simplicial_attention} \\[1.5pt]
    a_{\doarr} (\mathbf{h}_\sigma, \mathbf{h}_\tau) &= \alpha^{\doarr}_{\sigma, \tau} \mathbf{W}_{\doarr}, \label{eq:lower_simplicial_attention}  \\[3.5pt]
    \underset{\textbf{(a)}~\text{Upper Simplicial Attention}\vrule height 2.5ex width 0pt}{\mathbf{h}_{\uparr} = \underset{\tau \in \mathcal{N}_{\uparr}(\sigma)}{\agg} (\underbrace{a_{\uparr} (\mathbf{h}_\sigma, \mathbf{h}_\tau)}_{\text{upper attention}}} \, \mathbf{h}_\tau),& \quad \underset{\textbf{(b)}~\text{Lower Simplicial Attention}\vrule height 2.5ex width 0pt}{\mathbf{h}_{\doarr} = \underset{\tau \in \mathcal{N}_{\doarr}(\sigma)}{\agg} (\underbrace{a_{\doarr} (\mathbf{h}_\sigma, \mathbf{h}_\tau)}_{\text{lower attention}}} \, \mathbf{h}_\tau). \label{eq:simplicial_attention}
\end{align}

A pictorial overview of the simplicial attention mechanism is presented in~\Cref{fig:simp_att}.

It is unreasonable to think that a single attention head could be sufficient to capture the overall complexity of a phenomena of interest. To augment the expressive power of the simplicial attention operation and reduce instabilities, it is possible to compute $H$ distinct attention heads, which independently process the relationships within the upper and lower neighborhoods and aggregate the results through concatenation, sum, or mean.

\begin{align}\label{eq:multihead_simplicial_attention}
    \mathbf{h}_{\uparr} = \underset{\,\tau \in \mathcal{N}_{\uparr}(\sigma)}{\agg} (\underset{\,h}{\agg}({a^{(h)}_{\uparr} (\mathbf{h}_\sigma, \mathbf{h}_\tau)} \, \mathbf{h}_\tau)),  \\[1.5pt]
   \mathbf{h}_{\doarr} = \underset{\,\tau \in \mathcal{N}_{\doarr}(\sigma)}{\agg} (\underset{\,h}{\agg}({a^{(h)}_{\doarr} (\mathbf{h}_\sigma, \mathbf{h}_\tau)} \, \mathbf{h}_\tau)).  \\[1.5pt].
\end{align}

Notice that, if $\agg_{\, h}$ is implemented via concatenation, the output dimension is multiplied by a factor $H$, the number of attention heads.

\

\paragraph{Update and Readout}

Once upper and lower latent representations are obtained, they are combined together alongside with the current features to get the updated representation $\mathbf{h}_\sigma^{\textsf{new}}$.

\begin{equation}
    \mathbf{h}_\sigma^{\textsf{new}} = \com ( \mathbf{h}_\sigma, \mathbf{h}_{\uparr}, \mathbf{h}_{\doarr})\label{eq:combine_san}
\end{equation}

After $L$ layers of simplicial attention, the representation of the complex is computed as: 

\begin{equation}\label{eq:readout_san}
    \mathbf{h}_{\kph} = \out\Bigl( \{\{\{\mathbf{h}_\sigma^L\}\}\} \Bigr), 
\end{equation}

where $\{\{\mathbf{h}_\sigma^L\}\}$ is the multi-set of simplices's features at layer $L$ and $\out$ is a {\em readout function}. For each dimension of the complex, the representations of the simplices at dimension $k$ are computed by applying a max, mean, or sum readout operation, then the result is forwarded to a dense layer to obtain predictions.

\

In essence, the simplicial attention mechanism let messages to be sent from a simplex $\tau$ towards an adjacent simplex $\sigma$ and separately measures the relative importance of $\mathbf{h}_{\tau}$. Differently from previous graphs attention mechanisms~\citep{velivckovic2018graph, brody2021attentive}, simplicial complexes have an extended notion of adjacency. In particular, for two simplices $\sigma$ and $\tau$,  their relative connectivity in $\kph$ establishes if they are upper or lower neighbours. Notice also that $\sigma$ and $\tau$ might be both upper and lower neighbours without loss of generality. Consequently, simplicial neural networks equipped with the attention mechanism presented in this section are able to dynamically learn to attend neighbouring simplices according to the importance of their latent representations. Moreover, these architectures are able to address the relevance of the features based on both a {\em local context}~(via the lower scroes \Cref{eq:lower_attention_scores_v1}) and a {\em global context}~(via the upper scores \Cref{eq:upper_attention_scores_v1}).

%% file: Include/Chapters/4_Enhancing/4.1.2_cell_attention_networks.tex
\section{Cell Attention Networks}\label{sec:can}

The assumptions of Simplicial Attention Networks~(\Cref{sec:san}) require data as a simplicial complex $\kph$ with feature vectors $\xnorm_\sigma$ associated to its simplices. To drastically improve the flexibility of Simplicial Attention Networks, this section proposes Cell Attention Networks (CANs) to learn from graph data and perform topological representation learning tasks through {\em topological attention} on the messages exchanged by the edges of a cell complex. Cell Attention Networks are designed as a powerful learning tools that aim to extend Graph Attention Networks~\citep{velivckovic2018graph} by leveraging the connectivity induced by a cell complex $\cph$ to perform a \textbf{masked self-attention mechanisms over its edges}. Therefore, cell attention networks are designed with a {\em hierarchical} scheme. Since the aim of this method is to be able to process inputs as attributed graphs, a {\bf structural lift} embeds the input graphs into regular cell complexes. Then, since the masked self-attention will be defined on messages exchanged between edges,  a {\bf functional lift} operation is applied to node features for deriving edge features. After that, it performs the {\bf Cell Attention}, a message passing scheme able to attend neighbouring edges of the complex based on the features' importance. It is worth reminding that, as with the $1$-simplices in simplicial complexes, in a cell complex $\cph$, edges are equipped with two types of neighbourhoods: the upper and the lower, as discussed in~\Cref{sec:background:cell_complexes}. This implies that as the Simplicial Attention, the Cell Attention operation is composed by {\em two independent masked self-attention mechanism}, respectively responsible for the upper and the lower neighbourhoods of an edge $e$ in $\cph$.

\ 

To cope with scalability issues of message passing operations over cell complexes imposed by the huge amount of messages that flow within $\mathcal{N}_{\uparr}(e)$ and $\mathcal{N}_{\doarr}(e)$, after each layer of message passing a {\bf differentiable pooling} operation is applied to the edges. Moreover, by aggregating the features before pooling the complex, is possible to obtain collection of {\bf hierarchical representations} that describe the underling phenomena at {\bf different scales}. Finally, the sequence of representations is aggregated to obtain complex-wise predictions. 


\paragraph{Structural Lift}

To incorporate input graphs $\gph$ into regular cell complexes $\cph$, it is necessary to define an operation that attach two-dimensional disks as cells $\sigma$ to all the $R$-induced (or chordless) cycles of $\gph$ {\bf without compromising its original connectivity}. The parameter $R$ is referred as the maximum ring size of $\cph$ and can be considered a positive integer bounded by a small constant.

\begin{definition}[Structural Lifting Map~\citep{bodnar2021weisfeilercell}]\label{def:structural_lift} A structural lifting map $s: \gph \rightarrow \cph $ is a skeleton preserving function that incorporates a graph $\gph$ into a regular cell complex $\cph$, such that, for any graph $\gph$, the 1-skeleton (i.e., the underlying graph) of $\cph = s(\gph)$ and $\gph$ are isomorphic.
\end{definition}

\paragraph{Functional Lift}
In real-world applications, it is common to have data as attributed graphs without explicit edge features. To allow for message passing operations over the edges of $\cph$, after the structural lift it is necessary to populate edge features via a \emph{functional lift}. This operation assigns feature vectors $\mathbf{x}_{e} \in \mathbb{R}^{F_e}$ to each edge $e$ of $\cph$ by concatenating the features of the vertices $u, v \in \mathcal{B}(e)$ to be forwarded to an MLP with two dense layers.

\begin{definition}[Functional Lift]
    A functional lift is a \emph{learnable} function $f: \mathbb{R}^{F_n} \times \mathbb{R}^{F_n} \rightarrow \mathbb{R}^{F_e}$: 
\begin{equation}\label{eq:functional_lift}
    \mathbf{x}_{e} = f(\mathbf{x}_u,\mathbf{x}_v)= \sigma \big ( \Wnorm_1 \, [\mathbf{x}_u \parallel \mathbf{x}_v]\big ) \, \Wnorm_2 , \quad  u, v \in \mathcal{B}(e), \; \forall e \in \E,
\end{equation}
\end{definition}

where $\Wnorm_1 \in \mathbb{R}^{F_e \times 2 F_n}$, $\Wnorm_2 \in \mathbb{R}^{F_e \times F_e}$, and $\parallel$ denotes the concatenation operator. Since the order of the nodes connected by an edge does not alter the corresponding  edge features, $f$ {\em is invariant to node permutations}. It might happen that data comes naturally with edge features. In that case, they are concatenated to $\mathbf{x}_{e}$ and consider $F_e$ as the sum of the number of learned features and the provided ones. 

\

\paragraph{Cell Attention}\makeatletter\def\@currentlabel{Cell Attention}\makeatother\label{par:cell_attention}

\begin{figure}[!htb]
    \centering
    \includegraphics[width=\textwidth]{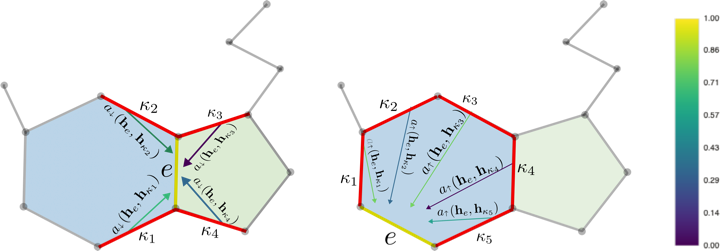}
     \caption{Illustration of the Cell Attention mechanism. The left panel illustrates the Lower Attention, it evaluates the reciprocal importance of two edges sharing a common node. The right panel showcases the Upper Attention, emphasizing the significance of edges within the same ring. In \textcolor{amber}{yellow} it is indicated the receiver while \textcolor{red}{red} is used for senders.}
     \label{fig:att}
\end{figure}


It is reasonable to think that the philosophy behind~\ref{par:simplicial_attention} can be naturally extended to cell complexes. In fact, this section addresses some critical considerations to straightforwardly adapt the principles of~\Cref{eq:simplicial_attention} for individually account the importance of edges' latent features when aggregating information coming from upper and lower neighbouroods. This operation takes the name of {\bf cell attention} and is exploits the connectivity of the edges (1-cells) within $\cph$ to design an efficient attention mechanism for topological message passing schemes over cell complexes.
As mentioned in~\Cref{sec:background:cell_complexes}, there are various types of adjacencies that can be taken into account when dealing with cell complexes. Here, for an edge $e$, only its upper and lower neighbourhoods are employed. This choice allows to capture long-range and higher-order relationships via the upper neighbourhood while the lower neighbourhood maintains local information. Moreover, this approach keeps the number of operations to be linear in the initial number of edges of the complex, which can be intended as a favorable trade-off between complexity and performance. 

\

It is extremely important to note that after each layer of message passing, a pooling operation reduces the size of the complex by aggregating edges. In particular, cell attention networks perform the topological message passing scheme on a sequence of cell complexes $\{\cph^{(l)}\}_{l=1}^{L}$ such that  $\cph^{(l+1)} \subseteq \cph^{(l)}$ since $\E^{(l+1)} \subseteq \E^{(l)}$ an the more deep the network is, the less elements the upper and lower neighbourhoods have. 
This reminder is provided because references to specific layers in the network's operation will be implicit and omitted in the equations to enhance clarity and reduce clutter.

\


At each layer, an \textbf{upper cell attention} $a_{\uparr}: \mathbb{R}^{F_e} \times \mathbb{R}^{F_e} \rightarrow \mathbb{R}$,  evaluates the reciprocal importance \textbf{lower cell attention} $a_{\doarr}: \mathbb{R}^{F_e} \times \mathbb{R}^{F_e}  \rightarrow \mathbb{R}$ measure the reciprocal importance of  of two edges that are part of the same ring and the importance of two edges's features that share a common node, respectively. Therefore, upper and lower embeddings are updated 
as:

\begin{equation}\label{eq:cell_attenton}
    \underset{\textbf{(a)}~\text{Upper Cell Attention}\vrule height 2.5ex width 0pt}{\mathbf{h}_{\uparr} = \underset{k \in \mathcal{N}_{\uparr}(e)}{\agg} \big ( a_\uparr(\mathbf{h}_{e},\mathbf{h}_{k})\, \mathbf{h}_{k} \big)}, \quad
        \underset{\textbf{(b)}~\text{Lower Cell Attention}\vrule height 2.5ex width 0pt}{\mathbf{h}_{\doarr} = \underset{k \in \mathcal{N}_{\doarr}(e)}{\agg} \big( a_\doarr(\mathbf{h}_{e},\mathbf{h}_{k}) \, \mathbf{h}_{k} \big)},
\end{equation}

where $\agg$ is a permutation invariant aggregation function (e.g., sum, mean, max), $\com$ is a learnable update function. Specifically, the upper~($a_\uparrow$,~\textcolor{darkred}{Equation}~\ref{eq:cell_attenton}\textcolor{darkred}{a}) and lower~($a_\downarrow$,~\textcolor{darkred}{Equation}~\ref{eq:cell_attenton}\textcolor{darkred}{b}) cell attention functions  can be implemented with the same spirit as~\ref{par:simplicial_attention}:
let $s_{\uparr} : \mathbb{R}^{F_e} \times \mathbb{R}^{F_e} \rightarrow \mathbb{R}$ be the \textbf{upper scoring} and  $s_{\doarr} : \mathbb{R}^{F_e} \times \mathbb{R}^{F_e} \rightarrow \mathbb{R}$ the \textbf{lower scoring}. In particular,  $s_{\uparr}$ and $s_{\doarr}$ are {\em responsible for learning the importance of edges' features} while computing the $\agg$ operation. Let $\mathbf{h}_e \in \mathbb{R}^{F_e}$ be a latent representation of edge $e$ and $\mathbf{h}_k \in \mathbb{R}^{F_e}$ the one for adjacent edge $k$. The scoring functions can be both implemented following~\cite{velivckovic2018graph} via {\bf cell attention}:

\begin{align}
s_{\uparr} (\mathbf{h}_{e}, \mathbf{h}_{k}) &= \text{LeakyReLU}\left( \mathbf{a}_{\uparr}^\top \big[  \mathbf{W}_{\uparr} \mathbf{h}_e \parallel \mathbf{W}_{\uparr}  \mathbf{h}_k \big]  \right),\label{eq:upper_cell_attention_scores_v1} \\[1.5pt]
s_{\doarr} (\mathbf{h}_{e}, \mathbf{h}_{k}) &= \text{LeakyReLU}\left( \mathbf{a}_{\doarr}^\top \big[  \mathbf{W}_{\doarr} \mathbf{h}_e \parallel \mathbf{W}_{\doarr}  \mathbf{h}_k \big]  \right),\label{eq:lower_cell_attention_scores_v1}
\end{align}

where $\mathbf{W}_{\doarr}, \mathbf{W}_{\uparr} \in \mathbf{F_e \times F_e}$ and $\mathbf{a}_{\doarr}, \mathbf{a}_{\uparr} \in \mathbb{R}^{2F_e}$. It is also possible to employ a {\bf dynamic cell attention} using scoring functions inspired by~\cite{brody2021attentive}: 

\begin{align}
s_{\uparr} (\mathbf{h}_e, \mathbf{h}_k) &= \mathbf{a}_{\uparr}^\top \text{LeakyReLU}\big( \mathbf{W}_{\uparr}  [\mathbf{h}_e \parallel \mathbf{h}_k]  \big),\label{eq:upper_cell_attention_scores_v2}  \\[1.5pt]
s_{\doarr} (\mathbf{h}_e, \mathbf{h}_k) &= \mathbf{a}_{\doarr}^\top \text{LeakyReLU}\big( \mathbf{W}_{\doarr} [\mathbf{h}_e \parallel \mathbf{h}_k]  \big),\label{eq:lower_cell_attention_scores_v2}
\end{align}

where $\mathbf{W}_{\doarr}, \mathbf{W}_{\uparr} \in \mathbb{R}^{F_e \times 2F_e}$ are learnable weight matrices and $\mathbf{a}_{\doarr}, \mathbf{a}_{\uparr} \in \mathbb{F_e}$ are two independent vectors of attention coefficients. As pointed out in~\cite{brody2021attentive}, by changing the order of the operations, the message passing scheme of dynamic cell attention is strictly more expressive than the one that involves~\Cref{eq:lower_attention_scores_v1} and~\Cref{eq:upper_cell_attention_scores_v1}. A pictorial example of the cell attention mechanism is provided in~\Cref{fig:att}.

Once the scores are obtained, to make them comparable across the neighbours, they are normalised using the softmax function:

\begin{align}
\alpha^{\uparr}_{e, k} &= \underset{e \in \mathcal{N}_{\uparr}(e)}{\text{softmax}}\,(s_{\uparr} (\mathbf{h}_e, \mathbf{h}_k)), \label{eq:softmax_upper_attention_can} \\[1.5pt]
\alpha^{\doarr}_{e, k} &= \underset{k \in \mathcal{N}_{\doarr}(e)}{\text{softmax}}\,(s_{\doarr} (\mathbf{h}_e, \mathbf{h}_k)). \label{eq:softmax_lower_attention_can}
\end{align}

Therefore, the upper and lower embeddings are computed as:

\begin{equation}\label{eq:cell_attenton_update}
   {\mathbf{h}_{\uparr} = \underset{k \in \mathcal{N}_{\uparr}(e)}{\agg} \big ( \alpha^{\uparr}_{e, k} \mathbf{W}_{\uparr}\, \mathbf{h}_{k} \big)}, \quad
        {\mathbf{h}_{\doarr} = \underset{k \in \mathcal{N}_{\doarr}(e)}{\agg} \big( \alpha^{\doarr}_{e, k} \mathbf{W}_{\doarr} \, \mathbf{h}_{k} \big)},
\end{equation}

As firstly proposed in~\cite{velivckovic2018graph}, multi-head attention can be employed to stabilize fluctuations within the self-attention mechanism. In particular, it consist in aggregating $H$ independent cell attentions~(\Cref{eq:cell_attenton}) using a concatenation, sum or averaging: 

\begin{align}\label{eq:multihead_cell_attention}
    \mathbf{h}_{\uparr} = \underset{\,k \in \mathcal{N}_{\uparr}(e)}{\agg} (\underset{\,h}{\agg}({a^{(h)}_{\uparr} (\mathbf{h}_e, \mathbf{h}_k)} \, \mathbf{h}_k)),  \\[1.5pt]
   \mathbf{h}_{\doarr} = \underset{\,k \in \mathcal{N}_{\doarr}(e)}{\agg} (\underset{\,h}{\agg}({a^{(h)}_{\doarr} (\mathbf{h}_e, \mathbf{h}_k)} \, \mathbf{h}_k)).  \\[1.5pt].
\end{align}

If a concatenation is used as aggregation function, the output dimension is multiplied by the number of attention heads involved. The latent representation of edge $e$ is therefore updated as:

\begin{equation}\label{eq:top_att:can:message_passing_scheme},
    \widetilde{\mathbf{h}}_{e} =  \com\big(\mathbf{h}_{e}, \, \mathbf{h}_{\uparr}, \mathbf{h}_{\doarr}\big)
\end{equation}

\

\paragraph{Edge Pooling}\makeatletter\def\@currentlabel{Edge Pooling}\makeatother\label{par:cell_attention_edge_pooling} 

\begin{figure}[!htb]
    \centering
    \includegraphics[width=\textwidth]{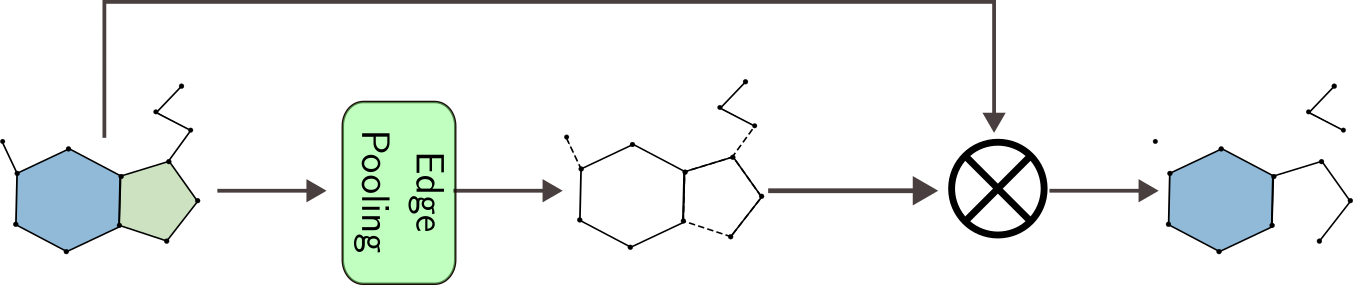}
    \caption{Visual representation of the edge pooling operation: At each layer, every edge of the complex is receives a score through a self-attention mechanism to determine its importance. Only the top-k scored edges are forwarded to the next layer. The structure of the complex is then adjusted: since the pooling affects the overall connectivity, a rewiring must be performed based on the topology of the edges removed.}
    \label{fig:pooling}
\end{figure}

To increase the scalability of the architecture, perform scale separation and learn a hierarchical representation of the complex, this section introduces a self-attention edge pooling technique. It extends the method used in~\cite{lee2019self} to compute a self-attention score $\gamma_{e} \in \mathbb{R}$ for each edge of the complex via a {\em learnable} function $a_{\wp}: \mathbb{R}^{F_e} \rightarrow \mathbb{R}$ : 
\begin{equation}\label{pooling_scores}
    \gamma_{e} = a_\wp \left(\widetilde{\mathbf{h}}_{e} \right).
\end{equation}

In particular, let $\rho \in (0, 1]$ be the \emph{pooling ratio}, that is the fraction of the edges that will be retained after the pooling layer. Moreover, let and $\aleph_e = \lceil \, \rho \cdot \, \vert \E \vert \,  \rceil$ the actual number of edges kept. Therefore, the edges that will be kept are the ones associated with the {\em top-k highest value} of the pooling scores.

At this point, the set of edges is updated as:  $\E^{\textsf{new}} = \{ e  : e \in \E \textrm{ and } \gamma_e \in \text{top-k}(\{\gamma_{e}\}, \aleph_e \}$, where $\text{top-k}(\cdot)$ is the set of the highest $\aleph_e$ self-attention scores. Finally, the latent representation of an edge $e$ kept after the pooling stage is scaled accordingly:
\begin{equation}\label{pooling_scaling}
    \mathbf{h}_{e}^{\textsf{new}} =  \gamma_{e}\widetilde{\mathbf{h}}_{e} , \;\; \forall  e  \in \E^{\textsf{new}}.
\end{equation}

The edge pooling stage {\bf alters the connectivity structure of $\cph$}. Thus, it has to be adjusted to obtain a consistent updated complex $\cph^{\textsf{new}}$. 
To this aim, the procedure depicted in~\Cref{fig:pooling} is applied: If an edge $e$ belongs to $\E$ but  is not contained in $ \E^{\textsf{new}}$, the lower connectivity is updated by disconnecting the nodes that are on the boundary of $e$, while the upper connectivity is updated by removing the rings that have $e$ on their boundaries. 

\

\paragraph{Readout} As~\cite{cangea2018towards}, a hierarchical version of the aforementioned attentional edge pooling operation is considered. To this aim, an intra-layer $\agg$ operation is applied on the latent features  $\mathbf{h}_e^{\textsf{new}}$ to obtain an embedding of the whole complex $\cph^{\textsf{new}}$ as:
\begin{equation}\label{local_readout}
    \mathbf{h}_{\cph^{\textsf{new}}} = \underset{e \in \E^{\textsf{new}}}{\agg} \big ( \mathbf{h}_{e}^{\textsf{new}} \big).
\end{equation} 

\

\begin{figure}[t]
    \centering
    \includegraphics[width=\textwidth]{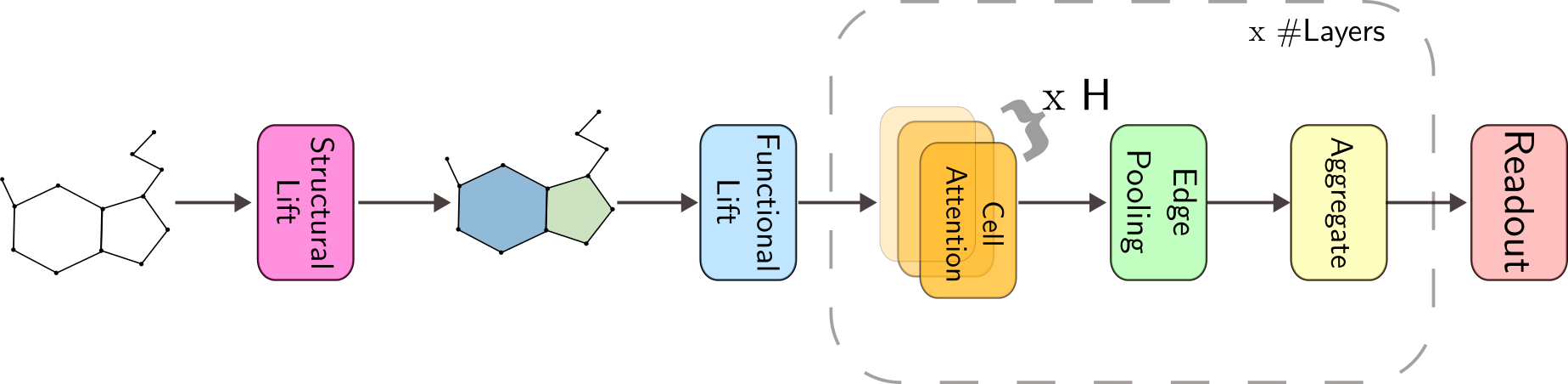}
    \caption{Schematic overview of the Cell Attention Network (CAN) architecture. The process begins with a structural lifting map, transforming a graph $\gph$ into a cell complex $\cph$. Following this, edge features are derived from node features through a functional lift. The core of the network consists of $m$ cell attention layers, each performing a message-passing operation, edge pooling stage, followed by an aggregation. The architecture finally combines the hierarchical features to obtain complex-wise prediction via readout.}
    \label{fig:tan}
\end{figure}

By integrating this operation to all the layers,it results in a sequence $\{\mathbf{h}_{\cph^{(l)}}\}$ of complex-wise hierarchical representation. After the last hidden layer, a final (global) readout operation is performed by aggregating all the previously computed complexes embeddings: 
\begin{equation}\label{global_readout}
\mathbf{h}_{\cph} = \underset{l}{\agg} \big( \mathbf{h}_{\cph^{(l)}} \big).
\end{equation}

Finally, $\mathbf{h}{\cph}$ is fed to a multi-layer perceptron (MLP) to obtain complex-wise predictions. The complete overview of the cell attention network architecture is pictored in~\Cref{fig:tan}.

%% file: Include/Chapters/4_Enhancing/4.2_topological_mp.tex
\section{Enhanced Topological Message Passing}

Graph Neural Networks excel at learning from graph-structured data but face limitations in handling long-range interactions and modeling higher-order structures. Cellular Isomorphism Networks (CINs) address these challenges through a message-passing scheme on a cell complex topology.

Despite their advantages, CINs make use only of boundary and upper messages which do not consider a direct interaction between the rings present in the underlying complex. Accounting for these interactions is critical for accurately learning representations of complex real-world phenomena such as the dynamics of supramolecular assemblies, neural activity within the brain, and gene regulation processes presented in~\Cref{sec:intro:tnns_4_science}. In this section, a powerful topological message passing scheme that accounts for ring interactions is introduced. This enhanced scheme overcomes these limitations by enabling cells within each layer to receive lower messages. By providing a more comprehensive representation of higher-order and long-range interactions, CIN++ achieves state-of-the-art results on large-scale and long-range chemistry benchmarks.

\begin{figure}[!htb]
    \centering
    \includegraphics[scale= 0.69]{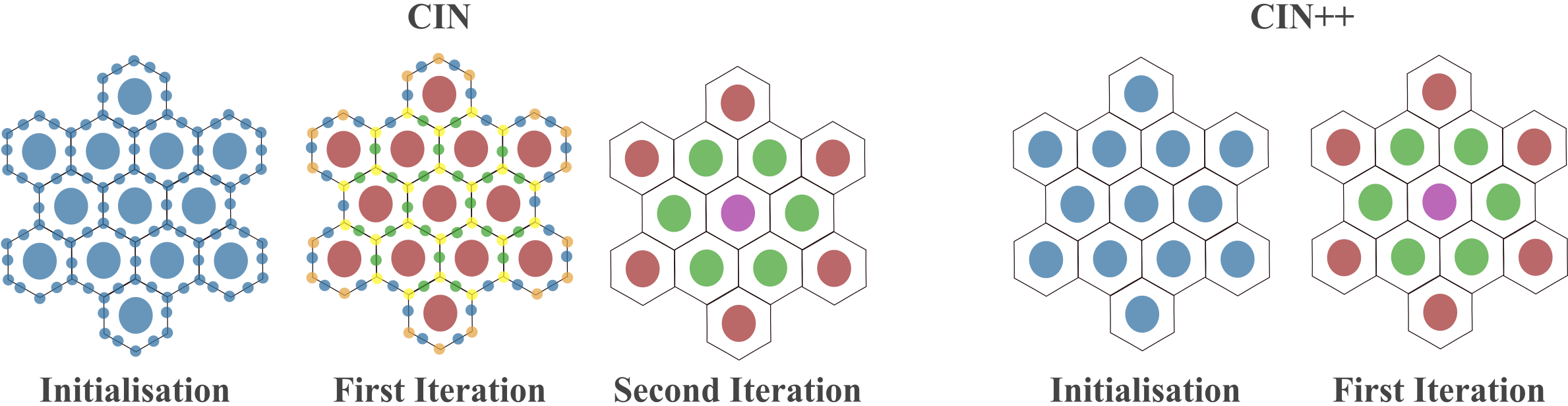}
    \caption{In molecular graphs featuring regions with a high concentration of rings, incorporating lower messages into cellular isomorphism networks expedites the convergence of the 2-cell colors.}
    \label{fig:ring coloring}
\end{figure}

\

\paragraph{Contribution} This section introduces a new message-passing scheme for cell complexes, leveraging the benefits of complex topological spaces. Motivated by the fact that cell complexes provide a natural framework to represent higher-dimensional structures and topological features that are inherent in the realm of chemistry, throughout this section, the focus is set on this domain. In particular, CIN++ includes messages that flow within the lower neighbourhood of the underlying cell complex. These messages are exchanged between edges that share a common vertex and between rings that are glued through an edge to better capture group interactions and to avoid potential bottlenecks. Experimental results, detailed later in~\Cref{sec:exp_cinpp}, demonstrate that CIN++ offers a deeper understanding of chemical systems compared to other models, showcasing top-tier performance on benchmarks, including ZINC and Peptides,. The ability of CIN++ model to understand higher-dimensional structures and topological features could have an immediate and significant impact in the areas of computational chemistry and drug discovery.

\paragraph{On the convergence speed of Cellular Isomorphism Networks} 

Cellular Isomorphism Networks (CINs) model higher-order signals through a proven, powerful hierarchical message-passing scheme in cell complexes. Examining CIN's coloring procedure reveals that edges initially receive messages from the upper neighborhood, and only in the subsequent iteration do they refine the ring colors~(\Cref{fig:ring coloring} (left)). Although this coloring refinement procedure holds the same expressive power (\cite{bodnar2021weisfeilercell}, Thm. 7), it is possible to achieve {\em faster convergence} by including messages from the cells' lower neighborhood. This allows for a direct interaction between the rings of the complex which removes the bottleneck caused by edges waiting for upper messages before updating ring colours~(\Cref{fig:ring coloring} (right)).

\subsubsection{Enhancing Topological Message Passing}\label{sec:etmp}

This section describes the operations involved in the enhanced topological message-passing scheme that regulates CIN++. In particular, the enhancement consists of the inclusion of lower messages in Cellular Isomorphism Networks~(CIN, \cite{bodnar2021weisfeilercell}). As will be shown later in this section, including lower messages will let the information flow within a broader neighbourhood of the complex via the messages exchanged between the rings that are lower adjacent and escaping potential bottlenecks~\citep{alon2020bottleneck} via messages between lower adjacent edges. 

\subsubsection{Boundary Messages}

\begin{figure}[!htb]
     \centering
     \begin{subfigure}[b]{0.45\textwidth}
         \centering
         \includegraphics[width=.5\textwidth]{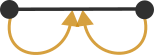}
         \caption{Boundary messages from nodes to the edge that joins them.}
         \label{fig:boundary_messages_edges}
     \end{subfigure}
     \hspace{5pt}
     \begin{subfigure}[b]{0.45\textwidth}
         \centering
         \includegraphics[width=.5\textwidth]{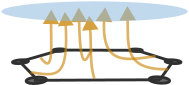}
         \caption{Boundary messages directed from edges to an inner ring.}
         \label{fig:boundary_messages_rings}
     \end{subfigure}
    \caption{Boundary message flow within a 2-dimensional cell complex: \textbf{(a)} from node pairs to their connecting edge and \textbf{(b)} from surrounding edges to enclosed rings.}
    \label{fig:bm}
\end{figure}

A cell $\sigma$ that is either an edge or a ring, receives messages from its boundary elements denoted by $\tau \in \mathcal{B}(\sigma)$. Thus, the feature vector $\mathbf{h}_{\mathcal{B}}$ is obtained through a permutation invariant aggregation that takes as input all the {\em boundary messages} $\mathsf{m}_{\mathcal{B}}$ between the feature vector $\mathbf{h}_{\sigma}$ and all the feature vectors of its boundary elements, $ \mathbf{h}_{\tau}$ as in~\Cref{fig:bm}. To reduce clutter and improve clarity, the particular cell $\sigma$ which receives the messages is left implicit.

\begin{equation}\label{eq:boundary_msg}
\mathbf{h}_{\mathcal{B}} = \underset{\tau \in \mathcal{B}(\sigma)}{\mathsf{agg}} \big(\mathsf{m}_{\mathcal{B}}\big( \mathbf{h}_{\sigma}, \mathbf{h}_{\tau}\big)\big).
\end{equation}

This operation is responsible for lifting the information from lower cells to higher-order ones, enabling bottom-up communication across the cells of the complex. Leveraging the theory developed in~\cite{xu2019powerful} for graphs and later on in~\cite{bodnar2021weisfeilercell} for regular cell complexes, to maximize the representational power of the underlying network, the boundary message function is implemented as:

\begin{equation}\label{eq:bm_impl}
    \mathbf{h}_{\mathcal{B}} =  \mlp_{\mathcal{B}}\big( (1+\epsilon_{\mathcal{B}}) \,  \mathbf{h}_{\sigma} + \sum_{\tau \in \mathcal{B}(\sigma)} \mathbf{h}_{\tau} \big), \nonumber
\end{equation}

where $\mlp_{\mathcal{B}}$ has 2 fully-connected layers. Considering that 0-cells (vertices)  do not have boundary elements, 1-cells (edges) have only two boundary elements and the maximum ring size of $\cph$ is bounded by a small constant, the number of boundary messages scales with $\mathcal{O}(\vert \cph \vert)$.  The number of parameters involved in this operation is $\mathcal{O}(d^2)$, provided by the outer Multi-Layer Perceptron (MLP). In this work, no parameter sharing is employed across the dimensions of the complex (i.e., a distinct MLP is used for each layer of the network and for each dimension of the complex). 

\subsubsection{Upper Messages}

\begin{figure}[!htb]
     \centering
     \begin{subfigure}[b]{0.45\textwidth}
         \centering
         \includegraphics[width=.5\textwidth]{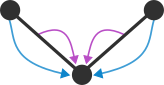}
         \caption{Upper messages flowing amongst nodes within a 2D cell complex. Co-boundary information from edges in common is also included.}
         \label{fig:upper_messages_nodes}
     \end{subfigure}
     \hspace{5pt}
     \begin{subfigure}[b]{0.45\textwidth}
         \centering
         \includegraphics[width=.5\textwidth]{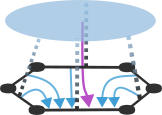}
         \caption{Visualization of upper messages shared between edges that form a cohesive ring in a 2D cell complex. Rings' messages are also included as information from the co-boundary neighbourhood.}
         \label{fig:upper_messages_edges}
     \end{subfigure}
    \caption{Schematic representation of upper message exchanges within a two-dimensional cell complex: \textbf{(a)} between nodes (i.e., the canonical message passing scheme), and  \textbf{(b)} between edges that bound a ring. The process also integrates messages from co-boundary adjacent cells.}
    \label{fig:um}
\end{figure}

These are the messages that each cell $\sigma$ receives from its upper neighbouring cells $\tau \in \mathcal{\mathcal{N}_{\uparr}}(\sigma)$ (i.e., the blue arrows in~\Cref{fig:um}) and from common co-boundary cells $\delta \in \mathcal{C}o(\sigma, \tau)$ (i.e., the purple arrows in~\Cref{fig:um}). The information coming from the upper neighbourhood of $\sigma$ and the common co-boundary elements is denoted as $\mathbf{h}_{\uparr}$. It obtained via a permutation invariant aggregation that takes as input all the {\em upper messages} $\mathsf{m}_{\uparr}$ between the feature vector $\mathbf{h}_\sigma$, all the feature vectors in its upper neighbourhood $\mathbf{h}_\tau$ and all the cells in the common co-boundary neighbourhood, $\mathbf{h}_\delta$. Formally:

\begin{equation}\label{eq:upp_msg}
\mathbf{h}_{\uparr} = \underset{\underset{{\delta \in \mathcal{C}o(\sigma, \tau)}}{\tau \in \mathcal{N}_{\uparr}(\sigma)}}{\mathsf{agg}}\big(\mathsf{m}_{\uparr}\big( \mathbf{h}_{\sigma}, \mathbf{h}_{\tau} ,  \mathbf{h}_{\delta}\big)\big)
\end{equation}

This operation will let the information flow within a {\em narrow} neighbourhood of $\sigma$, ensuring consistency and coherence with respect to the underlying topology of the complex. The function $m_{\uparr}$ is therefore implemented as:

\begin{equation}\label{eq:um_impl}              
    \mathbf{h}_{\uparr} = \mlp_{\uparr}\big( (1+\varepsilon_{\uparr}) \, \mathbf{h}_{\sigma} + \sum_{\stackrel{\tau \in \mathcal{N}_{\uparr}(\sigma)}{{\delta \in \mathcal{C}o(\sigma, \tau)}}} \mlp_{\mathsf{m}_{\uparr}} \big( \mathbf{h}_{\tau} \parallel \mathbf{h}_{\delta} \big) \big), \nonumber
\end{equation}

In this context, $\mlp_{\mathsf{m}_{\uparr}}$ denotes a single-layer fully-connected network, complemented by a point-wise non-linearity, while $\mlp_{\uparr}$ is implemented as a two-layer dense layer. The amount of upper messages that a cell $\tau \in \mathcal{B}(\sigma)$ exchanges with its adjacent cells is given by $2 \cdot \binom{\left| \mathcal{B}(\sigma) \right|}{2}$. Considering the assumption that the boundary of the cells is bounded by a fixed constant, the total number of messages correlates linearly with the magnitude of the complex, that is, the number of cells in $\cph$. The total number of learnable parameters is also on the order of $\mathcal{O}(d^2)$, a consequence of the two MLPs utilized in the message function.

\subsubsection{Lower Messages}

\begin{figure}[!htb]
     \centering
     \begin{subfigure}[b]{0.45\textwidth}
         \centering
         \includegraphics[width=.5\textwidth]{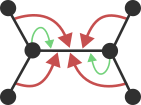}
         \caption{Edges exchanging lower messages based on shared nodes.}
         \label{fig:lower_messages_edges}
     \end{subfigure}
     \hspace{5pt}
     \begin{subfigure}[b]{0.45\textwidth}
         \centering
         \includegraphics[width=.7\textwidth]{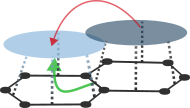}
         \caption{Rings communicating via lower messages through common bounding edges.}         \label{fig:lower_messages_rings}
     \end{subfigure}
    \caption{Visualization of lower message exchange in a 2D cell complex. \textbf{(a)} messages traverse edge pairs through shared nodes, and \textbf{(b)} between rings via shared boundary edges.}
    \label{fig:lm}
\end{figure}

These are the messages that each cell $\sigma$ receives from its lower neighbouring cells $\tau \in \mathcal{\mathcal{N}_{\doarr}}(\sigma)$ (i.e., the red arrows in~\Cref{fig:lm}) and from common boundary cells $\delta \in \mathcal{B}(\sigma, \tau)$ (i.e., the green arrows in~\Cref{fig:lm}). A function that aggregates the information coming from the upper neighbourhood of $\sigma$ and the common co-boundary elements is denoted as $m_{\doarr}$. It consists in a permutation invariant aggregation that takes as input all the {\em lower messages} $m_{\doarr}$ between the feature vector $\mathbf{h}_\sigma$, all the feature vectors in its lower neighbourhood $\mathbf{h}_\tau$ and all the cells in the common boundary neighbourhood, $\mathbf{h}_\delta$. Formally:

\begin{equation}\label{eq:lower_msg}
\mathbf{h}_{\doarr}  = \underset{\stackrel{\tau \in \mathcal{N}_\doarr(\sigma)}{{\delta \in \mathcal{B}(\sigma, \tau)}}}{\agg} \big(\mathsf{m}_{\doarr}\big( \mathbf{h}_{\sigma}, \mathbf{h}_{\tau},  \mathbf{h}_{\delta} \big)\big)
\end{equation}

As pictorially shown in~\Cref{fig:lm} (a), this operation would help a {\em broader} diffusion of the information between edges that are not necessarily part of a ring. Also, it will let the rings of the complex communicate directly~(\Cref{fig:lm} (b)). Similarly to the upper messages, $\mathsf{m}_{\doarr}$ is implemented via:

\begin{equation}\label{eq:lm_impl}              
    \mathbf{h}_{\doarr} = \mlp_{\doarr}\big( (1+\varepsilon_{\doarr}) \, \mathbf{h}_{\sigma} + \sum_{\stackrel{\tau \in \mathcal{N}_{\doarr}(\sigma)}{{\delta \in \mathcal{B}(\sigma, \tau)}}} \mlp_{\mathsf{m}_{\doarr}} \big( \mathbf{h}_{\tau} \parallel \mathbf{h}_{\delta} \big) \big), \nonumber
\end{equation}

As for the upper messages, $\mlp_{\mathsf{m}_{\doarr}}$ denotes a single-layer fully-connected network, succeeded by a point-wise non-linearity, while $\mlp_{\doarr}$ represents an MLP with two-layers fully connected. The amount of lower messages that a cell $\tau \in \mathcal{C}o(\sigma)$ exchanges with its neighbours is given by $2 \cdot \binom{\left| \mathcal{C}o(\sigma) \right|}{2}$. Since the assumptions include that the cells have a number of co-boundary neighbours that is bounded by a fixed constant, the total number of messages scales linearly with the number of cells in the complex. The two MLPs involved in the message function induces an amount of learnable parameters on the order of $\mathcal{O}(d^2)$.

\subsubsection{Update and Readout}
Update and Readout operations are performed as:

\begin{equation}\label{eq:update}
    \mathbf{h}_{\sigma}^{\textsf{new}} = \com\Bigl( \mathbf{h}_\sigma,
    \mathbf{h}_{\mathcal{B}}, \mathbf{h}_{\uparr}, \mathbf{h}_{\doarr}, \Bigr).
\end{equation}

The update function $\com$ is implemented using a single fully connected layer followed by a point-wise non-linearity that uses a different set of parameters for each layer of the model and for each dimension of the complex. Notice how the update function receives additional information provided by the messages that a cell $\sigma$ receives from its lower neighbourhood.
After $L$ layers, the representation of the complex is computed as: 

\begin{equation}\label{eq:readout}
    \mathbf{h}_{\cph} = \out\Bigl( \{\{\{\mathbf{h}_\sigma^L\}\}\}_{dim(\sigma)=0}^{2} \Bigr), 
\end{equation}

where $\{\{\mathbf{h}_\sigma^L\}\}$ is the multi-set of cell's features at layer $L$. In practice, the representation of the complex is computed in two stages: first, for each dimension of the complex, the representation of the cells at dimension $k$ is computed by applying a mean or sum readout operation. This results in one representation for the vertices $\mathbf{h}_\V$, one for the edges $\mathbf{h}_\E$ and one for the rings $\mathbf{h}_\mathsf{R}$. Then, a representation for the complex $\cph$ is computed as: $\mathbf{h}_{\cph} = \mlp_{\out,\V} \big( \mathbf{h}_\V \big) + \mlp_{\out,\E} \big( \mathbf{h}_\E \big) + \mlp_{\out,\mathsf{R}} \big( \mathbf{h}_\mathsf{R} \big)$, where each $\mlp_{\out,\cdot}$ is implemented as a single fully-connected layer followed by a non-linearity. Finally, $\mathbf{h}_{\cph}$ is forwarded to a final dense layer to obtain the predictions.

A neural architecture that updates the cell's representation using the message passing scheme defined in~\Cref{eq:update} and obtains complex-wise representations as in~\Cref{eq:readout} takes the name of {\em Enhanced Cell Isomorphism Network} (CIN++). The expressive power of CIN++ can be directly derived from the expressiveness results reported in~\cite{bodnar2021weisfeilercell}.

\begin{theorem}
\label{thm:cin++express}
Let $\mathcal{F} : \cph \rightarrow \mathbb{R}^d$ be a CIN++ network. With a sufficient number of layers and injective neighbourhood aggregators $\mathcal{F}$ is able to map any pair of
complexes $(\cph_1, \cph_2)$ in an embedding space that the Cellular Weisfeiler-Lehman (CWL) test is able to tell if $\cph_1$ and  $\cph_2$ are non-isomorphic.
\end{theorem}

%% file: Include/Chapters/5_Experiments/5_experiments.tex
\chapter{Experimental Analysis}\label{chap:exp}

\input{Include/Chapters/5_Experiments/5.1.1_experiments_on_oversq}

\input{Include/Chapters/5_Experiments/5.2.1_exp_san}

\input{Include/Chapters/5_Experiments/5.2.2_exp_can}
\input{Include/Chapters/5_Experiments/5.3.1_exp_cinpp}

%% file: Include/Chapters/5_Experiments/5.1.1_experiments_on_oversq.tex
\section{Experiments On Oversquashing}

The goal in the three graph transfer tasks - $\mathsf{Ring}$, $\mathsf{CrossedRing}$, and $\mathsf{CliquePath}$ - is for the $\MPNN$ to `transfer' the features contained at the target node to the source node. $\mathsf{Ring}$ graphs are cycles of size $n$, in which the target and source nodes are placed at a distance of $\lfloor n / 2 \rfloor$ from each other. $\mathsf{CrossedRing}$ graphs are also cycles of size $n$, but include "crosses" between the auxiliary nodes. Importantly, the added edges do not reduce the minimum distance between the source and target nodes, which remains $\lfloor n / 2 \rfloor$. $\mathsf{CliquePath}$ graphs contain a  $\lfloor n / 2  \rfloor$-clique and a path of length $\lfloor n / 2  \rfloor$. The source node is placed on the clique and the target node is placed at the end of the path. The clique and path are connected in such a way that the distance between the source and target nodes is $\lfloor n / 2  \rfloor + 1$, in other words the source node requires one hop to gain access to the path.

\begin{figure}[!htbp]
    \centering
    \includegraphics[width=\textwidth]{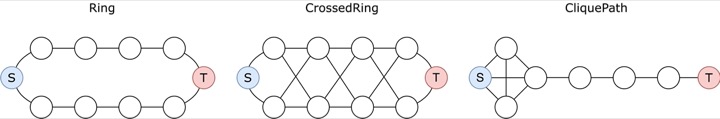}
    \caption{Topological structure of $\mathsf{RingTransfer}$, $\mathsf{CrossedRingTransfer}$, and $\mathsf{CliquePath}$. The nodes marked with an $S$ are the source nodes, while the nodes with a $T$ are the target nodes. All tasks are shown for a distance between the source and target nodes of $r=5$.}
    \label{fig:graph-transfer-example}
\end{figure}

\Cref{fig:graph-transfer-example} shows examples of the graphs contained in the $\mathsf{Ring}$, $\mathsf{CrossedRing}$, and $\mathsf{CliquePath}$ tasks, for when the distance between the source and target nodes is $r=5$. In the experiments the input dimension is fixed to $p=5$ and the target node is assigned a randomly one-hot encoded feature vector; for this reason, the random guessing baseline obtains $20\%$ accuracy. The source node is assigned a vector of all $0s$ and the auxiliary nodes are instead assigned vectors of $1s$. Following~\cite{bodnar2021weisfeilercell}, $5000$ graphs are generated for the training set and $500$ graphs for the test set for each task. In the experiments, are reported the mean accuracy over the test set. The train lasted for $100$ epochs, with depth of the $\MPNN$ equal to the distance between the source and target nodes $r$. Unless specified otherwise, the hidden dimension is fixed to $64$. During training and testing, a mask is applied over all nodes to focus only on the source node to compute losses and accuracy scores.

\subsection{Validating the impact of width}

This section validates empirically the message from~\Cref{cor:bound_MLP_MPNN}: if the task presents long-range dependencies, increasing the hidden dimension mitigates over-squashing and therefore has a positive impact on the performance. Consider the following 'graph transfer' task, building upon~\cite{bodnar2021weisfeilercell}: given a graph, consider source and target nodes placed at a distance $r$ from each other. Assign a one-hot encoded label to the target and a constant unitary feature vector to all other nodes. The goal is to assign to the source node the feature vector of the target. Partly due to over-squashing, performance is expected to degrade as $r$ increases. 

To validate that this holds irrespective of the graph structure, this is tested across three graph topologies, called  $\mathsf{CrossedRing}$, $\mathsf{Ring}$ and $\mathsf{CliquePath}$. While the topology is also expected to affect the performance (as confirmed in~\Cref{sec:depth}), given a fixed topology, it is expected that the model would benefit from an increase of hidden dimension.
 
To verify this behaviour,  the $\mathsf{GCN}$~\citep{kipf2017graph} architecture is employed on the three graph transfer tasks increasing the hidden dimension, but keeping the number of layers equal to the distance between source and target, as shown in~\Cref{fig:graph-transfer-hidden-dim}. The results verify the intuition from the theorem that a higher hidden dimension helps the $\mathsf{GCN}$ model solve the task across larger distances 
across the three graph-topologies.

\begin{figure}[t]
    \centering
    \includegraphics[width=0.8\textwidth]{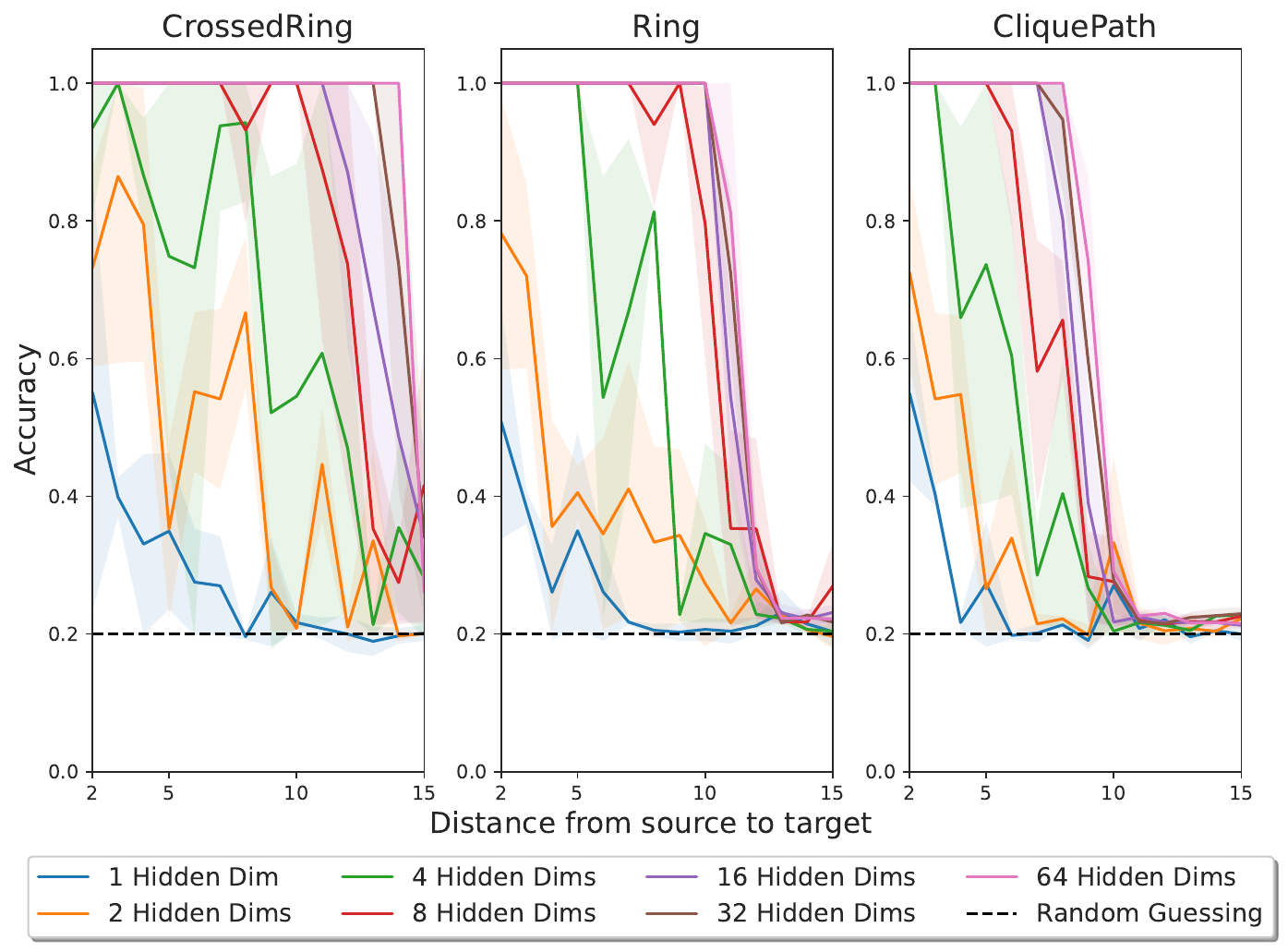}
    \caption{Performance of $\mathsf{GCN}$ on the $\mathsf{CrossedRing}$, $\mathsf{Ring}$, and $\mathsf{CliquePath}$ tasks obtained by varying the hidden dimension. Increasing the hidden dimension helps mitigate the over-squashing effect, in accordance with~\Cref{cor:bound_MLP_MPNN}.}
    \label{fig:graph-transfer-hidden-dim}
\end{figure}



\subsection{Validating the impact of depth}
The evidence in~\Cref{cor:over-squasing_distance}, provides a strong indication of difficulty of a task by calculating an upper bound on the Jacobian.  Consider the same graph transfer tasks introduced above, namely $\mathsf{CrossedRing}$, $\mathsf{Ring}$, and $\mathsf{CliquePath}$. For these special cases, consider a refined version of the r.h.s in~\Cref{eq:cor_distance}: in particular, $k = 0$ (i.e. the depth coincides with the distance among source and target) and the term $\gamma_{r}(v,u)(d_{\mathrm{min}})^{-r}$ can be replaced by the exact quantity $(\oper^{r})_{vu}$. 
Fixing a distance $r$ between source $u$ and target $v$ then, for example the $\mathsf{GCN}$-case has $\oper = \Anorm$ so that the term $(\oper^{r})_{vu}$ can be computed explicitly: 
\begin{alignat*}{2}
&(\oper^{r})_{vu} = (3/2)^{-(r-1)} \qquad &&\text{ for } \mathsf{CrossedRing}\\
&(\oper^{r})_{vu} =2^{-(r-1)} \quad &&\text{ for }  \mathsf{Ring}\\ 
& (\oper^{r})_{vu} = 2^{-(r-2)}/(r\sqrt{r-2}) &&\text{ for } 
 \mathsf{CliquePath}.
\end{alignat*}
Given 
an $\MPNN$, terms like $c_{\up},w,p$ entering~\Cref{cor:over-squasing_distance} are independent of the graph-topology and hence can be assumed to behave, roughly, the same across different graphs.  
As a consequence, over-squashing is likely to be more problematic for $\mathsf{CliquePath}$, followed by $\mathsf{Ring}$, and less prevalent comparatively in $\mathsf{CrossedRing}$.
~\Cref{fig:graph-transfer-main} shows the behaviour of $\mathsf{GIN}$~\citep{xu2019powerful}, $\mathsf{SAGE}$~\citep{hamilton2017inductive}, $\mathsf{GCN}$~\citep{kipf2017graph}, and $\mathsf{GAT}$~\citep{velivckovic2018graph} on the aformentioned tasks.  $\mathsf{CliquePath}$ is the consistently hardest task, followed by $\mathsf{Ring}$, and $\mathsf{CrossedRing}$. Furthermore, the decline in performance to the level of random guessing for the {\em same} architecture across different graph topologies highlights that this drop cannot be simply labelled as `vanishing gradients' since for certain topologies the same model can, in fact, achieve perfect accuracy. This validates that the underlying topology has a strong impact on the distance at which over-squashing is expected to happen. Moreover, this confirms that in the regime where the depth $m$ is comparable to the distance $r$, over-squashing will occur if $r$ is large enough. 


\begin{figure}[t]
    \centering
    \includegraphics[width=0.7\textwidth]{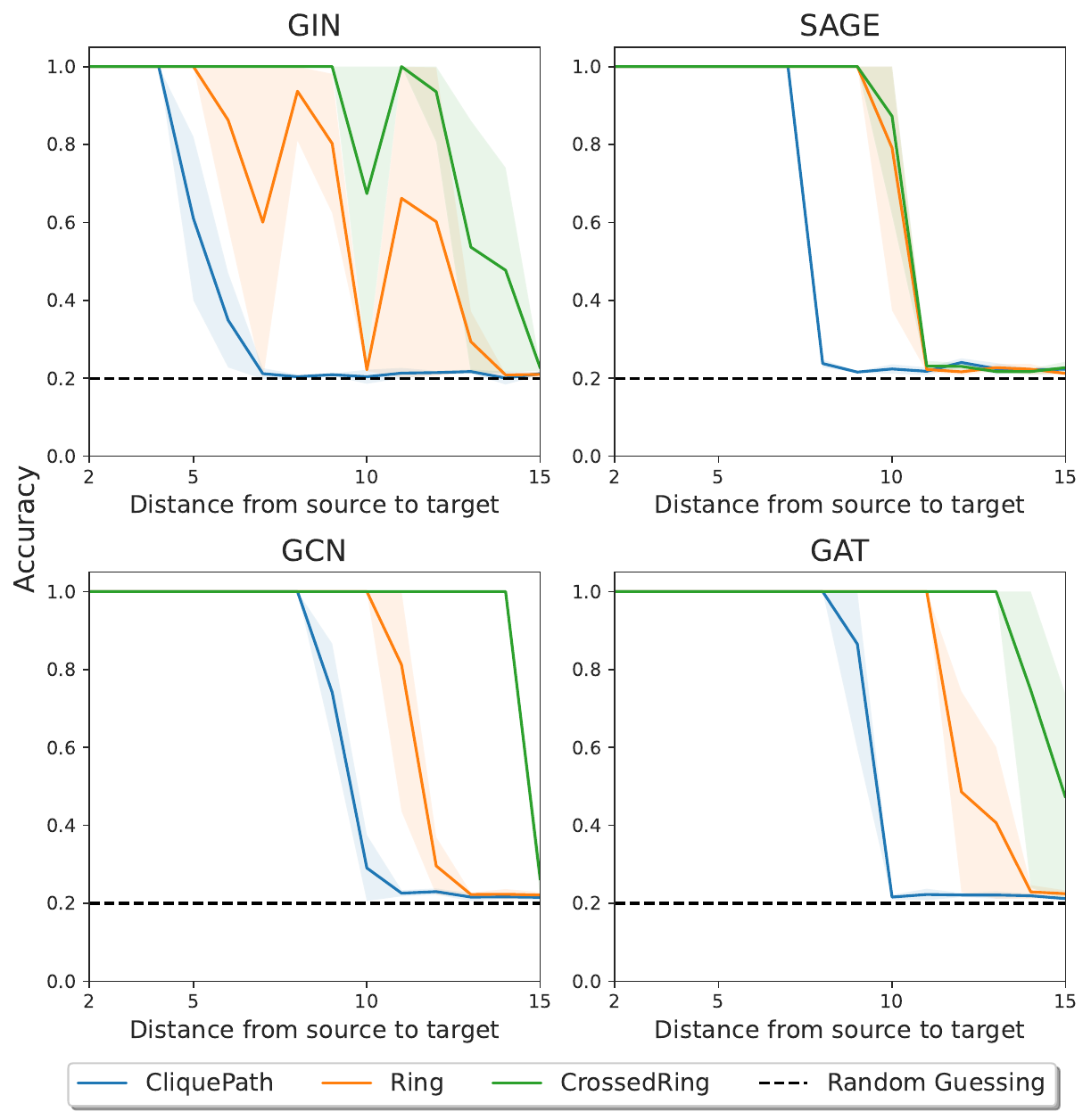}
    \caption{Performance of $\mathsf{GIN}$, $\mathsf{SAGE}$, $\mathsf{GCN}$, and $\mathsf{GAT}$ on the $\mathsf{CliquePath}$, $\mathsf{Ring}$, and $\mathsf{CrossedRing}$ tasks. In the case where depth and distance are comparable, over-squashing highly depends on the topology of the graph as the distance increases.}
    \label{fig:graph-transfer-main}
\end{figure}

\paragraph{Insights and observations.}  Finally, note that the results in~\Cref{fig:graph-transfer-main} also validate the theoretical findings of~\Cref{thm:effective_resistance}. If $v,u$ represent target and source nodes on the different graph-transfer topologies, then $\res(v,u)$ is highest for $\mathsf{CliquePath}$ and lowest for the $\mathsf{CrossedRing}$. Once again, the distance is only a partial information. Effective resistance provides a better picture for the impact of topology to over-squashing and hence the accuracy on the task; in~\Cref{app:signal-prop} the framework is further validate that via a synthetic experiment where the propagation of a signal in a $\MPNN$ is affected by the effective resistance of $\gph$.


\subsection{Validating the impact of topology}
\label{app:signal-prop}

\begin{figure}[t]
    \centering
    \includegraphics[width=.8\textwidth]{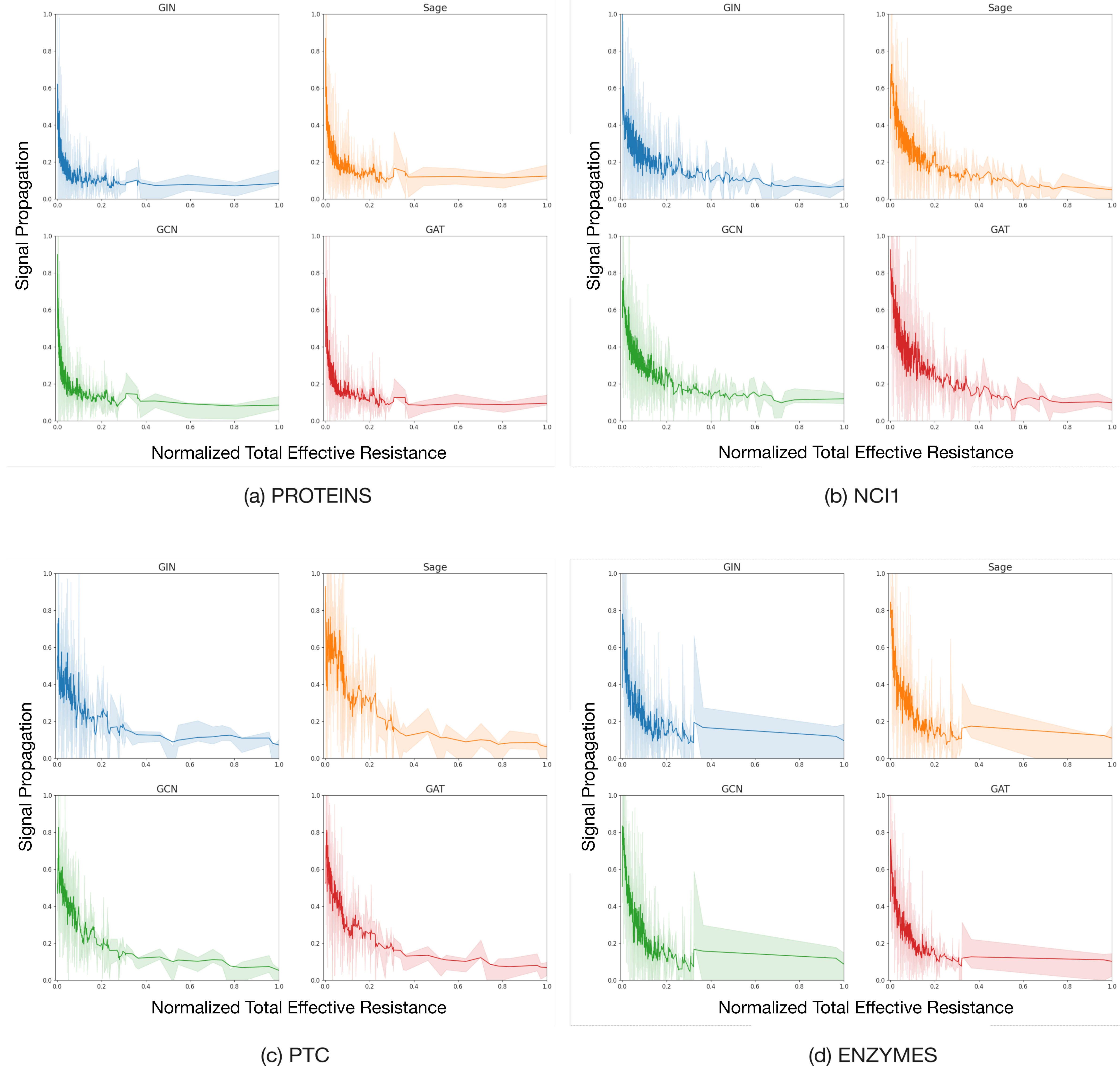}
    \caption{Decay of the amount of information propagated through the graphs w.r.t. the normalized total effective resistance (commute time) for: (a) $\mathsf{PROTEINS}$; (b) $\mathsf{NCI1}$; (c) $\mathsf{PTC}$; (d) $\mathsf{ENZYMES}$. For each dataset it is reported the decay for: (i) $\mathsf{GIN}$ (top-left); (ii) $\mathsf{SAGE}$ (top-right), (iii) $\mathsf{GCN}$  (bottom-left) and (iv)
    $\mathsf{GAT}$(bottom-right).}
    \label{fig:signal-diffusion-main}
\end{figure}

In this section, there are extensive synthetic experiments on the $\mathsf{PROTEINS}$, $\mathsf{NCI1}$, $\mathsf{PTC}$, $\mathsf{ENZYMES}$ datasets with the aim to provide empirical evidence to the fact that the total effective resistance of a graph, $\res_{\gph} = \sum_{v,u} \res{(v,u)}$~\citep{ellens2011effective}, is related to the ease of information propagation in an $\MPNN$. The experiment is designed as follows: first fix a source node $v\in\V$ assigning it a $p$-dimensional unitary feature vector, and assigning the rest of the nodes zero-vectors. Then 
consider the quantity 
\begin{equation*}
h^{(m)}_\odot = \frac{1}{p \max_{u\neq v} d_\gph (v,u)}  \sum_{f=1}^{p}\sum_{u \neq v}\frac{h_u^{(m), f}}{\| h_u^{(m),f} \|} d_\gph(v,u),
\end{equation*}

to be the amount of signal (or `information') that has been propagated through $\gph$ by an $\MPNN$ with $m$ layers. Then,the (normalized) propagation distance over $\gph$ is  measured by averaging it over all the $p$ output channels. Propagation distance refers to the average distance to which the initial 'unit mass' has been propagated to - a larger propagation distance means that on average the unit mass has travelled further w.r.t. to the source node. The goal is to show that $h^{(m)}_\odot$ is {\em inversely proportional to $\res_{\gph}$}. In other words, graphs with \emph{lower} total effective resistance should have a \emph{larger} propagation distance. The experiment is repeated for each graph $\gph$ that belongs to the dataset $\mathcal{D}$.  The process starts by randomly choosing the source node $v$, then set $\mathbf{h}_v$ to be an arbitrary feature vector with unitary mass (i.e. $\| \mathbf{h}_v \|_{L_1} = 1$) and assigning the zero-vector to all other nodes (i.e. $\mathbf{h}_u = \boldsymbol{0}, \; u \neq v$). The framework assumes $\MPNN$s with
a number of layers $m$ close to the average diameter of the graphs in the dataset, input and hidden dimensions $p=5$ and ReLU activations. In particular, the resistance of $\gph$ is estimated by sampling $10$ nodes with uniform probability for each graph and report $h^{(m)}_\odot$ accordingly. ~\Cref{fig:signal-diffusion-main} shows that {\em $\MPNN$s are able to propagate information further when the effective resistance is low}, validating empirically the impact of the graph topology on over-squashing phenomena. It is worth to emphasize that in this experiment, the parameters of the $\MPNN$ are randomly initialized and without an underlying training task. This implies that this setup isolates the problem of propagating the signal throughout the graph, separating it from vanishing gradient phenomenon.

%% file: Include/Chapters/5_Experiments/5.2.1_exp_san.tex
\section{Experiments Simplicial Attention Networks}

In this section, the performance of simplicial attention networks is assessed on two different tasks: trajectory prediction~\citep{schaub2020random} (inductive learning), and missing data imputation in citation complexes~\citep{ebli2020simplicial, yang2022simplicial} (transductive learning)\footnote{SAN implementation \& datasets are available at \url{https://github.com/lrnzgiusti/Simplicial-Attention-Networks}}. A summary of the datasets and the tasks is presented in~\Cref{tab:summary}.

\subsection{Benchmarks and Datasets}

\input{Tables/san_summary}

\subsubsection{Trajectory Prediction}
Trajectory prediction tasks are used to address many problems in location-based services, e.g., route recommendation~\citep{zheng2014modeling}, or inferring the missing portions of a given trajectory~\citep{wu2016probabilistic}. 
Inspired by~\cite{schaub2020random}, the studies in~\cite{roddenberry2021principled, bodnar2021weisfeiler, goh2022simplicial} utilize simplicial neural networks to tackle trajectory prediction. In the sequel, the same experimental setup of~\cite{bodnar2021weisfeiler} is employed for a fair comparison.\smallskip\\

\begin{figure}[!htb]
    \centering
    \includegraphics[scale=.3]{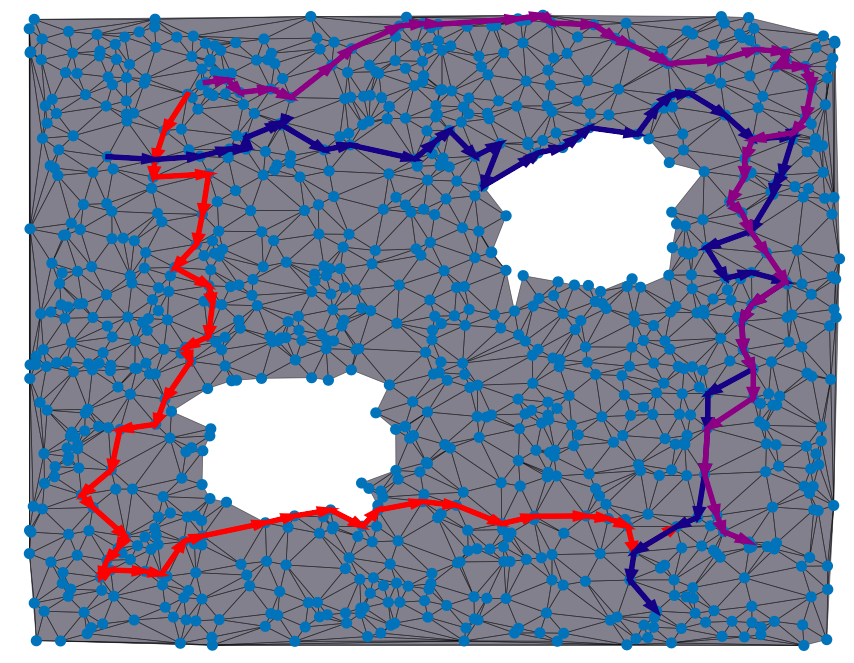}
    \caption{Illustration of the synthetic flow dataset. Points are uniformly sampled within a unit square and connected using a Delaunay triangulation to form the domain. Trajectories start from the top-left and progress to the bottom-right, closely approaching one of two distinct holes. The learning goal is to discern which hole a given trajectory is closest to.}
    \label{fig:synth}
\end{figure}

\paragraph{Synthetic Flow} The architecture is firstly tested on the synthetic flow dataset from~\cite{bodnar2021weisfeiler}. The simplicial complex is generated by sampling $400$ points uniformly at random in the unit square, and then a Delaunay triangulation is applied to obtain the domain of the trajectories. The set of trajectories is generated on the simplicial complex shown in~\Cref{fig:synth}: Each trajectory starts from the top left corner and goes through the entire map until the bottom right corner, passing close to either the bottom-left hole or the top-right hole. Thus, the learning task is to identify which of the two holes is the closest one on the path. The dataset has $1000$ training examples and $200$ test examples. \smallskip\\

\begin{figure}[!htb]
    \centering
    \includegraphics[scale=.3]{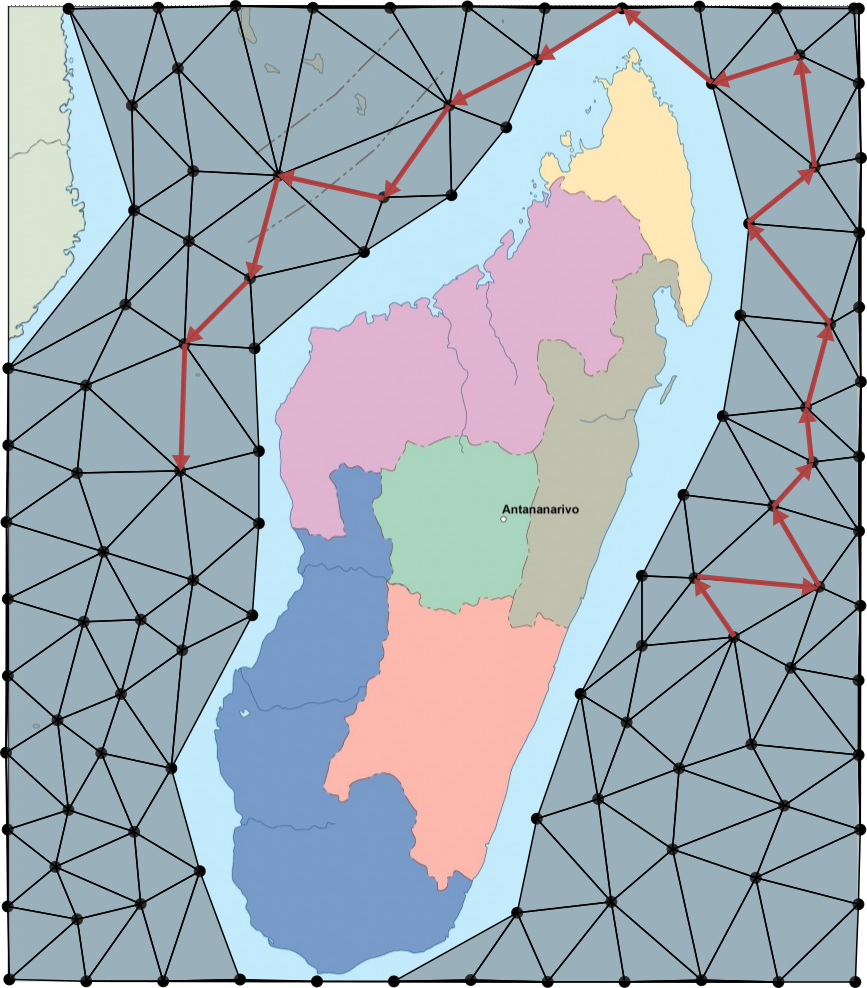}
    \caption{Discretized map of ocean drifter tracks near Madagascar, represented as a simplicial complex with a central and top-left islands. The learning objective is to distinguish between clockwise and counter-clockwise flow motions around the island.}
    \label{fig:ocean}
\end{figure}

\paragraph{Ocean Drifters} Another dataset examined involves real-world ocean drifter tracks near Madagascar from 2011 to 2018~\citep{schaub2020random}. The map surface is discretized into a simplicial complex with a hole in the centre, which represents the presence of the island. The discretization process is done by tiling the map into a regular hexagonal grid. Each hexagon represents a $0$-simplex (vertex), and if there is a nonzero net flow from one hexagon to its surrounding neighbors, a $1$-simplex (edge) is placed between them. All the 3-cliques of the $1$-simplex are considered to be $2$-simplex (triangles) of the simplicial complex shown in~\Cref{fig:ocean}. Thus, following the experimental setup of~\cite{bodnar2021weisfeiler}, the learning task is to distinguish between the clockwise and counter-clockwise motions of flows around the island. The dataset is composed of $160$ training trajectories and $40$ test trajectories. The flows belonging to each trajectory of the test set use random orientations.

\

\input{Tables/trajectory_accuracies}

Both experiments are inductive learning problems. In particular, it is employed a single layer simplicial attention network with a single attention head, $4$ output features, and upper and lower filter lengths $K^{\doarr} = K^{\uparr} = 3$. To perform the classification task, an MLP is used as a readout layer with softmax non-linearity.  The network is trained via ADAM optimizer~\citep{kingma2014adam} and cross-entropy loss, with initial learning rate set to $0.01$, a step reduction of $0.77$, and a patience of $10$ epochs. To avoid overfitting, an $l_2$ regularization with $\lambda_{l_2} = 0.003$ is used and Dropout~\citep{srivastava2014dropout} with  probability equal to $p_{\text{drop}} = 0.6$. In~\Cref{tab:trajectory} there is a comparison between the accuracy of the simplicial attention network averaged over $5$ different seeds. For each seed, the network is trained with an early stopping criteria with a patience of $100$ epochs. The architecture is therefore compared alongside with MPSN~\citep{bodnar2021weisfeiler}, SCN~\citep{yang2022simplicial}, and SAT~\cite{goh2022simplicial}. For the MPSN architecture, the metrics are the ones reported in~\cite{bodnar2021weisfeiler}. As shown in~\Cref{tab:trajectory}, simplicial attention networks achieves the best results among the state of the art models in both the synthetic and real-world scenarios. In particular, for the synthetic example, SAN architecture achieves 100\% of accuracy independently on the used non-linearity.

\paragraph{Citation Complex Imputation}

\input{Tables/cc_imputation}

Missing data imputation is a learning task that consists of estimating missing values in a dataset. 
GNN can be used to tackle this task as in~\cite{spinelli2020missing}, but recently the works~\cite{ebli2020simplicial, yang2022simplicial} have handled the missing data imputation problem using simplicial complexes. Here it is used the same experimental settings of~\cite{ebli2020simplicial}. The task consist in estimating the number of citation of a collaboration between $k+1$ authors over a co-authorship complex. This is as a transductive learning setup, where the labels of the $k$-simplex are the number of citation of the $k+1$ authors. To address this task, a simplicial attention network with $4$ layers, $256$ hidden features for the first three layers, and a filter length over upper and lower neighborhoods $K^{\doarr} = K^{\uparr} = 2$. The final layer computes a single output feature that will be used as estimate of the $k$-simplex' labels. ReLU non-linearities are placed as activation after each layer.  To train the network, a Xavier initialization~\citep{glorot2010understanding} is used, sampling from a uniform distribution with a gain of $\sqrt{2}$, ADAM optimizer~\citep{kingma2014adam} with $0.1$ as initial learning rate equipped with a step reduction on plateaus with a patience of $100$ epochs, and masked $\ell_1$ loss with an early stopping criteria with patience of $500$ epochs. Accuracy is computed by considering a citation value correct if its estimate is within $\pm5\%$ of the true value. In~\Cref{tab:imputation_cc}, it is reported the mean performance and the standard deviation of simplicial attention network averaged over $10$ different masks for missing data. The results are compared with SNN~\citep{ebli2020simplicial}, SCNN~\citep{yang2022simplicial}, and SAT~\cite{goh2022simplicial} for different simplex orders and percentages of missing data. Both SAT and the proposed simplicial attenton network exploit single-head attention. To fairly evaluate the benefits of the attention mechanism, the proposed method is compare with with SCNN~\citep{yang2022simplicial} (denoted as "SCNN (ours)" ) using the same experimental setup. From~\Cref{tab:imputation_cc}, is possible to notice that simplicial attention networks achieve the best performance for each order and percentage of missing data, with huge gains as the order and the percentage grow, illustrating the importance of incorporating self-attention mechanisms in simplicial neural networks.

%% file: Tables/san_summary.tex
\begin{table}[t]
\normalsize
\caption{Summary of datasets and tasks of our experiments.}
\centering
\label{tab:summary}
\begin{tabular}{|cccc|}
\hline

\begin{tabular}[l]{@{}c@{}}Info   \end{tabular}           & 
\begin{tabular}[c]{@{}c@{}}Synthetic Flow \end{tabular}                                       & 
\begin{tabular}[c]{@{}c@{}}Ocean Drifters \end{tabular}                                       & 
\multicolumn{1}{c|}{\begin{tabular}[c]{@{}c@{}}Citation Complex \end{tabular}}            \\ \hline

\begin{tabular}[c]{@{}l@{}} \textbf{Type of task}  \\ \textbf{\#Nodes} \\ \textbf{\#Edges} \\ \textbf{\#Triangles}  \\ \textbf{\#Classes} \\ \textbf{\#Training Nodes} \\ \textbf{\#Test Nodes}    \end{tabular} & 
\begin{tabular}[c]{@{}c@{}}Inductive \\ 186 \\ 527 \\ 340 \\ 2 \\ 1000 \\  200 \end{tabular} &
\begin{tabular}[c]{@{}c@{}}Inductive \\ 133 \\ 320 \\ 186 \\ 2 \\ 160 \\ 40 \end{tabular} & 
\multicolumn{1}{c|}{\begin{tabular}[c]{@{}c@{}}Trasductive \\ 352 \\ 1474 \\ 3285 \\ - \\ - \\ - \end{tabular}}\\ \hline
\end{tabular}
\end{table}

%% file: Tables/trajectory_accuracies.tex
\begin{table}[t]
\normalsize
\caption{Trajectory classification test accuracy.}
\centering
\label{tab:trajectory}
\begin{tabular}{llcc}
\hline

\begin{tabular}[c]{@{}c@{}}Model\end{tabular} & 
\begin{tabular}[c]{@{}c@{}}Activation   \end{tabular}           & 
\begin{tabular}[c]{@{}c@{}}Synthetic Flow (\%)\end{tabular}                                       & 
\multicolumn{1}{c|}{\begin{tabular}[c]{@{}c@{}}Ocean Drifters (\%)\end{tabular}}            \\ \hline

MPSN~\citep{bodnar2021weisfeiler}                                                 &
\begin{tabular}[c]{@{}c@{}}Id\\ ReLU \\ Tanh \end{tabular} & 
\begin{tabular}[c]{@{}c@{}}82.6 $\pm$ 3.0\\ 50.0  $\pm$ 0.0\\ 95.2 $\pm$ 1.8 \end{tabular} & 
\multicolumn{1}{c}{\begin{tabular}[c]{@{}c@{}}73.0 $\pm$ 2.7\\ 46.5 $\pm$ 5.7\\ 72.5 $\pm$ 0.0\end{tabular}}\\ \hline

SCNN~\citep{yang2022simplicial}                                              & 
\begin{tabular}[c]{@{}c@{}}Id \\ ReLU\\ Tanh\end{tabular} & 
\begin{tabular}[c]{@{}c@{}}66.5 $\pm$ 0.16\\ 100 $\pm$ 0.0\\ 67.2  $\pm$ 0.16 \end{tabular} & 
\begin{tabular}[c]{@{}l@{}}98.1 $\pm$ 0.01\\ 97.0 $\pm$ 0.01\\ 97.0  $\pm$ 0.16  \end{tabular}           \\ \hline

SAT~\citep{goh2022simplicial}                                            & 
\begin{tabular}[c]{@{}c@{}}Id\\ ReLU \\ Tanh \end{tabular} & 
\begin{tabular}[c]{@{}c@{}}99.7 $\pm$ 0.0\\ 100 $\pm$ 0.0\\ 100 $\pm$ 0.0\end{tabular} & 
\begin{tabular}[c]{@{}l@{}}97.0 $\pm$ 0.01\\ 95.0 $\pm$ 0.00\\ 95.0 $\pm$ 0.01 \end{tabular}                      \\ \hline


SAN                                                & 
\begin{tabular}[c]{@{}c@{}}Id \\ ReLU\\ Tanh \end{tabular} & 
\begin{tabular}[c]{@{}c@{}}\second{100 $\pm$ 0.0}\\ \third{100 $\pm$ 0.0}\\ \first{100 $\pm$ 0.0}\end{tabular} & 
\begin{tabular}[c]{@{}l@{}}\first{99.0 $\pm$ 0.01}\\ \second{98.5 $\pm$ 0.01}\\ \third{98.5 $\pm$ 0.01}\end{tabular}             \\ \hline
\end{tabular}
\end{table}

%% file: Tables/cc_imputation.tex
\begin{table}[t]
\normalsize
\caption{Missing Data Imputation test accuracy}
\centering
\label{tab:imputation_cc}

\begin{adjustbox}{max width=\linewidth}
\begin{tabular}{clcccccc}
\hline

\begin{tabular}[c]{@{}c@{}}\%Miss/Order\\ $N_k$\end{tabular} & 
\begin{tabular}[c]{@{}c@{}}Method   \end{tabular}           & 
\begin{tabular}[c]{@{}c@{}}0\\ 352\end{tabular}                                       & 
\begin{tabular}[c]{@{}c@{}}1\\ 1474\end{tabular}                                      & 
\begin{tabular}[c]{@{}c@{}}2\\ 3285\end{tabular}                                      & 
\begin{tabular}[c]{@{}c@{}}3\\ 5019\end{tabular}                                      & 
\begin{tabular}[c]{@{}c@{}}4\\ 5559\end{tabular}                                      & 
\multicolumn{1}{c}{\begin{tabular}[c]{@{}c@{}}5\\ 4547\end{tabular}}            \\ \hline

10\%                                                 &
\begin{tabular}[c]{@{}c@{}}SNN \; \citep{ebli2020simplicial}\\ SCNN\;\citep{yang2022simplicial}\\SCNN (ours)\\SAT\;\citep{goh2022simplicial} \\SAN\end{tabular} & 
\begin{tabular}[c]{@{}c@{}}\second{91 $\pm$ 0.3}\\ \third{91  $\pm$ 0.4}\\ 90 $\pm$ 0.3\\ 18 $\pm$ 0.0 \\ \first{91 $\pm$ 0.4} \end{tabular} & 
\begin{tabular}[c]{@{}c@{}}\third{91 $\pm$ 0.2}\\ \second{91  $\pm$ 0.2}\\ 91 $\pm$ 0.3\\ 31 $\pm$ 0.0 \\ \first{95 $\pm$ 1.9}\end{tabular} & 
\begin{tabular}[c]{@{}c@{}}\second{91 $\pm$ 0.2}\\ \third{91  $\pm$ 0.2}\\ 91 $\pm$ 0.3\\ 28 $\pm$ 0.1 \\ \first{95 $\pm$ 1.9}\end{tabular} & 
\begin{tabular}[c]{@{}c@{}}91 $\pm$ 0.2\\ \third{91  $\pm$ 0.2}\\ \second{93 $\pm$ 0.2}\\ 34 $\pm$ 0.1 \\ \first{97 $\pm$ 1.6}\end{tabular} & 
\begin{tabular}[c]{@{}c@{}}\third{91 $\pm$ 0.2}\\ 91  $\pm$ 0.2\\ \second{92 $\pm$ 0.2}\\ 53 $\pm$ 0.1 \\ \first{98 $\pm$ 0.9}\end{tabular} & 
\multicolumn{1}{c}{\begin{tabular}[c]{@{}c@{}}90 $\pm$ 0.4\\ \third{91 $\pm$ 0.2}\\ \second{94 $\pm$ 0.1}\\ 55 $\pm$ 0.1 \\ \first{98 $\pm$ 0.7} \end{tabular}}\\ \hline

20\%                                                 & 
\begin{tabular}[c]{@{}c@{}}SNN \citep{ebli2020simplicial}\\ SCNN \citep{yang2022simplicial}\\SCNN (ours)\\SAT\;\citep{goh2022simplicial} \\ SAN\end{tabular} & 
\begin{tabular}[c]{@{}c@{}}\second{81 $\pm$ 0.6}\\ 81 $\pm$ 0.7\\ \third{81 $\pm$ 0.6}\\ 18 $\pm$ 0.0   \\ \first{82 $\pm$ 0.8} \end{tabular} & 
\begin{tabular}[c]{@{}c@{}}\third{82 $\pm$ 0.3}\\ 82 $\pm$ 0.3\\ \second{83 $\pm$ 0.7}\\ 30 $\pm$ 0.0   \\ \first{91 $\pm$ 2.4} \end{tabular} & 
\begin{tabular}[c]{@{}c@{}}\second{81 $\pm$ 0.6}\\ 81 $\pm$ 0.7\\ \third{81 $\pm$ 0.6}\\ 29 $\pm$ 0.1   \\ \first{82 $\pm$ 0.8} \end{tabular} & 
\begin{tabular}[c]{@{}c@{}}\third{82 $\pm$ 0.3}\\ 82 $\pm$ 0.3\\ \second{88 $\pm$ 0.4}\\ 35 $\pm$ 0.1   \\ \first{96 $\pm$ 0.4} \end{tabular} & 
\begin{tabular}[c]{@{}c@{}}\third{81 $\pm$ 0.6}\\ 81 $\pm$ 0.7\\ \second{86 $\pm$ 0.7} \\ 50 $\pm$ 0.1  \\ \first{96 $\pm$ 1.3} \end{tabular} & 
\begin{tabular}[c]{@{}l@{}}82 $\pm$ 0.5\\ \third{83 $\pm$ 0.3}\\ \second{89 $\pm$ 0.6} \\ 58 $\pm$ 0.1  \\ \first{97 $\pm$ 0.9} \end{tabular}                      \\ \hline

30\%                                                 & 
\begin{tabular}[c]{@{}c@{}}SNN \citep{ebli2020simplicial}\\ SCNN \citep{yang2022simplicial}\\SCNN (ours)\\SAT\;\citep{goh2022simplicial} \\ SAN\end{tabular} & 
\begin{tabular}[c]{@{}c@{}}\third{72 $\pm$ 0.6}\\ \second{72 $\pm$ 0.5}\\ 72  $\pm$ 0.6 \\ 19 $\pm$ 0.0 \\ \first{75 $\pm$ 2.1}\end{tabular} & 
\begin{tabular}[c]{@{}c@{}}\third{73 $\pm$ 0.4}\\ 73 $\pm$ 0.4\\ \second{76  $\pm$ 0.6} \\ 33 $\pm$ 0.1 \\ \first{89 $\pm$ 2.1}\end{tabular} & 
\begin{tabular}[c]{@{}c@{}}\second{81 $\pm$ 0.6}\\ 81 $\pm$ 0.7\\ \third{81 $\pm$ 0.6} \\ 25 $\pm$ 0.1 \\ \first{82 $\pm$ 0.8} \end{tabular} & 
\begin{tabular}[c]{@{}c@{}}\second{82 $\pm$ 0.3}\\ \third{82 $\pm$ 0.3}\\ 82 $\pm$ 1.2 \\ 33 $\pm$ 0.0 \\ \first{94 $\pm$ 0.4} \end{tabular} & 
\begin{tabular}[c]{@{}c@{}}\second{81 $\pm$ 0.6}\\ \third{81 $\pm$ 0.7}\\ 80 $\pm$ 0.7 \\ 47 $\pm$ 0.1 \\ \first{95 $\pm$ 0.5} \end{tabular} & 
\begin{tabular}[c]{@{}l@{}}73 $\pm$ 0.5\\ \third{74 $\pm$ 0.3}\\ \second{86  $\pm$ 0.8}  \\ 53 $\pm$ 0.1 \\ \first{96 $\pm$ 0.5}\end{tabular}           \\ \hline

40\%                                                 & 
\begin{tabular}[c]{@{}c@{}}SNN \citep{ebli2020simplicial}\\ SCNN \citep{yang2022simplicial}\\SCNN (ours)\\SAT\;\citep{goh2022simplicial} \\ SAN\end{tabular} & 
\begin{tabular}[c]{@{}c@{}}\third{63 $\pm$ 0.7}\\ \second{63 $\pm$ 0.6}\\ 63 $\pm$ 0.7 \\ 20 $\pm$ 0.0 \\ \first{67 $\pm$ 1.9}\end{tabular} & 
\begin{tabular}[c]{@{}c@{}}\third{64 $\pm$ 0.3}\\ 64 $\pm$ 0.3\\ \second{67 $\pm$ 1.1} \\ 29 $\pm$ 0.0 \\ \first{85 $\pm$ 2.8}\end{tabular} & 
\begin{tabular}[c]{@{}c@{}}\second{81 $\pm$ 0.6}\\ 81 $\pm$ 0.7\\ \third{81 $\pm$ 0.6} \\ 22 $\pm$ 0.0 \\ \first{82 $\pm$ 0.8} \end{tabular} & 
\begin{tabular}[c]{@{}c@{}}\second{82 $\pm$ 0.3}\\ \third{82 $\pm$ 0.3}\\ 79 $\pm$ 1.0 \\ 43 $\pm$ 0.1 \\ \first{91 $\pm$ 0.9} \end{tabular} & 
\begin{tabular}[c]{@{}c@{}}\second{81 $\pm$ 0.6}\\ \third{81 $\pm$ 0.7}\\ 74 $\pm$ 1.1 \\ 51 $\pm$ 0.1  \\ \first{93 $\pm$ 1.1} \end{tabular} & 
\begin{tabular}[c]{@{}l@{}}65 $\pm$ 0.3\\ \third{65 $\pm$ 0.2}\\ \second{83 $\pm$ 0.9} \\ 50 $\pm$ 0.1 \\ \first{95 $\pm$ 1.6}\end{tabular}          \\ \hline

50\%                                                 & 
\begin{tabular}[c]{@{}c@{}}SNN \citep{ebli2020simplicial}\\ SCNN \citep{yang2022simplicial}\\SCNN (ours)\\SAT\;\citep{goh2022simplicial} \\ SAN\end{tabular} & 
\begin{tabular}[c]{@{}c@{}}54 $\pm$ 0.7\\ \third{54 $\pm$ 0.6}\\ \second{55 $\pm$ 0.9} \\ 19 $\pm$ 0.0 \\ \first{61 $\pm$ 1.9}\end{tabular} & 
\begin{tabular}[c]{@{}c@{}}55 $\pm$ 0.5\\ \third{55 $\pm$ 0.4}\\ \second{60 $\pm$ 1.1} \\ 30 $\pm$ 0.1 \\ \first{79 $\pm$ 4.3}\end{tabular} & 
\begin{tabular}[c]{@{}c@{}}\second{81 $\pm$ 0.6}\\ 81 $\pm$ 0.7\\ \third{81 $\pm$ 0.6} \\ 22 $\pm$ 0.0 \\ \first{82 $\pm$ 0.8} \end{tabular} & 
\begin{tabular}[c]{@{}c@{}}\second{82 $\pm$ 0.3}\\ \third{82 $\pm$ 0.3}\\ 71 $\pm$ 1.3 \\ 32 $\pm$ 0.1 \\ \first{88 $\pm$ 1.5} \end{tabular} & 
\begin{tabular}[c]{@{}c@{}}\second{81 $\pm$ 0.6}\\ \third{81 $\pm$ 0.7}\\ 68 $\pm$ 1.3 \\ 43 $\pm$ 0.0 \\ \first{92 $\pm$ 0.7} \end{tabular} & 
\begin{tabular}[c]{@{}l@{}}\third{56 $\pm$ 0.3}\\ 56 $\pm$ 0.3\\ \second{79 $\pm$ 2.0} \\ 48 $\pm$ 0.1 \\ \first{94 $\pm$ 1.1}\end{tabular}             \\ \hline
\end{tabular}
\end{adjustbox}
\end{table}

%% file: Include/Chapters/5_Experiments/5.2.2_exp_can.tex
\section{Experiments Cell Attention Networks}

\input{Tables/tu_stats}

\paragraph{Computational Resources and Code Assets} 

In all experiments an NVIDIA\textsuperscript{\textregistered} RTX 3090 GPU with 10,496 CUDA cores and 24GB of GPU memory was used on a personal computing platform with an Intel\textsuperscript{\textregistered} Xeon\textsuperscript{\textregistered} Gold 5218 CPU @ 2.30GHz using Ubuntu 22.04 LTS 64-bit.

The model was implemented in PyTorch~\citep{NEURIPS2019_9015} by building on top of the Simplicial Attention Networks library\footnote{\url{https://github.com/lrnzgiusti/Simplicial-Attention-Networks}}~\citep{giusti2022simplicial} and PyTorch Geometric library\footnote{\url{https://github.com/pyg-team/pytorch_geometric/}}~\citep{fey2019fast}. High-performance lifting operations utilize the graph-tool Python library\footnote{\url{https://graph-tool.skewed.de/}} and are parallelised via Joblib\footnote{\url{https://joblib.readthedocs.io/en/latest/}}.
PyTorch, NumPy~\citep{harris2020array}, SciPy~\citep{virtanen2020scipy} and Joblib are made available under the BSD license, Matplotlib~\citep{hunter2007matplotlib} under the PSF license, graph-tool under the GNU LGPL v3 license. CW Networks and PyTorch Geometric are made available under the MIT license.

\subsection{Benchmarks and Datasets}

\begin{figure}[!htb]
    \centering
    \includegraphics[width=.6\textwidth]{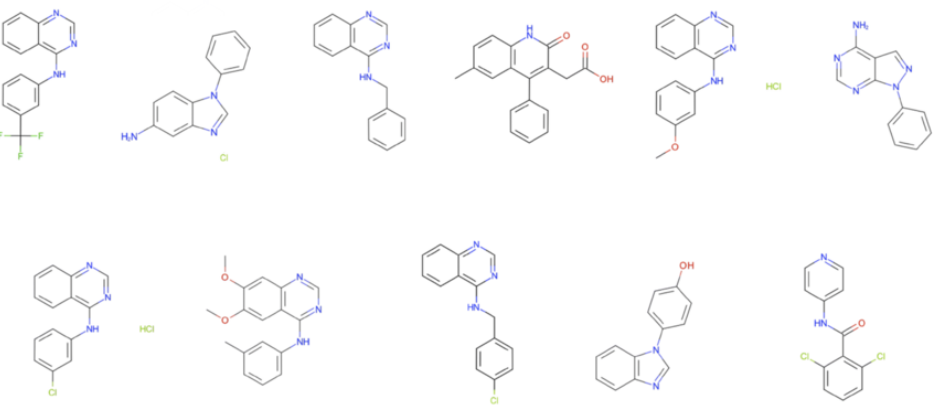}
    \caption{The \textsf{TUDataset} molecular benchmark is a set of five different datasets composed mainly by small molecular compounds in which the learning task is to classify the attributed graph that represent the molecule. Here, node features represent the atom type while edge features encode the type of molecular bonding between the atoms.}
    \label{fig:tu_dataset_small_molec}
\end{figure}

The performance of Cell Attention Network models~\Cref{sec:can} is evaluated on several real-world graph classification problems, focusing on \textsf{TUDataset} molecular benchmarks~\citep{morris2020tudataset}. In every experiment, if the dataset is equipped with edge features, they are concatenated to the result of the functional lift~(\Cref{eq:functional_lift})). The benchmark is composed of small molecules with class labels such as \textsf{MUTAG}~\citep{kazius2005derivation} and \textsf{PTC}~\citep{helma2001predictive}. In the former dataset, the task is  to identify mutagenic molecular compounds for potentially commercial drugs, while in the latter the goal is to identify chemical compounds based on their carcinogenicity in rodents. The \textsf{PROTEINS} dataset~\citep{dobson2003distinguishing} is composed mainly by macromolecules. Here, nodes represent secondary structure elements and are annotated by their type. Nodes are connected by an edge if the two nodes are neighbours on the amino acid sequence or one of three nearest neighbors in space; the task is to understand if a protein is an enzyme or not. Using this type of data in a cell complex based architecture has an underlying importance since molecules have polyadic structures. Finally, \textsf{NCI1} and \textsf{NCI109} are two datasets aimed at identifying chemical compounds against the activity of non-small lung cancer and ovarian cancer cells~\citep{wale2008comparison}. Considering the aforementioned datasets, cell attention is compared with other state of the art techniques in graph representation learning. Since there are no official splits for training and inference phases, to validate the proposed architecture, it is used a 10-fold cross-validation reporting the maximum of the average validation accuracy across folds as in~\cite{bodnar2021weisfeilercell}.
\subsection{Comparative Performance Analysis}

\input{Tables/hparams_can_exp}

The performance of the CAN model is reported in~\Cref{tab:tud} and the hyperparameters used are in~\Cref{tab:hyper-params-can}. The proposed architecture is compared along with those of graph kernel methods: Random Walk Kernel (RWK,~\cite{gartner2003graph}), 
Graph Kernel (GK,~\cite{shervashidze2009efficient}),
Propagation Kernels (PK,~\cite{neumann2016propagation}), 
Weisfeiler-Lehman graph kernels (WLK,~\cite{shervashidze2011weisfeiler}); other GNNs: Diffusion-Convolutional Neural Networks (DCNN,~\cite{atwood2016diffusion}),  Deep Graph Convolutional Neural Network (DGCNN,~\cite{zhang2018end}), Invariant and Equivariant Graph Networks (IGN,~\cite{maron2019invariant}), Graph Isomorphism Networks (GIN,~\cite{xu2019powerful}), Provably Powerful Graph Networks (PPGNs,~\cite{maron2019provably}), Natural Graph Networks (NGN,~\cite{de2020natural}), Graph Substructure Network (GSN~\cite{bouritsas2022improving}) and topological networks: Convolutional Cell Complex Neural Networks (CCNN~\cite{hajij2020cell}), Simplicial Isomorphism Network (SIN,~\cite{bodnar2021weisfeiler},
Cell Isomorphism Network (CIN,~\cite{bodnar2021weisfeilercell}). As shown in~\Cref{tab:tud}, Cell Attention Networks achieves very high performance on this benchmark, and performs similarly to Cell Isomorphism Networks in the last experiment (i.e.,  \textsf{NCI109})\footnote{The code implementation for the proposed architecture is available at: \url{https://github.com/lrnzgiusti/can}}.

\subsection{Ablation Study}

\begin{figure}[t]
    \centering
    \includegraphics[width=.8\textwidth]{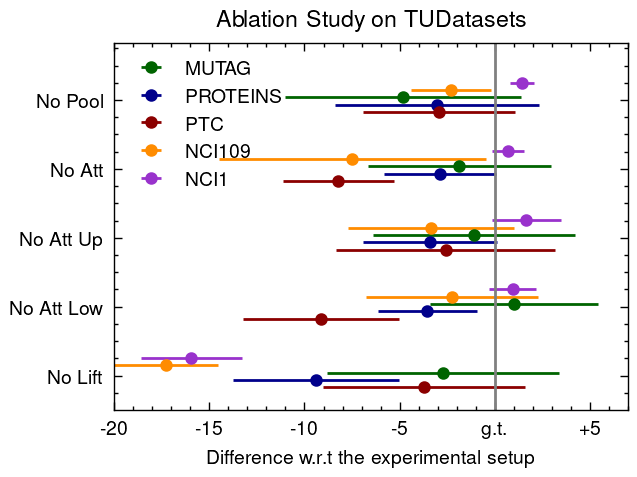}
    \caption{TUDataset: Results of the ablation of different CAN features with respect to~\Cref{tab:tud} (g.t.). The ablation study shows the benefits of incorporating all the proposed operations into the message passing procedure when operating on data defined over cell complexes.}
    \label{fig:ablation}
\end{figure}

This section dedicates a detailed look at the performance of each operation involved in cell attention networks by performing different ablation studies and show their individual importance and contributions. To perform the ablation study, the hyper-parameters are kept fixed as in~\Cref{tab:hyper-params-can} and  the cell attention network operations are sequentially removed one-by-one. Removing the functional lift refers to assign a feature $\mathbf{x}_e$ to an edge $e$ using a scalar product between the features $\mathbf{x}_u$ and $\mathbf{x}_v$ for $u,v \in \mathcal{B}(e)$ (i.e., $\mathbf{x}_e = \langle \mathbf{x}_u, \mathbf{x}_v \rangle$). Removing the attention means setting  $\mathbf{a}_{\uparr}$ and/or $\mathbf{a}_{\doarr}$ to yield the coefficients of the upper and lower Laplacians indexed by the label of their input. The case in which both attentions are removed can be seen as a particular implementation of the cell complex neural network~\citep{hajij2020cell}. Removing the pooling is equal to set the pooling ratio $\rho$ in~\Cref{par:cell_attention_edge_pooling} equal to $1$ and remove eventual intermediate readout computations involved in the hierarchical pooling setup. The ablation in~\Cref{fig:ablation} shows a drastic drop in the overall performance when removing parts of Cell Attention Network. Of particular interest is the study on \textsf{NCI1}, which shows a slightly higher accuracy in every case  he attention coefficients are kept fixed and without the pooling, but a drastic drop in the performance is observed when the edge features are no longer learned. Moreover there are no evident patterns inside the ablation study except for \textsf{NCI109}, which shows the same behavior as \textsf{NCI1} when the lift layer is removed. This fact can be explained by noticing that the aforementioned datasets experience, on average, a very similar topology~(\Cref{tab:tu_dataset_details}).

%% file: Tables/tu_stats.tex
\begin{table}[t]
\begin{center}
\caption{Details of the datasets used in our experiments.}
\label{tab:tu_dataset_details}
\begin{tabular}{lccccc}
\toprule
\multicolumn{1}{c}{Info}  & \multicolumn{1}{c}{MUTAG} & \multicolumn{1}{c}{PTC} & \multicolumn{1}{c}{PROTEINS} & \multicolumn{1}{c}{NCI1} & \multicolumn{1}{c}{NCI109} \\  \bottomrule
\textbf{\# Graphs} & 
\multicolumn{1}{c}{$188$} & 
\multicolumn{1}{c}{$336$} & 
\multicolumn{1}{c}{$1113$} & 
\multicolumn{1}{c}{$4110$} & 
\multicolumn{1}{c}{$4127$} \\
\textbf{\# Classes}            & 
\multicolumn{1}{c}{$2$} & 
\multicolumn{1}{c}{$2$} & 
\multicolumn{1}{c}{$2$}& 
\multicolumn{1}{c}{$2$} & 
\multicolumn{1}{c}{$2$} \\
\textbf{\# Node Feat.}            & 
\multicolumn{1}{c}{$7$} & 
\multicolumn{1}{c}{$20$} & 
\multicolumn{1}{c}{$3$}& 
\multicolumn{1}{c}{$37$} & 
\multicolumn{1}{c}{$38$} \\
\textbf{\# Edge Feat.}            & 
\multicolumn{1}{c}{$4$} & 
\multicolumn{1}{c}{$4$} & 
\multicolumn{1}{c}{$0$}& 
\multicolumn{1}{c}{$0$} & 
\multicolumn{1}{c}{$0$} \\
\textbf{Avg. Nodes}            & 
\multicolumn{1}{c}{$17.93$} & 
\multicolumn{1}{c}{$13.97$} & 
\multicolumn{1}{c}{$39.06$}& 
\multicolumn{1}{c}{$29.87$} & 
\multicolumn{1}{c}{$29.68$} \\
\textbf{Avg. Edges}            & 
\multicolumn{1}{c}{$19.79$} & 
\multicolumn{1}{c}{$14.32$} & 
\multicolumn{1}{c}{$72.82$}& 
\multicolumn{1}{c}{$32.30$} & 
\multicolumn{1}{c}{$32.13$} \\
\textbf{Avg. 3 Cells.}            & 
\multicolumn{1}{c}{$0.00$} & 
\multicolumn{1}{c}{$0.04$} & 
\multicolumn{1}{c}{$27.40$}& 
\multicolumn{1}{c}{$0.04$} & 
\multicolumn{1}{c}{$0.04$} \\
\textbf{Avg. 4 Cells.}            & 
\multicolumn{1}{c}{$0.00$} & 
\multicolumn{1}{c}{$0.01$} & 
\multicolumn{1}{c}{$14.08$}& 
\multicolumn{1}{c}{$0.03$} & 
\multicolumn{1}{c}{$0.03$} \\
\textbf{Avg. 5 Cells.}            & 
\multicolumn{1}{c}{$0.36$} & 
\multicolumn{1}{c}{$0.19$} & 
\multicolumn{1}{c}{$5.68$}& 
\multicolumn{1}{c}{$0.75$} & 
\multicolumn{1}{c}{$0.74$} \\
\textbf{Avg. 6 Cells.}            & 
\multicolumn{1}{c}{$2.5$} & 
\multicolumn{1}{c}{$1.12$} & 
\multicolumn{1}{c}{$8.72$}& 
\multicolumn{1}{c}{$2.66$} & 
\multicolumn{1}{c}{$2.7$} \\
\bottomrule
\end{tabular}
\end{center}
\end{table}

%% file: Tables/hparams_can_exp.tex
\begin{table}[t]
\centering
\caption{Hyperparameter used for the experiments on TUDatasets.}
\label{tab:hyper-params-can}
\begin{tabular}{lccccc}
\toprule 
\multicolumn{1}{l}{Parameter}  & \multicolumn{1}{c}{MUTAG} & \multicolumn{1}{c}{PTC} & \multicolumn{1}{c}{PROTEINS} & \multicolumn{1}{c}{NCI1} & \multicolumn{1}{c}{NCI109} \\ \bottomrule
Lift Heads                   & \multicolumn{1}{c}{$1$} & \multicolumn{1}{c}{$32$} & \multicolumn{1}{c}{$256$} & \multicolumn{1}{c}{$128$} & \multicolumn{1}{c}{$128$} \\
Lift Activation            & \multicolumn{1}{c}{\textit{ELU}} & \multicolumn{1}{c}{\textit{ELU}} & \multicolumn{1}{c}{\textit{ELU}} & \multicolumn{1}{c}{\textit{ELU}} & \multicolumn{1}{c}{\textit{ELU}} \\
Lift Dropout                    & \multicolumn{1}{c}{$0.0$} & \multicolumn{1}{c}{$0.0$} & \multicolumn{1}{c}{$0.05$} & \multicolumn{1}{c}{$0.2$} & \multicolumn{1}{c}{$0.2$}\\
Hidden Dim.           & \multicolumn{1}{c}{$[32,32]$} & \multicolumn{1}{c}{$[32,32]$} & \multicolumn{1}{c}{$[128,128]$} & \multicolumn{1}{c}{$[32,32,32,32]$} & \multicolumn{1}{c}{$[32,32,32,32]$} \\
Att.  Heads             & \multicolumn{1}{c}{$[1,1]$} & \multicolumn{1}{c}{$[2,2]$} & \multicolumn{1}{c}{$[1,1]$} & \multicolumn{1}{c}{$[4,4,4,4]$} & \multicolumn{1}{c}{$[4,4,4,4]$} \\
Att. Aggregation      & \multicolumn{1}{c}{\textit{-}} & \multicolumn{1}{c}{\textit{cat}} & \multicolumn{1}{c}{\textit{-}} & \multicolumn{1}{c}{\textit{cat}} & \multicolumn{1}{c}{\textit{cat}} \\
Att.  Activation       & \multicolumn{1}{c}{\textit{LReLU}} & \multicolumn{1}{c}{\textit{LReLU}} & \multicolumn{1}{c}{\textit{Tanh}} & \multicolumn{1}{c}{\textit{Tanh}} & \multicolumn{1}{c}{\textit{Tanh}} \\
$\com$ Activation                & \multicolumn{1}{c}{ELU} & \multicolumn{1}{c}{ELU} & \multicolumn{1}{c}{Tanh} & \multicolumn{1}{c}{ELU} & \multicolumn{1}{c}{ELU} \\
Classif. Dim.                & \multicolumn{1}{c}{$8$} & \multicolumn{1}{c}{$4$} & \multicolumn{1}{c}{$128$} & \multicolumn{1}{c}{$256$} & \multicolumn{1}{c}{$32$} \\
Batch Size                 & \multicolumn{1}{c}{$64$} & \multicolumn{1}{c}{$128$} & \multicolumn{1}{c}{$128$} & \multicolumn{1}{c}{$128$} & \multicolumn{1}{c}{$128$}\\
Neg. Slope                & \multicolumn{1}{c}{$0.1$} & \multicolumn{1}{c}{$0.1$} & \multicolumn{1}{c}{$0.3$} & \multicolumn{1}{c}{$0.08$} & \multicolumn{1}{c}{$0.07$}\\
Pool Ratio         & 
\multicolumn{1}{c}{$1.0$} & \multicolumn{1}{c}{$0.75$} & \multicolumn{1}{c}{$0.6$} & \multicolumn{1}{c}{$0.5$} & \multicolumn{1}{c}{$0.75$} \\
Pool Type               & \multicolumn{1}{c}{\textit{Hier.}} & \multicolumn{1}{c}{\textit{Glob.}} & \multicolumn{1}{c}{\textit{Hier.}} & \multicolumn{1}{c}{\textit{Glob.}} & \multicolumn{1}{c}{\textit{Glob.}} \\
Dropout                    & \multicolumn{1}{c}{$0.1$} & \multicolumn{1}{c}{$0.6$} & \multicolumn{1}{c}{$0.3$} & \multicolumn{1}{c}{$0.15$} & \multicolumn{1}{c}{$0.05$}\\
Learning Rate               & \multicolumn{1}{c}{$3e^{-3}$} & \multicolumn{1}{c}{$1e^{-3}$} & \multicolumn{1}{c}{$3e^{-3}$} & \multicolumn{1}{c}{$3e^{-4}$} & \multicolumn{1}{c}{$3e^{-3}$}    \\
\bottomrule
\end{tabular}
\end{table}

%% file: Include/Chapters/5_Experiments/5.3.1_exp_cinpp.tex
\section{Experiments CIN++}\label{sec:exp_cinpp}

\subsection{Experimental Setup}
\input{Tables/zinc-500k}

This section proposed an empirical validation of the properties of the proposed message-passing scheme in different real-world scenarios involving graph-structured data. The experiments are performed on a large-scale molecular benchmark (ZINC)~\citep{dwivedi2020benchmarking} and a long-range graph benchmark (Peptides)~\citep{dwivedi2022long}. Unless otherwise specified, in each Multi-Layer Perceptron,  Batch Normalization~\citep{ioffe2015batch} between the linear transformations and ReLU activations Adam is used with a starting learning rate of $0.001$, which is halved whenever the validation loss reaches a plateau after a patience value set to $20$. Moreover, an early stopping criterion is employed. It terminates the training when the learning rate reaches a threshold. Unless stated otherwise, the early stopping threshold is fixed to $1e^{-5}$.

\paragraph{Computational Resources and Code Assets} 

All the experiments were performed using an NVIDIA\textsuperscript{\textregistered} Tesla V100 GPUs with 5,120 CUDA cores and 32GB GPU memory on a personal computing platform with an Intel\textsuperscript{\textregistered} Xeon\textsuperscript{\textregistered} Gold 5218 CPU @ 2.30GHz using Ubuntu 18.04.6 LTS.

The model has been implemented in PyTorch~\citep{NEURIPS2019_9015} by building on top of CW Networks library\footnote{\url{https://github.com/twitter-research/cwn/}}~\citep{bodnar2021weisfeilercell} and PyTorch Geometric library\footnote{\url{https://github.com/pyg-team/pytorch_geometric/}}~\citep{fey2019fast}. High-performance lifting operations use the graph-tool\footnote{\url{https://graph-tool.skewed.de/}} Python library and are parallelized via Joblib\footnote{\url{https://joblib.readthedocs.io/en/latest/}}.
PyTorch, NumPy~\citep{harris2020array}, SciPy~\citep{virtanen2020scipy} and Joblib are made available under the BSD license, Matplotlib~\citep{hunter2007matplotlib} under the PSF license, graph-tool under the GNU LGPL v3 license. CW Networks and PyTorch Geometric are made available under the MIT license.

\subsection{Benchmarks and Datasets}

\subsubsection{Large-Scale Molecular Benchmarks} 

\paragraph{ZINC}
Topological message passing is here evaluated on a large-scale molecular benchmark from the {\em ZINC} database~\citep{ZINCdataset}. The benchmark is composed of two datasets: {\em ZINC-Full} (consisting of 250K molecular graphs) and {\em ZINC-Subset} (an extract of 12k graphs from ZINC-Full) from~\cite{dwivedi2020benchmarking}. 

The number of nodes (or atoms) in the graphs ranges from 3 to 132, with an average size of approximately 24 nodes. The majority of the graphs have between 10 and 30 nodes. The average degree in the graphs is approximately 2 and the average diameter of the graphs is approximately 12.4 nodes (or atoms) and the maximum diameter was 62 nodes. Regarding the edges (or bonds), the average number of edges in the graphs is approximately 50 composed of 98\% by single bonds, while the remaining 2\% are aromatic bonds.

\

\begin{figure}[!htb]
    \centering
    \includegraphics[width=.8\textwidth]{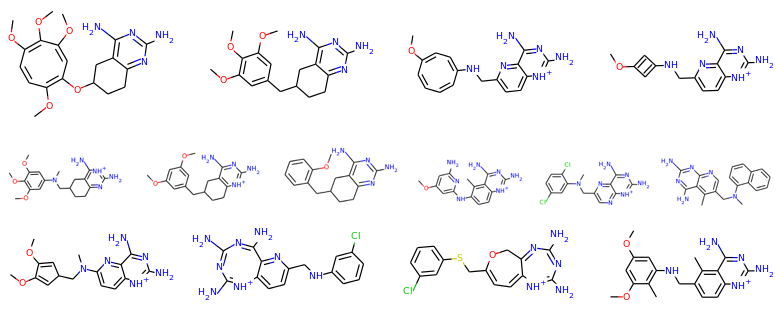}
    \caption{Visualization of molecular graphs contained in the ZINC dataset. Each graph represents a unique molecule with atoms as nodes and chemical bonds as edges. The graph-level targets are the penalized water-octanol partition coefficient (logP) that characterizes a molecule's drug-likeness.}
    \label{fig:zinc}
\end{figure}

\begin{figure}[!htb]
    \centering
    \includegraphics[width=1\textwidth]{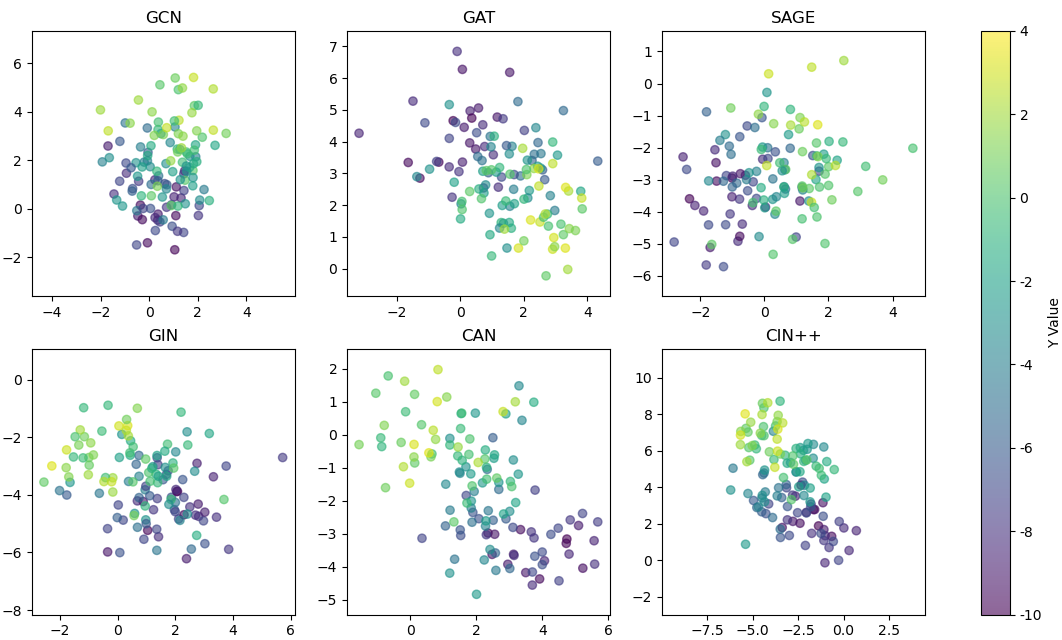}
    \caption{t-SNE visualizations representing the hidden features  from six different trained models on the ZINC dataset, displaying the clustering of molecular structures by their penalized logP values. CIN++ outperforms others with distinct clustering, followed by CAN, while GCN, GAT, SAGE, and GIN show greater overlap, suggesting a gradation in the models' ability to exploit complex chemical properties.}
    \label{fig:tsne_zinc}
\end{figure}

These are two graph regression task datasets for drug-constrained solubility prediction, built on top of the ZINC database provided by the Irwin and Shoichet Laboratories in the Department of Pharmaceutical Chemistry at the University of California, San Francisco (UCSF)~\cite{ZINCdataset}. Each graph represents a molecule, where the features over the nodes specify which atom it represents while edge features specify the type of chemical bond between two atoms~\Cref{fig:zinc}. Graph-level targets correspond to the penalised water-octanol partition coefficient -- logP, an important metric in drug design that depends on chemical structures and molecular properties and characterizes the drug-likeness of a molecule~\citep{gomez2018automatic}.

\paragraph{MOLHIV}
\input{Tables/mol_table}

\begin{figure}[!htb]
    \centering
    \includegraphics[width=.8\textwidth]{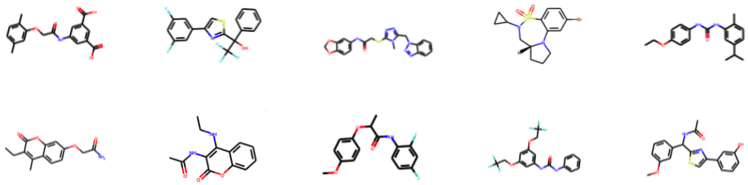}
    \caption{Representation of molecules in the \texttt{ogbg-molhiv} dataset from the Open Graph Benchmark. Individual nodes denote atoms, while edges depict chemical bonds. Various node and edge features such as atomic number, chirality, bond type, and stereochemistry are utilized to encapsulate the chemical properties of the molecule. Adapted from~\cite{hu2021ogb}.}
    \label{fig:mol-hiv}
\end{figure}

The model is further validated experimentally using the \texttt{ogbg-molhiv} molecular dataset from the Open Graph Benchmark~\citep{hu2020open}. Each graph is a representation of a molecule, where the nodes stand for atoms and the edges for chemical bonds~\Cref{fig:mol-hiv}. The node features, which are 9-dimensional, include: the atomic number, chirality, and other atom-specific attributes such as formal charge and ring inclusion. The edge features, which are 3-dimensional, incorporate the bond type, bond stereochemistry, and an additional feature that indicates the presence of a conjugated bond. The statistics of the graphs in the dataset are similar to the ones discussed for the ZINC benchmark. The task is to predict the ability of compounds to inhibit HIV replication.

\subsubsection{Long-Range Graph Benchmarks}

\input{Tables/peptides_table}
To test the effectiveness of enhanced topological message passing for discovering long-range interactions CIN++ is evaluated on a long-range molecular benchmark~\citep{dwivedi2022long}. The datasets used from the benchmark are derived from 15,535 peptides that compose the SATPdb database~\citep{singh2016satpdb}.  In both tasks of this benchmark, each graph corresponds to a peptide molecule~\citep{dwivedi2022long}. \\

\begin{figure}[!htb]
    \centering
    \includegraphics[width=.7\textwidth]{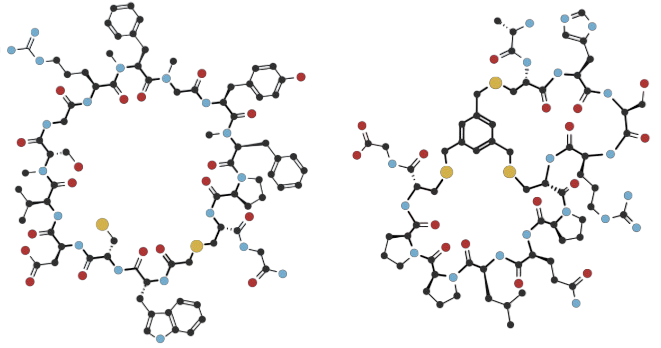}
    \caption{Graph Representation of two peptides made up on arrangements of amino acids connected through peptide linkages. Each node represents a heavy atom, while the edges show the covalent bonds between them. It worth emphasize the complexity of peptide molecular structures in contrast to smaller drug-like molecules. Adapted from~\cite{vinogradov2019macrocyclic}.}

    \label{fig:peptides}
\end{figure}

Peptides, in the realm of biology, are depicted as compact polymers of amino acids, which are covalently bonded through peptide linkages formed between the carboxyl group of one amino acid and the amino group of another~\Cref{fig:peptides}. These molecules execute a diverse spectrum of functions in living organisms, serving as signaling molecules~\citep{feng1997cloning}, protective agents of the immune system~\citep{janeway1997immunobiology}, structural constituents~\citep{o1993peptide}, transporters~\citep{torchilin2008tat}, enzymes~\citep{rastelli2010fast}, and even as a nutritional source~\citep{erdmann2008possible}. \\

Since each amino acid is composed of many heavy atoms, the molecular graph of a peptide is much larger than that of a small drug-like molecule. The long-range molecular benchmark proposes two datasets for \peptides property prediction where the graphs are derived such that the nodes correspond to the heavy (non-hydrogen) atoms of the peptides while the edges represent the bonds that join them. \\

The peptides datasets have a diameter about 5 times larger ($\approx$ 57) and contain 6 times more atoms than the molecular graphs present in the {\em ZINC} benchmark, with an average node degree of 2.04. The average shortest path is 20.89. The requirements for long-range interactions and sensitivity to the graph's global properties are met through the three-dimensional structural dependencies intrinsic to the peptide chains combined with a substantial raise in number of nodes in the graphs.

\subsubsection{TUDataset}
\input{Tables/tu_table}

\begin{figure}[!htb]
    \centering
    \includegraphics[width=.8\textwidth]{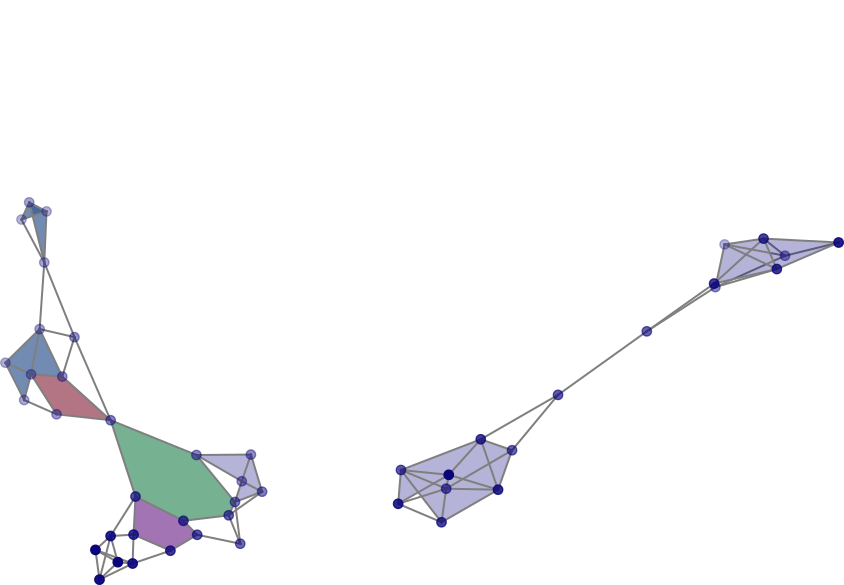}
    \caption{The protein complexes lifted from graphs in the \textsf{PROTEINS} datasets from the TUDataset molecular benchmark. In \textcolor{mblue}{blue} are denoted rings with three nodes (triangles), \textcolor{mred}{red} squares, \textcolor{mpurp}{purple} pentagons, \textcolor{mgreen}{green} hexagons,  Notably, the right structure resembles the graph in~\Cref{fig:effective-resistance}, characterized only by triangles as 2-cells.}
    \label{fig:protein_complexes}
\end{figure}

The \texttt{TUDataset}~\citep{morris2020tudataset} is a rich repository of graph-based datasets, serving as a benchmark for learning tasks on graph-structured data. Specifically, the assessment is performed on dataset composed of small molecules and bioinformatics. The \texttt{MUTAG} dataset, for instance, comprises nitroaromatic compounds, where the task is to predict their mutagenicity on Salmonella typhimurium~\citep{debnath1991structure}. \\

The dataset is structured as graphs, with vertices representing atoms labeled by atom type and edges representing bonds between the corresponding atoms, consisting of 188 samples of chemical compounds with 7 discrete node labels. Another dataset used is \texttt{PTC}, a collection of 344 chemical compounds, each represented as a graph, with the goal to report carcinogenicity for rodents, and 19 node labels for each node~\citep{toivonen2003statistical}.

The \texttt{NCI1} and \texttt{NCI109} and dataset, from the cheminformatics domain, represents each chemical compound as a graph, where vertices and edges respectively representing atoms and bonds between atoms. The dataset pertains to anti-cancer screens with chemicals evaluated for their effectiveness against cell lung cancer~\citep{wale2008comparison}. Each vertex label denotes the corresponding atom type, encoded via a one-hot-encoding scheme into a binary vector. The \texttt{PROTEINS} dataset~\Cref{fig:protein_complexes} is utilized in the field of bioinformatics for protein function prediction~\citep{borgwardt2005protein}. The task is to predict functional class membership of enzymes and non-enzymes.

\subsection{Comparative Performance Analysis}

\paragraph{ZINC}
These experiments follow the experimental setup of~\cite{bodnar2021weisfeilercell} with the exception that the architecture uses $3$ layers with a hidden dimension of $64$. This restricts the parameter budget of the model to $500$K parameters. The training and evaluation follow the specification in~\cite{dwivedi2020benchmarking}. All results are illustrated in~\Cref{tab:zinc-500k} and in~\Cref{tab:mol_dataset}. {\em Without any use of feature augmentation} such as positional encoding, the proposed model exhibits particularly strong performance on these benchmarks: it attains state-of-the-art results by a significant margin on {\em ZINC-Subset},  outperforming other models by a significant margin and is on par with the best baselines for {\em ZINC-Full}. For the {\em ZINC-Subset}, a qualitative result is also reported in~\Cref{fig:tsne_zinc}, where the feature representations from various models are visualized through t-SNE~\cite{van2008visualizing}. The figure shows that CIN++ exhibits the most clear clustering of data points, suggesting a superior qualitative result in capturing the molecular characteristics relevant to the penalized logP values, as compared to Cell Attention Netowrk and the other graph neural networks.

\paragraph{MOLHIV}
For this dataset, a maximum ring size of $6$ assign nodes as $2$-cells. The architecture and hyperparameter settings mirror those referenced in previous studies~\citep{bodnar2021weisfeilercell, Fey2020_himp}. In~\Cref{tab:mol_dataset}, it is presented the average test ROC-AUC metrics at the epoch of optimal validation performance across $10$ random weight initializations. For this dataset, a lower performance than CIN is achieved, but superior to many other established models.

As evidenced in~\Cref{tab:mol_dataset}, CIN++ performs significantly well on the \texttt{ogbg-molhiv} dataset, making it the second-best performing model. The simpler version, \textbf{CIN++-small}, also demonstrates commendable results with an average test ROC-AUC, surpassing several other models and landing it in the top three. This illustrates that while the CIN model is the front-runner, the proposed models have effectively made use of the inherent graph structures and features to make predictive assessments about the molecules' capabilities to inhibit HIV replication.

\paragraph{Peptides}
For this benchmark, the proposed method is evaluated on the tasks of peptide structure prediction (\pepstruct) and peptide function prediction (\pepfunc). {\em For both datasets, any feature augmentation is employed such as positional or structural encoding}. The parameter budget has been constrained to $500$K. The assessment is then repeated with $4$ different seeds and reported the mean of the test AP and MAEs at the time of early stopping in~\Cref{tab:experiments_peptides}. For \pepstruct, a cellular lifting map is used that considers all induced cycles of dimension up to $8$ as rings. Here, CIN++ implements $3$ layers with $64$ as a hidden dimension, a batch size of $128$ and a sum aggregation to obtain complex-level embeddings. For \pepfunc, 2 cells are attached to all the induced cycles of dimension up to $6$. For this dataset was employed a CIN++ model with $4$ layers with an embedding dimension of $50$, and a batch size of $64$. A Dropout~\citep{srivastava2014dropout} with a probability of 0.15 is inserted.  With respect to the other benchmarks, the starting learning rate was set to $4e^{-4}$, with a weight decay of $5e^{-5}$.. The final readout is performed with a mean aggregation. As shown in~\Cref{tab:experiments_peptides} this model achieves very high performance on these tasks even without any use of feature augmentation.

\paragraph{TUDataset}
Moreover, the performance of enhanced topological message passing scheme is also assessed against graph kernel methods, graph neural networks as well as topological neural networks. In this set of experiments, the model employs the same model configurations used in~\cite{bodnar2021weisfeilercell}.  Therefore, in~\Cref{tab:tud} it is reported that the proposed scheme achieves state of the art results on four out of five different evaluations. The exception is for \texttt{NCI1} where the propsed method achieves the second place after WL kernel~\citep{shervashidze2011weisfeiler}.

%% file: Tables/zinc-500k.tex
\begin{table}[t]
\vspace{-18pt}
\caption{Performance results on ZINC benchmark. The best performance are indicated with gold \tikzcircle[gold,fill=gold]{2pt}, silver \tikzcircle[silver,fill=silver]{2pt}, and bronze \tikzcircle[bronze,fill=bronze]{2pt} colors.} 
  \vspace{2px}
  \label{tab:zinc-500k}
  \small
  \centering
  \resizebox{0.98\textwidth}{!}{ \renewcommand{\arraystretch}{1.0}
    \begin{tabular}{clcccc}
    \toprule
    \multirow{2}{*}{Method} & \multirow{2}{*}{Model} & \multirow{2}{*}{Time (s)} & \multirow{2}{*}{Params} & \multicolumn{2}{c}{Test MAE}\\
    & & & & ZINC-Subset & ZINC-Full \\ \midrule
    
    \multirow{7}{*}{MPNNs}
    & GIN~\citep{xu2019powerful} & 8.05 & 509,549 & 0.526$\pm$0.051 &  0.088$\pm$0.002 \\
    & GraphSAGE~\citep{hamilton2017inductive} & 6.02 & 505,341 & 0.398$\pm$0.002 & 0.126$\pm$0.003 \\
    & GAT~\citep{velivckovic2018graph} & 8.28 & 531,345 & 0.384$\pm$0.007 & 0.111$\pm$0.002 \\
    & GCN~\citep{kipf2017graph} & 5.85 & 505,079 & 0.367$\pm$0.011 & 0.113$\pm$0.002 \\

    & MoNet~\citep{monti2017geometric} & 7.19 & 504,013 & 0.292$\pm$0.006 & 0.090$\pm$0.002 \\
    & \textls[-25]{GatedGCN-PE\citep{bresson2017residual}} & 10.74 & 505,011 & 0.214$\pm$0.006 & - \\
    & MPNN(sum)~\citep{gilmer2017neural} & - &  480,805 & 0.145$\pm$0.007 & - \\
    & PNA~\citep{corso2020principal} & - & 387,155 & 0.142$\pm$0.010 & - \\ 
    \midrule
    \multirow{2}{*}{\begin{tabular}[c]{@{}c@{}}Higher-order\\GNNs\end{tabular}} & RingGNN~\citep{chen2019equivalence} & 178.03 & 527,283 &  0.353$\pm$0.019 & - \\
    & 3WLGNN~\citep{maron2019provably} & 179.35 & 507,603 &  0.303$\pm$0.068 & - \\
    \midrule
    Substructure GNNs & GSN~\citep{bouritsas2022improving} & - & $\sim$500k & 0.101$\pm$0.010 & - \\
    \midrule
    
    \multirow{5}{*}{\begin{tabular}[c]{@{}c@{}}Subgraph\\GNNs\end{tabular}}
    & NGNN~\citep{zhang2021nested} & - & $\sim$500k & 0.111$\pm$0.003 & 0.029$\pm$0.001 \\
    & DSS-GNN~\citep{bevilacqua2022equivariant} & - & 445,709 & 0.097$\pm$0.006 & - \\
    & GNN-AK~\citep{zhao2022stars} & - & $\sim$500k & 0.105$\pm$0.010 & - \\
    & GNN-AK+~\citep{zhao2022stars} & - & $\sim$500k & 0.091$\pm$0.011 & - \\
    & SUN~\citep{frasca2022Understanding} & 15.04 & 526,489 & 0.083$\pm$0.003 & - \\ \midrule
    
    \multirow{4}{*}{\begin{tabular}[c]{@{}c@{}}Graph\\Transformers\end{tabular}}
    & GT~\citep{dwivedi2021generalization} & - & 588,929 & 0.226$\pm$0.014 & - \\
    & SAN~\citep{kreuzer2021rethinking} & - & 508,577 & 0.139$\pm$0.006 & - \\
    & Graphormer~\citep{ying2021transformers} & 12.26 & 489,321 & 0.122$\pm$0.006 & 0.052$\pm$0.005 \\ 
    & URPE~\citep{luo2022your} & 12.40 & 491,737 & 0.086$\pm$0.007 & \third{0.028$\pm$0.002} \\
    \midrule
    GD-WL & Graphormer-GD~\citep{zhang2023rethinking}   & 12.52 & 502,793 & ~~\second{0.081$\pm$0.009}& ~~\first{0.025$\pm$0.004} \\
    \midrule 

    \multirow{4}{*}{\begin{tabular}[c]{@{}c@{}}Topological NNs\end{tabular}} & CIN-Small~\citep{bodnar2021weisfeilercell} & - & $\sim$100k & 0.094$\pm$0.004 & 0.044$\pm$0.003 \\ 
    & CIN~\citep{bodnar2021weisfeilercell}  & 7.96 & 475,316  & ~~\third{0.081$\pm$0.006} & ~~0.029$\pm$0.007 \\
    & CAN~\citep{giusti2022cell}  & 9.34 & 499,912  & ~~0.087$\pm$0.004 & ~~0.038$\pm$0.005 \\
     & CIN++  & 8.29 & 501,967 & ~~\first{0.077$\pm$0.004} & ~~\second{0.027$\pm$0.007} \\\bottomrule
    \end{tabular}
    }
    \vspace{-10pt}
\end{table}

%% file: Tables/mol_table.tex
\begin{table}[!t]
    \centering
    \begin{minipage}[t]{0.80\textwidth}
        \centering
         \caption{ZINC-Subset (MAE), ZINC-Full (MAE) and Mol-HIV.}
        \label{tab:mol_dataset}
           \resizebox{\columnwidth}{!}{
          \begin{tabular}{l ccc}
            \toprule
            \multirow{2}{*}{Model} & 
            
            ZINC-Subset  &
            ZINC-Full & 
            MOLHIV\\
            
            &
            (MAE $\downarrow$) &
            (MAE $\downarrow$)&
            (ROC-AUC $\uparrow$) \\
            \midrule
            
            GCN~\citep{kipf2017graph} & 
            0.469$\pm$0.002 &
            N/A &
            76.06$\pm$0.97 \\

            GAT~\citep{velivckovic2018graph} & 
            0.463$\pm$0.002 &
            N/A &
            N/A \\
            
            GatedGCN~\citep{bresson2017residual} &
            0.363$\pm$0.009 &
            N/A &
            N/A \\

            GIN~\citep{xu2019powerful}  & 
            0.252$\pm$0.014 &
            0.088$\pm$0.002 &
            77.07$\pm$1.49 \\
            
            PNA~\citep{corso2020principal} & 
            0.188$\pm$0.004 &
            N/A &
            79.05$\pm$1.32 \\
            
            DGN~\citep{beaini2020directional} & 
            0.168$\pm$0.003 &
            N/A &
            79.70$\pm$0.97 \\
            
            HIMP~\citep{Fey2020_himp} &
            0.151$\pm$0.006 &
            0.036$\pm$0.002 &
            78.80$\pm$0.82\\

            GSN~\citep{bouritsas2022improving} & 
            0.108$\pm$0.018 &
            N/A &
            77.99$\pm$1.00 \\
            
            \midrule
            
            CIN-small~\citep{bodnar2021weisfeilercell} & 
            0.094$\pm$0.004 &
            \third{0.044$\pm$0.003} &
            \third{80.55$\pm$1.04} \\
            
            CIN~\citep{bodnar2021weisfeilercell} & 
            \second{0.079$\pm$0.006} &
            \second{0.022$\pm$0.002} &
            \first{80.94$\pm$0.57} \\

            \midrule

            \textbf{CIN++-small} & 
            \third{0.091$\pm$0.003} &
            0.044$\pm$0.004 &
            80.26$\pm$1.02 \\
            
            \textbf{CIN++}  & 
            \first{0.074$\pm$0.004} &
            \first{0.021$\pm$0.001} &
            \second{80.63$\pm$0.94} \\
            
            \bottomrule
          \end{tabular}%
          }
    \end{minipage}
\end{table}

%% file: Tables/peptides_table.tex
\begin{table}[t]
    \caption{Performance results for \pepfunc (graph classification) and \pepstruct (graph regression). Best scores are highlighted using gold \tikzcircle[gold,fill=gold]{2pt}, silver \tikzcircle[silver,fill=silver]{2pt}, and bronze \tikzcircle[bronze,fill=bronze]{2pt} colors.}
    \label{tab:experiments_peptides}
    \begin{adjustwidth}{-2.5 cm}{-2.5 cm}\centering
    \scalebox{0.9}{
    \setlength\tabcolsep{4pt} 
    \begin{tabular}{l c c c c}\toprule
    \multirow{2}{*}{\textbf{Model}} & \multicolumn{2}{c}{\pepfunc} & \multicolumn{2}{c}{\pepstruct} \\\cmidrule(lr){2-3}\cmidrule(lr){4-5}
    &\textbf{Train AP} &\textbf{Test AP $\uparrow$} &\textbf{Train MAE} &\textbf{Test MAE $\downarrow$} \\\midrule
    MLP & 0.4217$\pm$0.0049 & 0.4060$\pm$0.0021 & 0.4273$\pm$0.0011 & 0.4351$\pm$0.0008 \\
    GCN &0.8840$\pm$0.0131 &0.5930$\pm$0.0023 &0.2939$\pm$0.0055 &0.3496$\pm$0.0013 \\
    GCNII  & 0.7271$\pm$0.0278 & 0.5543$\pm$0.0078 & 0.2957$\pm$0.0025 & 0.3471$\pm$0.0010\\
    GINE  &0.7682$\pm$0.0154 &0.5498$\pm$0.0079 &0.3116$\pm$0.0047 &0.3547$\pm$0.0045 \\
    GatedGCN  &0.8695$\pm$0.0402 &0.5864$\pm$0.0077 &0.2761$\pm$0.0032 &0.3420$\pm$0.0013 \\
    GatedGCN+RWSE  &0.9131$\pm$0.0321 &0.6069$\pm$0.0035 &0.2578$\pm$0.0116 &0.3357$\pm$0.0006 \\ \midrule
    Transformer+LapPE  &0.8438$\pm$0.0263 & 0.6326$\pm$0.0126 &0.2403$\pm$0.0066 &\third{0.2529$\pm$0.0016} \\
    SAN+LapPE  &0.8217$\pm$0.0280 &\third{0.6384$\pm$0.0121} &0.2822$\pm$0.0108 &0.2683$\pm$0.0043 \\
    SAN+RWSE  &0.8612$\pm$0.0219 &\second{0.6439$\pm$0.0075} &0.2680$\pm$0.0038 & 0.2545$\pm$0.0012 \\ \midrule
    
    CIN  &0.8076$\pm$0.0109& 0.6323$\pm$0.0054 &0.2309$\pm$0.0028 &\second{0.2523$\pm$0.0007} \\
    
    CIN++  &0.8943$\pm$0.0226&\first{0.6569$\pm$0.0117} &0.2290$\pm$0.0079 &\first{0.2523$\pm$0.0013} \\
    \bottomrule
    \end{tabular}
    }
    \vspace{-10pt}
    \end{adjustwidth}
\end{table}

%% file: Tables/tu_table.tex
\begin{table}[!t]
    \centering
    \caption{TUDatasets. The first part shows the performance of graph kernel methods. The second assess graph neural networks while the third part is for topological neural networks. The best performance are indicated with gold \tikzcircle[gold,fill=gold]{2pt}, silver \tikzcircle[silver,fill=silver]{2pt}, and bronze \tikzcircle[bronze,fill=bronze]{2pt} colors }
    \label{tab:tud}
    \resizebox{\linewidth}{!}{%
    \begin{tabular}{l  lllll}
        \toprule
        Model & 
        MUTAG &
        PTC\_MR &
        PROTEINS &
        NCI1 &
        NCI109\\
        \midrule       
        RWK~\citep{gartner2003graph} & 
         79.2$\pm$2.1 & 
         55.9$\pm$0.3 & 
         59.6$\pm$0.1 & 
         $>$3 days & 
         N/A \\
        
        GK ($k=3$)~\citep{shervashidze2009efficient} &
        81.4$\pm$1.7 & 
        55.7$\pm$0.5 & 
        71.4$\pm$0.3 &  
        62.5$\pm$0.3 & 
        62.4$\pm$0.3  \\

        PK~\citep{neumann2016propagation} & 
         76.0$\pm$2.7& 
         59.5$\pm$2.4 & 
         73.7$\pm$0.7 & 
         82.5$\pm$0.5 & 
         N/A  \\

        WL kernel~\citep{shervashidze2011weisfeiler} &
          90.4$\pm$5.7 & 
          59.9$\pm$4.3 & 
          75.0$\pm$3.1 & 
          \first{86.0$\pm$1.8} & 
          N/A \\

        \midrule
         
        DCNN~\citep{atwood2016diffusion} & 
        N/A&  
        N/A & 
        61.3$\pm$1.6 & 
        56.6$\pm$1.0 & 
        N/A  \\

        DGCNN~\citep{zhang2018end} & 
        85.8$\pm$1.8 & 
        58.6$\pm$2.5 & 
        75.5$\pm$0.9 & 
        74.4$\pm$0.5 & 
        N/A \\
        
        IGN~\citep{maron2019invariant} &
        83.9$\pm$13.0 & 
        58.5$\pm$6.9 & 
        76.6$\pm$5.5 & 
        74.3$\pm$2.7 & 
        72.8$\pm$1.5 \\
        
        GIN~\citep{xu2019powerful} & 
        89.4$\pm$5.6 & 
        64.6$\pm$7.0 &  
        76.2$\pm$2.8 & 
        82.7$\pm$1.7 &  
        N/A \\

        PPGNs~\citep{maron2019provably} &
        90.6$\pm$8.7 & 
        66.2$\pm$6.6 & 
        \third{77.2$\pm$4.7} &  
        83.2$\pm$1.1 & 
        82.2$\pm$1.4  \\

        Natural GN~\citep{de2020natural} &
        89.4$\pm$1.6 & 
        66.8$\pm$1.7 & 
        71.7$\pm$1.0 & 
        82.4$\pm$1.3 &  
        N/A \\

        GSN~\citep{bouritsas2022improving} &
        92.2 $\pm$ 7.5 & 
        68.2 $\pm$ 7.2 & 
        76.6 $\pm$ 5.0 & 
        83.5 $\pm$ 2.0 &  
        N/A \\ 
       
        \midrule
        
        SIN~\citep{bodnar2021weisfeiler} & 
        N/A  & 
        N/A &  
        76.4 $\pm$ 3.3 & 
        82.7 $\pm$ 2.1 & 
        N/A \\ 

        CIN~\citep{bodnar2021weisfeilercell} & 
        \third{92.7 $\pm$ 6.1} & 
        \third{68.2 $\pm$ 5.6} & 
        77.0 $\pm$ 4.3 & 
        83.6 $\pm$ 1.4 & 
        \second{84.0 $\pm$ 1.6}  \\ 
        
        CAN~\citep{giusti2022cell} & 
        \second{94.1 $\pm$ 4.8} & 
        \second{72.8 $\pm$ 8.3} & 
        \second{78.2 $\pm$ 2.0} &  
        \third{84.5 $\pm$ 1.6}  & 
        \third{83.6 $\pm$ 1.2}  \\ 

        \midrule
       {{\bf CIN++}} & 
        \first{94.4 $\pm$ 3.7} & 
        \first{73.2 $\pm$ 6.4} & 
        \first{80.5 $\pm$ 3.9} & 
        \second{85.3 $\pm$ 1.2} & 
        \first{84.5 $\pm$ 2.4}  \\ 
        \bottomrule

    \end{tabular}
    }
\end{table}

%% file: Include/Chapters/6_Conclusions/6_conclusions.tex
\chapter{Conclusions}\label{chap:conclusions}

This thesis has presented an innovative approach to measure the impact of several factors that reduce the capabilities of message-passing neural networks in capturing long-range interactions. In particular, it has been shown how the width and depth of the $\MPNN$, and the underlying graph topology can influence the performance of structured learning tasks that depend on long-range interactions. Moreover, to cope with such limitations, this thesis has proposed topological approaches to {\bf naturally decouple the computational graph from the input graph}. In particular, it is possible to mitigate the bottlenecks of graph neural networks while capturing higher-order relationships without a significant increase in the complexity of the underlying model by designing proper message-passing schemes on discrete topological spaces. However, current state-of-the-art models do not naturally account for a principled way to model efficient topological message passing schemes accounting both higher-order interactions and the feature's importance. To this aim, this thesis has addressed these challenges by advancing the methods presented in the {\em topological deep learning} literature including attentional schemes on topological spaces and an enhancement of Cellular Isomorphism Networks~\citep{bodnar2021weisfeilercell}. The newly proposed topological message passing scheme, named CIN++, enables a direct interaction within high-order structures of the underlying cell complex, by letting messages flow within its lower neighbourhood without sacrificing the model's expressivity. By allowing the exchange of messages between higher-order structures, the model's capacity to capture multi-way relationships in the data is significantly enhanced. We have demonstrated that the ability to model long-range and group interactions is critical for capturing real-world chemistry-related problems. In particular, the natural affinity of cellular complexes for representing higher-dimensional structures and topological features will provide a more detailed understanding of complex chemical systems compared to traditional models.

\section{Broader Impacts} 
This work provides evidence of how the proposed topological message-passing schemes allows the integration of local and global information within a discrete topological space. In particular, the proposed architecture capture complex dependencies and long-range interactions more effectively. This work is foreseen to have a broad impact within the fields of computational chemistry, network neuroscience, and physics, as it offers a robust and versatile framework for predicting meaningful properties of complex systems by accurately modeling group dependencies and capturing long-range interactions.

\section{Limitations} 
While this thesis demonstrates that topological message-passing effectively models higher-order dependencies and long-range interactions in complex systems, it is reasonable to acknowledge that the complexity of the proposed methods inherently increases due to additional operations performed on top of those provided by classic $\MPNN$s. For example, the structural lifting maps~(\Cref{def:structural_lift}) and the additional messages sent throughout the complex~(\ref{par:simplicial_attention},~\ref{par:cell_attention}). However, much of the computational overhead introduced by cellular lifting can be mitigated by mapping all graphs present in the datasets into cell complexes in a preprocessing stage and storing them for later use. Additionally, the overhead of the topological message-passing schemes is mitigated by the fact that the operations within the same layer are naturally decoupled. Efficient network implementations make it possible to update the representation of a cell $\sigma$ in a concurrent execution~\citep{Besta22parallel}, amortizing the cost to be proportional to the largest neighbourhood of $\sigma$.

\section{Recommendations for Future Research}

One of the most promising avenues for further exploration is the application of topological neural networks to the fields of science mentioned in~\Cref{sec:intro:tnns_4_science}. Moreover, the field of algorithmic topology aims to solve problems that are often computationally intractable or fall within the NP-hard complexity class.Moreover, it turns out that algorithmic knot theory can also be used in the very same fields of science mentioned before. For example, in chemistry, molecular chirality can be determined by the nodes' chirality within the knots~\cite{patone2011know}; in physics, through the relationship between the Yang-Baxter equation~\citep{jimbo1989introduction} and knots' invariants. Finally, algorithmic knot theory can be used to model an essential biological process such as DNA recombination~\citep{sumners2020role}. By combining the techniques presented in this thesis, combined with the principles of neural algorithmic reasoning~\citep{velivckovic2021neural} it should be possible to approximate solutions to algorithmic topology problems with a feasible amount of computational resources.

%% file: Include/Backmatter/Appendix/Glossary.tex
\chapter{Glossary}\label{app:glossary}
\input{Tables/notation}

%% file: Tables/notation.tex
\begin{table}[!htb]
\begin{center}
\caption{Summary of Notations: {\em Structural Elements}. Notation for topological constructs such as graphs, simplicial complexes, regular cell complexes, and their associated components.}
\label{tab:structural_notation}
\begin{tabular}{ll}
\hline
\multicolumn{2}{c}{\textbf{Structual Elements}}\\
\toprule
$\gph = (\V,\E)$ & A graph, $\V$ and $\E$ are respectively the sets of nodes and edges.\\\vspace{2pt}
$\kph = (\V, \mathsf{S})$ & A simplicial complex, $\mathsf{S}$ is the ensemble set of simplices.\\\vspace{2pt}
$\cph = (\V, \mathcal{P}_{\cph})$ & A regular cell complex, $\mathcal{P}_{\cph}$ is the set of cells.\\\vspace{2pt}
$v_i$ & A node, element of $\V$. \\\vspace{2pt}
$e_i  = (v_i, v_j) $ & An edge, element of $\E$. \\\vspace{3pt}
$\sigma_i^k = (\sigma_1^{k-1}, \ldots, \sigma_k^{k-1})$ & A k-simplex,  element of $\mathsf{S}$. It holds $\sigma_{j}^{k-1} \face \sigma_{i}^{k}$. \\\vspace{2pt}
$r_i = (e_1, \ldots, e_{\vert r_i \vert} )$ & A ring,  element of $\pph$. $\vert r_i \vert$ is the size of the $i$-th ring. \\\vspace{2pt}
$\mathcal{B}({\sigma})$ & Boundary of $\sigma$.  \\\vspace{2pt}
$\mathcal{C}o({\sigma})$ &  Co-boundary of $\sigma$. \\ \vspace{2pt}
$\mathcal{N}_{\uparr}({\sigma})$ & Upper neighbourhood of $\sigma$.  \\\vspace{2pt}
$\mathcal{N}_{\doarr}(\sigma)$ & Lower neighbourhood of $\sigma$.  \\\vspace{2pt}
$\sigma \face \tau$ & Boundary relationship (i.e. $\sigma \in \mathcal{B}(\tau)$ ). \\\vspace{2pt}
$\mathcal{B}({\sigma, \tau})$ & Boundary elements in common between $\sigma$ and $\tau$.  \\\vspace{2pt}
$\mathcal{C}o({\sigma, \tau})$ & Co-boundary elements in common between $\sigma$ and $\tau$. \\
\bottomrule

\end{tabular}
\end{center}
\end{table}

\newpage

\begin{table}[!htb]
\begin{center}
\caption{Summary of Notations: {\em Functional Elements}. Notation used for functional aspects, including feature vectors, information exchange, and message passing operations.}

\label{tab:functional_notation}
\begin{tabular}{ll}
\hline
\multicolumn{2}{c}{\textbf{Functional Elements}}\\
\toprule
$\xnorm_v$ & Graph signal defined over a node $v$. \\\vspace{2pt}
$\mathbf{h}_v$ & Latent representation of a node $v$. \\\vspace{2pt}
$\mathcal{R}$ & Rewiring map. \\\vspace{2pt}
$\Snorm$ & Graph Shift Operator. \\\vspace{2pt}
$\mathsf{GNN}_{\theta}$ & Graph neural network parametrized by $\theta$. \\\vspace{2pt}
$\mathbf{W}_{\doarr}$. & Learnable weight matrix in $\R^{d' \times d}$. \\\vspace{2pt}
$\mathsf{m}$ & Message function. \\\vspace{2pt}
$\agg$ & Permutation invariant aggregation function. \\\vspace{2pt}
$\com$ & Update function.\\\vspace{2pt}
$\out$ & Readout function. \\\vspace{2pt}
$\lvert \partial \mathbf{h}_{v}^{(r)} / \partial \mathbf{h}_{u}^{(0)}\rvert$ & Sensitivity of node $v$ to the features of node $u$ after $r$ layers. \\\vspace{2pt}
$\xnorm_{\sigma}$ & Topological signal defined over a cell $\sigma$.\\\vspace{2pt}
$\mathbf{h}_{\sigma}$ & Latent representation of the cell $\sigma$.\\\vspace{2pt}
$\mathbf{h}_{\mathcal{B}}, \mathbf{h}_{\mathcal{C}o}, \mathbf{h}_{\uparr}, \mathbf{h}_{\doarr}$ & Boundary, Co-Boundary, Upper, Lower latent representations.\\\vspace{2pt}
$\mathsf{m}_{\mathcal{B}}, \mathsf{m}_{\mathcal{C}o}, \mathsf{m}_{\uparr}, \mathsf{m}_{\doarr}$ & Boundary, Co-Boundary, Upper, Lower message functions.\\\vspace{2pt}
$\mathbf{W}_{\uparr}, \mathbf{W}_{\doarr}$ & Upper and lower weight matrices in $\R^{d' \times d}$. \\\vspace{2pt}
$\mathbf{a}_{\uparr}, \mathbf{a}_{\doarr}$. & Upper and lower vectors of attention coefficients.\\\vspace{2pt}
$s_{\uparr}, s_{\doarr}$ & Upper and lower scoring functions. \\\vspace{2pt}
$a_{\uparr}, a_{\doarr}$ & Upper and lower attention functions. \\\vspace{2pt}
$\alpha^{\uparr}_{\sigma, \tau}, \alpha^{\doarr}_{\sigma, \tau}$ & Upper and lower weight coefficients between simplices/cells $\sigma$ and $\tau$. \\\vspace{2pt}
$\mathbf{h}_{\kph}$ & Latent representation of a simplicial complex $\kph$.\\\vspace{2pt}
$\mathbf{h}_{\cph}$ & Latent representation of a cell complex $\cph$.\\\vspace{2pt}
$\mathbf{h}_{\xph}$ & Latent representation of a discrete topological space $\xph$. \\
\bottomrule
\end{tabular}
\end{center}
\end{table}

%% file: Include/Backmatter/Appendix/Appendix_Bottlenecks.tex
\chapter{Appendix of On Oversquashing in MPNNs}\label{app:on_oversq}

\section{General preliminaries}\label{app:sec_preliminaries}
Assume a graph $\gph$ with nodes $ \mathsf{V}$ and edges $\mathsf{E}\subset \V \times \V$, to be simple, undirected, and connected. Let $n = \lvert \V \rvert$ and 
write $[n]:= \{1,\ldots,n\}$. Denote the adjacency matrix by $\mathbf{A}\in\R^{n\times n}$. Compute the degree of $v\in \V$ by $d_v = \sum_{u}A_{vu}$ and write 
$\mathbf{D} = \mathrm{diag}(d_1,\ldots,d_n)$. One can take different normalizations of $\mathbf{A}$, so write $\Anorm\in\R^{n\times n}$ for a Graph Shift Operator (GSO), i.e., an $n\times n$ matrix satisfying $\Anorm_{vu} \neq 0$ if and only if $(v,u)\in \mathsf{E}$; typically, $\Anorm \in \{\mathbf{A},\mathbf{D}^{-1}\mathbf{A},\mathbf{D}^{-1/2}\mathbf{A}\mathbf{D}^{-1/2}\}$. Finally, $d_{\gph}(v,u)$ is the {\bf shortest walk} ({\bf geodesic}) distance between nodes $v$ and $u$. 

\paragraph{Graph spectral properties: the eigenvalues.} The (normalized) graph Laplacian is defined as $\Lnorm = \mathbf{I} - \mathbf{D}^{-\frac{1}{2}}\mathbf{A}\mathbf{D}^{-\frac{1}{2}}$. This is a symmetric, positive semi-definite operator on $\gph$. Its eigenvalues can be ordered as $\lambda_0 < \lambda_1 \leq \ldots \leq \lambda_{n-1}$. The smallest eigenvalue $\lambda_0$ is always zero, with multiplicity given by the number of connected components of $\gph$ \citep{chung1997spectral}. Conversely, the largest eigenvalue $\lambda_{n-1}$ is always strictly smaller than $2$ whenever the graph is not bipartite. Finally, recall that the smallest, positive, eigenvalue $\lambda_1$ is known as the {\bf spectral gap}. Several of the proofs presented here rely on this quantity to provide convergence rates. Also, recall that the spectral gap is related to the Cheeger constant -- introduced in~\Cref{def:cheeger} -- of $\gph$ via the Cheeger inequality:
\begin{equation}\label{eq:cheeger_inequality}
    2\cheeg \geq \lambda_{1} > \frac{\cheeg^{2}}{2}.
\end{equation}

\paragraph{Graph spectral properties: the eigenvectors.}
Throughout this section, let $\{\eigen_{\ell}\}$ be a family of orthonormal eigenvectors of $\Lnorm$. In particular, note that the eigenspace associated with $\lambda_0$ represents the space of signals that respect the graph topology the most (i.e. the smoothest signals), so that is possible to write $(\eigen_0)_v = \sqrt{d_v}/2\lvert\E\rvert$, for any $v\in\V$.

From now on, assume that the graph is {\em not} bipartite, so that $\lambda_{n-1} < 2$. 
Let $\mathbf{H}^{(0)}\in\R^{n\times p}$ be the matrix representation of node {\em features}, with $p$ denoting the hidden dimension. Features of node $v$ produced by layer $l$ of an $\MPNN$ are denoted by $\mathbf{h}_{v}^{(l)}$ and write their components as $(\mathbf{h}_{v}^{(l)})^{\alpha} := h_{v}^{(l),\alpha}$, for $\alpha \in [p]$.

\paragraph{Einstein summation convention.} To ease notations when deriving the bounds on the Jacobian, the proof below often rely on Einstein summation convention, meaning that, unless specified otherwise, sums are always repeated across indices: for example, when writing terms like $x_{\alpha}y^{\alpha}$, the symbol $\sum_{\alpha}$ is left implicit.

\section{\texorpdfstring{Proofs of \Cref{sec:width}}{}}\label{app:sec_width}

This Section demonstrates the results in~\Cref{sec:width}. In fact, it will be derived a sensitivity bound far more general than \Cref{cor:bound_MLP_MPNN} that, in particular, extends to $\MPNN$s that can stack multiple layers (MLPs) in the aggregation phase. Let's introduce a class of $\MPNN$s of the form:
\begin{equation}\label{eq:isotropic_MPNN}
    \mathbf{h}_{v}^{(l)} = \mathsf{up}^{(l)}\Big(\mathsf{rs}^{(l)}(\mathbf{h}_{v}^{(l-1)}) + \mathsf{mp}^{(l)}\Big(\sum_{u}\Anorm_{vu}\mathbf{h}_{u}^{(l-1)}\Big) \Big)
\end{equation}
\noindent for learnable update, residual, and message-passing maps $\mathsf{up}^{(l)},\mathsf{rs}^{(l)},\mathsf{mp}^{(l)}:\R^{p}\rightarrow \R^{p}$. Note that~\Cref{eq:isotropic_MPNN} includes common $\MPNN$s like $\mathsf{GCN}$ \citep{kipf2017graph}, $\mathsf{SAGE}$ \citep{hamilton2017inductive}, and $\mathsf{GIN}$ \citep{xu2019powerful}, where $\Anorm$ is $\mathbf{D}^{-1/2}\mathbf{A}\mathbf{D}^{-1/2}$, $\mathbf{D}^{-1}\mathbf{A}$ and $\mathbf{A}$, respectively. An $\MPNN$ usually has Lipschitz maps, with Lipschitz constants typically depending on regularization of the weights to promote generalization. An $\MPNN$ as in~\Cref{eq:isotropic_MPNN} is $(c_{\mathsf{up}},c_{\mathsf{rs}},c_{\mathsf{mp}})$-regular, if for $t\in[m]$ and $\alpha\in[p]$, it holds
\begin{align*}
    \| \nabla (\mathsf{up}^{(l)})^{\alpha}\|_{L_{1}} \leq c_{\mathsf{up}}, \quad \| \nabla (\mathsf{rs}^{(l)})^{\alpha}\|_{L_{1}} \leq c_{\mathsf{rs}}, \quad
    \| \nabla (\mathsf{mp}^{(l)})^{\alpha}\|_{L_{1}} \leq  c_{\mathsf{mp}}.
\end{align*}
\noindent As in~\cite{xu2018representation, topping2022understanding}, the interest is on the propagation of information in the $\MPNN$ via the Jacobian of node features after $m$ layers. A small derivative of $\mathbf{h}_v^{(m)}$ with respect to $\mathbf{h}_u^{(0)}$ means that -- {\bf at the first-order} -- the representation at node $v$ is mostly insensitive to the information contained at $u$ (e.g. its atom type, if $\gph$ is a molecule).
\begin{theorem}\label{thm:bound_general} Given a $(c_{\mathsf{up}},c_{\mathsf{rs}},c_{\mathsf{mp}})$-regular $\MPNN$ for $m$ layers and nodes $v,u\in\mathsf{V}$, it holds 
\begin{equation}
   \left\|\frac{\partial \mathbf{h}_{v}^{(m)}}{\partial \mathbf{h}_{u}^{(0)}}\right\|_{L_1} \leq p\cdot c^{m}_{\mathsf{up}}\left(\left(c_{\mathsf{rs}}\mathbf{I} +  c_{\mathsf{mp}}\Anorm\right)^{m}\right)_{vu}.
\end{equation}
\end{theorem}

\begin{proof}
The result above will be proven by induction on the number of layers $m$. Fix $\alpha,\beta \in [p]$. In the case of $m=1$, get (omitting to write the arguments where the maps are being evaluated, and {\bf using the Einstein summation convention over repeated indices}):
\begin{equation*}
    \Big| \frac{\partial h_v^{(1),\alpha}}{\partial h_u^{(0),\beta}}\Big| = \Big| \partial_{p}\mathsf{up}^{(0),\alpha}\Big(\partial_r\mathsf{rs}^{(0),p}\frac{\partial h_v^{(0),r}}{\partial h_u^{(0),\beta}} + \partial_q\mathsf{mp}^{(0),p}\Anorm_{vz}\frac{\partial h_z^{(0),q}}{\partial h_u^{(0),\beta}}\Big) \Big|,
\end{equation*}
\noindent which can be readily reduced to
\begin{equation*}
     \Big| \frac{\partial h_v^{(1),\alpha}}{\partial h_u^{(0),\beta}}\Big| = \Big| \partial_{p}\mathsf{up}^{(0),\alpha}\Big(\partial_\beta\mathsf{rs}^{(0),p}\Lnorm_{vu} + \partial_\beta\mathsf{mp}^{(0),p}\Anorm_{vu}\Big) \Big| \leq c_{\mathsf{up}}\left(c_{\mathsf{rs}}\mathbf{I}+ c_{\mathsf{mp}}\Anorm\right)_{vu},
\end{equation*}
\noindent thanks to the Lipschitz bounds on the $\MPNN$, which confirms the case of a single layer (i.e. $m=1$). Also, assume the bound to be satisfied for $m$ layers and use induction to derive
\begin{align*}
    \Big| \frac{\partial h_v^{(m+1),\alpha}}{\partial h_u^{(0),\beta}}\Big| &= \Big| \partial_{p}\mathsf{up}^{(m),\alpha}\Big(\partial_r\mathsf{rs}^{(m),p}\frac{\partial h_v^{(m),r}}{\partial h_u^{(0),\beta}} + \partial_q\mathsf{mp}^{(m),p}\Anorm_{vz}\frac{\partial h_z^{(m),q}}{\partial h_u^{(0),\beta}}\Big) \Big| \\
    &\leq \Big|\partial_{p}\mathsf{up}^{(m),\alpha}\Big|\Big(\left\vert \partial_r\mathsf{rs}^{(m),p}\right\vert \left(c_{\mathsf{up}}^{m}\left(\left(c_{\mathsf{rs}}\mathbf{I} + c_{\mathsf{mp}}\Anorm\right)^{m}\right)_{vu}\right) + \Big|\partial_q\mathsf{mp}^{(m),p} \Big| \Anorm_{vz}\left(c_{\mathsf{up}}^{m}\left(\left(c_{\mathsf{rs}}\mathbf{I} + c_{\mathsf{mp}}\Anorm\right)^{m}\right)_{zu}\right) \Big) \\
    &\leq \Big|\partial_{p}\mathsf{up}^{(m),\alpha}\Big|\Big(c_{\mathsf{rs}} \left(c_{\mathsf{up}}^{m}\left(\left(c_{\mathsf{rs}}\mathbf{I} + c_{\mathsf{mp}}\Anorm\right)^{m}\right)_{vu}\right) + c_{\mathsf{mp}} \Anorm_{vz}\left(c_{\mathsf{up}}^{m}\left(\left(c_{\mathsf{rs}}\mathbf{I} + c_{\mathsf{mp}}\Anorm\right)^{m}\right)_{zu}\right) \Big) \\
    &\leq c_{\mathsf{up}}^{m+1}\left(c_{\mathsf{rs}}\left(\left(c_{\mathsf{rs}}\mathbf{I} + c_{\mathsf{mp}}\Anorm\right)^{m}\right)_{vu} + c_{\mathsf{mp}}\Anorm_{vz}\left(\left(c_{\mathsf{rs}}\mathbf{I} + c_{\mathsf{mp}}\Anorm\right)^{m}\right)_{vu}\right) \\
    & = c_{\mathsf{up}}^{m+1}\Big(\left(c_{\mathsf{rs}}\mathbf{I} + c_{\mathsf{mp}}\Anorm\right)^{m+1}\Big)_{vu},
\end{align*}
\noindent using the Lipschitz bounds on the maps $\mathsf{up},\mathsf{rs},\mathsf{mp}$. This completes the induction argument.
\end{proof}

\noindent From now on the focus will be on the class of $\MPNN$ adopted in~\Cref{sec:on_oversq}, whose layer are report below for convenience:
\begin{equation*}
    \mathbf{h}_{v}^{(l+1)} = \up\Big(c_{\rs}\W_{\rs}^{(l)} \mathbf{h}_{v}^{(l)} + c_{\mpas}\W_{\mpas}^{(l)}\sum_{u}\Anorm_{vu}\mathbf{h}^{(l)}_{u}\Big).
\end{equation*}
\noindent The general argument can be adapted to derive \Cref{cor:bound_MLP_MPNN}.

\begin{proof}[Proof of \Cref{cor:bound_MLP_MPNN}] One can follow the steps in the proof of \Cref{thm:bound_general} and, again, proceed by induction. The case $m=1$ is straightforward, so consider the inductive step and assume the bound to hold for $m$ arbitrary. Given $\alpha,\beta \in [p]$, it holds
\begin{align*}
 \Big| \frac{\partial h_v^{(m+1),\alpha}}{\partial h_u^{(0),\beta}}\Big| &\leq \lvert \up' \rvert \Big( c_{\rs}\left\vert(\W_{\rs})^{(m)}_{\alpha\gamma}\right\vert\Big|\frac{\partial h_v^{(m),\gamma}}{\partial h_u^{(0),\beta}}\Big| + c_{\mpas}\left\vert (\W_{\mpas})^{(m)}_{\alpha\gamma}\right\vert \Anorm_{vz} \Big| \frac{\partial h_z^{(m),\gamma}}{\partial h_u^{(0),\beta}}\Big|\Big) \\
 &\leq c_{\up}w\left(c_{\rs}\left\|  \frac{\partial \mathbf{h}_v^{(m)}}{\partial \mathbf{h}_u^{(0)}}\right\|_{L_1} + c_{\mpas}\Anorm_{vz}\left\|  \frac{\partial \mathbf{h}_z^{(m)}}{\partial \mathbf{h}_u^{(0)}}\right\|_{L_1}  \right) \\
 &\leq c_{\sigma}w\left(c_{\sigma}wp\right)^{m}\left( c_{\rs}\left(\left(c_{\rs}\mathbf{I} + c_{\mpas}\Anorm\right)^{m}\right)_{vu} + c_{\mpas}\Anorm_{vz}\left(\left(c_{\rs}\mathbf{I} + c_{\mpas}\Anorm\right)^{m}\right)_{zu}\right) \\
 &\leq c_{\sigma}w\left(c_{\sigma}wp\right)^{m} \Big(\left(c_{\rs}\mathbf{I} + c_{\mpas}\Anorm\right)^{m+1}\Big)_{vu}.
\end{align*}
\noindent By summing over $\alpha$ on the left will conclude the proof (this will generate an extra $p$ factor on the right hand side).

\end{proof}
\section{\texorpdfstring{Proofs of \Cref{sec:depth}}{}}\label{app:sec_depth}

{\bf Convention:} From now on always consider $\Anorm = \mathbf{D}^{-1/2}\mathbf{A}\mathbf{D}^{-1/2}$. The bounds in this Section extend easily to $\mathbf{D}^{-1}\mathbf{A}$ in light of the similarity of the two matrices since $\Anorm^{k} = \mathbf{D}^{1/2}\left(\mathbf{D}^{-1}\mathbf{A}\right)^{k}\mathbf{D}^{-1/2}$. For the unnormalized matrix $\mathbf{A}$ instead, things are slightly more subtle. In principle, this matrix is not normalized, and in fact, the entry $(\mathbf{A}^{k})_{vu}$ coincides with the number of walks from $v$ to $u$ of length $k$. In general, this will not lead to bounds decaying exponentially with the distance. However, in expectation over the computational graph as in \citet{xu2018representation}, Appendix A of \citet{topping2022understanding} and \Cref{sec:topology}, one finds that nodes at smaller distance will still have sensitivity exponentially larger than nodes at large distance. This is also confirmed by $\mathsf{Graph}$ $\mathsf{Transfer}$ synthetic experiments, where $\mathsf{GIN}$ struggles with long-range dependencies (in fact, even slightly more than $\mathsf{GCN}$, which uses the symmetrically normalized adjacency $\Anorm$).

A sharper bound for~\Cref{eq:cor_distance} will be proven foreword, it is important to notice that it contains~\Cref{cor:over-squasing_distance} as a particular case.

\begin{theorem}
Given an $\MPNN$ as in~\Cref{eq:MPNN_mlp}, let $v,u\in\mathsf{V}$ be at distance $r$. Let $c_{\up}$ be the Lipschitz constant of $\sigma$, $w$ the maximal entry-value over all weight matrices, $d_{\mathrm{min}}$ be the minimal degree, and $\gamma_{\ell}(v,u)$ be the number of walks from $v$ to $u$ of maximal length $\ell$. For any $0 \leq k < r$, it holds

\begin{equation}
   \left\| \frac{\partial \mathbf{h}_{v}^{(r+k)}}{\partial \mathbf{h}_{u}^{(0)}}\right\|_{L_1} \leq \gamma_{r+k}(v,u) (c_{\up}(c_{\rs} + c_{\mpas})w p(k+1))^k \Big(\frac{2c_{\up}wpc_{\mpas}}{d_{\mathrm{min}}}\Big)^r.
\end{equation}
\end{theorem}
\begin{proof}
    Fix $v,u\in\V$ as in the statement and let $0 \leq k < r$. By using the sensitivity bounds in \Cref{cor:bound_MLP_MPNN} and writing that 
    \begin{equation*}
         \left\| \frac{\partial \mathbf{h}_{v}^{(r+k)}}{\partial \mathbf{h}_{u}^{(0)}}\right\|_{L_1} \leq \left(c_{\up}wp\right)^{r+k}\Big(\left(c_{\rs}\mathbf{I} + c_{\mpas}\Anorm\right)^{r+k}\Big)_{vu} = \left(c_{\up}wp\right)^{r+k}\sum_{i = 0}^{r+k}\binom{r+k}{i}c_{\rs}^{r+k - i}c_{\mpas}^{i}(\Anorm^{i})_{vu}.
    \end{equation*}
\noindent Since nodes $v,u$ are at distance $r$, the first $r$ terms of the sum above vanish. Since $\Anorm = \mathbf{D}^{-1/2}\mathbf{A}\mathbf{D}^{-1/2}$, the polynomial in the previous equation can be bounded by
\begin{align*}
    \sum_{i = 0}^{r+k}\binom{r+k}{i}c_{\rs}^{r+k - i}c_{\mpas}^{i}(\Anorm^{i})_{vu} &= \sum_{i = r}^{r+k}\binom{r+k}{i}c_{\rs}^{r+k - i}c_{\mpas}^{i}(\Anorm^{i})_{vu} \leq \gamma_{r+k}(v,u)\sum_{i = r}^{r+k}\binom{r+k}{i}c_{\rs}^{r+k - i}\left(\frac{c_{\mpas}}{d_{\mathrm{min}}}\right)^{i} \\
    &= \gamma_{r+k}(v,u) \sum_{q = 0}^{k}\binom{r+k}{r+q}c_{\rs}^{k - q}\left(\frac{c_{\mpas}}{d_{\mathrm{min}}}\right)^{r + q} \\ 
    &= \gamma_{r+k}(v,u)\left(\frac{c_{\mpas}}{d_{\mathrm{min}}}\right)^{r} \sum_{q = 0}^{k}\binom{r+k}{r+q}c_{\rs}^{k - q}\left(\frac{c_{\mpas}}{d_{\mathrm{min}}}\right)^{q}. 
\end{align*}
\noindent A simple estimate for 
\begin{align*}
    \binom{r+k}{r+q} &= \frac{(r+k)(r-1+k)\cdots(1+k)}{(r+q)(r-1+q)\cdots(1+q)}\binom{k}{q} \leq \frac{(r+k)(r-1+k)\cdots(1+k)}{r!}\binom{k}{q}  \\
    &\leq \left(1 + \frac{k}{r}\right)\cdots (1+k) \binom{k}{q} \leq \left(1+\frac{k}{k+1}\right)^{r-k}(1+k)^{k}\binom{k}{q} \\
\end{align*}
\noindent can be provided by expanding the polynomial above can be as:
\begin{align*}
 \sum_{i = 0}^{r+k}\binom{r+k}{i}c_{\rs}^{r+k - i}c_{\mpas}^{i}(\Anorm^{i})_{vu} &\leq \gamma_{r+k}(v,u)\left(1+\frac{k}{k+1}\right)^{r-k}(1+k)^{k}\left(\frac{c_{\mpas}}{d_{\mathrm{min}}}\right)^{r} \sum_{q = 0}^{k}\binom{k}{q}c_{\rs}^{k - q}\left(\frac{c_{\mpas}}{d_{\mathrm{min}}}\right)^{q} \\
 &=  \gamma_{r+k}(v,u)\left(\frac{(1+k)^{2}}{2k+1}\left(c_{\rs} + \frac{c_{\mpas}}{d_{\mathrm{min}}}\right)\right)^{k}\left(\left(1+\frac{k}{k+1}\right)\frac{c_{\mpas}}{d_{\mathrm{min}}}\right)^{r} \\
 &\leq \gamma_{r+k}(v,u)\left(\frac{(1+k)^{2}}{2k+1}\left(c_{\rs} + \frac{c_{\mpas}}{d_{\mathrm{min}}}\right)\right)^{k}\left(\frac{2c_{\mpas}}{d_{\mathrm{min}}}\right)^{r}.
 \end{align*}
 \noindent Combining all the ingredients together, the bound can be written as 
 \begin{align*}
     \left| \frac{\partial \mathbf{h}_{v}^{(r+k)}}{\partial \mathbf{h}_{u}^{(0)}}\right\|_{L_1} &\leq\gamma_{r+k}(v,u) \left(c_{\up}wp\right)^{r+k}\left(\frac{(1+k)^{2}}{2k+1}\left(c_{\rs} + \frac{c_{\mpas}}{d_{\mathrm{min}}}\right)\right)^{k}\left(\frac{2c_{\mpas}}{d_{\mathrm{min}}}\right)^{r} \\
     & = \gamma_{r+k}(v,u)\left(c_{\up}wp\,\frac{(1+k)^{2}}{2k+1}\left(c_{\rs} + \frac{c_{\mpas}}{d_{\mathrm{min}}}\right)\right)^{k}\left(\frac{2c_{\up}wpc_{\mpas}}{d_{\mathrm{min}}}\right)^{r} \\
     &\leq \gamma_{r+k}(v,u)\left(c_{\up}\left(c_{\rs} + c_{\mpas}\right)wp(1+k)\right)^{k}\left(\frac{2c_{\up}wpc_{\mpas}}{d_{\mathrm{min}}}\right)^{r},
 \end{align*}
 \noindent which completes the proof. Notice that this also proves~\Cref{cor:over-squasing_distance}.
\end{proof}

\subsection{Vanishing gradients result}

Here it will be reported and proven a more explicit version of~\Cref{thm:vanishing}.

\begin{theorem}[\textbf{Vanishing gradients}]\label{thm:main_vanishing_app} Consider an $\MPNN$ as in Eq.~\eqref{eq:MPNN_mlp} for $m$ layers with a quadratic loss $\mathcal{L}$. Assume that (i) $\sigma$ has Lipschitz constant $c_{\up}$ and $\sigma(0) = 0$, and (ii) that all weight matrices have spectral norm bounded by $\mu > 0$. 
Given any weight $\theta$ entering a layer $k$, there exists a constant $C > 0$ independent of $m$, such that
\begin{align}
    \left\vert \frac{\partial \mathcal{L}}{\partial \theta}\right\vert &\leq C\left(c_{\up}\mu(c_{\rs} + c_{\mpas})\right)^{m-k}\left(1  +\left(c_{\up}\mu(c_{\rs} + c_{\mpas})\right)^{m}\right),
\end{align}
\noindent where $\lvert\lvert \mathbf{H}^{(0)}\rvert\rvert_{F}$ is the Frobenius norm of the input node features.
\end{theorem}
\begin{proof}
Consider a quadratic loss $\mathcal{L}$ of the form 
\begin{equation*}
    \mathcal{L}(\mathbf{H}^{(m)}) = \frac{1}{2}\sum_{v\in\V}\| \mathbf{h}_v^{(m)} - \mathbf{y}_v\|^{2},
\end{equation*}
\noindent and let $\mathbf{Y}$ represent the node ground-truth values. Given a weight $\theta$ entering layer $k < m$, it is possible to write the gradient of the loss as
\begin{equation*}
    \Big| \frac{\partial \mathcal{L}(\mathbf{H}^{(m)})}{\partial\theta}\Big| = \Big| \sum_{v,u\in\V}\sum_{\alpha,\beta\in[p]}\frac{\partial \mathcal{L}}{\partial h_v^{(m),\alpha}}\frac{\partial h_v^{(m),\alpha}}{\partial h_u^{(k),\beta}}\frac{\partial h_u^{(k),\beta}}{\partial \theta}\Big|.
\end{equation*}
\noindent Once $k$ is fixed,, the term $\lvert \partial h_u^{(k),\beta} / \partial \theta \rvert$ is independent of $m$ and is possible to bound it by some constant $C$. Since the loss is quadratic, to bound $\partial \mathcal{L} / \partial h_v^{(m),\alpha}$, it suffices to bound the solution of the $\MPNN$ after $m$ layers. First, use the Kronecker product formalism to rewrite the $\MPNN$-update in matricial form as 
\begin{equation}
    \mathbf{H}^{(m)} = \up\left(\left(c_{\rs}\OMEga^{(m)}\otimes \mathbf{I} + c_{\mpas}\W^{(m)}\otimes \Anorm\right)\mathbf{H}^{(m-1)}\right).
\end{equation}
\noindent Thanks to the Lipschitzness of $\up$ and the requirement $\up(0) = 0$, is possible to derive
\begin{equation*}
  \| \mathbf{H}^{(m)} \|_{F} \leq c_{\up}\| c_{\rs}\OMEga^{(m)}\otimes \mathbf{I} + c_{\mpas}\W^{(m)}\otimes \Anorm\|_{2} \|\mathbf{H}^{(m-1)}\|_{F},
\end{equation*}
\noindent where $F$ indicates the Frobenius norm. Since the largest singular value of $\mathbf{B}\otimes \mathbf{C}$ is bounded by the product of the largest singular values, it is easy to deduce that -- recall that the largest eigenvalue of $\Anorm = \mathbf{D}^{-1/2}\mathbf{A}\mathbf{D}^{-1/2}$ is $1$:
\begin{equation}\label{app:eq:norm_bounded}
    \| \mathbf{H}^{(m)} \|_{F} \leq c_{\up}\mu(c_{\rs} + c_{\mpas})\|\mathbf{H}^{(m-1)}\|_{F} \leq (c_{\up}\mu(c_{\rs} + c_{\mpas}))^{m}\|\mathbf{H}^{(0)}\|_{F},
\end{equation}
\noindent which affords a control of the gradient of the loss w.r.t. the solution at the final layer being the loss quadratic. Then, find
\begin{align}
     \Big| \frac{\partial \mathcal{L}(\mathbf{H}^{(m)})}{\partial\theta}\Big| &\leq C\Big| \sum_{v,u\in\V}\sum_{\alpha,\beta\in[p]}\frac{\partial \mathcal{L}}{\partial h_v^{(m),\alpha}}\frac{\partial h_v^{(m),\alpha}}{\partial h_u^{(k),\beta}}\Big| \notag \\ 
     &\leq C \sum_{v,u\in\V}\sum_{\beta\in[p]}\left\|\frac{\partial \mathcal{L}}{\partial \mathbf{h}_v^{(m)}}\right\|\left\|\frac{\partial \mathbf{h}_v^{(m)}}{\partial h_u^{(k),\beta}}\right\| \notag \\
     &\leq C \sum_{v,u\in\V}\sum_{\beta\in[p]}\Big(\|\mathbf{H}^{(m)}\|_{F} + \|\mathbf{Y}\|_{F}\Big)\left\|\frac{\partial \mathbf{h}_v^{(m)}}{\partial h_u^{(k),\beta}}\right\| \notag \\
     &\leq C \sum_{v,u\in\V}\sum_{\beta\in[p]}\Big((c_{\up}\mu(c_{\rs} + c_{\mpas}))^{m}\|\mathbf{H}^{(0)}\|_{F} + \|\mathbf{Y}\|_{F}\Big)\left\|\frac{\partial \mathbf{h}_v^{(m)}}{\partial h_u^{(k),\beta}}\right\| \label{eq:app:proof:vanishing:1}
\end{align}
\noindent where in the last step used~\Cref{app:eq:norm_bounded}. Now it will be given a {\bf new} bound on the sensitivity -- {\em differently from the analysis in earlier Sections. 
Given that is necessary to integrate over all possible pairwise contributions to compute the gradient of the loss, the topological information depending on the choice of $v,u$ is no loger needed.} The idea below, is to apply the Kronecker product formalism to derive a single operator in the tensor product of feature and graph space acting on the Jacobian matrix -- this allows to derive much sharper bounds. Note that, once a node $u$ is fixed and a $\beta\in[p]$, is possible to write
\begin{align*}
\left\|\frac{\partial \mathbf{H}^{(m)}}{\partial h_u^{(k),\beta}}\right\|^{2} &\leq \sum_{v\in\V}\sum_{\alpha\in [p]}c_{\up}^{2}\Big( c_{\rs}\OMEga^{(m)}_{\alpha\gamma}\frac{\partial h_v^{(m-1),\gamma}}{\partial h_{u}^{(k),\beta}} + c_{\mpas}\W^{(m)}_{\alpha\gamma}\Anorm_{vz}\frac{\partial h_z^{(m-1),\gamma}}{\partial h_{u}^{(k),\beta}}  \Big)^{2} \\
&= c_{\up}^{2}\sum_{v\in\V}\sum_{\alpha\in [p]}\Big(\Big(c_{\rs}\OMEga^{(m)}\otimes \mathbf{I} + c_{\mpas}\W^{(m)}\otimes \Anorm \Big)\frac{\partial \mathbf{H}^{(m-1)}}{\partial h_u^{(k),\beta}}\Big)_{v,\alpha}^{2} \\
&\leq c_{\up}^{2}\| c_{\rs}\OMEga^{(m)}\otimes \mathbf{I} + c_{\mpas}\W^{(m)}\otimes \Anorm \|_{2}^{2} \left\|\frac{\partial \mathbf{H}^{(m-1)}}{\partial h_u^{(k),\beta}}\right\|_{F}^{2} 
\end{align*}
\noindent meaning that 
\begin{equation*}
    \left\|\frac{\partial \mathbf{H}^{(m)}}{\partial h_u^{(k),\beta}}\right\| \leq (c_{\up}\mu(c_{\rs} + c_{\mpas}))^{m-k},
\end{equation*}
\noindent where (i) the largest singular value of the weight matrices is $\mu$, (ii) that the largest eigenvalue of $c_{\rs}\mathbf{I} + c_{\mpas}\Anorm$ is $c_{\rs} + c_{\mpas}$ (as follows from $\Anorm = \mathbf{I} - \Lnorm$, and the spectral analysis of $\Lnorm$), (iii) that $\| \partial\mathbf{H}^{(k)} / \partial h_u^{(k),\beta} \| = 1$. The proof is complexed once the term $\| \mathbf{Y}\|$ is absorbed in the constant $C$ in~\Cref{eq:app:proof:vanishing:1}.
\end{proof}

\section{\texorpdfstring{Proofs of \Cref{sec:topology}}{}}\label{app:sec_topology}
\noindent This Section considers the convolutional family of $\MPNN$ in~\Cref{eq:mpnn_simplified}.
Before proving the main results of this Section, it is necessary to comment the main assumption on the nonlinearity and formulate it more explicitly. Take $k < m$. When the sensitivity of $\mathbf{h}_v^{(m)}$ w.r.t $\mathbf{h}_u^{(k)}$ is computed, it yields a sum of different terms over all possible paths from $v$ to $u$ of length $m-k$. In this case, the derivative of 
$\mathsf{ReLU}$ acts as a Bernoulli variable evaluated along all these possible paths. Similarly to \citet{kawaguchi2016deep,xu2018representation}, it is necessary that  following assumption holds:


\begin{assumption}\label{assumption} Assume that all paths in the computation graph
of the model are activated with the same probability of
success $\rho$. 
The expectation $\mathbb{E}[\partial \mathbf{h}_{v}^{(m)} / \partial \mathbf{h}_{u}^{(k)}]$, means taking the average over such Bernoulli variables. 
\end{assumption}

\noindent Thanks to~\Cref{assumption}, is possible to follow the same argument in the proof of Theorem 1 in~\cite{xu2018representation} to derive 
\begin{equation*}
   \mathbb{E}\left[\frac{\partial \mathbf{h}_{v}^{(m)}}{\partial\mathbf{h}_{u}^{(k)}}\right] = \rho\prod_{s = k+1}^{m}\W^{(s)}(\oper^{m-k})_{vu}.
\end{equation*}

\noindent Now, let's proceed to prove the relation between sensitivity analysis and access time.

\begin{proof}[Proof of \Cref{thm:access}]
Under \Cref{assumption}, The term $\mathbf{J}_k^{(m)}(v,u)$ can be rewritten as 
\begin{align*}
    \mathbb{E}\Big[\mathbf{J}_k^{(m)}(v,u)\Big] &= \mathbb{E}\Big[ \frac{1}{d_v}\frac{\partial \mathbf{h}_{v}^{(m)}}{\partial \mathbf{h}_{v}^{(k)}} - \frac{1}{\sqrt{d_v d_u}}\frac{\partial \mathbf{h}_{v}^{(m)}}{\partial \mathbf{h}_{u}^{(k)}}\Big] \\ 
    &= \rho\prod_{s = k+1}^{m}\W^{(s)}\Big(\frac{1}{d_v}(\oper^{m-k})_{vv} - \frac{1}{\sqrt{d_v d_u}}(\oper^{m-k})_{vu} \Big).
\end{align*}
\noindent 
\noindent Since $\oper = c_{\rs}\mathbf{I} + c_{\mpas}\mathbf{D}^{-1/2}\mathbf{A}\mathbf{D}^{-1/2}$, the spectral decomposition of the graph Laplacian can be employed -- see the conventions and notations introduced in \Cref{app:sec_preliminaries} -- to write
\begin{equation*}
    \oper = \sum_{\ell = 0}^{n-1}\left(c_{\rs} + c_{\mpas}(1-\lambda_{\ell})\right)\eigen_{\ell}\eigen_{\ell}^\top,
\end{equation*}
\noindent where $\Lnorm \eigen_\ell = \lambda_\ell \eigen_\ell$. Therefore, is possible to bound (in {\bf expectation}) the Jacobian obstruction by
\begin{align*}
    \obst^{(m)}(v,u) &= \sum_{k=0}^{m}\| \mathbf{J}^{(m)}_k(v,u) \| \geq \sum_{k = 0}^{m} \rho\nu^{m-k}\Big|\Big(\frac{1}{d_v}(\oper^{m-k})_{vv} - \frac{1}{\sqrt{d_v d_u}}(\oper^{m-k})_{vu} \Big)\Big| \\
    &\geq \rho \Big|\sum_{k = 0}^{m} \nu^{m-k}\Big(\frac{1}{d_v}(\oper^{m-k})_{vv} - \frac{1}{\sqrt{d_v d_u}}(\oper^{m-k})_{vu} \Big)\Big| \\
    &= \rho \Big|\sum_{k = 0}^{m} \nu^{m-k}\sum_{\ell = 0}^{n-1}\Big(c_{\rs} + c_{\mpas}(1-\lambda_{\ell})\Big)^{m-k}\left(\frac{\eigen^{2}_\ell(v)}{d_v} - \frac{\eigen_\ell(v)\eigen_\ell(u)}{\sqrt{d_u d_v}}\right)\Big| \\
    &= \rho\Big| \sum_{\ell = 0}^{n-1}\Big(\sum_{k = 0}^{m} \nu^{m-k}\left(c_{\rs} + c_{\mpas}(1-\lambda_{\ell})\right)^{m-k}\Big)\Big(\frac{\eigen^{2}_\ell(v)}{d_v} - \frac{\eigen_\ell(v)\eigen_\ell(u)}{\sqrt{d_u d_v}}\Big)\Big|\\
    &= \rho\Big|\sum_{\ell = 1}^{n-1}\sum_{k = 0}^{m} \left(\nu(c_{\rs} + c_{\mpas}(1-\lambda_{\ell}))\right)^{m-k}\Big(\frac{\eigen^{2}_\ell(v)}{d_v} - \frac{\eigen_\ell(v)\eigen_\ell(u)}{\sqrt{d_u d_v}}\Big)\Big|,
\end{align*}
\noindent where the last equality uses $\eigen_0(v) = \sqrt{d_v}/(2\lvert \E\rvert)$ for each $v\in\V$. By expanding the geometric sum using the assumption $\nu(c_{\rs} + c_{\mpas}) = 1$ and writing
\begin{equation*}
\obst^{(m)}(v,u) \geq \rho\Big|\sum_{\ell = 1}^{n-1}\frac{1 - \left(\nu(c_{\rs} + c_{\mpas}(1-\lambda_\ell))\right)^{m+1}}{1 - \nu(c_{\rs} + c_{\mpas}) +\nu c_{\mpas}\lambda_\ell}\Big(\frac{\eigen^{2}_\ell(v)}{d_v} - \frac{\eigen_\ell(v)\eigen_\ell(u)}{\sqrt{d_u d_v}}\Big)\Big|;
\end{equation*}
\noindent since $\nu(c_{\rs} + c_{\mpas}) = 1$, is possible to simplify the lower bound as
\begin{equation*}
\obst^{(m)}(v,u) \geq \rho\Big|\sum_{\ell = 1}^{n-1}\frac{1}{\nu c_{\mpas}\lambda_\ell} \Big(\frac{\eigen^{2}_\ell(v)}{d_v} - \frac{\eigen_\ell(v)\eigen_\ell(u)}{\sqrt{d_u d_v}}\Big)\Big| - \rho\Big| \sum_{\ell = 1}^{n-1}\frac{\left(\nu(c_{\rs} + c_{\mpas}(1-\lambda_\ell))\right)^{m+1}}{\nu c_{\mpas}\lambda_\ell}\Big(\frac{\eigen^{2}_\ell(v)}{d_v} - \frac{\eigen_\ell(v)\eigen_\ell(u)}{\sqrt{d_u d_v}}\Big) \Big|. 
\end{equation*}
\noindent By \citet[Theorem 3.1]{lovasz1993random}, the first term is equal to $\lvert (\nu c_{\mpas})^{-1}\mathsf{t}(u,v)/2\lvert\E\rvert \rvert$ which is a positive number. Concerning the second term, recall that the eigenvalues of the graph Laplacian are ordered from smallest to largest and that $\eigen_\ell$ is a unit vector, so 
\begin{equation*}
\obst^{(m)}(v,u) \geq \frac{\rho}{\nu c_{\mpas}}\frac{\mathsf{t}(u,v)}{2\lvert \E\rvert} - \frac{\rho(1 - \nu c_{\mpas}\lambda^\ast)^{m+1}}{\nu c_{\mpas}\lambda_1}\frac{n-1}{d_{\mathrm{min}}},
\end{equation*}
\noindent with $\lambda^\ast$ such that $\lvert 1 - \lambda^\ast \rvert = \max_{\ell > 0}\lvert 1 - \lambda_\ell \rvert$ which completes the proof.

\end{proof}
\begin{proof}[Proof of \Cref{thm:effective_resistance}]
This proof follows the same strategy used in the proof of \Cref{thm:access}. Under \Cref{assumption}, the term $\mathbf{J}_k^{(m)}(v,u)$ can be written as 
\begin{align*}
    \mathbb{E}\Big[\mathbf{J}_k^{(m)}(v,u)\Big] &= \mathbb{E}\left[ \frac{1}{d_v}\frac{\partial \mathbf{h}_{v}^{(m)}}{\partial \mathbf{h}_{v}^{(k)}} - \frac{1}{\sqrt{d_v d_u}}\frac{\partial \mathbf{h}_{v}^{(m)}}{\partial \mathbf{h}_{u}^{(k)}} + \frac{1}{d_u}\frac{\partial \mathbf{h}_{u}^{(m)}}{\partial \mathbf{h}_{u}^{(k)}}  - \frac{1}{\sqrt{d_v d_u}}\frac{\partial \mathbf{h}_{u}^{(m)}}{\partial \mathbf{h}_{v}^{(k)}}\right] \\
    &= \rho\prod_{s = k+1}^{m}\W^{(s)}\Big(\frac{1}{d_v}(\oper^{m-k})_{vv} + \frac{1}{d_u}(\oper^{m-k})_{uu} - 2(\oper^{m-k})_{vu} \Big)
\end{align*}
\noindent where it has been used the symmetry of $\oper$. Notice that the term within brackets can be equivalently reformulated as
\begin{equation*}
    \frac{1}{d_v}(\oper^{m-k})_{vv} + \frac{1}{d_u}(\oper^{m-k})_{uu} - 2(\oper^{m-k})_{vu} = \langle \frac{\mathbf{e}_v}{\sqrt{d_v}} - \frac{\mathbf{e}_u}{\sqrt{d_u}}, \oper^{m-k}\Big(\frac{\mathbf{e}_v}{\sqrt{d_v}} - \frac{\mathbf{e}_u}{\sqrt{d_u}}\Big)\rangle
\end{equation*}
\noindent where $\mathbf{e}_v$ is the vector with $1$ at entry $v$, and zero otherwise. In particular,
Please notice an {\bf important fact}: since, by assumption, $c_{\rs} \geq c_{\mpas}$ and $\lambda_{n-1} < 2$, whenever $\gph$ is not bipartite, $\oper$ is a {\em positive definite operator}. Then a bound (in {\bf expectation}) the Jacobian obstruction can be computed by
\begin{align*}
    \tilde{\obst}^{(m)}(v,u) &= \sum_{k=0}^{m}\| \mathbf{J}^{(m)}_k(v,u) \| \leq \sum_{k = 0}^{m} \rho\mu^{m-k}\sum_{\ell = 0}^{n-1}\left(c_{\rs} + c_{\mpas}(1-\lambda_{\ell})\right)^{m-k}\Big(\frac{\eigen_\ell(v)}{\sqrt{d_v}} - \frac{\eigen_\ell(u)}{\sqrt{d_u}}\Big)^{2} \\
    &= \rho\sum_{\ell = 0}^{n-1}\left(\sum_{k = 0}^{m} \mu^{m-k}\left(c_{\rs} + c_{\mpas}(1-\lambda_{\ell})\right)^{m-k}\right)\Big(\frac{\eigen_\ell(v)}{\sqrt{d_v}} - \frac{\eigen_\ell(u)}{\sqrt{d_u}}\Big)^{2} \\
    &= \rho\sum_{\ell = 1}^{n-1}\left(\sum_{k = 0}^{m} \mu^{m-k}\left(c_{\rs} + c_{\mpas}(1-\lambda_{\ell})\right)^{m-k}\right)\Big(\frac{\eigen_\ell(v)}{\sqrt{d_v}} - \frac{\eigen_\ell(u)}{\sqrt{d_u}}\Big)^{2},
\end{align*}
\noindent where in the last equality $\eigen_0(v) = \sqrt{d_v}/(2\lvert \E\rvert)$. Therefore, is possible to expand the geometric sum by using the assumption $\mu(c_{\rs} + c_{\mpas})\leq 1$ and write
\begin{align*}
\tilde{\obst}^{(m)}(v,u) &\leq \rho\sum_{\ell = 1}^{n-1}\frac{1 - \left(\mu(c_{\rs} + c_{\mpas}(1-\lambda_\ell))\right)^{m+1}}{1 - \mu(c_{\rs} + c_{\mpas}) +\mu c_{\mpas}\lambda_\ell}\Big(\frac{\eigen_\ell(v)}{\sqrt{d_v}} - \frac{\eigen_\ell(u)}{\sqrt{d_u}}\Big)^{2} \\
&\leq \sum_{\ell = 1}^{n-1}\frac{\rho}{\mu c_{\mpas}\lambda_\ell}\Big(\frac{\eigen_\ell(v)}{\sqrt{d_v}} - \frac{\eigen_\ell(u)}{\sqrt{d_u}}\Big)^{2} \\
& = \frac{\rho}{\mu c_{\mpas}}\sum_{\ell = 1}^{n-1}\frac{1}{\lambda_\ell}\Big(\frac{\eigen_\ell(v)}{\sqrt{d_v}} - \frac{\eigen_\ell(u)}{\sqrt{d_u}}\Big)^{2} \\
& = \frac{\rho}{\mu c_{\mpas}}\res(v,u)
\end{align*} 
\noindent where in the last step, the spectral characterization of the effective resistance derived in~\cite{lovasz1993random} has been used -- which was also leveraged in~\cite{arnaiz2022diffwire} to derive a novel rewiring algorithm. Since by~\cite{chandra1996electrical} it holds $2\res(v,u)\lvert\E\rvert = \tau(v,u)$, this completes the proof of the upper bound. The lower bound case follows by a similar argument. In fact, one arrives at the estimate
\begin{align*}
\tilde{\obst}^{(m)}(v,u) &\geq \rho\sum_{\ell = 1}^{n-1}\frac{1 - \left(\nu(c_{\rs} + c_{\mpas}(1-\lambda_\ell))\right)^{m+1}}{1 - \nu(c_{\rs} + c_{\mpas}) +\nu c_{\mpas}\lambda_\ell}\Big(\frac{\eigen_\ell(v)}{\sqrt{d_v}} - \frac{\eigen_\ell(u)}{\sqrt{d_u}}\Big)^{2}. 
\end{align*} 
\noindent  Derive
 \[
 1 - \left(\nu(c_{\rs} + c_{\mpas}(1-\lambda_\ell))\right)^{m+1} \geq 1 - \left(\nu(c_{\rs} + c_{\mpas}(1-\lambda^\ast))\right)^{m+1},
 \]
 \noindent where $\lvert 1 -\lambda^\ast\rvert = \max_{\ell > 0} \lvert 1 - \lambda^\ast\rvert$.
 \noindent Next, find that 
 \[
 \frac{1}{1 - \nu(c_{\rs} + c_{\mpas}) +\nu c_{\mpas}\lambda_\ell} \geq \frac{\epsilon}{\nu c_{\mpas}\lambda_\ell} \Longleftrightarrow \lambda_\ell \geq \frac{\epsilon}{1-\epsilon}\, \frac{1-\nu (c_{\rs} + c_{\mpas})}{\nu c_{\mpas}}. 
 \]
 \noindent Since the eigenvalues are ordered from smallest to largest, it suffices that 
  \[
 \lambda_1 \geq \frac{\epsilon}{1-\epsilon}\, \frac{1-\nu (c_{\rs} + c_{\mpas})}{\nu c_{\mpas}} \Longleftrightarrow \epsilon \leq \epsilon_\gph := \frac{\lambda_1}{\lambda_1 + \frac{1 - \nu(c_{\rs}+c_{\mpas})}{\nu c_{\mpas}}}.
 \]
 \noindent This completes the proof.
\end{proof}

\noindent {\em It worth emphasize that without the degree normalization, the bound would have an extra-term (potentially diverging with the number of layers) and simply proportional to the degrees of nodes $v,u$. The extra-degree normalization is off-setting this uninteresting contribution given by the steady state of the Random Walks.}




%% file: Include/Backmatter/Appendix/Appendix_Simplicial_Attention_Networks.tex
\chapter{On the Symmetries of Topological Neural Networks}
\label{appendix:simmetries_tnns}

\section{Primer on Category Theory}
\label{primer:category_theory}

Category theory is a branch of mathematics that deals with abstract structures and relationships between them. It provides a unified framework to study mathematical concepts in a way that emphasizes their relationships, rather than the objects themselves. While the  main goal of this thesis, is to study topological neural networks, a grasp of category theory can provide a bird-eye view of the underlying symmetries of these models.

\subsubsection{Objects and Morphisms}

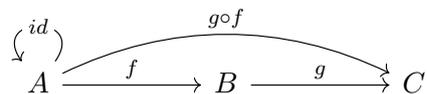
\begin{figure}[!htb]
        \centering
        \begin{tikzcd}
        A \arrow["id"{description}, loop, distance=2em, in=125, out=55] && B && C
        \arrow["f", from=1-1, to=1-3]
        \arrow["g", from=1-3, to=1-5]
        \arrow["{g \circ f}", curve={height=-24pt}, from=1-1, to=1-5]
        \end{tikzcd}
        \caption{Illustration of the Composition concept. If there is a morphism $ f $ from object $ A $ to $ B $ and another morphism $ g $ from $ B $ to $ C $, then there is a morphism $ g \circ f $ from $ A $ to $ C $.}
        \label{fig:composition}
    \end{figure}

The foundation of category theory groudns on the fundamental notion of
a category $\mathcal{C}$ composed by objects and morphisms.

\textbf{Objects} can be thought of as mathematical entities or structures. For the purposes of this thesis, think of them as containers or placeholders that represent mathematical objects, such as sets, groups, rings and so on. In this case, objects are discrete topological spaces equipped with signals. In~\Cref{fig:composition}, the objects are denoted wiht $A$, $B$ and $C$.

\textbf{Morphisms} are the relationships between objects. They can be described as transformations between two objects within the catorgy $\mathcal{C}$. Morphisms must satisfy two properties: {\bf composition}: if there exists a morphism $f$ from object $A$ to $B$ and another morphism $g$ from $B$ to $C$, then a morphism $g \circ f$ from $A$ to $C$ {\em must} also exist~(\Cref{fig:composition}) and {\bf identity} For every object, there exists a morphism that maps it to itself, called the identity morphism. Since often this is omitted from the diagrams, in~\Cref{fig:composition} the identity morphism is represented only for the object $A$. 

\begin{definition}[Category] A \textbf{category} $\mathcal{C}$  consists in a collection of objects and morphisms with the condition that morphisms can be composed, and this composition is associative. Each object has an associated identity morphism. 
\end{definition}

In essence, a category captures a mathematical world where objects and their relationships live. To relate different mathematical structures, the concept of {\em functor} bridges the gap between seemingly unrelated categories .

\begin{definition}[Functor]
    A \textbf{functor} $F$ is a map between two categories that maintains the object-morphism structure. Think of it as a transformer ({\em not the transformer architecture~\citep{vaswani2017attention}}) that takes objects and morphisms from one category and produces corresponding objects and morphisms in another category while preserving their relationships.
\end{definition}

\subsection{Why Category Theory for Topological Neural Networks?}

By modeling the symmetries of topological neural networks within the framework of category theory, is possible to exploit powerful mathematical concepts to elegantly express and prove the equivariance of such models. When objects like simplicial or cell complexes and morphisms like permutation matrices are considered, the problem is naturally embedded into a categorical framework, giving a rich language and toolkit to work with.

\subsubsection{A Categorical Perspective on the Symmetries of Topological Neural Networks}

Let $\mathcal{C}$ be a category such that, in the context of message passing schemes on discrete topological spaces the objects in $\mathcal{C}$ are complexes $\xph$ equipped with sequences of boundary matrices, $\mathbf{B} = (\Bnorm_1, \ldots, \Bnorm_K)$, and feature matrices $\mathbf{H} = (\mathbf{H}_0, \ldots, \mathbf{H}_K) $. The morphisms in $\mathcal{C}$ are sequences of permutation matrices $\mathbf{P} = (\Pnorm_0, \ldots, \Pnorm_K)$. These act on the complexes, permuting them as: $\mathbf{P}: \xph \to P(\xph)$.

Given this structure, permutation equivariance and invariance can be defined in terms of functoriality.

\begin{definition}[Permutation Equivariance in Category Theory]
A functor $F: \mathcal{C} \to \mathcal{C} $ is equivariant if, for any object $\xph$ in $\mathcal{C}$ and any morphism $\mathbf{P}: \xph \to \mathbf{P}(\xph)$ in $\mathcal{C}$, it holds

 \begin{equation}
     F(\mathbf{P}(\xph)) = (F  \circ \mathbf{P})(\xph) =( F \circ \mathbf{P}   )(\xph) = \mathbf{P}\, F(\xph), 
 \end{equation}

such that the following diagram commutes:
\end{definition}

\begin{center}
\begin{tikzcd}
\xph \arrow{r}{\mathbf{P}} \arrow[swap]{d}{F} & P(\xph) \arrow{d}{F} \\
F(\xph) \arrow{r}{F(\mathbf{P})} & F(P(\xph))
\end{tikzcd}
\end{center}

\begin{definition}[Permutation Invariance in Category Theory]
A functor $ F: \mathcal{C} \to \mathcal{C} $ is invariant if, for any object $\xph$ in $\mathcal{C}$ and any morphism $ \mathbf{P}: \xph \to \mathbf{P}(\xph) $ in $\mathcal{C}$,

\begin{equation}
 F(\mathbf{P}(\xph)) = (F  \circ \mathbf{P})(\xph) =( F \circ \mathbf{P}   )(\xph) = F(\xph)
\end{equation}

such that the following diagram commutes:

\begin{center}
\begin{tikzcd}
	\xph && \mathbf{P}(\xph) \\
	\\
	&& F(\xph)
	\arrow["F(\mathbf{P})"', from=1-1, to=3-3]
	\arrow["F", from=1-3, to=3-3]
	\arrow["\mathbf{P}", from=1-1, to=1-3]
\end{tikzcd}
\end{center}
\end{definition}

\begin{proof}[Permutation Equivariance for Topological Neural Networks]
Assume a topological neural network $\mathsf{TNN}_{\theta}$ as in~\Cref{eq:tnn_def} which acts as a functor $F$. Consider any complex $\xph$ with boundary $\mathbf{B}$ and feature matrices $\mathbf{H}$. When the sequence of permutation matrices $\mathbf{P}$ acts on $\xph$, it results in a permuted complex $P(\xph)$.

For permutation equivariance, it holds:

\begin{equation}
    \mathsf{TNN}_{\theta}(\mathbf{P} \mathbf{H}, \mathbf{P} \mathbf{B} \mathbf{P}^\top) = \mathbf{P} \, \mathsf{TNN}_{\theta}(\mathbf{H}, \mathbf{B}),
\end{equation}

which is analogous to:
\begin{equation}
    F(\mathbf{P}(\xph)) =   (\mathbf{P} \circ F)(\xph).
\end{equation}

When $\mathbf{P}$ is applied on $\xph$, the topological neural network represented with the functor $F$ respects the permuted relationships and produces an output that is a permuted version of the original. Thus, the network $\mathsf{TNN}_{\theta}$ satisfies the condition of being equivariant as defined in category theory.

This concludes the proof of permutation equivariance for topological neural networks from a categorical perspective.

\end{proof}

%% file: Include/Backmatter/Appendix/Appendix_Cell_Attention_Networks.tex
\chapter{Computational Complexity and Learnable Parameters of Cell Attention Networks}\label{app:comp_tan}

\paragraph{Structural Lift}

Although this operation can be pre-computed for the entire dataset and the connectivity results stored for later use, it is worth eliciting its complexity, noting that for some applications, the storage of the upper and lower connectivity for the \emph{entire} dataset might not be possible.
The chord-less cycles in a graph can thus be enumerated in $\mathcal{O}((|E| + |V| R) \, \textrm{polylog} |V|)$ time~\citep{ferreira2014amortized} where $R$ is upper bounded by a small constant. Thus, the complexity of this operation can be approximated to be linear in the size of the complex (i.e., the overall number of cells $\sigma \in \cph$. Intuitively, structural lifts do not involve any parameter to be learned during training.

\paragraph{Functional Lift}

The complexity of this operation is equivalent to a multi-head attention message passing scheme over the entire graph. For a single node pair $i,j \in \V$ connected by an edge $e \in \E$, the functional lift defined in~\Cref{eq:functional_lift} can be decomposed into $F_e$ independent transformations. Each map requires $\mathcal{O}(F_n)$ computations, where $F_n$ is the number of input node features. Thus, for the pair $i,j$, the functional lift is performed in $\mathcal{O}(F_e F_n)$, where $F_e$ is a parameter to be chosen as the number of input edge features. Accounting all the edges of the complex yields an amount of $\mathcal{O}(\vert \E \vert F_e F_n)$ operations to lift node features into edge ones. In the context of lifting a pair of graph node features $\mathbf{x}_u, \mathbf{x}_v \in \mathbb{R}^{F_n}$ to obtain edge features $\mathbf{x}_{e} \in \mathbb{R}^{F_e}$, the attention parameters are involved. The parameters to learn the transformations $\Wnorm_1 \in \mathbb{R}^{2 F_n \times F_e}$ and $\Wnorm_2 \in \mathbb{R}^{F_e \times F_e}$ are therefore on the order of $\Theta \left( F_e F_n \right)$. 

\paragraph{Cell Attention}

This operation consists in two independent masked self-attention message passing schemes over the upper and lower neighbourhoods of the complex, namely cell attention, an inner linear transformation of the edges' features and an outer point-wise nonlinear activation~(\Cref{eq:top_att:can:message_passing_scheme}). For a layer $l$, the number of messages that an edge $e$ receives from its lower neighbourhood is equal to $\vert \mathcal{N}_{\downarrow}(e) \vert$, the number of edges that share a common node with $e$.  The same computation yields for the upper neighbourhood: edge $e$ receives $\vert \mathcal{N}_{\doarr}(e) \vert$ messages, from edges that are in the same cell's boundaries as $e$. Recalling that $\E^{(l+1)} \subseteq \E^{(l)}$ and $R$ is upper bounded by a small constant~\cite{bodnar2021weisfeilercell}, in a single message passing the number of messages that and edge $e$ receives is bounded by $\mathcal{O}(\vert \E \vert )$, where $\E$ is the initial number of edges of $\cph$. The inner linear transformation that propagates the information contained in $\mathbf{h}_e$ is upper bounded by $\mathcal{O}(F_e^2)$. Extending this to all edges of the complex, the complexity of a cell attention layer can be rewritten as $\mathcal{O}(|\E|\, F_e^2)$. In the case of a multi-head cell attention, the complexity receives an overhead induced by the number of attention heads involved within the layer, i.e., a multiplication by a factor $H$, the number of cell attention heads. In terms of learnable parameters and in the case of the  GAT-like attention functions~\citep{velivckovic2018graph}, a single \emph{cell attention layer} is composed of: two independent vectors of attention coefficients $\mathbf{a}_{\doarr}, \mathbf{a}_{\uparr} \in \mathbb{R}^{2F_e}$ for properly weighting the lower and upper neighbourhoods, respectively. Moreover, the layer is equipped with three linear transformations, $\mathbf{W}, \mathbf{W}_{\doarr}, \mathbf{W}_{\uparr} \in \mathbb{R}^{F_e \times F_e}$ acting respectively on: $\mathbf{h}_e$, the latent representation of edge $e$  and the hidden features $\mathbf{h}_k$ in the lower and upper neighbourhoods of the edge $e$. If instead the dynamic attention proposed in~\cite{brody2021attentive} is used, the size of the weight matrices increases to $\mathbf{W}, \mathbf{W}_{\doarr}, \mathbf{W}_{\uparr} \in \mathbb{R}^{F_e \times 2 F_e}$ while the vectors of attention coefficients reduce to $\mathbf{a}_{\doarr}, \mathbf{a}_{\uparr} \in \mathbb{R}^{F_e}$.. Thus, independently on the particular graph attention mechanism employed, the number of learnable parameters of a cell attention layer is $\mathcal{O}(F_e^2)$.

\paragraph{Edge Pool}

The operations involved in the pooling layer can be decomposed in: (i) computing the self-attention scores for each edge of the complex ($\gamma_e$ from~\Cref{pooling_scores}); (ii) select the highest $\lceil k \, \vert \E \vert \, \rceil$ values from a collection of self-attention scores ($\text{top-k}(\{\gamma_e\}_{e \in \E}, \lceil k |\E| \rceil)$); and (iii) adjust the connectivity of the complex~(\Cref{fig:pooling}). To compute the computational complexity of this layer, it is convenient to view the selection operation as a combination of a sorting algorithm over a collection of self-attention scores and the selection of the first $\lceil k , \vert \E \vert , \rceil$ elements from the sorted collection. Since the computations involved in (i) and (iii) are linear in the dimension of the complex, the overall complexity of this layer in can be upper bounded by the sorting algorithm, i.e., $\mathcal{O}(|\E| \; log (|\E|))$. Please notice that, in the context of the edge pooling operation the number of elements of $\E$ is reduced after each layer. For this operation, learnable parameters are employed only in computing the self-attention scores ($\gamma_e$~(\Cref{pooling_scores})). In the case of the GAT-like attention functions~\citep{velivckovic2018graph}, they consist of a shared vector of attentional scores' coefficients $\mathbf{a}_p \in \mathbb{R}^{F_e}$, similarly to the lift layer, leading to $\Theta \left( F_e \right)$.

%% file: Include/Backmatter/Appendix/Appendix_CINpp.tex
\chapter{Appendix CIN++}\label{app:cinpp}

\section{Expressive Power}\label{app:exp}
This section analyse the expressive power of enhanced topological message passing. Two complexes $\cph_1$ and $\cph_2$ are said to be {\em isomorphic} (written $\cph_1 \simeq \cph_2$) if there exists a bijection $\varphi:\mathcal{P}_{\cph_1} \to \mathcal{P}_{\cph_2}$ such that $\sigma \in \cph_1 \iff \varphi(\sigma) \in \cph_2$~\citep{bodnar2021weisfeiler, bodnar2021weisfeilercell}. Also, a cell coloring $c$ {\em refines} a cell coloring $d$, written $c \sqsubseteq d$, if $c(\sigma) = c(\tau) \implies d(\sigma) = d(\tau)$ for every $\sigma,\tau \in \cph$. 
Two colorings are {\em equivalent} if $c \sqsubseteq d$ and $d \sqsubseteq c$, and it is written as: $c \equiv d$~\citep{morris2019weisfeiler}.

\begin{proof}[\textbf{Proof of Theorem~\ref{thm:cin++express}}]
Let $c^l$ be the colouring of CWL~\citep{bodnar2021weisfeilercell} at iteration $l$ and $h^l$ the colouring (i.e., the multi-set of features) provided by a CIN++ network at layer $l$ as in~\Cref{sec:etmp}.

To show that CIN++ inherits all the properties of Cellular Isomorphism Networks~\citep{bodnar2021weisfeilercell} it is necessary to show that the proposed topological message passing scheme produces a colouring of the complex that satisfies Lemma 26 of~\cite{bodnar2021weisfeilercell}. 

To show $c^t \sqsubseteq h^t$ by induction, assume $h^l = h^L$ for all $l > L$, where $L$ is the number of the network's layers.
Let also $\sigma, \tau$ be two arbitrary cells with $c^{l+1}_\sigma = c^{l+1}_\tau$. Then, $c^l_\sigma = c^l_\tau$, $c^l_\mathcal{B}(\sigma) = c^l_\mathcal{B}(\tau)$, $c^l_\uparrow(\sigma) = c^l_\uparrow(\tau)$ and $c^l_\downarrow(\sigma) = c^l_\downarrow(\tau)$. By the induction hypothesis, $h^l_\sigma = h^l_\tau$, $h^l_\mathcal{B}(\sigma) = h^l_\mathcal{B}(\tau)$, $h^l_\uparrow(\sigma) = h^l_\uparrow(\tau)$ and $h^l_\downarrow(\sigma) = h^l_\downarrow(\tau)$. 

If $l+1 > L$, then $h^{l+1}_\sigma = h^{l}_\sigma = h^l_\tau = h^{l+1}_\tau$. Otherwise, $h^{l+1}$ is given by the update function in Equation \ref{eq:update}. Given that the inputs passed to these functions are equal for $\sigma$ and $\tau$, $h^{l+1}_\sigma = h^{l+1}_\tau$.

For showing  $h^l \sqsubseteq c^l$, suppose the aggregation from Equation \ref{eq:update} is injective and the model is equipped with a sufficient number of layers such that the convergence of the colouring is guaranteed. Let $\sigma, \tau$ be two cells with $h^{l+1}_\sigma = h^{l+1}_\tau$. Then, since the local aggregation is injective $h^l_\sigma = h^l_\tau$, $h^l_\mathcal{B}(\sigma) = h^l_\mathcal{B}(\tau)$, $h^l_\uparrow(\sigma) = h^l_\uparrow(\tau)$ and $h^l_\downarrow(\sigma) = h^l_\downarrow(\tau)$. By the induction hypothesis, $c^l_\sigma = c^l_\tau$, $c^l_\mathcal{B}(\sigma) = c^l_\mathcal{B}(\tau)$,  $c^l_\uparrow(\sigma) = c^l_\uparrow(\tau)$ and   $c^l_\downarrow(\sigma) = c^l_\downarrow(\tau)$ which implies that $c^{l+1}_\sigma = c^{l+1}_\tau$. 

Given that $c^t \sqsubseteq h^t$ and $h^l \sqsubseteq c^l$, is possible to conclude that $h^l \equiv c^l$.
\end{proof}

As a result, CIN++ inherits all the properties of Cellular Isomorphism Networks, in accordance with Lemma 26 from~\cite{bodnar2021weisfeilercell}.

\section{A Categorical Interpretation: Sheaves}
It worth to notice that CIN++ can be seen as a particular case of a message passing scheme over a cellular sheaf. 
Let $\cph$ be a regular cell complex. A cellular sheaf is a mathematical object that attaches data spaces to the
cells of $\cph$ together with relations that specify when assignments to these data
spaces are consistent.

\begin{definition}[Cellular Sheaf]~\citep{hansen2019toward}\label{def:sheaf} A \textit{cellular sheaf} of vector spaces on a regular cell complex $\cph$ is an
assignment of a vector space $\mathcal{F}(\sigma)$ to each cell $\sigma$ of $\cph$
together with a linear transformation $\mathcal{F}_{\sigma\face\tau}\colon \mathcal{F}(\sigma)
\to \mathcal{F}(\tau)$ for each incident cell pair $\sigma\,\face\,\tau$. These must
satisfy both an identity relation $\mathcal{F}_{\sigma\face\sigma}={id}$ and the
composition condition:
\[
     \rho\,\face\,\sigma\,\face\,\tau
  ~~\Rightarrow~~
  \mathcal{F}_{\rho\face\tau} = \mathcal{F}_{\sigma\face\tau}\circ\mathcal{F}_{\rho\face\sigma}.
\]
\end{definition}

It is also natural to consider a dual construction to a
cellular sheaf to preserves stalk data but reverses
the direction of the face poset, and with it, the restriction maps.
\begin{definition}[Cellular Cosheaf]~\citep{hansen2019toward}\label{def:cosheaf}
A cellular cosheaf of vector spaces on a regular cell complex $\cph$ is an
assignment of a vector space $\mathcal{F}(\sigma)$ to each cell $\sigma$ of  $\cph$
together with linear maps $\mathcal{F}^{\text{op}}_{\sigma\face\tau}\colon \mathcal{F}(\tau) \to
\mathcal{F}(\sigma)$ for each incident cell pair $\sigma\,\face\,\tau$ which satisfies
the identity ($\mathcal{F}^{\text{op}}_{\sigma\face\sigma}={id}$) and composition condition:
\[
     \rho\,\face\,\sigma\,\face\,\tau
  ~~\Rightarrow~~
  \mathcal{F}^{\text{op}}_{\rho\face\tau} = \mathcal{F}^{\text{op}}_{\rho\face\sigma}\circ\mathcal{F}^{\text{op}}_{\sigma\face\tau} .
\]
\end{definition}

The vector space $\mathcal{F}(\sigma)$ is called the {\em stalk} of $\mathcal{F}$ at $\sigma$ and will encode the features supported over $\sigma$. The maps $\mathcal{F}_{\sigma\face\tau}$ and $\mathcal{F}^{\text{op}}_{\sigma\face\tau}$ are called the {\em restriction maps} and will provide a principled way to respectively move features from lower dimensional cells to higher dimensional ones and vice-versa.

From a categorical perspective, a cellular sheaf is a functor $\mathcal{F} : \mathcal{P}_{\cph} \rightarrow \mathbf{Vect}_{\mathbb{R}}$ that maps the indexing set $\mathcal{P}_{\cph}$ to the category of vector spaces over $\mathbb{R}$ while a cellular cosheaf is a functor $ \mathcal{F}^{\text{op}} : \mathcal{P}_{\cph}^{\text{op}} \to
\mathbf{Vect}_{\mathbb{R}}$ such that, for a two dimensional regular cell complex $\cph$, a sheaf $(\mathcal{F}, \mathbb{R})$ and its dual cosheaf $(\mathcal{F}^{\text{op}}, \mathbb{R})$ on  $\cph$, the following diagram commutes:

\begin{wrapfigure}[5]{r}{0.4\textwidth}
    \begin{subfigure}[t!]{1.0\linewidth}
        \centering
        \includegraphics[width=1.0\textwidth]{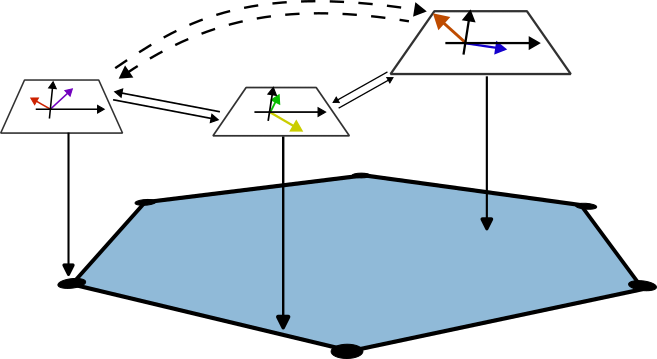}
    \end{subfigure}
     \caption{Pictorial example of a sheaf and cosheaf of vector spaces structure on a ring of a regular cell complex $\cph$.}
    \label{fig:sheaf}
    \vspace{+30pt}
\end{wrapfigure}

\
\

\[\begin{tikzcd}[sep=2.25em]
	{\mathcal{F}(v)} &&& {\mathcal{F}(e)} \\
	\\
	\\
	&&& {\mathcal{F}(r)}
	\arrow["{\mathcal{F}_{v\face e}}", from=1-1, to=1-4]
	\arrow["{\mathcal{F}_{e \face r}}", from=1-4, to=4-4]
	\arrow["{\mathcal{F}_{v \face r}}"', shift right=1, curve={height=30pt}, dashed, from=1-1, to=4-4]
	\arrow["{\mathcal{F}_{v \face r}^{\text{op}}}"', shift right=1, curve={height=-30pt}, dashed, from=4-4, to=1-1]
	\arrow["{\mathcal{F}_{e \face r}^{\text{op}}}", shift left=2, from=4-4, to=1-4]
	\arrow["{\mathcal{F}_{v \face e}^{\text{op}}}", shift left=2, from=1-4, to=1-1]
\end{tikzcd}\]

\vspace{25pt}
In the given commutative diagram, the arrow is dashed to indicate that the morphism (map) it represents is not explicitly defined in the diagram, but rather it is implied by the other morphisms. In this case, the dashed arrow is used to show the existence of a unique morphism that makes the diagram commute. This relationship is important in the context of cellular sheaves, where these morphisms represent restrictions on different cells and their overlaps. The dashed arrows shows that there is a unique way to go from $\mathcal{F}(v)$ to $\mathcal{F}(r)$  and back that is consistent with the other restrictions, even if it is not directly defined in the diagram.